\documentclass[11pt]{amsart}
\usepackage[letterpaper,left=0.9in,right=0.9in,top=1in,bottom=1in]{geometry}

\usepackage[utf8]{inputenc}
\usepackage[T1]{fontenc}
\usepackage{microtype}
\usepackage{mathtools}
\usepackage{amssymb,amsfonts}
\usepackage{graphicx}
\usepackage{booktabs}
\usepackage{tabularx}
\usepackage{enumitem}
\usepackage{xcolor}
\usepackage{tikz}
\usepackage{nicefrac}
\usepackage{url}
\usepackage{subcaption}
\usepackage{hyperref}

\usetikzlibrary{arrows.meta,positioning,calc,fit,backgrounds,decorations.pathreplacing,matrix}
\numberwithin{equation}{section}
\allowdisplaybreaks

\theoremstyle{plain}
\newtheorem{theorem}{Theorem}[section]
\newtheorem{proposition}[theorem]{Proposition}
\newtheorem{lemma}[theorem]{Lemma}
\newtheorem{corollary}[theorem]{Corollary}

\theoremstyle{definition}
\newtheorem{definition}[theorem]{Definition}
\newtheorem{assumption}[theorem]{Assumption}
\newtheorem{remark}[theorem]{Remark}

\newcommand{\R}{\mathbb{R}}
\newcommand{\E}{\mathbb{E}}
\newcommand{\Pp}{\mathbb{P}}
\newcommand{\cX}{\mathcal{X}}
\newcommand{\cV}{\mathcal{V}}
\newcommand{\cW}{\mathcal{W}}
\newcommand{\cZ}{\mathcal{Z}}
\newcommand{\cH}{\mathcal{H}}
\newcommand{\bN}{\mathbf{N}}
\newcommand{\subop}{\operatorname{sub}}
\newcommand{\templ}{\mu_{\mathrm{tmplt}}}
\newcommand{\loss}{\ell_{\log}}
\newcommand{\norm}[1]{\left\lVert #1\right\rVert}

\newcommand{\diag}{\operatorname{diag}}

 \newcommand{\Sub}{\operatorname{sub}}
\newcommand{\1}{\mathbf 1}

\newcommand{\Var}{\operatorname{Var}}
 
\newcommand{\KL}{\operatorname{KL}}  
\newcommand{\fr}{\mathrm{fr}}
\newcommand{\id}{\mathrm{id}}
\newcommand{\op}{\mathrm{op}}
\newcommand{\lit}{\mathrm{lit}} % literal-token events, not a proof-path name
\newcommand{\corr}{\mathrm{corr}}
\newcommand{\bdens}{\mathrm{bd}}      % block-density / quotient concentration
\newcommand{\degact}{\mathrm{deg}}    % weighted degree-action
\newcommand{\cspec}{\mathrm{cs}}      % curvature-spectral
\newcommand{\hanova}{\mathrm{anova}}  % Hoeffding--ANOVA
\newcommand{\bfdec}{\mathrm{bf}}      % bias--fluctuation / backbone correction
\newcommand{\ew}{\mathrm{ew}}      % edge--wedge envelope

\providecommand{\cA}{\mathcal A}

\providecommand{\cY}{\mathcal Y}

\providecommand{\subop}{\operatorname{sub}}
\providecommand{\templ}{\mu_{\mathrm{tmplt}}}

\providecommand{\Anc}{\operatorname{Anc}}

\newcommand{\eff}{\mathrm{eff}}

\newcommand{\Ccol}{\mathsf C}
\newcommand{\Wcol}{\mathsf W}
\newcommand{\Bad}{\mathsf{Bad}}

\newcommand{\Gcol}{\mathcal G_{\mathrm{coll}}}
\newcommand{\Back}{\mathsf B}
\newcommand{\Res}{\mathsf R}
\newcommand{\supp}{\operatorname{supp}}
\newtheorem{mainthm}{Main Theorem}
\usepackage[nameinlink,noabbrev]{cleveref}

\makeatletter
\def\@setauthors{%
  \begingroup
  \trivlist
  \centering\footnotesize
  \@topsep30\p@\relax
  \advance\@topsep by -\baselineskip
  \item\relax
  \andify\authors
  \uppercasenonmath\authors
  \authors
  \par\medskip
  {\normalfont\small
  Department of Statistics and Data Science, Cornell University \\
  \texttt{wg285@cornell.edu} \quad \texttt{jelena.bradic@cornell.edu}\par}
  \endtrivlist
  \endgroup
}
\makeatother

\tikzset{
  tok/.style={draw,rounded corners=1pt,minimum height=4.6mm,minimum width=7.6mm,inner sep=1pt,font=\scriptsize},
  lit/.style={tok,fill=black!5},
  wild/.style={tok,fill=blue!7},
  seqframe/.style={draw,rounded corners,inner sep=4pt},
  tinylabel/.style={font=\scriptsize},
  pmatrixbox/.style={draw,rounded corners,inner sep=5pt,fill=white}
}

\title[When Symbol Names Should Not Matter]{When Symbol Names Should Not Matter: \\
A Logistic Theory of Fresh-Symbol Classification}

\author{Wenjie Guan}
\address{Department of Statistics and Data Science, Cornell University, Ithaca, NY 14850}
\email{wg285@cornell.edu}

\author{Jelena Bradic}
\address{Department of Statistics and Data Science, Cornell University, Ithaca, NY 14850}
\email{jelena.bradic@cornell.edu}

\date{}

\begin{document}

% The table of contents is printed at the start of the appendix.
% Hide all main-text section entries from that appendix contents list;
% appendix entries are re-enabled immediately after \appendix below.
\addtocontents{toc}{\protect\setcounter{tocdepth}{-10}}

\maketitle

\begin{abstract}
	Template tasks have emerged as a clean testbed for asking whether transformers
	reason with abstract symbols rather than concrete token names. We study the
	fixed-label classification version of this problem, where train and test
	examples share latent templates but may use disjoint vocabularies. Unlike
	next-token prediction, the model need not emit unseen symbols; it must learn a
	decision rule invariant to symbol renaming. We analyze regularized kernel
	logistic classification in the transformer-kernel regime. Our main result
	decomposes the learned predictor into an ideal template-level classifier and a
	finite-sample perturbation caused by accidental token overlaps in the training
	data. We encode these overlaps by a colored collision graph and prove
	high-probability margin-transfer guarantees for fresh-symbol classification.
	This perspective extends template-based analyses to logistic classification and
	refines scalar diversity conditions: vocabulary size controls the average rate
	of collisions, but collision geometry controls whether the ideal classification
	margin is preserved. Synthetic template experiments illustrate the predicted
	roles of regularization, sample size, and transformer-kernel structure.
	
%We study unseen-symbol generalization: test examples are generated from the same latent templates as the training data, but use symbols that never appear during training. This setting isolates a clean form of systematic out-of-distribution generalization beyond memorizing token identities. We establish classification guarantees for ridge-regularized kernel logistic regression on template tasks. The analysis combines a template-space separation result, a perturbation theorem for kernel logistic regression, and a diversity-based comparison between the empirical kernel and an ideal block model over templates. Under token symmetry, nondegenerate template geometry, and sufficiently diverse substitutions, these ingredients yield high-probability margin and correct-classification guarantees on unseen-symbol test inputs. We further show that the perturbation method extends to multiclass softmax losses and connects to frozen-feature transformer kernels on fixed finite datasets.
\end{abstract}

\section{Introduction}\label{sec:introduction}

The Transformer attention mechanism is a fundamental component that allows large language models (LLMs) to capture context information \cite{vaswani2017attention,bahdanau2014neural}.  Yet a recurring finding is that sequence models can rely on surface patterns and token identity, leading to brittle behavior under small perturbations \cite{mirzadeh2025gsm,jiang2024peek,Dziri2023,HosseiniEtAl2022ICL}.  This concern is continuous with the classical systematicity and compositionality debate \cite{FodorPylyshyn1988,LakeBaroni2018} and with recent taxonomies and surveys of generalization in neural language models \cite{HupkesEtAl2023,McCurdyEtAl2024}.  Template-based tasks provide a controlled way to isolate the issue: examples are generated from latent structures, while concrete symbols are randomized \cite{Boix2024,Keysers2020,KimLinzen2020,RuisEtAl2020}.

In this work, we focus on a particularly clean instance: \textbf{unseen-symbol (fresh-symbol) generalization}, where test examples follow the same latent templates as training but use entirely new symbols \cite{Boix2024}.  This setting differs from broad compositional splits because the quotient rule is shared across train and test; the support shift is only in the nuisance names.  Prior work shows that transformers can succeed in this regime under sufficient data and symmetry assumptions, and also identifies copying or emission failures when the output itself is an unseen symbol \cite{Boix2024,LazicEtAl2026Unseen}.  Existing explanations based on token symmetry or vocabulary diversity \cite{Boix2024,AbbeBengioLotfiRizk2024} do not fully account for the finite-sample effects caused by concrete substitutions.

We consider template-based classification tasks, where each example is generated by substituting symbols into latent templates. For instance, in a binary task with templates $\alpha\alpha \mapsto +1$ and $\alpha\beta \mapsto -1$, the string CC should receive the same label as AA (+1), even if the symbols C and D never appear during training. This setting follows the fresh-symbol framework of Boix-Adserà et al. \cite{Boix2024}, where generalization requires learning rules invariant to symbol identity. Prior work shows that transformers can succeed in this regime under sufficient data, but also identifies failure modes when models rely on memorizing token patterns rather than abstract structure.

In this work, we study this phenomenon in the classification setting under the transformer-kernel regime. Unlike next-token prediction, the model does not need to generate unseen symbols, but must instead learn a decision function whose output is invariant under symbol renaming. This allows us to isolate the role of finite-sample effects in determining when generalization succeeds.

The examples used throughout the paper should be read in this way.  Program variables, equality-pattern strings, symbolic copying tasks, and literal scaffolds all define a quotient template problem in which the concrete names are nuisance variables.  The statistical question is then whether finite training substitutions introduce enough accidental name overlap to move the empirical kernel away from the quotient problem.  This is the role of the collision graph developed below.

\subsection{Our contributions}

For a token-renaming-invariant kernel \(K\), fresh representatives of the same pair of templates
 induces a finite template matrix \(\bN\).  The ridge-regularized population minimizer is determined only by the pairwise kernel values between fresh template instantiations, the template frequencies, and the labels; under template-level separability and sufficiently small regularization, it achieves positive signed margin.

Our main theorem is the classification analogue of the unseen-symbol guarantee of
Boix-Adser\`a et al.~\cite{Boix2024}, but with logistic loss and a margin conclusion rather
than squared-loss regression.

\noindent
\textbf{Informal theorem.}
\emph{Suppose the ideal quotient-space logistic classifier has signed margin
\(\Gamma=\varepsilon+\tau\) on template \(a\).  If the collision graph is sufficiently benign so that
\(\Delta_{\rm cg}(x,\delta)\le \tau\), then}
$
    \Pp\bigl(y_a  \widehat f_\lambda(x)>\varepsilon\bigr)
    \ge
    1-r\exp(-n p_{\min}/8)-\delta .
$
\emph{In particular, for \(\varepsilon=0\), the empirical classifier correctly labels the fresh
string \(x\in\cX_a\) with the same probability lower bound.}

The formal theorem replaces \(\Delta_{\rm cg}\) by an explicit high-probability bound, whose
terms are obtained from collision-graph control.  Thus the theorem separates the semantic
condition, i.e. positive margin on the quotient template problem, from the statistical condition, finite
substitutions must not create a large collision perturbation.

\begin{figure}[h]
    \centering
    \includegraphics[scale=0.3]{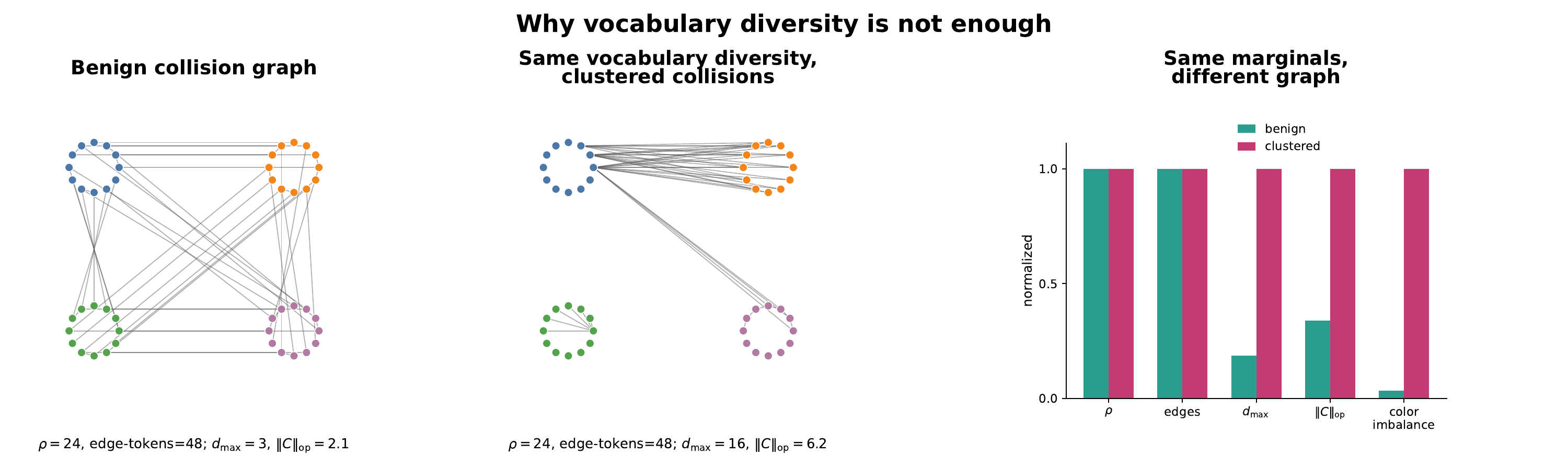}
    \caption{Pairwise token reuse preserves $\rho$ and edge-token counts, but collision structure differs in maximum degree, $d_{\max}$, color imbalance, and centered spectral norm, $\lVert C\rVert_{\mathrm{op}}$.}
    \label{fig:collision-vocab}
\end{figure}

We also show that vocabulary diversity \(\rho\), as defined in \cite{Boix2024}, is only a marginal proxy: it controls the largest marginal probability of any token occupying a wildcard slot, but not how collisions are distributed. Figure~\ref{fig:collision-vocab} shows that two schemes can have identical token frequencies and edge-token collision counts, yet induce graphs with very different degree structure, color-pair imbalance, spectral norm, and generalization behavior.

\subsection{Related work}

\textbf{Systematic and fresh-symbol generalization.}
The paper is closest to work on systematic generalization as invariance to nuisance symbol identity.  This line connects classical arguments about systematicity and compositionality with modern benchmarks for semantic parsing, grounded language, and equality-pattern reasoning \cite{FodorPylyshyn1988,LakeBaroni2018,BahdanauEtAl2019,Keysers2020,KimLinzen2020,RuisEtAl2020,HupkesEtAl2023,McCurdyEtAl2024}.  Boix-Adser\`a et al. formalize the fresh-symbol template setting for transformers; our result keeps the same quotient-template viewpoint, but studies fixed-label logistic classification and finite-sample collision geometry \cite{Boix2024,AbbeBengioLotfiRizk2024}.  This distinction is important because models may infer a rule yet still fail when they must copy or emit unseen variables \cite{LazicEtAl2026Unseen}.

\textbf{Mechanisms, invariance, and symbol binding.}
Several recent works design or analyze representations that are insensitive to token names, including renaming-invariant embeddings and open-vocabulary symbol-invariant architectures \cite{IsikEtAl2025Interchangeable,IsikLi2026SymbolInvariant}.  More broadly, our token-symmetry assumption is a kernel-level analogue of invariance and equivariance ideas in neural architectures and invariant kernels \cite{GensDomingos2014,CohenWelling2016,ZaheerEtAl2017,HaasdonkBurkhardt2007,MrouehVoineaPoggio2015}.  Mechanistic studies of induction heads, variable binding, and symbolic abstraction identify model components that can implement related equality and role-binding operations \cite{OlssonEtAl2022,WuGeigerMilliere2025,YangEtAl2025SymbolicMechanisms,AlSaeediHarma2025AbstractionHeads,SmolenskyEtAl2025SymbolProcessing}.  Work on in-context abstraction and in-context algebra asks how such symbolic structure can be induced at test time \cite{AnandEtAl2025DualProcess,ChanEtAl2025IWI,ToddEtAl2025InContextAlgebra,ParkEtAl2025ICLRRepresentations}.

\textbf{Kernel and logistic learning theory.}
The kernel side of the paper is related to string and convolution kernels for structured inputs \cite{Haussler1999,LodhiEtAl2002}, random-feature and neural tangent kernel limits \cite{RudiRosasco2017RandomFeatures,JacotGabrielHongler2018,AroraEtAl2019CNTK,ChizatBach2019}, and classical KRR generalization theory \cite{SchoelkopfHerbrichSmola2001,CaponnettoDeVito2007,LiangRakhlin2020Ridgeless,ZhangLiLuLin2025KRR,PanditWangZhu2025KernelRegression,ZhangShiZhangLiangZou2025NonIIDKRR}.  Standard KLR theory usually targets surrogate-risk consistency, self-concordant analysis, or primal RKHS rates \cite{Zhang2004ConvexRisk,BartlettJordanMcAuliffe2006,SteinwartChristmann2008,Bach2010LogisticSelfConcordant,WenBetkenHang2025KLR}.  Here the objective is different: we use the logistic entropy dual to control how a structured Gram-matrix perturbation changes a fresh-symbol margin.

\textbf{Graph viewpoint.}
The collision graph is inspired by the fact that density alone does not determine spectral or degree behavior, a theme familiar from quasi-random graphs and matrix concentration \cite{ChungGrahamWilson1989,BoucheronLugosiMassart2013,Tropp2015}.  Unlike a standard random graph, however, our graph is colored by latent templates and carries signed kernel weights.  The certificates therefore track not only how many collisions occur, but also which template pairs they connect and how they act on the logistic dual weights.

\section{Problem setup and idealized analysis}\label{sec:setup}

Let $\cV$ be a token vocabulary and $\cW$ a finite wildcard alphabet. A template family is a finite collection $\cZ=\{z_1,\dots,z_r\}\subset (\cV\cup\cW)^k$, where the templates are assumed to be pairwise disjoint. Each template $z_a$ defines a class $\cX_a \subset \cV^k$ obtained by injectively substituting tokens into wildcard positions while respecting literals. Thus every admissible string belongs to exactly one class $\cX_a$. %The full formal definition of admissible substitutions, matching, pairwise disjointness, and freshness is given in the supplementary appendix. For the main text it suffices to note that every admissible string belongs to exactly one class $\cX_a$.

\begin{remark}[Template examples and collision channels]
\label{rem:template-examples-main}
The formalism covers several standard symbolic-generalization examples.  A real-label program template may be written as
\[
    z_+(\alpha,\beta)
    =
    [\alpha=1;\ \beta=-1;\ \mathtt{print}(\alpha)],
    \qquad
    z_-(\alpha,\beta)
    =
    [\alpha=1;\ \beta=-1;\ \mathtt{print}(\beta)].
\]
The concrete variable names are nuisance symbols; the label is determined only by which assigned variable is queried.  With uniform injective substitutions from an alphabet of size \(m\), both templates have two active wildcards, so \(\rho=m/2\), while the wildcard-set collision probability is
\[
    q^{\rm wild}_{++}=q^{\rm wild}_{+-}=q^{\rm wild}_{--}
    =1-{\binom{m-2}{2}\over \binom{m}{2}}
    ={4m-6\over m(m-1)}
    \simeq {2\over \rho}.
\]
Symbolic-label or copying tasks add a second channel.  In
\[
    z_\alpha=[\alpha=``\gamma'';\ \beta=``\delta'';\ \mathtt{print}(\alpha)]\to\gamma,
    \qquad
    z_\beta=[\alpha=``\gamma'';\ \beta=``\delta'';\ \mathtt{print}(\beta)]\to\delta,
\]
variable-name collisions perturb the selector, while value-token collisions perturb the copied output.  Thus a signed discrepancy may decompose as
\[
    \Wcol_{ij}=\eta_{\rm var}\Wcol^{\rm var}_{ij}+\eta_{\rm val}\Wcol^{\rm val}_{ij}.
\]
The same template notation also covers equality-pattern tasks such as
\[
    z_{\rm same}(\alpha)=\alpha\alpha,
    \qquad
    z_{\rm diff}(\alpha,\beta)=\alpha\beta .
\]
Here \(w_{\rm same}=1\), \(w_{\rm diff}=2\), and \(\rho=m/2\), but the full color-pair collision matrix is
\[
    Q^{\rm wild}
    =
    \begin{pmatrix}
        1/m & 2/m \\
        2/m & (4m-6)/(m(m-1))
    \end{pmatrix}.
\]
A single diversity number hides these color-pair differences; the collision graph keeps them.
\end{remark}

We first develop the theory for binary classification.  Section~\ref{sec:multiclass-softmax-extension} gives the corresponding multiclass softmax formulation and classification guarantee without repeating the collision-graph estimates.  For the binary setup, we assign labels $y_1,\dots,y_r\in\{-1,+1\}$. A training observation is generated by first drawing a latent template $Z\sim \templ$ with masses $p_a>0$, then drawing an admissible substitution map $S$ conditional on $Z$, and finally setting $
X:=\subop(Z,S)$ and $Y:=y_Z$. The sample $(X_i,Y_i)_{i=1}^n$ is i.i.d. from this distribution.

Let $K:\cX\times\cX\to\R$ be a positive-semidefinite kernel with RKHS $\cH$. We assume \emph{token symmetry}: for every permutation $\pi:\cV\to\cV$ and every $x,x'\in\cX$,
\[
K(x,x')=K(\pi(x),\pi(x')).
\]
Under token symmetry, the quantity
\[
N_{ab}:=K\bigl(\subop(z_a,s),\subop(z_b,s')\bigr)
\]
is well defined whenever $(s,s')$ is a fresh admissible pair for $(z_a,z_b)$. This gives the template matrix $\bN=(N_{ab})_{a,b=1}^r$.  Thus token symmetry is used here as a quotient device: it removes the nuisance action of renaming concrete vocabulary items, in the same spirit as invariant kernels and equivariant architectures, but without requiring a special finite-width architecture \cite{HaasdonkBurkhardt2007,MrouehVoineaPoggio2015,GensDomingos2014,CohenWelling2016}. %Intuitively, token symmetry implies that the kernel depends only on the underlying template structure rather than on the concrete symbol names themselves. In the ideal fresh-symbol setting, two independently generated realizations of the same template pair therefore induce the same kernel value. This allows the learning problem to collapse from the sample level to a finite-dimensional template-level problem characterized by the matrix $\bN$.

\begin{assumption}[Standing assumptions]
	\label{ass:standing}
	Throughout the main results we assume:
	the template family is pairwise disjoint;  $p_a>0$ for every template $a$; the kernel $K$ is token-symmetric;  the template matrix $\bN$ is symmetric positive definite; the diagonal is bounded, in the sense that
	$
	\|K\|_\infty:=\sup_{x\in\cX}K(x,x)<\infty.
	$
\end{assumption}

The estimator is ridge-regularized kernel logistic regression (KLR):
\[
 \widehat f_\lambda
:=
\arg\min_{f\in\cH}
\left\{
\frac1n\sum_{i=1}^n \loss(Y_i f(X_i))
+
\frac{\lambda}{2}\norm{f}^2_{\cH}
\right\},
\qquad
\loss(u):=\log(1+e^{-u}),
\]
where $\lambda>0$.

The same kernel formulation is obtained from the frozen-feature transformer used in the experiments.  At finite width, if \(\Psi_\theta(x)\) is the frozen feature vector before the trainable unembedding, then fitting the unembedding with the same logistic loss and ridge penalty is linear logistic regression in \(\Psi_\theta\), or equivalently KLR with the empirical transformer kernel
\(
  \hat K_{\mathrm{trans}}(x,x')=\langle \Psi_\theta(x),\Psi_\theta(x')\rangle .
\)
Appendix~\ref{app:transformer-kernel-interface} proves the finite-sample random-feature limit \(\hat K_{\mathrm{trans}}\to K_{\mathrm{trans}}\) in probability on every fixed dataset.  This places the model in the same lazy-training/random-feature tradition as NTK and CNTK analyses, while keeping only the final linear head trainable \cite{JacotGabrielHongler2018,AroraEtAl2019CNTK,RudiRosasco2017RandomFeatures,ChizatBach2019}.  Hence the infinite-width frozen-transformer model corresponds to the special choice \(K=K_{\mathrm{trans}}\); the results below are stated for a general token-symmetric \(K\) to separate the collision analysis from this random-feature limit.

Under token symmetry, all fresh realizations of the same template are equivalent from the perspective of the kernel. Consequently, the ideal classification problem can be expressed entirely in terms of template identities rather than concrete strings. By the representer theorem, the corresponding RKHS minimizer lies in the span of the template-level kernel evaluations, and its prediction on every fresh string $x \in \cX_a^{\fr}$ reduces to a scalar score $g_a$. The signed quantity $y_a g_a$ therefore plays the role of a template-level classification margin.

Here we formulate this idea rigorously. The ideal finite-dimensional problem replaces concrete strings by template identities. For \(q\in\Delta_r\), define
\[
\Phi_{q,\lambda}(g)
:=
\sum_{a=1}^r q_a\,\loss(y_a g_a)
+
\frac\lambda2 g^\top\bN^{-1}g,
\qquad g\in\R^r.
\]
The objective $\Phi_{q,\lambda}$ is the quotient-space analogue of kernel logistic regression. The first term measures logistic classification error at the template level, weighted by template frequencies, while the quadratic penalty is induced by the RKHS norm through the template kernel matrix $\bN$. 
Let \(g_\lambda^{\id}\) be the minimizer of \(\Phi_{\widehat p,\lambda}\), where $\hat{p} \in \R^r$ indicates the empirical mass of each template. For a globally fresh test point \(x\in\cX_a^{\fr}\), the ideal scorer is
$
S_\lambda^{\id}(x)=(g_\lambda^{\id})_a.
%\  x\in\cX_a^{\fr}.
$
The small-ridge argument is entirely finite-dimensional: the logistic objective can drive the signed template scores past any fixed threshold before the RKHS penalty becomes dominant.  Let
\[
  R_*^2
  :=
  \min\{g^\top\bN^{-1}g:\ y_a g_a\ge1\text{ for all }a\}.
\]

\begin{proposition} 
\label{prop:template_margin}
Fix \(\underline p>0\) and \(\Gamma\ge0\).  There is
\(\lambda_\Gamma(\underline p)>0\) such that, whenever
\(q_a\ge\underline p\) for all \(a\) and
\(0<\lambda<\lambda_\Gamma(\underline p)\), the minimizer
\(g_{q,\lambda}\) of \(\Phi_{q,\lambda}\) satisfies
\[
  \min_a y_a(g_{q,\lambda})_a>\Gamma .
\]
 \end{proposition}

The following theorem turns this deterministic separation statement into the ideal margin used by the collision-graph transfer theorem.   

\begin{theorem} 
\label{thm:ideal-margin}
\label{thm:idealmargin}
For every \(\Gamma\ge0\), there exists
\(\lambda_\Gamma^{\id}>0\), depending only on
\(\Gamma,p_{\min},\bN\), and \(y\), such that on \(E_{\rm rep}\),
every \(0<\lambda<\lambda_\Gamma^{\id}\) satisfies
\[
  y_a S_\lambda^{\id}(x)>\Gamma
  \qquad
  \text{for all }a\in[r]\text{ and all }x\in\cX_a^{\fr}.
\]
Moreover, for $  E_{\rm rep}:=\{\widehat p_b\ge p_b/2 \text{ for all }b\}.
$ we have 
$
  \Pp(E_{\rm rep})\ge 1-r e^{-np_{\min}/8}.
$
\end{theorem}
 
To analyze the margin of $\hat{f}_{\lambda}$, we want to show that the template scorer behaves as an idealized version of empirical predictor, and derive an upper bound on their difference.  
 Let
\[
    \widehat K_n=(K(X_i,X_j))_{i,j=1}^n,
    \qquad
    k_x=(K(X_1,x),\ldots,K(X_n,x))^\top .
\]
where $x\in\cX$ is a test point. Define the binary entropy potential
$
\varphi(c):=c\log c+(1-c)\log(1-c),\  c\in[0,1].
$ Unlike square loss, logistic loss does not yield a linear system for the kernel
coefficients, so the standard kernel-ridge perturbation argument cannot be used
directly.  This is the point at which our analysis departs from least-squares RKHS theory and from risk-consistency treatments of logistic classification \cite{CaponnettoDeVito2007,Zhang2004ConvexRisk,BartlettJordanMcAuliffe2006,SteinwartChristmann2008}.  Our key analytic device is to work with the entropy dual of kernel
logistic regression, which is closely related to the convex-analytic and self-concordant viewpoint on logistic loss \cite{Bach2010LogisticSelfConcordant,WenBetkenHang2025KLR}.
The following dual statement is the analytic interface between kernel logistic regression and the collision graph.  It says that the trained scorer is controlled by a bounded entropy-dual vector, rather than by an unconstrained linear-system solution.

\begin{theorem} 
\label{thm:dual}
Let
\[
  D_{\widehat K_n}(c)
  :=
  {1\over n}\sum_i\varphi(c_i)
  +
  {1\over 2\lambda n^2}
  (\mathbf y\odot c)^\top
  \widehat K_n
  (\mathbf y\odot c),
  \qquad c\in[0,1]^n .
\]
Then \(D_{\widehat K_n}\) has a unique minimizer
\(c_{\widehat K_n}\in(0,1)^n\), and
\[
  \widehat f_\lambda(\cdot)
  =
  {1\over\lambda n}
  \sum_{i=1}^n c_{\widehat K_n,i}Y_iK(X_i,\cdot)
\qquad \mbox{and} \qquad 
  \widehat f_\lambda(x)
  =
  {1\over\lambda n}
  k_x^\top(\mathbf y\odot c_{\widehat K_n}).
\]
\end{theorem}

To compare the empirical predictor with the ideal template scorer, we introduce an idealized sample-level kernel matrix obtained by replacing empirical token interactions with their template-level counterparts. Let $\tau(i)\in\{1,\dots,r\}$ be the latent template index of the $i$th training observation, so that $Z_i=z_{\tau(i)}$. Define the ideal sample-level block matrix and the corresponding test vector by
$$
M_n:=\bigl(N_{\tau(i)\tau(j)}\bigr)_{i,j=1}^n,\ 
m_a:=\bigl(N_{a,\tau(i)}\bigr)_{i=1}^n.
$$
At the ideal level, let \(c_{M_n}\) be the dual minimizer associated with \(M_n\).  The block-constant structure of \(M_n\) identifies this sample-level dual problem exactly with the finite template problem.
We next present the exact ideal block reduction.
\begin{theorem}
\label{thm:idealreduction}
Let \(x\in\cX_a^{\fr}\).  If \(c_{M_n}\) is the dual minimizer associated
with \(M_n\), then
\[
  {1\over\lambda n}
  m_a^\top(\mathbf y\odot c_{M_n})
  =
  S_\lambda^{\id}(x).
\]
If a block is empty, the corresponding coordinate \(\theta_b\) may be chosen
arbitrarily; all expressions above are independent of this choice because
\(\widehat p_b=0\).
\end{theorem}

Thus \(S_\lambda^{\mathrm{id}}(x)\) is the
template-level analogue of the empirical predictor \(\widehat f_\lambda(x)\),
but with all concrete token identities removed.
The empirical logistic scorer can be compared directly to the template scorer by
\[
\widehat f_\lambda(x)-S_\lambda^{\id}(x)
=
\frac1{\lambda n}\zeta^\top(\mathbf y\odot c_{M_n})
+
\frac1{\lambda n}k_x^\top \mathbf y\odot(c_{\widehat K_n}-c_{M_n}).
\]
The first term is the direct test--train collision contribution; the second term is the movement of the logistic dual optimizer caused by the train--train Gram perturbation.  The following two deterministic estimates are the main analytic tools used later by the graph certificates.

For any matrix \(G\succeq0\),  let \[
  s_{k,G}={1\over\lambda n}k^\top(\mathbf y\odot c_G)
\]
where 
\(c_G, \) is the corresponding dual minimizer of \(D_{ G}\).

\begin{lemma} 
\label{thm:detperturb}
Let \(G,M\succeq0\), let \(k,m\in\mathbb R^n\), and let
\(c_G,c_M\) be the corresponding dual minimizers.  If \(|k_i|,|m_i|\le \kappa^2\) for all \(i\), then
\[
  |s_{k,G}-s_{m,M}|
  \le
  {\kappa^2\over4\lambda^2 n}\|G-M\|_{\op}
  +
  {1\over\lambda\sqrt n}\|k-m\|_2 .
\]
\end{lemma}

Statistically, Lemma~\ref{thm:detperturb} is a deterministic stability bound
for a fixed design.  The first term measures how much the fitted score changes
when the empirical kernel design is perturbed, and the second term measures the
change in the test representer.  For kernel ridge regression this type of
comparison follows from a linear normal equation, but logistic regression has no
such linear coefficient map \cite{SchoelkopfHerbrichSmola2001,CaponnettoDeVito2007}.  The entropy dual supplies the missing curvature: the Bernoulli entropy
has uniform second derivative lower bound on the dual cube, so the dual
optimizer is Lipschitz in the Gram matrix without assuming a margin condition or
an invertible unregularized Fisher matrix.  This differs from the usual
consistency or excess-risk perspective for convex classification losses, where
the main object is population risk rather than a finite-sample out-of-support
score \cite{Zhang2004ConvexRisk,BartlettJordanMcAuliffe2006,SteinwartChristmann2008}.  

\begin{theorem} 
\label{thm:curvature-spectral-perturbation}
Let \(G,M\succeq0\), let \(k,m\in\mathbb R^n\), and let
\(c_G,c_M\) be the associated dual minimizers.  Set $  \mathbf Y:=\diag(\mathbf y), \Delta:=G-M$ and 
%\[
%  s_G={1\over\lambda n}k^\top(\mathbf y\odot c_G),
%  \qquad
%  s_M={1\over\lambda n}m^\top(\mathbf y\odot c_M),
%\]
\[
   C_M:={4\over n}I+{1\over\lambda n^2}\mathbf YM\mathbf Y,
  \qquad
  E_\Delta:={1\over\lambda n^2}\mathbf Y\Delta\mathbf Y , \qquad  b:={1\over\lambda n^2}\mathbf Y\Delta\mathbf Yc_M .
\] 
For \(a\in\mathbb R^n\), set
\[
  \mathcal E_M(a;b)
  :=
  {1\over 2(1-\gamma_{G|M})}
  \left(
    |a^\top C_M^{-1}b|
    +
    \sqrt{a^\top C_M^{-1}a}
    \sqrt{b^\top C_M^{-1}b}
  \right), \qquad  \gamma_{G|M}
  :=
  \bigl\|
    C_M^{-1/2}E_\Delta C_M^{-1/2}
  \bigr\|_{\op},
\]
when \(\gamma_{G|M}<1\), and set
\(\mathcal E_M(a;b)=+\infty\) otherwise.
With
$
  \zeta:=k-m,
  \
  a_m:={1\over\lambda n}\mathbf Ym,
  \
  a_\zeta:={1\over\lambda n}\mathbf Y\zeta,
$ 
\begin{equation}
\label{eq:curvature-spectral-perturbation-main}
  |s_G-s_M|
  \le
  {1\over\lambda n}
  |\zeta^\top\mathbf Yc_M|
  +
  \mathcal E_M(a_m;b)
  +
  \mathcal E_M(a_\zeta;b).
\end{equation}
\end{theorem}

The curvature-spectral estimate is a local, information-geometric refinement of
the preceding stability bound.  Here
\(C_M\) is the ideal dual Hessian, the regularized finite-sample analogue of a
Fisher information matrix for the logistic problem, and
$
  b 
$
is the score perturbation induced by train--train collisions.  Thus we control \(\Delta\)  action on
the ideal residual vector \(\mathbf Yc_M\).  The factor
\(C_M^{-1}\) is the statistical sensitivity: perturbations lying in
high-curvature directions are damped, while perturbations aligned with
weak-curvature directions can move the margin.   

  The curvature argument here is not a standard self-concordant or convex-risk
analysis of logistic regression.  Generic convexity or self-concordance analysis 
\cite{Bach2010LogisticSelfConcordant,WenBetkenHang2025KLR} would control the movement
of the estimator through a global norm of the Gram perturbation, such as
\(\|\Delta\|_{\mathrm{op}}\), and would therefore charge all token collisions
uniformly.  This is too coarse for fresh-symbol classification.
  Two samples can have the same scalar diversity \(\rho\) and
similar edge counts but different classification behavior because they produce
different vectors \(b\) after weighting by \(\mathbf Yc_M\) and filtering by
\(C_M^{-1}\).

The dual variables also have a useful template-importance interpretation.  Since
\(M_n\) is block-constant, the ideal dual optimizer is constant on each
template block: \(c_{M_n,i}=\theta_{\tau(i)}\). Therefore the ideal score can
be rewritten for every \(x\in\cX_a^{\fr}\) as
\[
    S_\lambda^{\id}(x)
    =
    \sum_{b=1}^r
    \widehat p_b\,N_{ab}\,y_b\,\theta_b / \lambda.
\]

\section{Collision-Graph Control of Fresh-Symbol Generalization}\label{subsec:collision_graph}

The empirical error \(\widehat f_\lambda(x)-S_\lambda^{\id}(x)\) is driven by the train--train Gram error \(\Delta:=\widehat K_n-M_n\) and the test--train error \(\zeta:=k_x-m_a\), for \(x\in\cX_a^{\fr}\).  For off-diagonal train pairs, these errors can occur only through substitution collisions: if \(X_i\) and \(X_j\) are jointly fresh relative to their templates, then token symmetry gives \(K(X_i,X_j)=N_{\tau(i),\tau(j)}\), and hence \(\Delta_{ij}=0\).  Similarly, if the fresh test string \(x\) does not collide with \(X_i\), then \(K(x,X_i)=N_{a,\tau(i)}\), and hence \(\zeta_i=0\).  Thus the collision graph records the support of the finite-sample perturbation away from the ideal template problem.  The scalar diversity parameter \(\rho:=\min_{1\le a\le r}\min_{t\in\cV}\{1/\Pp_{S\sim\mu_{\mathrm{sub},a}}[t\in S(\cW(z_a))]\}\) controls the marginal probability of token reuse, and is the quantity used for block approximation in the square-loss theory of \cite{Boix2024}.  For logistic classification, however, this marginal rate is not enough: the effect of collisions also depends on where they occur, which template colors they connect, and how their signed kernel errors act on the dual weights.  The colored collision graph keeps this finite-sample geometry.  The distinction mirrors the gap between edge density and graph discrepancy or spectral behavior in random-graph theory, but here the graph is colored, signed, and filtered through the logistic dual \cite{ChungGrahamWilson1989,BoucheronLugosiMassart2013,Tropp2015}.

%The empirical error
%$
%    \widehat f_\lambda(x)-S_\lambda^{\id}(x)
%$
%is controlled by two discrepancies: the train--train Gram error
%\(\Delta=\widehat K_n-M_n\) and the test--train error
%\(\zeta=k_x-m_x\). The only source of discrepancy is accidental reuse of concrete symbols.  If two %realized strings are jointly fresh, token symmetry makes their empirical
%kernel value equal to the corresponding template-kernel value.  Thus the Gram
%error can occur only on colliding pairs.  The collision graph records precisely
%these finite-sample locations where the ideal template reduction may fail.
% The usual scalar measure of substitution diversity is
%$
%\rho
%:=
%\min_{1\le a\le r} \ \min_{t\in\cV}   $ $ 
%{1}/{\Pp_{S\sim\mu_{\mathrm{sub},a}}\bigl[t\in S(\cW(z_a))\bigr]}.
%$
%Large \(\rho\) means that no token is too likely to occupy a wildcard slot.
%This is the same quantity that controls block approximation in the square-loss
%theory of \cite{Boix2024}. For logistic classification, however, the marginal collision rate is not
%enough.  Two samples can have the same value of \(\rho\), or even the same
%number of collision pairs, but induce very different perturbations of the
%learned classifier.  The colored collision graph records the finite-sample locations where the ideal
%template reduction can fail. 

\begin{definition}[Training and test collision graphs]
	\label{def:collision-graph}
	For \(i\ne j\), with \(\tau(i)=b\) and \(\tau(j)=c\), let \(\Bad_{ij}\) be the event that the two realized strings \(X_i,X_j\) are not jointly fresh relative to the template pair \((z_b,z_c)\).  Equivalently, at least one wildcard token in one string hits a literal token in the other string, or the two wildcard substitution sets share a token.  Define
	\[
	\Ccol_{ij}:=\1_{\Bad_{ij}},\qquad \Ccol_{ii}:=0 .
	\]
	For a fixed fresh test string \(x\), define \(\Bad_{0i}\) analogously for the pair \((x,X_i)\), and set \(\Ccol_{0i}=\1_{\Bad_{0i}}\).  The graph \(\Gcol=(\{0,1,\ldots,n\},\Ccol,\tau)\) is the colored collision graph. 
\end{definition}

\begin{remark}[From scalar diversity to collision geometry]
\label{rem:rho-to-graph-main}
For uniform injective substitutions of \(w_a\) active wildcards into an alphabet of size \(m\),
\[
    \rho_a={m\over w_a},
    \qquad
    \rho={m\over w_{\max}},
    \qquad
    w_{\max}:=\max_a w_a .
\]
For two templates using \(w_a\) and \(w_b\) active wildcards, the probability that their wildcard sets intersect is
\[
    q^{\rm wild}_{ab}(m)
    =
    1-{\binom{m-w_a}{w_b}\over\binom{m}{w_b}}
    \le
    {w_aw_b\over m-w_b+1}.
\]
When \(w_a=w_b=w\), this scale is \(q^{\rm wild}_{ab}(m)\simeq w^2/m=w/\rho\).  Thus \(\rho\) gives a first-order collision rate, but the graph refines it in four ways: it keeps the full color-pair matrix rather than a worst-case scalar; it records finite-sample degrees and hubs; it keeps signed kernel discrepancies so cancellations are not lost; and it separates deterministic literal overlap from random wildcard collisions.
\end{remark}

 The empirical problem differs from the ideal template problem only through the signed
collision graph
$
\Wcol=(\Delta,\zeta),
$
\(\Delta\) and \(\zeta\) vanish away from collision edges. If \(\Bad_{ij}\) does not occur, then the wildcard tokens appearing in \(X_i\) and \(X_j\) are distinct from each other and from the other template's literal tokens. For \(i\ne j\),
$
\Bad_{ij}^c\ \Longrightarrow\  \Delta_{ij}=0 .
$
For the test edge,
$
\Bad_{0i}^c\ \Longrightarrow\  \zeta_i=0 .
$

\begin{center}
	\resizebox{0.86\textwidth}{!}{%
		\begin{tikzpicture}[
			x=0.92cm,y=0.92cm,
			every node/.style={font=\scriptsize},
			tok/.style={draw,rounded corners=0.8pt,minimum height=3.8mm,minimum width=6.2mm,inner sep=0.6pt,font=\scriptsize},
			lit/.style={tok,fill=black!5},
			wild/.style={tok,fill=blue!7},
			tinylabel/.style={font=\tiny}
		]
			% Column headings
			\node[font=\bfseries\scriptsize,anchor=west] at (0.00,4.15) {(A) strings};
			\node[font=\bfseries\scriptsize,anchor=west] at (5.35,4.15) {(B) collision};
			\node[font=\bfseries\scriptsize,anchor=west] at (10.45,4.15) {(C) graph};
			
			% Panel A: short strings
			\node[tinylabel,anchor=west] at (0.00,3.72) {$z=[\mathtt{if},\alpha,\mathtt{then},\beta]$};
			\node[tinylabel,anchor=east] at (0.65,3.10) {$X_i=$};
			\node[lit]  (xi1) at (1.10,3.10) {$\mathtt{if}$};
			\node[wild] (xi2) at (1.80,3.10) {$a$};
			\node[lit]  (xi3) at (2.60,3.10) {$\mathtt{then}$};
			\node[wild] (xi4) at (3.55,3.10) {$b$};
			
			\node[tinylabel,anchor=east] at (0.65,2.30) {$X_j=$};
			\node[lit]  (xj1) at (1.10,2.30) {$\mathtt{if}$};
			\node[wild] (xj2) at (1.80,2.30) {$c$};
			\node[lit]  (xj3) at (2.60,2.30) {$\mathtt{then}$};
			\node[wild] (xj4) at (3.55,2.30) {$b$};
			
			\node[tinylabel,anchor=east] at (0.65,1.50) {$X_k=$};
			\node[lit]  (xk1) at (1.10,1.50) {$\mathtt{if}$};
			\node[wild] (xk2) at (1.80,1.50) {$d$};
			\node[lit]  (xk3) at (2.60,1.50) {$\mathtt{then}$};
			\node[wild] (xk4) at (3.55,1.50) {$e$};
			
			\draw[red!75!black,thick] (xi4.south) -- (xj4.north);
			\node[tinylabel,red!75!black,anchor=west] at (4.15,2.70) {$b=b$};
			\node[tinylabel,anchor=west] at (0.15,0.72) {$S_i\cap S_j=\{b\}$,\qquad $S_i\cap S_k=\varnothing$};
			
			% Flow arrow A -> B
			\draw[-{Latex[length=2.3mm]},thick] (4.35,2.45) -- (5.15,2.45);
			
			% Panel B: compact collision test
			\node[draw,rounded corners,fill=orange!5,anchor=north west,inner sep=3pt] at (5.35,3.40) {\(
				\begin{array}{c|c|c}
					(u,v) & S_u\cap S_v & (\Ccol_{uv},\Wcol_{uv})\\ \hline
					(i,j) & \{b\} & (1,\Delta_{ij})\\[-0.05em]
					(i,k) & \varnothing & (0,0)
				\end{array}
				\)};
			
			% Flow arrow B -> C
			\draw[-{Latex[length=2.3mm]},thick] (10.35,2.45) -- (11.15,2.45);
			
			% Panel C: graph
			\node[draw,circle,minimum size=5.8mm,fill=blue!10] (gi) at (11.75,2.95) {$i$};
			\node[draw,circle,minimum size=5.8mm,fill=blue!10] (gj) at (13.35,2.95) {$j$};
			\node[draw,circle,minimum size=5.8mm,fill=blue!10] (gk) at (12.65,1.45) {$k$};
			\draw[red!75!black,thick] (gi)--(gj) node[midway,above,tinylabel] {$\Delta_{ij}$};
			\draw[gray!60,dashed] (gi)--(gk);
			\draw[gray!60,dashed] (gj)--(gk);
			\node[tinylabel,align=center] at (12.55,0.72) {$\Ccol=\supp(\Wcol)$};
			
			% Bottom support lemma
			\node[draw,rounded corners,fill=yellow!10,align=center,inner sep=3pt] at (6.80,-0.05) {\(
				\Bad_{uv}^{c}\Longrightarrow K(X_u,X_v)=N_{\tau(u),\tau(v)}\Longrightarrow \Delta_{uv}=0
				\)};
		\end{tikzpicture}%
	}
	
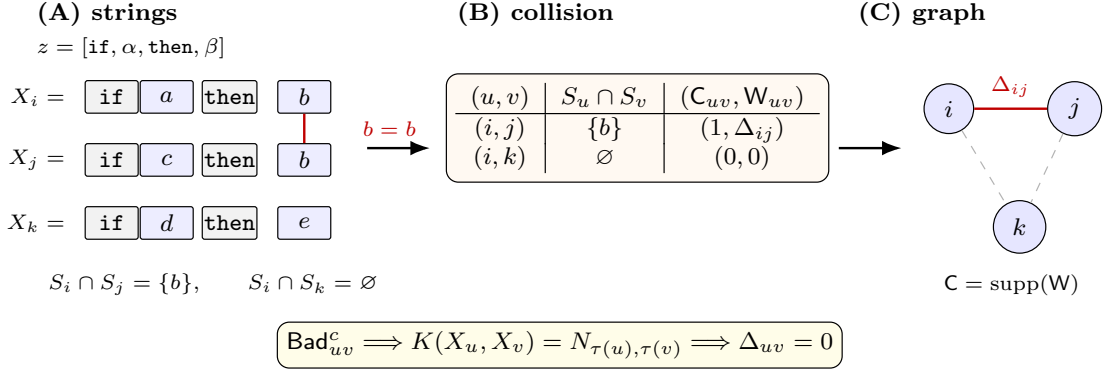
\captionof{figure}{\textbf{What the collision graph records.}  A graph edge appears only for a non-fresh pair.  Here $(i,j)$ collides through the wildcard token $b$, while $(i,k)$ is fresh and has zero Gram discrepancy.}
	\label{fig:collision-overview}
\end{center}

\begin{remark}[Equality selectors and signed collision weights]
\label{rem:selector-view-main}
The collision graph can be read as a sample-level abstraction of equality selectors.  This connects the present perturbation view to programmatic transformer descriptions and to induction-head or variable-binding mechanisms that also rely on equality matching \cite{WeissGoldbergYahav2021RASP,OlssonEtAl2022,WuGeigerMilliere2025,YangEtAl2025SymbolicMechanisms}.  A RASP-style selector has the form
\[
    S_{pq}=\1\{x_p=x_q\},
\]
where \(p,q\) are sequence positions.  The sample-level edge aggregates this predicate over wildcard slots:
\[
    \Ccol_{ij}
    =
    \1\{S_i(\cW(z_{\tau(i)}))\cap S_j(\cW(z_{\tau(j)}))\ne\emptyset\}.
\]
Thus \(\Ccol\) records where equality information is available, while the signed object \(\Wcol=(\Delta,\zeta)\) records how the kernel weights those equalities.  The curvature-spectral certificate asks whether the signed selector-induced perturbation aligns with weakly curved directions of the logistic objective.
\end{remark}

%\section{Generalization via Collision-Induced Perturbations}

%\subsection{Margin for ideal template scorer}
%
%{\color{red}\textbf{Remark:} The derivation of the target function $\Phi_{q, \lambda}$ should probably be explained in more detail. We need to at least explain why $g_j$ is the margin with representer theorem.}
%
%For a probability vector $q=(q_1,\dots,q_r)$ and ridge parameter $\lambda>0$, define
%\[
%\Phi_{q,\lambda}(g)
%:=
%\sum_{a=1}^r q_a\,\loss(y_a g_a)
%+
%\frac{\lambda}{2}g^\top \bN^{-1}g,
%\qquad g\in\R^r.
%\]
%Write $m_{\min}(g):=\min_{1\le a\le r} y_a g_a$. Let $\hat p_a=n^{-1}\#\{i:Z_i=z_a\}$ be the empirical template frequencies, let $g_\lambda^{\mathrm{id}}$ denote the minimizer of $\Phi_{\hat p,\lambda}$, and define the ideal empirical score by
%\[
%S_\lambda^{\mathrm{id}}(x):=(g_\lambda^{\mathrm{id}})_a\qquad\text{for }x\in\cX_a.
%\]
%
%The following theorem suggests the trained template scorer will achieve a positive lower bound with high probability.
%
%\begin{theorem}[Ideal empirical scorer has positive template margin]
%	\label{thm:idealmargin}
%	Fix $\Gamma\ge 0$. Then there exists
%	$
%	\lambda_\Gamma^{\mathrm{id}}>0,
%	$
%	depending only on $\Gamma$, $p_{\min}$, $\bN$, and $y$, such that on the event $E_{\mathrm{rep}}$, every ridge parameter $0<\lambda<\lambda_\Gamma^{\mathrm{id}}$ satisfies
%	\[
%	y_a\,S_\lambda^{\mathrm{id}}(x)>\Gamma
%	\qquad\text{for every }a\in\{1,\dots,r\}\text{ and every }x\in\cX_a.
%	\]
%	Moreover,
%	$
%	\Pp(E_{\mathrm{rep}})\ge 1-r e^{-n p_{\min}/8}.
%	$
%\end{theorem}

The role of the collision graph can be seen from the linearized dual equations.
Writing \(c_M:=c_{M_n}\), the first-order
variation of the dual optimizer around the ideal matrix \(M_n\) is
\[
    c_{\widehat K_n}-c_M
    \approx
    -
    C_M^{-1}
    \left(
        {1\over \lambda n^2}Y\Delta Y
    \right)c_M,
    \qquad
    C_M
    :=
    {4\over n}I
    +
    {1\over \lambda n^2}YM_nY .
\]
Thus train--train collisions affect the classifier through the action of
\(\Delta\) on the ideal signed dual weights \(Yc_M\), filtered by the local
curvature \(C_M\) of the logistic dual objective.  The test--train collisions
enter separately through the term
$
    {1\over \lambda n}\zeta^\top Yc_M .
$

The next two corollaries translate this generic perturbation formula into
quantities that are visible in the colored collision graph. Let \(P\) be the sample-template membership matrix, \(P_{ib}=\1_{\tau(i)=b}\),
and let \(Q=\diag(\widehat p)\).  Because the ideal Gram matrix is
block-constant, \(M_n=P\bN P^\top\), and the ideal dual optimizer is also
block-constant, \(c_M=P\theta\).  We write
\[
    \beta:=y\odot\theta,
    \qquad
    Y_r:=\diag(y_1,\ldots,y_r),
    \qquad
    n_a^{\rm tmplt}:=(N_{a1},\ldots,N_{ar})^\top .
\]
For any matrix perturbation \(\Delta\), define its color-pair block average by
$
    \overline\Delta_{bc}
    :=
    {1\over n_bn_c}
    \sum_{i\in I_b}\sum_{j\in I_c}\Delta_{ij}.
$

The vector \(g_\Delta:=Y_r\overline\Delta Q\beta\) is the quotient-level
train--train collision action: it first averages the empirical Gram error over
template color pairs and then weights the result by the signed ideal dual
coordinates.  The matrix \(R_{0,\lambda,\widehat p}\) below is the inverse
curvature operator after restricting the logistic dual Hessian to the
block-constant, template-level subspace.  The residual \(b_\perp\) is the part
of the train--train perturbation that is invisible to block averages and must be
controlled by row, degree, or spectral graph estimates.

\begin{corollary}
\label{cor:curvature-template-identities}
Assume \(M=M_n\), \(m=m_a=Pn_a^{\rm tmplt}\),
and \(c_M=P\theta\).  Suppose also that every block is nonempty.  Let
\[
  s_a:=Y_r n_a^{\rm tmplt},
 \quad \mbox{and} \quad
  \mathcal A_{0,\lambda,\widehat p}
  :=
  4I
  +
  {1\over\lambda}
  Y_rQ^{1/2}\bN Q^{1/2}Y_r,
  \qquad
  R_{0,\lambda,\widehat p}
  :=
  Q^{1/2}\mathcal A_{0,\lambda,\widehat p}^{-1}Q^{1/2}.
\]
If \(b_\parallel=P\bar b\) is the block-constant projection of \(b\), with
$
  \bar b={1\over\lambda n}g_\Delta,
  \
  b_\perp:=b-b_\parallel,
$
then
\begin{align}
\label{eq:curvature-spectral-template-identities}
  a_m^\top C_M^{-1}a_m
  &=
  {1\over\lambda^2}
  s_a^\top R_{0,\lambda,\widehat p}s_a,
 \qquad
  a_m^\top C_M^{-1}b
   =
  {1\over\lambda^2}
  s_a^\top R_{0,\lambda,\widehat p}g_\Delta,
  \nonumber\\
  b^\top C_M^{-1}b
  &=
  {1\over\lambda^2}
  g_\Delta^\top R_{0,\lambda,\widehat p}g_\Delta
  +
  {n\over4}\|b_\perp\|_2^2 .
\end{align}
\end{corollary}

The identities separate three statistically distinct objects.  The first term is
the curvature leverage of the fresh test template.  The second is the signed
alignment between the block-averaged collision perturbation and that fresh-template
direction.  The third decomposes the total curvature-weighted perturbation into a
quotient-level block action, \(g_\Delta\), and a within-block residual,
\(b_\perp\). 

To state the resulting transfer bound without interrupting the main argument,
we collect the exact envelope definitions in Appendix~\ref{sec:graph-budgets}.
Here \(A_\Delta\) denotes the train--train action size,; \(E_*(u)\) is the worst block-average
Gram-error envelope; \(X_\star(u)\) is the aggregate test--train collision
envelope; and \(\gamma_{\cspec}\) is the relative curvature size of the Gram
perturbation.

\begin{corollary} 
\label{cor:Mn-level-curvature-spectral-upper-bound}
Condition on a template assignment with \(n_b>0\) for all \(b\), and fix
\(x\in\cX_a^{\fr}\).  With
\(\Delta,\zeta,C_M,\gamma_{\cspec},d_i,d_2\), and \(A_\Delta\) as in
Section \ref{sec:graph-budgets}, on the KL block/test event of
Theorem \ref{thm:klbdens-kl-block-envelope},
\[
  \bigl|
    \widehat f_\lambda(x)-S_\lambda^{\id}(x)
  \bigr|
  \le
  \operatorname{Curv}_\lambda
  \bigl(A_\Delta,E_*(u),X_\star(u),\gamma_{\cspec}\bigr)
  =
  B_{\cspec}(\lambda;x,u)
\]
for every \(u>0\).  In particular, when \(\gamma_{\cspec}<1\),
\[
  \bigl|
    \widehat f_\lambda(x)-S_\lambda^{\id}(x)
  \bigr|
  \le
  {L_*\over\lambda}X_\star(u)
  +
  {K_*E_*(u)
    +(K_*+2L_*\sqrt{X_\star(u)})A_\Delta
  \over
  8\lambda^2(1-\gamma_{\cspec})}.
\]
%Also,
%\begin{equation}
%\label{eq:Mn-level-gamma-simple-bound}
%  \gamma_{\cspec}
%  \le
%  {\|\Delta\|_{\op}\over4\lambda n}.
%\end{equation}
\end{corollary}
 
The bound has a direct statistical interpretation.  The term
\(L_*X_\star(u)/\lambda\) is the direct test--train collision contribution.
The term involving \(E_*(u)\) is the block-level train--train distortion seen by
the fresh template.  The term involving \(A_\Delta\) is the residual action of
the empirical Gram perturbation after the block projection.  The denominator
\(1-\gamma_{\cspec}\) is a curvature safety factor: if the Gram perturbation is
small relative to the ideal logistic Hessian, the local sensitivity calculation
is stable. 

\begin{remark}[Literal backbones and bias--fluctuation]
\label{rem:literal-backbone-main}
Literal tokens show another limitation of scalar diversity.  Two template families can use the same number of wildcards, and hence have the same \(\rho\), while one contains deterministic literals such as \(\mathtt{if}\), \(\mathtt{then}\), or \(\mathtt{print}\).  The signed discrepancy can then be decomposed as
\[
    \Wcol
    =
    \underbrace{\Back_{\lit}}_{\text{shared-literal scaffold}}
    +
    \underbrace{\Res_{\fr}}_{\text{centered fresh-token residual}}.
\]
The residual \(\Res_{\fr}\) decreases with the alphabet size, whereas the literal backbone \(\Back_{\lit}\) persists even when wildcard collisions are rare.  The bias--fluctuation certificate uses this decomposition by comparing the empirical problem to a literal-corrected population target and then concentrating the centered residual.
\end{remark}

We now state the two main results.  The first is a transfer result: if the
realized collision graph is benign, then the empirical logistic classifier
inherits the ideal template margin.  
For a fresh test string $x\in\cX_a^{\fr}$ and confidence level
$\delta\in(0,1)$, and 
$
u_\delta=\log{11r^2+5r\over\delta}
$
let 
\begin{equation}
	\label{eq:main-new-budget}
	B^\sharp_\lambda(x,\delta)
	:=
	\min\{B_{\cspec},B_{\degact},B_{\bdens},B_{\hanova},B_{\bfdec}\}(\lambda;x,u_\delta),
\end{equation}
with \(B^\sharp_\lambda(x,\delta)=+\infty\) outside
$
E_{\rm rep}:=\{\widehat p_b\ge p_b/2 \text{ for all } b\}, $
denote the minimum of the five collision-graph certificates:
curvature-spectral, degree-action, block-density, Hoeffding--ANOVA, and
bias--fluctuation.

\begin{theorem}[Fresh-symbol classification from the collision graph]
	\label{thm:main}
	Assume the standing model conditions above.  Fix a globally fresh test string \(x\in\cX_a^{\fr}\), a confidence level \(\delta\in(0,1)\), and a ridge parameter \(\lambda>0\).  Then
	\[
	\Pp\!\left(
	\bigl|\widehat f_\lambda(x)-S_\lambda^{\id}(x)\bigl|
	\le B^\sharp_\lambda(x,\delta)
	\right)
	\ge
	1-r e^{-n p_{\min}/8}-\delta.
	\]
\end{theorem}
In particular, whenever $\lambda$ is small enough that the ideal template
margin exceeds $\varepsilon+s$, a certificate of size at most $s$, $ B^\sharp_\lambda(x,\delta)\le s$ guarantees
correct fresh-symbol classification with margin $\varepsilon$, $y_a\widehat f_\lambda(x)>\varepsilon$.
\begin{figure}[t]
	\centering
	\includegraphics[width=\linewidth]{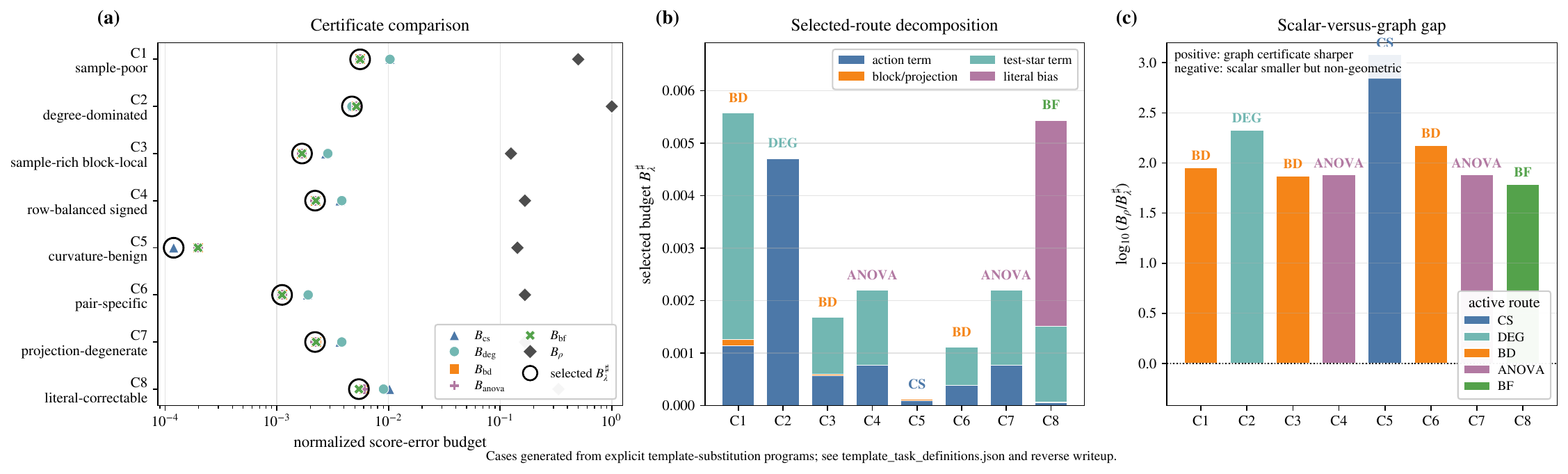}
	\caption{\textbf{Collision-graph bounds.} Eight explicit finite template tasks are generated by wildcard substitutions and literal hits, after which
the colored collision graph and all five Theorem~\ref{thm:main} certificates are computed.  Panel (a) compares the
five graph certificates to the scalar proxy $B_\rho=w_{\max}/\rho$.  Panel (b) decomposes the selected
route.  Panel (c) plots $\log_{10}(B_\rho/B^\sharp_\lambda)$, with positive bars indicating an improvement
over the scalar proxy.}
	\label{fig:reverse-template-certificates}
\end{figure}
Figure~\ref{fig:reverse-template-certificates} is a mechanism check for the certificate, not a new benchmark.  All eight examples use the same three-template family and labels; only the concrete substitutions are changed, so the ideal template problem is held fixed while the collision geometry varies.  The selected route tracks the obstruction that remains after the ideal margin is factored out: block-density wins when the collisions are few or confined to relevant color pairs; degree-action wins for a dense core whose main effect is row action; ANOVA wins when signed collisions cancel in the block projections; curvature-spectral wins for a single weak perturbation; and bias--fluctuation wins when many apparent edges come from a literal-induced backbone that can be absorbed into a corrected target.  This explains the gap in Panel~(c): \(B_\rho=w_{\max}/\rho\) sees only the worst marginal token-reuse scale, whereas \(B^\sharp_\lambda\) asks which realized collisions survive the dual weights, the template colors, and the local logistic curvature.  Full construction details appear in Appendix~\ref{sec:supp-reverse-template-examples}.

The second result shows that this random certificate is itself controlled
with high probability by explicit edge--wedge collision probabilities.  Let
$
B_{\lambda} (x,\delta,\eta)
$
denote the resulting  collision envelope separating collision-driven gains from row-action, projection,
literal-bias, and ellipsoid-geometry effects. 

\begin{theorem}[Edge--wedge control of the certificate]
	\label{thm:main-edge-wedge}
	Assume the standing model conditions.  Fix a globally fresh test string
	$x\in\cX_a^{\fr}$, a ridge parameter $\lambda>0$, and
	$\delta,\eta\in(0,1)$.  Then
	$
	\Pp (
	B^\sharp_\lambda(x,\delta)
	\le   B_{\lambda} (x,\delta,\eta)
	)
	\ge
	1-\eta .
	$
\end{theorem}
Combining this bound with Theorem~\ref{thm:main} gives
\[
\Pp\!\left(
|\widehat f_\lambda(x)-S_\lambda^{\id}(x)|
\le  B_{\lambda}  (x,\delta,\eta)
\right)
\ge
1-r e^{-np_{\min}/8}-\delta-\eta .
\]
Thus the empirical logistic score is controlled by an explicit edge--wedge
collision envelope.

 Taking \(\varepsilon=0\), fixint \(\tau>0\) and \(\eta\in(0,1)\), and setting
$
  v_\eta:=\log {(r^2+r^3+1)/\eta} 
$
as long as 
$
  0<\lambda<\lambda_\tau^{\id}
$ and $
  B_{\lambda} (x,\delta,v_\eta)\le \tau,
$
then the transformer is correctly classifying the unseen symbol  a 
\(x\in\cX_a^{\fr}\) i.e.
$$
  \Pp\!\left(
    \operatorname{sign}(\widehat f_\lambda(x))\ne y_a
  \right)
  \le
  r e^{-n p_{\min}/8}+\delta+\eta .
$$

\section{Multiclass Classification}\label{sec:multiclass-softmax-extension}

The binary formulation keeps the notation light, but the collision-graph
argument is not tied to two labels.  The graph only controls how concrete
substitutions perturb the sample kernel.  For a fixed label set \([L]\), the
same comparison is applied to class contrasts of a softmax classifier.  The
only changes are the form of the quotient estimator and the dual residuals that
weight the collision graph.  This fixed-label setting should be distinguished from next-token or copy-style tasks, where the output alphabet itself may contain unseen symbols \cite{Boix2024,LazicEtAl2026Unseen}.

Let the template labels be \(y_a\in[L]\).  For a score vector \(u\in\R^L\), set
\[
  \ell_{\rm sm}(u,y)
  :=
  \log\!\left(\sum_{\ell=1}^L e^{u_\ell}\right)-u_y .
\]
For \(q\in\Delta_r\), the ideal template problem is
\begin{equation}
\label{eq:sm-template-objective-main}
  \Phi^{\rm sm}_{q,\lambda}(G)
  :=
  \sum_{a=1}^r q_a\ell_{\rm sm}(G_{a\cdot},y_a)
  +
  {\lambda\over2}
  \sum_{\ell=1}^L
  G_{\cdot\ell}^{\top}\bN^{-1}G_{\cdot\ell},
  \qquad
  G\in\R^{r\times L}.
\end{equation}
We write \(G_\lambda^{\id,{\rm sm}}\) for the minimizer of
\eqref{eq:sm-template-objective-main} with \(q=\widehat p\).  Thus a fresh
string \(x\in\cX_a^{\fr}\) has ideal score row
\[
  S_\lambda^{\id,{\rm sm}}(x)
  :=
  (G_\lambda^{\id,{\rm sm}})_{a\cdot}\in\R^L .
\]
The ideal multiclass margin is
\[
  m_{\rm sm}(G)
  :=
  \min_{a\in[r]}\min_{\ell\ne y_a}
  \{G_{a,y_a}-G_{a,\ell}\}.
\]
Exactly as in the binary quotient problem, small ridge regularization gives a
positive template margin.  For every \(\Gamma\ge0\), there is
\(\lambda_{\Gamma,L}^{\id}>0\), depending only on
\(\Gamma,L,p_{\min},\bN\), and the template labels, such that on
\(E_{\rm rep}\), every \(0<\lambda<\lambda_{\Gamma,L}^{\id}\) satisfies
\[
  m_{\rm sm}(G_\lambda^{\id,{\rm sm}})>\Gamma .
\]
The proof is the same comparison argument as Theorem~\ref{thm:ideal-margin},
using \(\log(1+(L-1)e^{-t})\) for a score vector with margin \(t\).

We next record the corresponding dual representation.  Let
\(E_Y\in\R^{n\times L}\) have row \(e_{Y_i}^\top\), and let
\[
  \Delta_L:=\{p\in[0,1]^L:\sum_{\ell=1}^L p_\ell=1\},
  \qquad
  \psi(p):=\sum_{\ell=1}^L p_\ell\log p_\ell .
\]
For a positive semidefinite sample matrix \(\mathsf K\), define
\begin{equation}
\label{eq:sm-dual-main}
  D_{\mathsf K}^{\rm sm}(\mathsf P)
  :=
  {1\over n}\sum_{i=1}^n\psi(\mathsf P_{i\cdot})
  +
  {1\over2\lambda n^2}
  \operatorname{tr}\!\left(
    R(\mathsf P)^\top \mathsf K R(\mathsf P)
  \right),
  \qquad
  R(\mathsf P):=E_Y-\mathsf P,
\end{equation}
with \(\mathsf P\in\Delta_L^n\).  The entropy term is strongly convex on each
simplex row, hence \(D_{\mathsf K}^{\rm sm}\) has a unique minimizer
\(\mathsf P_{\mathsf K}\).  The fitted score vector at a test point is
\begin{equation}
\label{eq:sm-empirical-score-main}
  \widehat s_\lambda(x)
  =
  {1\over\lambda n}
  R(\mathsf P_{\widehat K_n})^\top k_x
  \in\R^L,
\end{equation}
and the predicted class is \(\arg\max_\ell \widehat s_{\lambda,\ell}(x)\).

The ideal reduction remains finite-dimensional.  Let
\(\Pi_{ia}=\1_{\{\tau(i)=a\}}\), so \(M_n=\Pi\bN\Pi^\top\), and let
\(E_y\in\R^{r\times L}\) have row \(e_{y_a}^\top\).  Since the ideal dual
objective is invariant under permutations inside each template block,
\[
  \mathsf P_{M_n}=\Pi\Theta
  \qquad\text{for some }\Theta\in\Delta_L^r .
\]
With
\[
  B:=E_y-\Theta,
  \qquad
  Q:=\diag(\widehat p),
\]
the quotient score matrix is
\begin{equation}
\label{eq:sm-ideal-score-main}
  G_\lambda^{\id,{\rm sm}}
  =
  {1\over\lambda}\bN Q B .
\end{equation}
Indeed, \(\Theta_{a\cdot}=\operatorname{softmax}((G_\lambda^{\id,{\rm sm}})_{a\cdot})\), and
\eqref{eq:sm-ideal-score-main} is equivalent to the first-order condition
\[
  Q(\Theta-E_y)+\lambda\bN^{-1}G_\lambda^{\id,{\rm sm}}=0 .
\]
Consequently, for \(x\in\cX_a^{\fr}\),
\[
  {1\over\lambda n}
  R(\mathsf P_{M_n})^\top m_a
  =
  (G_\lambda^{\id,{\rm sm}})_{a\cdot}^\top .
\]
Thus the scalar ideal score \(g_\lambda^{\id}\) is replaced by the row-valued
estimator \(G_\lambda^{\id,{\rm sm}}\), but the template reduction is otherwise
unchanged.

It remains to transfer the ideal margin through the collision perturbation.  For
\(x\in\cX_a^{\fr}\) and \(\ell\ne y_a\), define the class contrast
\[
  v_{a\ell}:=e_{y_a}-e_\ell .
\]
The binary proof applies to the scalar score
\(v_{a\ell}^\top\widehat s_\lambda(x)\).  In the five certificates, the binary
signed residual \(\mathbf Yc_M=P\beta\) is replaced by the softmax contrast
residual
\begin{equation}
\label{eq:sm-residual-replacement-main}
  R(\mathsf P_{M_n})v_{a\ell}
  =
  \Pi Bv_{a\ell}.
\end{equation}
The support lemma, block-density bounds, ANOVA bounds, literal correction, and
edge--wedge envelope are unchanged, since none of them uses the labels.  Let
\(B_\lambda^{\sharp,{\rm sm}}(x,a,\ell;\delta)\) denote the resulting contrast
certificate, with
\[
  u_{\delta,L}:=
  \log{(L-1)(11r^2+5r)\over\delta}
\]
in place of \(u_\delta\).  Equivalently, this is the binary certificate applied
to the contrast after the replacement \eqref{eq:sm-residual-replacement-main}.

\begin{theorem}[Multiclass fresh-symbol transfer]
\label{thm:multiclass-main}
Assume the standing model conditions, with labels in \([L]\).  Fix
\(x\in\cX_a^{\fr}\), \(\lambda>0\), and \(\delta\in(0,1)\).  Then
\[
  \Pp\!\left(
  \max_{\ell\ne y_a}
  \left|
    v_{a\ell}^\top
    \{\widehat s_\lambda(x)-S_\lambda^{\id,{\rm sm}}(x)\}
  \right|
  \le
  \max_{\ell\ne y_a}
  B_\lambda^{\sharp,{\rm sm}}(x,a,\ell;\delta)
  \right)
  \ge
  1-r e^{-np_{\min}/8}-\delta .
\]
On this event, if for some \(\varepsilon\ge0\) and \(s>0\),
\[
  (G_\lambda^{\id,{\rm sm}})_{a,y_a}
  -
  \max_{\ell\ne y_a}(G_\lambda^{\id,{\rm sm}})_{a,\ell}
  >
  \varepsilon+s
, \qquad \mbox{ and } \qquad 
  \max_{\ell\ne y_a}
  B_\lambda^{\sharp,{\rm sm}}(x,a,\ell;\delta)
  \le s,
\]
then
\[
  \widehat s_{\lambda,y_a}(x)
  -
  \max_{\ell\ne y_a}\widehat s_{\lambda,\ell}(x)
  >
  \varepsilon .
\]
In particular, for \(\varepsilon=0\), the empirical softmax classifier predicts
the correct fresh-symbol class.
\end{theorem}

This theorem is not a new graph estimate; it is the same collision-graph
transfer applied to the \(L-1\) class contrasts.  The union bound over those
contrasts is the only source of the additional factor \(L-1\) in
\(u_{\delta,L}\).  Combining the theorem with the ideal multiclass margin above
gives the same conclusion as in the binary case: once the quotient margin is
larger than the multiclass contrast certificate, fresh-symbol classification
transfers from the ideal template problem to the empirical softmax classifier.

\section{Margin Gains from Abstract Prompting}\label{sec:abstract-prompting-main}
We next extend the collision-perturbation viewpoint to abstraction-based
reasoning methods such as AbstRaL~\cite{gao2026abstral},
Chain-of-Abstraction~\cite{gao2025coa}, and SyReLM~\cite{dutta2024syrelm}.
This connects to a broader tradition of modular, tool-using, and neuro-symbolic approaches that insert intermediate structure before prediction \cite{AndreasEtAl2016,SantoroEtAl2017,LiuEtAl2020,NyeEtAl2020}.  Our question is narrower: when does such an abstraction improve fresh-symbol classification in the kernel regime?  The answer is not automatic.  Abstraction helps only when it increases the quotient-template margin by more than the additional collision and attention interactions introduced by the abstraction tokens.

Let \(\cA\) be an abstraction space.  An abstraction channel is a Markov kernel
$
    \nu_a(\cdot\mid x)
    =
    \mathcal L(A\mid X=x,\ Z=z_a),
    \  x\in\cX_a .
$
We write
$
    A=\Anc(X,\xi),
$
where \(\xi\) denotes prompting, decoding, parsing, retrieval, or other internal
randomness.  In the AbstRaL example, \(A\) may be interpreted as the generated
abstract answer, the retrieved symbolic abstraction, or the tuple consisting of
the recognized conditions, abstract question, abstract answer, and retrieved
abstraction.  In Chain-of-Abstraction, \(A\) may be the abstract placeholder
chain.  In formalize-then-solve methods, \(A\) may be the formal-language
expression passed to a symbolic solver.

We use the limiting frozen-feature transformer kernel derived in
Appendix~\ref{app:transformer-kernel-interface}.  Recall that
\[
    K_{\mathrm{trans}}(X,Y)
    =
    \E_{u,v}[\sigma(u)\sigma(v)],
\]
where \((u,v)\) is a centered Gaussian pair with covariance
$
    \begin{pmatrix}
        K_{\mathrm{attn}}(X,X) & K_{\mathrm{attn}}(X,Y)\\
        K_{\mathrm{attn}}(Y,X) & K_{\mathrm{attn}}(Y,Y)
    \end{pmatrix}
$ and  \(\sigma\in C^2\).
The attention kernel \(K_{\mathrm{attn}}\) is
\[
K_{\mathrm{attn}}(X,Y)
=
\E_{m(X),m(Y)}
\left[
\operatorname{softmax}(\beta m(X))^\top
(XY^\top+\gamma^2I)
\operatorname{softmax}(\beta m(Y))
\right].
\]

 Let
$
    U_\eta(x,a)
$
be the fixed-length embedded sequence obtained from the concrete input \(x\)
and abstraction \(a\), where \(\eta\ge0\) controls the strength of the
abstraction prompt.  We assume the prompt is turned off at \(\eta=0\):
$
    U_0(x,a)=U_{\mathrm{base}}(x)
   \  \text{for all }a.
$
Define the prompted transformer kernel by composition:
\[
    K_{\mathrm{trans},\eta}((x,a),(x',a'))
    :=
    K_{\mathrm{trans}}
    \!\left(
        U_\eta(x,a),U_\eta(x',a')
    \right).
\]
The effect of the
abstraction may pass through attention between original tokens and abstraction
tokens, through the MLP covariance map, or through both. We will assume bounded abstraction prompts inserted
through a smooth gate. For templates \(z_a,z_b\), let
$
    X_a=\Sub(z_a,S_a),
    \ 
    A_a=\Anc(X_a,\xi_a),
$
and let \((X_b',A_b')\) be an independent fresh augmented instantiation of
\(z_b\).  Define
\[
    N_{\eta,ab}
    :=
    \E\!\left[
        K_{\mathrm{trans},\eta}
        ((X_a,A_a),(X_b',A_b'))
    \right].
\]
Since \(U_0(x,a)=U_{\mathrm{base}}(x)\), the abstraction is absent at
\(\eta=0\), and \(N_{0,ab}=N_{ab}\). In the Appendix \ref{sec:abstract-prompting-ktrans} we formally show 
\[
    \bN_\eta=\bN+\eta H+o(\eta),
\]
with \[
    H_{ab}
    :=
    \E\!\left[
        \dot K_{\mathrm{trans},0}
        ((X_a,A_a),(X_b',A_b'))
    \right],
    \qquad
    \dot K_{\mathrm{trans},0}
    :=
    {d\over d\eta}K_{\mathrm{trans},\eta} |_{\eta=0}.
\]
The matrix \(H\) is the first-order template-level effect of abstract prompting.
It includes abstraction-token effects, cross-attention effects, and the
downstream MLP covariance response.    In general \(H\) need not be positive
semidefinite.

The next result makes the limitation precise.  The matrix \(H\) is the first-order change induced in the
fresh-template transformer kernel by the prompted abstraction, including any
cross-attention between the original input and the abstraction tokens.  Thus the
sign of the margin gain is determined by how this induced template direction
aligns with the original ideal classifier.

\begin{theorem}[When abstract prompting helps]
\label{thm:main-abstract-prompting}
Let \(g_\eta\) be the ideal template logistic scorer associated with
\(\bN_\eta\).  Then, at \(\eta=0\),
\[
    \left.{d\over d\eta}g_\eta\right|_{\eta=0}
    =
    \lambda J_0^{-1}\bN^{-1}H\bN^{-1}g_0,
\qquad \mbox{ where,} \ 
    J_0
    =
    \diag\!\left(q_a\loss''(y_a(g_0)_a)\right)_{a=1}^r
    +
    \lambda\bN^{-1}.
\]
Thus the first-order gain in signed margin on template \(a\) is
$
    \mathfrak C_a(H)
    :=
    y_a e_a^\top
    \lambda J_0^{-1}\bN^{-1}H\bN^{-1}g_0.
$\
\end{theorem}
After subtracting the augmented certificate, the certified margin for template
\(a\) is
\[
    \mathcal M_a(\eta)
    :=
    y_a(g_\eta)_a-B_{\lambda,\eta}.
\]
where \(B_{\lambda,\eta}\) is a  augmented collision certificate (follows from the results on Section \ref{subsec:collision_graph}).
Then, the result above gives us that 
$
   d\mathcal M_a(\eta)/ d \eta |_{\eta=0}
    =
    \mathfrak C_a(H)
    -
 dB_{\lambda,\eta}/d\eta|_{\eta=0}. 
$
Thus prompted abstraction has two competing first-order effects: it changes the
ideal template margin by \(\mathfrak C_a(H)\), and it changes the finite-sample
certificate by \(dB_{\lambda,\eta}/d\eta|_{\eta=0}\).  It improves the
certified fresh-symbol margin exactly when the former dominates the latter,
namely when
\[
    \mathfrak C_a(H)
    >
    dB_{\lambda,\eta}/d\eta|_{\eta=0}.
\]
In this sense   abstraction must be template-aligned: additional
reasoning tokens help only if the template-level margin they induce is larger
than the extra certificate cost they introduce.

\section{Experiments}\label{sec:experiments}

The simulation is designed as a controlled study of how well different sequence models can generalize symbolic rules from limited data. We use synthetic template-based classification tasks, where each sample is generated by substituting symbols into wildcard templates. Specifically, we focus on the two following tasks, each representing template classification and NTP tasks: the two-template binary classification $\alpha \beta \alpha \mapsto 1$ and $\alpha \beta \beta \mapsto -1$, and finding the majority, where the templates are $\{ \alpha \alpha \alpha, \alpha \alpha \beta, \alpha \beta \alpha, \alpha \beta \beta \}$, and the output should be the majority token (the substitution for $\alpha, \alpha, \alpha$ and $\beta$ respectively). All models are one-layer transformers where only the unembedding layer is trained, corresponding to kernel logistic regression in the infinite-width limit.  The binary equality-pattern task is deliberately close to same--different and abstract-relation probes used in visual and symbolic reasoning, while the majority/copy task stresses the additional output-symbol channel \cite{BarrettEtAl2018,KimRicciSerre2018,Boix2024}.  The detailed simulation settings and results will be shown in Appendix \ref{app:simulation}.

\begin{remark}[A worked collision ledger for \(\alpha\beta\alpha\) versus \(\alpha\beta\beta\)]
\label{rem:aba-abb-ledger-main}
The binary classification experiment is a minimal example where scalar diversity, equality-pattern structure, and signed graph cancellation can be seen simultaneously.  Let
\[
    z_+=\alpha\beta\alpha\to +1,
    \qquad
    z_-=\alpha\beta\beta\to -1 .
\]
Both templates use two active wildcards, hence
\[
    \rho={m\over2},
    \qquad
    q^{\rm wild}
    =1-{(m-2)(m-3)\over m(m-1)}
    \simeq {2\over \rho}.
\]
The ideal rule distinguishes whether the repeated position is the first wildcard or the second wildcard.  Consider four wildcard-pair realizations
\[
    (a,b),\quad (a,c),\quad (d,b),\quad (d,c)
\]
on both sides.  If \((u_i,v_i)\) denotes the wildcard pair in a positive example and \((s_j,t_j)\) denotes the wildcard pair in a negative example, a signed two-channel residual may be written as
\[
    R_{ij}=\1\{u_i=s_j\}-\1\{v_i=t_j\}.
\]
For the ordering \((a,b),(a,c),(d,b),(d,c)\), this gives
\[
    R
    =
    \begin{pmatrix}
        0& 1&-1& 0\\
        1& 0& 0&-1\\
        -1& 0& 0& 1\\
        0&-1& 1& 0
    \end{pmatrix},
    \qquad
    C:=\1_{\{R\ne0\}}
    =
    \begin{pmatrix}
        0&1&1&0\\
        1&0&0&1\\
        1&0&0&1\\
        0&1&1&0
    \end{pmatrix}.
\]
The support matrix \(C\) has density \(8/16\), but the signed residual cancels in every row and column:
\[
    R\mathbf 1=0,
    \qquad
    \mathbf 1^\top R=0,
    \qquad
    \mathbf 1^\top R\mathbf 1=0.
\]
A support-only block-density argument sees a dense cross-class block, while the signed residual is invisible to the block average and the all-ones degree direction.  This is precisely the setting where Hoeffding--ANOVA and bias--fluctuation routes improve on scalar support counting.
\end{remark}

We compare one-layer transformers with same hyperparameters except for different combination of multipliers on K-Q and V-O inner products $W_K W_Q^\top + a I$ and $W_V W_O^\top + b I$, with $a, b \in \{0, 100\}$. Figure \ref{fig:accuracy} shows their test accuracy on these two tasks mentioned above. For each set of multipliers, Figure \ref{fig:loss_curves_binary} shows the learning curves for transformers with different embedding dimensions on classification tasks.

\begin{figure}[h]
	\centering
	\begin{subfigure}{0.48\textwidth}
		\centering
		\includegraphics[width=\textwidth]{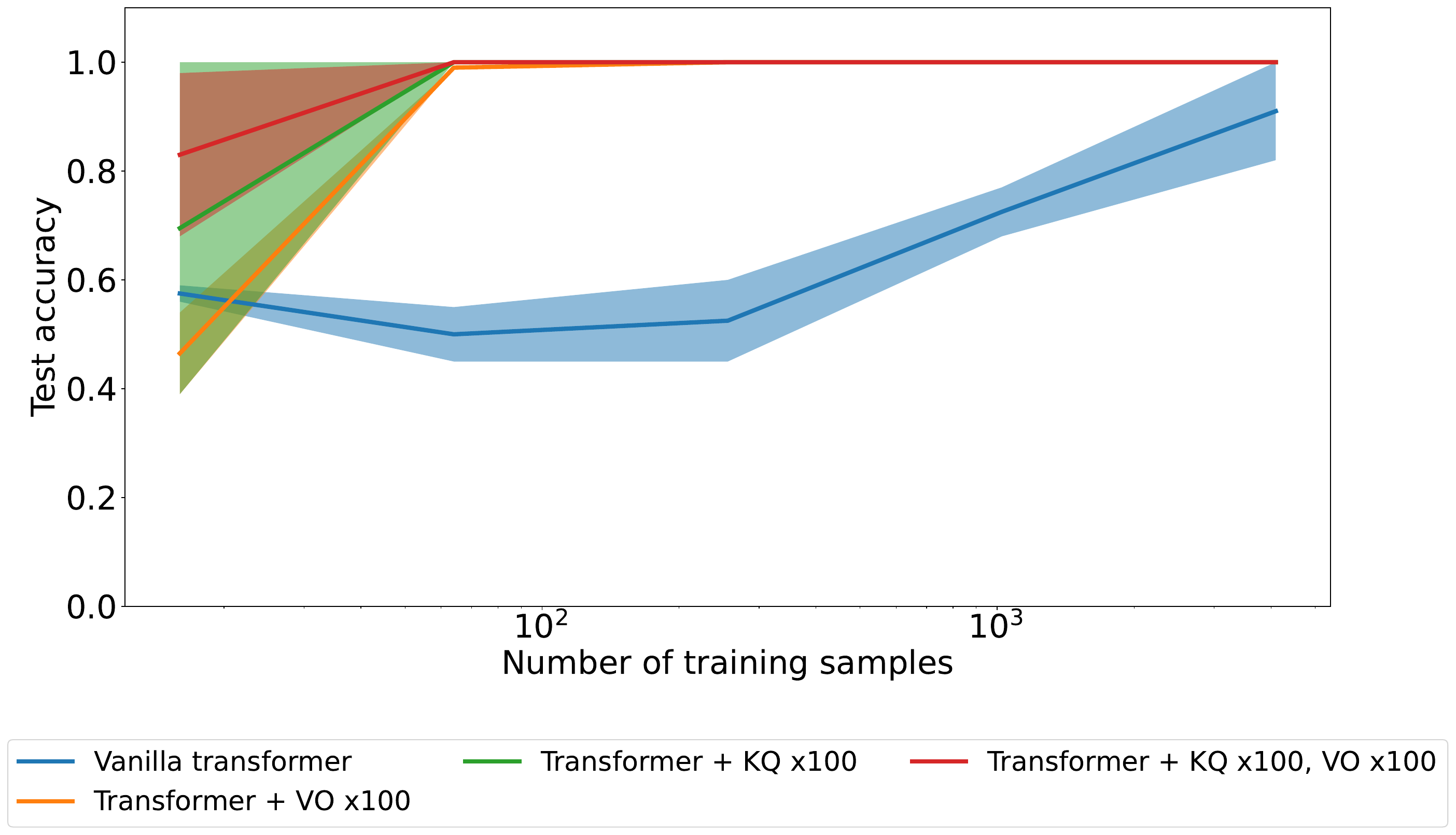}
	\end{subfigure}
	\hfill
	\begin{subfigure}{0.48\textwidth}
		\includegraphics[width=\textwidth]{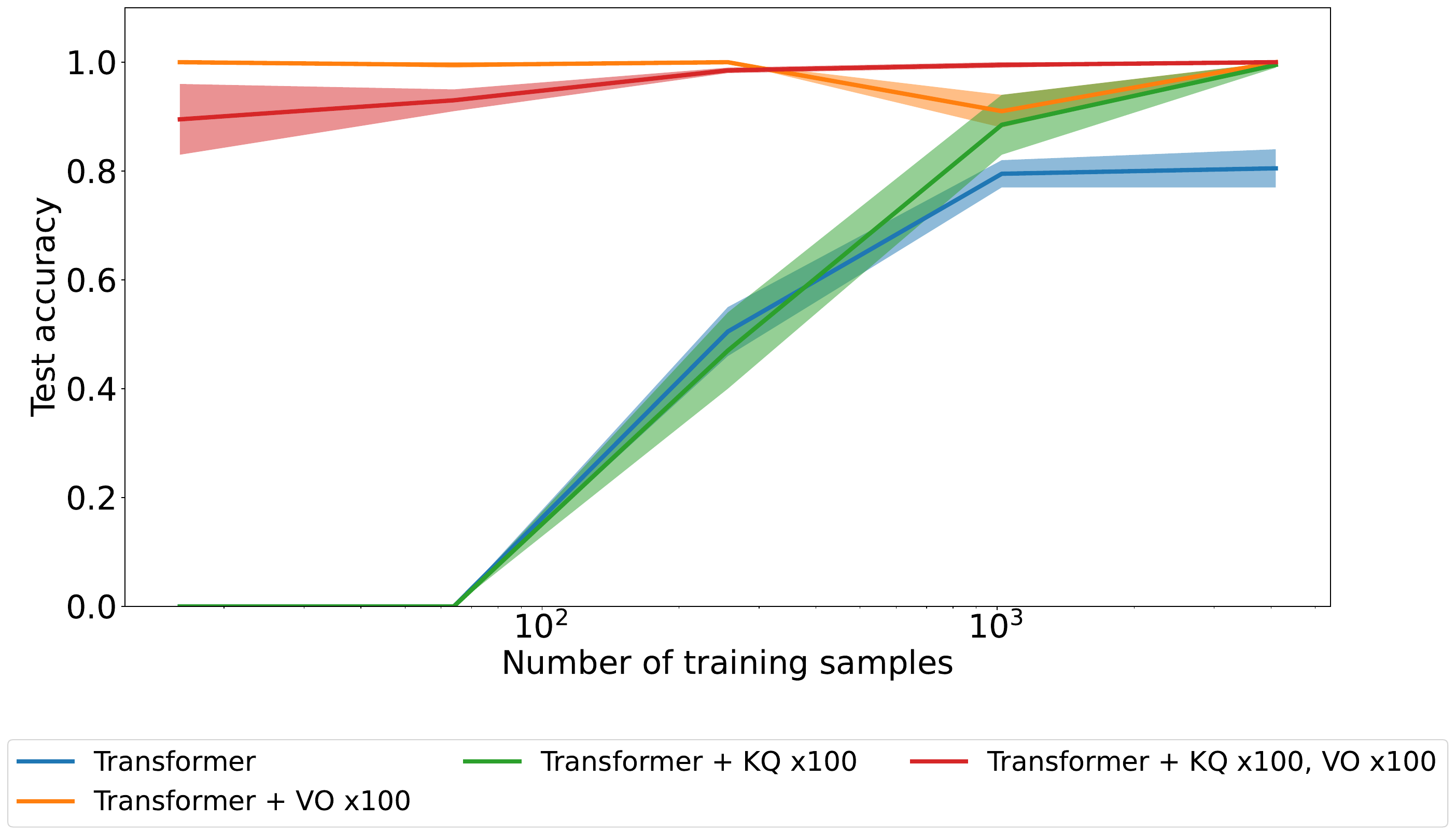}
	\end{subfigure}
	\caption{Test accuracy of one-layer transformers with different K-Q and V-O multipliers. Left: accuracy for binary classification. Right: accuracy for finding the majority.}
	\label{fig:accuracy}
\end{figure}

\begin{figure}[h]
	\centering
	\begin{subfigure}{0.48\textwidth}
		\centering
		\includegraphics[width=\textwidth]{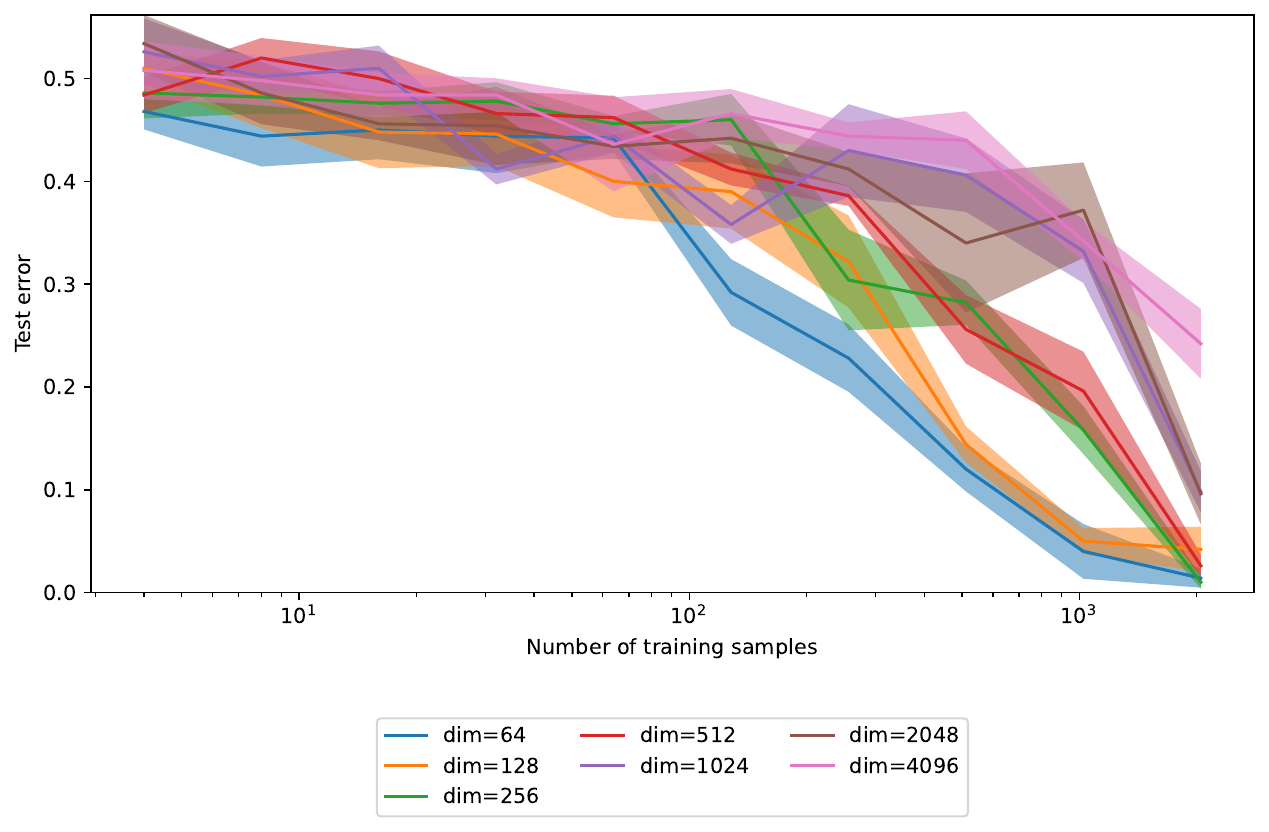}
		\caption{$(a,b)=(0,0)$}
	\end{subfigure}
	\hfill
	\begin{subfigure}{0.48\textwidth}
		\centering
		\includegraphics[width=\linewidth]{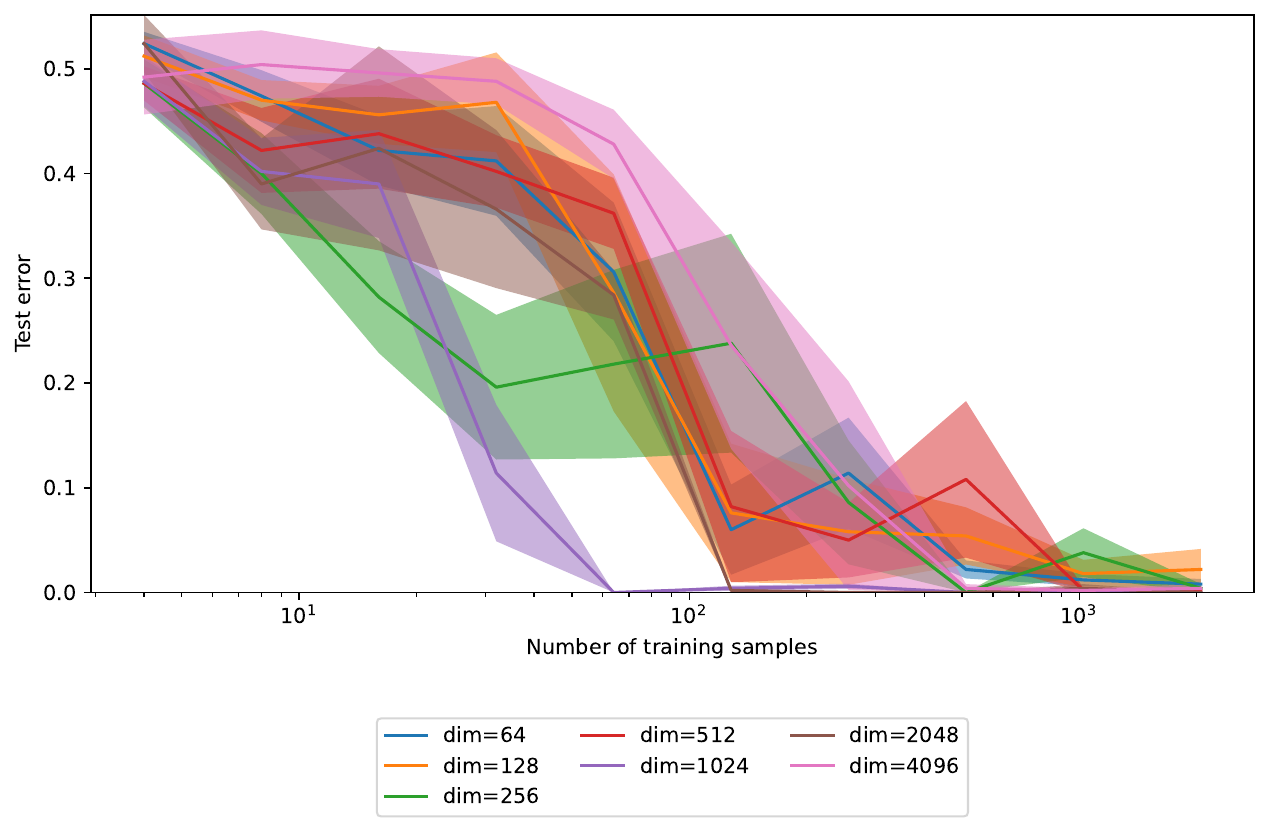}
		\caption{$(a,b)=(0,100)$}
	\end{subfigure}
	\\
	\begin{subfigure}{0.48\textwidth}
		\centering
		\includegraphics[width=\textwidth]{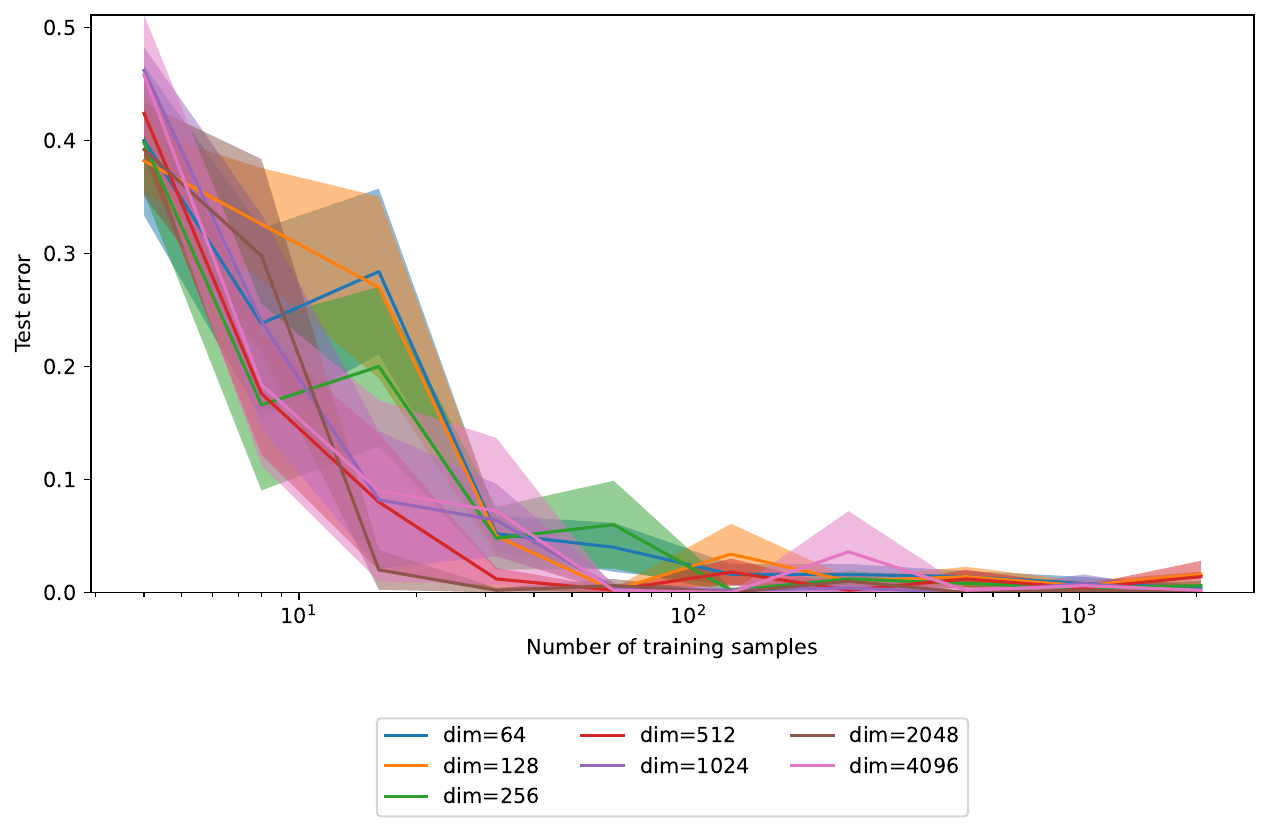}
		\caption{$(a,b)=(100,0)$}
	\end{subfigure}
	\hfill
	\begin{subfigure}{0.48\textwidth}
		\centering
		\includegraphics[width=\textwidth]{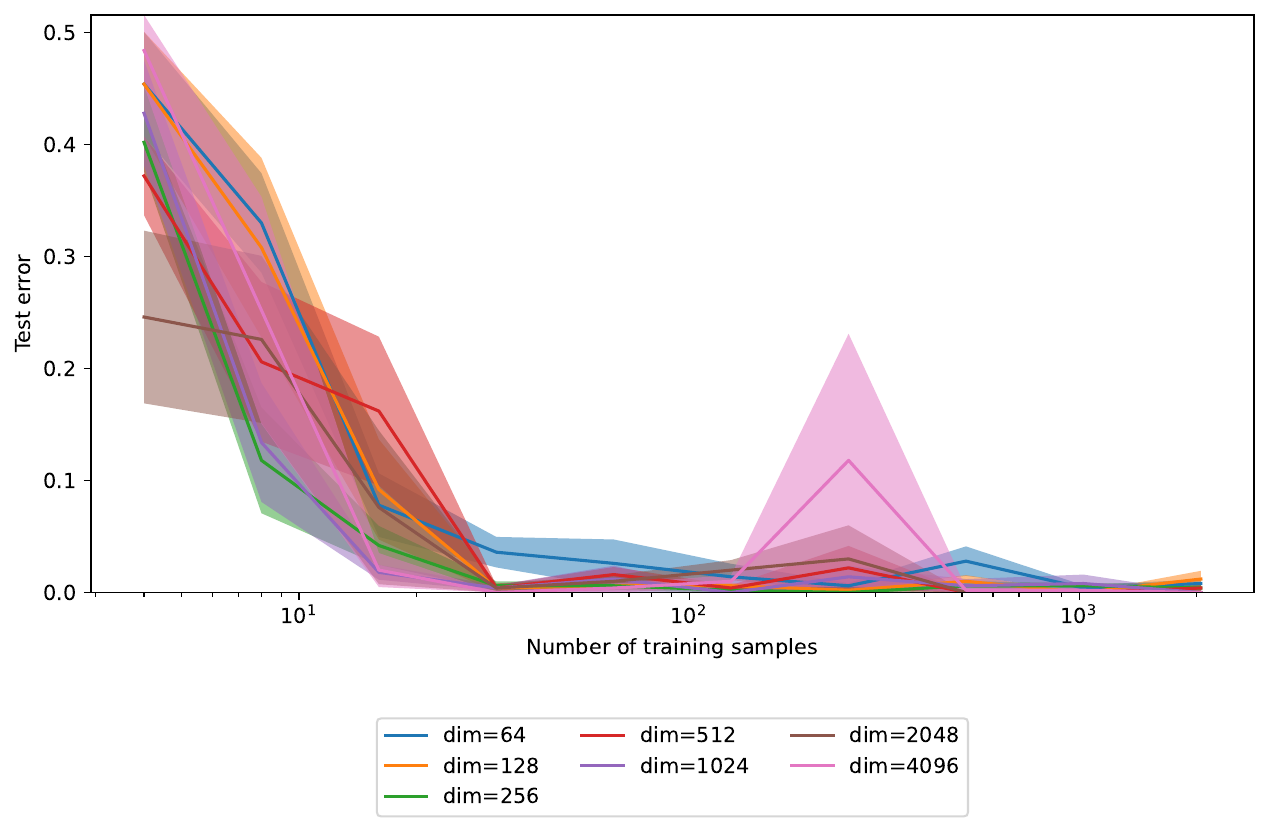}
		\caption{$(a,b)=(100,100)$}
	\end{subfigure}
	\caption{Test error curves for one-layer transformers with different embedding dimensions for binary classification.}
	\label{fig:loss_curves_binary}
\end{figure}

Based on the simulation, we have the following findings. First, vanilla one-layer transformers can generalize to unseen symbols, but requires large sample. This is similar to the finding for the copyout task from \cite{Boix2024}, transformers without any multipliers cannot learn well at the early stages of training. Second, multipliers will substantially improve the generalization performance. The test accuracy increases immediately from the first few training samples (or equivalently, training time) for transformers with almost every nonzero combinations of multipliers. Third, although two multipliers both introduce implicit regularization in the transformer kernel, their impact is different and task-dependent. As shown in Figure \ref{fig:accuracy}. K-Q multiplier plays a dominant role in achieving high accuracy for the classification task. In contrast, for the next-token prediction task, V-O multiplier is more critical, especially in the small-sample regime, where it enables generalization even with very limited data. However, for both tasks, adding K-Q and V-O multipliers together will be the overall best choice compared to adding either of them alone.

\section{Discussion}\label{sec:discussion}

Our results suggest that fresh-symbol generalization is governed not only by token diversity, but by the geometry of finite-sample substitution collisions. The collision graph provides a mechanism-level explanation for when empirical transformer kernels preserve the margin of the ideal template scorer. From this perspective, unseen-symbol generalization is fundamentally a perturbation problem. The ideal template scorer captures semantic structure, while the empirical training process introduces collision-induced deviations whose geometry determines whether the ideal margin survives finite sampling.

\begin{remark}[Relation to SCAN, COGS, and CFQ-style splits]
\label{rem:scan-cogs-cfq-main}
The theorem separates two mechanisms that are often entangled in systematic-generalization benchmarks.  The first is \emph{template coverage}: whether the ideal quotient problem contains enough colors and relations to support the target rule.  This is a property of \(\bN\), the template distribution, and the ideal margin.  The second is \emph{finite-substitution perturbation}: whether accidental concrete token overlaps move the empirical kernel away from the quotient problem.  This is a property of the collision graph.  Thus the result should be read as
\[
    \text{ideal template representation}
    \quad+\quad
    \text{small collision-graph perturbation}
    \quad\Longrightarrow\quad
    \text{fresh-symbol classification}.
\]
For SCAN, COGS, CFQ, and related splits, a model may fail because it has not learned the relevant compound rule, because the sample induces misleading token-level correlations, or both \cite{LakeBaroni2018,KimLinzen2020,Keysers2020}.  The present framework isolates the second obstruction conditional on a positive ideal template margin.
\end{remark}

Our analysis has several limitations. First, the theory is developed in the one-layer transformer kernel regime, where the feature map remains effectively fixed during training. While this enables explicit control of collision-induced perturbations, it does not capture representation learning in fully trained deep transformers \cite{JacotGabrielHongler2018,ChizatBach2019}. Second, our experiments focus on synthetic template tasks designed to isolate token-renaming invariance and symbolic structure. Although these tasks provide a controlled setting for studying fresh-symbol generalization, they do not reflect the full complexity of natural language. Finally, our framework studies frozen-feature kernel logistic classification and therefore does not address optimization dynamics, finite-width effects, or implicit bias in end-to-end transformer training.

Several directions may extend the present framework. A natural next step is to study collision-induced perturbations beyond the frozen-kernel regime, where collisions may interact with learned representations and attention dynamics \cite{Xie2022,VonOswaldEtAl2023}. It would also be interesting to investigate whether similar collision structures arise in broader forms of systematic generalization, including variable binding, compositional reasoning, and in-context learning \cite{SwaminathanEtAl2023,WebbSinhaCohen2021,WuGeigerMilliere2025}. More broadly, our results suggest that generalization may depend not only on global diversity statistics, but also on the geometry of finite-sample interactions induced by the data distribution.

\appendix

% Re-enable appendix entries for the appendix-only table of contents.
\addtocontents{toc}{\protect\setcounter{tocdepth}{2}}
\setcounter{tocdepth}{2}
\tableofcontents

\section{Notation and standing objects}
\label{sec:model-ideal}

This supplement uses the notation of the main paper.  For standalone reading, all global notation is collected here once.  Later statements introduce only theorem-local variables.

\subsection{Notation summary}
\label{sec:notation-summary}

The table below is a reference for symbols used throughout the supplement.  Generic variables that are local to a single perturbation lemma, such as an arbitrary pair of Gram matrices \(G,M\), are defined in the corresponding statement.

\begin{center}
\small
\renewcommand{\arraystretch}{1.16}
\begin{tabularx}{\textwidth}{@{}p{0.24\textwidth}X@{}}
\toprule
\multicolumn{2}{@{}l}{\textbf{Templates, substitutions, and labels}}\\
\midrule
\(\cV,\cW,k\) & Vocabulary, wildcard alphabet, and common sequence length.\\
\(\cZ=\{z_1,\ldots,z_r\}\) & Finite template family.  Template indices are usually denoted by \(a,b,c,d\in[r]\).\\
\(\cW(z),L(z)\) & Active wildcards and literal-token set of template \(z\).\\
\(W_a,L_a,w_a,w_*\) & Shorthand \(W_a=\cW(z_a)\), \(L_a=L(z_a)\), \(w_a=|W_a|\), and \(w_* =\max_a w_a\).\\
\(R\) & Global literal set \(R=\bigcup_a L_a\).\\
\(s,\Sub(z,s)\) & Admissible substitution map and the string obtained from template \(z\) by applying \(s\).\\
\(\cX_a,\cX_a^{\fr}\) & Matching class of \(z_a\), and its globally fresh subclass.\\
\(Z,S,X,Y\) & Latent template, substitution, observed string, and label.  The label is \(Y=y_Z\).\\
\(p_a,p_{\min},\templ\) & Template mass \(p_a=\Pp(Z=z_a)\), minimum mass \(p_{\min}=\min_a p_a\), and template law \(\templ\).\\
\(y_a\) & Fixed binary label of template \(z_a\), with \(y_a\in\{-1,+1\}\).\\
\(\rho\) & Coarse diversity parameter, the reciprocal of the largest token marginal in a wildcard slot.\\
\bottomrule
\end{tabularx}
\end{center}

\begin{center}
\small
\renewcommand{\arraystretch}{1.16}
\begin{tabularx}{\textwidth}{@{}p{0.24\textwidth}X@{}}
\toprule
\multicolumn{2}{@{}l}{\textbf{Kernel and ideal template objects}}\\
\midrule
\(K,\cH\) & Token-symmetric positive-semidefinite kernel and its RKHS.\\
\(K_*,L_*\) & Diagonal bound \(K_* = \sup_x K(x,x)\) and discrepancy bound \(L_*=2K_*\).\\
\(N_{ab},\bN\) & Fresh template kernel value and template matrix \(\bN=(N_{ab})\).\\
\(D_a,\delta_*\) & Self-kernel constant and diagonal discrepancy bound, \(\delta_* = \max_a |D_a-N_{aa}|\).\\
\(\Phi_{q,\lambda}\) & Ideal finite-template logistic objective at template frequencies \(q\).\\
\(g_\lambda^{\id}\) & Minimizer of \(\Phi_{\widehat p,\lambda}\).\\
\(S_\lambda^{\id}(x)\) & Ideal template score; if \(x\in\cX_a^{\fr}\), then \(S_\lambda^{\id}(x)=(g_\lambda^{\id})_a\).\\
\(\widehat f_\lambda\) & Ridge-regularized kernel logistic regression estimator.\\
\(\loss,\varphi\) & Logistic loss \(\loss(u)=\log(1+e^{-u})\) and dual entropy \(\varphi(c)=c\log c+(1-c)\log(1-c)\).\\
\bottomrule
\end{tabularx}
\end{center}

\begin{center}
\small
\renewcommand{\arraystretch}{1.16}
\begin{tabularx}{\textwidth}{@{}p{0.24\textwidth}X@{}}
\toprule
\multicolumn{2}{@{}l}{\textbf{Sample-level notation}}\\
\midrule
\(\tau(i)\) & Latent template index of training example \(i\).\\
\(I_b,n_b,\widehat p_b,Q\) & Template block, block size, empirical mass, and \(Q=\diag(\widehat p)\).\\
\(P\) & Membership matrix, \(P_{ib}=\1_{\tau(i)=b}\).\\
\(E_{\rm rep}\) & Representation event \(\{\widehat p_b\ge p_b/2\text{ for all }b\}\).\\
\(\mathbf y,\mathbf Y,Y\) & Label vector, diagonal label matrix, and the same diagonal matrix in compact displays.\\
\(M_n,m_a\) & Ideal sample Gram matrix \(M_n=P\bN P^\top\) and ideal fresh test vector \(m_a=(N_{a,\tau(i)})_i\).\\
\(\widehat K_n,k_x\) & Empirical Gram matrix and empirical test vector \((K(x,X_i))_i\).\\
\(\Delta,\zeta\) & Train--train and train--test discrepancies: \(\Delta=\widehat K_n-M_n\), \(\zeta=k_x-m_a\).\\
\(c_{\widehat K_n},c_M\) & Logistic dual minimizers for \(\widehat K_n\) and \(M_n\).\\
\(\theta,\beta,Y_r\) & Block dual coordinates \(c_M=P\theta\), signed coordinates \(\beta=y\odot\theta\), and \(Y_r=\diag(y_1,\ldots,y_r)\).\\
\bottomrule
\end{tabularx}
\end{center}

\begin{center}
\small
\renewcommand{\arraystretch}{1.16}
\begin{tabularx}{\textwidth}{@{}p{0.24\textwidth}X@{}}
\toprule
\multicolumn{2}{@{}l}{\textbf{Collision graph and concentration quantities}}\\
\midrule
\(\Bad_{ij},\Bad_{0i}\) & Train--train and test--train nonfresh-pair events.\\
\(\Ccol,\Wcol,\Gcol\) & Collision support, signed kernel-weighted collision object, and colored collision graph.\\
\(p_{b,t}\) & Token marginal \(\Pp_{S\sim\mu_{{\rm sub},b}}\{t\in S(W_b)\}\).\\
\(\ell_{b\to c},\chi_{bc}\) & Literal-hit mass from template \(b\) to literals of \(c\), and wildcard-overlap mass.\\
\(q_{bc},q_{bx}\) & Pair-specific train--train and train--test collision probability envelopes.\\
\(D_{\rm B},U_{\KL}\) & Bernoulli relative entropy and its inverse KL upper envelope.\\
\(N_{bc}^{\eff},E_{bc},T_b\) & Effective sample size and KL block/test envelopes.\\
\(X_\star,E_*\) & Aggregated test collision envelope and maximum block envelope.\\
\(d_i,d_2\) & Collision degree of training vertex \(i\) and its \(\ell_2\) aggregate.\\
\bottomrule
\end{tabularx}
\end{center}

\begin{center}
\small
\renewcommand{\arraystretch}{1.16}
\begin{tabularx}{\textwidth}{@{}p{0.24\textwidth}X@{}}
\toprule
\multicolumn{2}{@{}l}{\textbf{Certificate and corrected-target notation}}\\
\midrule
\(A_{\lambda,\widehat p},r_a\) & Finite-template Hessian surrogate and target sensitivity vector.\\
\(A_\Delta,\gamma_{\cspec}\) & Train--train action and curvature-spectral relative perturbation.\\
\(D_{\bdens},Z_{\bdens}\) & Weighted block-density train and test slots.\\
\(D_{\hanova},Z_{\hanova}\) & Weighted Hoeffding--ANOVA train and test slots.\\
\(N^\circ,m_x^\circ,M_n^\circ,\Delta^\circ\) & Literal-corrected template matrix, test vector, sample matrix, and empirical discrepancy.\\
\(A_\Delta^\circ,D_{\bfdec},Z_{\bfdec}\) & Corrected action and corrected bias--fluctuation slots.\\
\(B_{\mathrm{bias}}\) & Deterministic corrected-to-fresh bias envelope.\\
\(\operatorname{Ord}_\lambda,\operatorname{Curv}_\lambda\) & Ordinary and curvature score functionals used by the five certificates.\\
\(B_{\cspec},B_{\degact}\) & Curvature-spectral and degree-action transfer certificates.\\
\(B_{\bdens},B_{\hanova},B_{\bfdec}\) & Block-density, Hoeffding--ANOVA, and bias--fluctuation certificates.\\
\(B_\lambda^\sharp\) & Minimum realized certificate.\\
\(Q_E,S_\wedge\) & Edge and wedge density envelopes.\\
\(A_{\ew},A_{\op},A_0,A_{\corr}\) & Edge--wedge action envelopes.\\
\(B_{\bullet}^{\ew},B_\lambda^{\ew},v_\eta\) & Pathwise edge--wedge envelopes, their minimum, and the high-probability choice \(v_\eta=\log((r^2+r^3+1)/\eta)\).\\
\bottomrule
\end{tabularx}
\end{center}

\section{The collision graph as the perturbation object}
\label{sec:collision-graph}

The discrepancies \(\Delta\) and \(\zeta\) from the notation table are the train--train and train--test perturbations away from the fresh template problem.  The collision graph records exactly where these discrepancies may be nonzero.

\paragraph{Collision notation recalled from the main paper.}
The main paper defines the pairwise bad events and the collision graph.  For the proofs below we use the following explicit conventions.  After conditioning on a realized template assignment \(\tau\), the substitutions \(S_1,\ldots,S_n\) are independent, with
\[
  S_i\sim \mu_{{\rm sub},\tau(i)} .
\]
For \(b,c\in[r]\) and admissible substitutions \(s,t\), write \(s(W_b)=\{s(w):w\in W_b\}\) and set
\begin{align}
\label{eq:collision-bad-bc-recall}
  \Bad_{bc}(s,t)
  :=&\ \{s(W_b)\cap L_c\ne\varnothing\}
  \cup \{t(W_c)\cap L_b\ne\varnothing\}
  \nonumber\\
  &\cup \{s(W_b)\cap t(W_c)\ne\varnothing\}.
\end{align}
If \(\tau(i)=b\) and \(\tau(j)=c\), then
\[
  \Bad_{ij}:=\Bad_{bc}(S_i,S_j),
  \qquad
  \Ccol_{ij}:=\1_{\Bad_{ij}}
  \qquad (i\ne j).
\]
For the fixed globally fresh test string \(x=\Sub(z_a,s_x)\), put \(T_x=s_x(W_a)\) and
\begin{equation}
\label{eq:collision-bad-test-recall}
  \Bad_{b,x}(s)
  :=
  \{s(W_b)\cap L_a\ne\varnothing\}
  \cup
  \{s(W_b)\cap T_x\ne\varnothing\}.
\end{equation}
If \(\tau(i)=b\), then
\[
  \Bad_{0i}:=\Bad_{b,x}(S_i),
  \qquad
  \Ccol_{0i}:=\1_{\Bad_{0i}}.
\]
The notation \(\Bad_{bc}(s,V)\) used later means the same event with the second substitution random, \(V\sim\mu_{{\rm sub},c}\).

With these conventions, \(\Ccol_{ij}\) depends only on \((S_i,S_j)\), and \(\Ccol_{0i}\) depends only on \(S_i\).  Hence, conditionally on \(\tau\), any family of train--train collision indicators indexed by vertex-disjoint edges is jointly independent.  Likewise, any family of wedge indicators \(\Ccol_{ij}\Ccol_{ik}\) indexed by triples \(\{i,j,k\}\) that are pairwise vertex-disjoint is jointly independent.  Finally, after conditioning additionally on \(S_i=s\), the variables \(\{\Ccol_{ij}:j\ne i\}\) are independent and satisfy, for \(\tau(i)=b\) and \(\tau(j)=c\),
\[
  \Pp(\Ccol_{ij}=1\mid S_i=s,\tau)
  =
  \Pp_{V\sim\mu_{{\rm sub},c}}\bigl(\Bad_{bc}(s,V)\bigr).
\]
These are the conditional-independence facts used by the matching, hypergraph-coloring, and maximum-degree arguments below.

\begin{lemma}
\label{lem:collision-support-front}
For \(i\ne j\),
$
  \Bad_{ij}^c\ \Longrightarrow\  \Delta_{ij}=0 .
$
For the test edge,
$
  \Bad_{0i}^c\ \Longrightarrow\  \zeta_i=0 .
$
Moreover,
$
  |\Delta_{ij}|\le L_*\Ccol_{ij}\quad(i\ne j),
  \
  |\zeta_i|\le L_*\Ccol_{0i},
$
and the diagonal discrepancy is bounded separately by
$
  |\Delta_{ii}|\le \delta_* .
$
\end{lemma}

\begin{proof}[Proof of Lemma \ref{lem:collision-support-front}]
If \(\Bad_{ij}\) does not occur, then the wildcard tokens appearing in \(X_i\) and \(X_j\) are distinct from one another and from the other template's literal tokens.  We can therefore choose a vocabulary permutation that fixes the relevant literal tokens and maps the realized pair \((X_i,X_j)\) to any reference fresh admissible pair of substitutions of \((z_{\tau(i)},z_{\tau(j)})\).  Token symmetry gives
\[
  K(X_i,X_j)=N_{\tau(i),\tau(j)},
\]
so \(\Delta_{ij}=0\).  The same argument applies to \((x,X_i)\), because the test string is globally fresh.  On a bad event, Cauchy--Schwarz in the RKHS gives \(|K(u,v)|\le K_*\) and \(|N_{bc}|\le K_*\), so every off-diagonal discrepancy has magnitude at most \(2K_*=L_*\).  The diagonal entry compares a self-kernel with a fresh two-copy kernel and is not a collision between independent substitutions; it is charged separately by \(\delta_*\).
\end{proof}

The support lemma is the first compression step: all off-diagonal error lives on \(\Gcol\).

\section{Certificate budgets used by the main theorems}
\label{sec:graph-budgets}

This section records the finite-sample quantities that enter the two transfer bounds.  On \(E_{\rm rep}^c\) we
set \(B^\sharp_\lambda(x,\delta)=+\infty\).

\subsection{Block averages and KL envelopes}

Set
\[
  \overline\Delta_{bc}
  :=
  {1\over n_bn_c}
  \sum_{i\in I_b}\sum_{j\in I_c}\Delta_{ij},
  \qquad
  \overline\zeta_b
  :=
  {1\over n_b}\sum_{i\in I_b}\zeta_i .
\]
The empirical collision densities are
\[
  \widehat q_{bc}
  :=
  \begin{cases}
  \displaystyle
  {1\over n_bn_c}
  \sum_{i\in I_b}\sum_{j\in I_c}\Ccol_{ij},
  & b\ne c,\\[.8em]
  \displaystyle
  {2\over n_b(n_b-1)}
  \sum_{\substack{i<j\\ i,j\in I_b}}\Ccol_{ij},
  & b=c,\ n_b\ge2,
  \end{cases}
  \qquad
  \widehat q_{0b}
  :=
  {1\over n_b}\sum_{i\in I_b}\Ccol_{0i}.
\]
By Lemma \ref{lem:collision-support-front},
\[
  |\overline\Delta_{bc}|
  \le
  \begin{cases}
  L_*\widehat q_{bc}, & b\ne c,\\[.4em]
  \displaystyle
  {\delta_*\over n_b}
  +
  {2L_*\binom{n_b}{2}\over n_b^2}\widehat q_{bb},
  & b=c,
  \end{cases}
  \qquad
  |\overline\zeta_b|
  \le
  L_*\widehat q_{0b}.
\]

For each token \(t\), let
\[
  p_{b,t}
  :=
  \Pp_{S\sim\mu_{{\rm sub},b}}
  \{t\in S(\cW(z_b))\}.
\]
Define
\[
  \ell_{b\to c}:=\sum_{t\in L_c}p_{b,t},
  \qquad
  \chi_{bc}:=\sum_{t\in\cV}p_{b,t}p_{c,t},
\  \mbox{and}
  q_{bc}
  :=
  \min\{1,\ell_{b\to c}+\ell_{c\to b}+\chi_{bc}\}.
\]
For \(x=\Sub(z_a,s_x)\), put \(T_x=s_x(\cW(z_a))\) and
$
  q_{bx}
  :=
  \min\left\{
    1,\ell_{b\to a}+\sum_{t\in T_x}p_{b,t}
  \right\}.
$
The scalar diversity bound is recovered from these quantities by the crude
estimate \(p_{b,t}\le 1/\rho\).

Let
$
  D_{\rm B}(p\Vert q)
  :=
  p\log {p / q}
  +(1-p)\log {(1-p) / (1-q)}
$, and set
\[
  U_{\KL}(N,q,u)
  :=
  \inf\{p\in[q,1]:ND_{\rm B}(p\Vert q)\ge u\},
  \qquad
  U_{\KL}(0,q,u):=1.
\]
If the feasible set is empty, the value is taken to be \(1\).  The effective
sample sizes are
\[
  N^\eff_{bc}
  =
  \begin{cases}
  \min(n_b,n_c), & b\ne c,\\[.2em]
  \binom{n_b}{2}/m_{n_b}, & b=c,\ n_b\ge2,\\[.2em]
  0, & b=c,\ n_b=1,
  \end{cases}
\]
where \(m_{n_b}=n_b-1\) for even \(n_b\) and \(m_{n_b}=n_b\) for odd \(n_b\).

Finally define
\[
  E_{bc}(u)
  =
  \begin{cases}
  L_*U_{\KL}(N^\eff_{bc},q_{bc},u),
  & b\ne c,\\[.4em]
  \displaystyle
  {\delta_*\over n_b}
  +
  \1_{n_b\ge2}
  {2L_*\binom{n_b}{2}\over n_b^2}
  U_{\KL}(N^\eff_{bb},q_{bb},u),
  & b=c,
  \end{cases}
\]
and
\[
  T_b(u):=L_*U_{\KL}(n_b,q_{bx},u),
  \qquad
  X_\star(u):=\sum_b\widehat p_b U_{\KL}(n_b,q_{bx},u),
  \qquad
  E_*(u):=\max_{b,c}E_{bc}(u).
\]

\subsection{Sensitivity and action terms}

By Theorem \ref{thm:idealreduction}, \(c_M=P\theta\).  Put
\[
  \beta:=y\odot\theta,
  \qquad
  Y_r:=\diag(y_1,\ldots,y_r),
  \qquad
  n_a^{\rm tmplt}:=(N_{a1},\ldots,N_{ar})^\top .
\]
With \(\varphi(c)=c\log c+(1-c)\log(1-c)\), define
\[
  A_{\lambda,\widehat p}
  :=
  \diag(\varphi''(\theta_1),\ldots,\varphi''(\theta_r))
  +
  {1\over\lambda}Y_r\bN QY_r,
  \qquad
  r_a
  :=
  Y_rA_{\lambda,\widehat p}^{-1}Y_r n_a^{\rm tmplt}.
\]
The block-density terms are
\[
  D_{\bdens}(u)
  :=
  \sum_{b,c}
  \widehat p_b\widehat p_c
  |r_{ab}|\,|\beta_c|\,E_{bc}(u),
  \qquad
  Z_{\bdens}(u)
  :=
  \sum_b\widehat p_b|\beta_b|T_b(u).
\]

Let
\[
  d_i:=\sum_{j\ne i}\Ccol_{ij},
  \qquad
  d_2:=\left(\sum_i d_i^2\right)^{1/2}.
\]
Set
\[
  A_\Delta
  :=
  \min\left\{
    {\|\Delta\|_{\op}\over n},
    {\delta_*\over n}+{L_*d_2\over n^{3/2}}
  \right\},
 \qquad \mbox{and} \qquad
  \gamma_{\cspec}
  :=
  \left\|
  C_M^{-1/2}
  \left({1\over\lambda n^2}Y\Delta Y\right)
  C_M^{-1/2}
  \right\|_{\op}.
\]

For the corrected comparison, put
\[
  Q_q:=(q_{bc})_{b,c=1}^r,
  \qquad
  \Delta^\circ:=\widehat K_n-M_n^\circ,
\qquad \mbox{and }
  A_\Delta^\circ
  :=
  {1\over n^{3/2}}
  \|\Delta^\circ Yc_{M_n^\circ}\|_2.
\]
The terms \(D_{\hanova},Z_{\hanova}\) and \(D_{\bfdec},Z_{\bfdec}\) are
defined in \eqref{eq:hanova-weighted-proj-def} and
\eqref{eq:bfdec-weighted-lit-cent-def}, respectively.

\subsection{Score budgets}

For \(A,D,Z\ge0\), let
\begin{equation}
\label{eq:ordinary-score-functional}
  \operatorname{Ord}_\lambda(A,D,Z)
  :=
  {K_*\over4\lambda^2}A
  +
  {D\over\lambda^2}
  +
  {Z\over\lambda}.
\end{equation}
The curvature budget is
\begin{equation}
\label{eq:curvature-score-functional}
  \operatorname{Curv}_\lambda(A,E,X,\gamma)
  :=
  \begin{cases}
  \displaystyle
  {L_*X\over\lambda}
  +
  {K_*E+(K_*+2L_*\sqrt X)A
  \over
  8\lambda^2(1-\gamma)},
  & \gamma<1,\\[1.2em]
  +\infty,
  & \gamma\ge1.
  \end{cases}
\end{equation}
Also set
\[
  Z_\star(u):=L_*X_\star(u).
\]
The deterministic bias term for the corrected comparison is
\begin{equation}
\label{eq:main-bias-envelope}
  B_{\mathrm{bias}}(\lambda;x)
  =
  {K_*L_*\over4\lambda^2}
  \|Q^{1/2}Q_qQ^{1/2}\|_{\op}
  +
  {L_*\over\lambda}
  \left(\sum_b\widehat p_bq_{bx}^2\right)^{1/2}.
\end{equation}

The five budgets are
\begin{align}
\label{eq:five-certificates-compact}
  B_{\cspec}(\lambda;x,u)
  &:=
  \operatorname{Curv}_\lambda
  \bigl(A_\Delta,E_*(u),X_\star(u),\gamma_{\cspec}\bigr),
  \tag{CS}\\
  B_{\degact}(\lambda;x,u)
  &:=
  \operatorname{Ord}_\lambda
  \bigl(A_\Delta,0,Z_\star(u)\bigr),
  \tag{DEG}\\
  B_{\bdens}(\lambda;x,u)
  &:=
  \operatorname{Ord}_\lambda
  \bigl(A_\Delta,D_{\bdens}(u),Z_{\bdens}(u)\bigr),
  \tag{BD}\\
  B_{\hanova}(\lambda;x,u)
  &:=
  \operatorname{Ord}_\lambda
  \bigl(A_\Delta,D_{\hanova}(u),Z_{\hanova}(u)\bigr),
  \tag{ANOVA}\\
  B_{\bfdec}(\lambda;x,u)
  &:=
  \operatorname{Ord}_\lambda
  \bigl(A_\Delta^\circ,D_{\bfdec}(u),Z_{\bfdec}(u)\bigr)
  +
  B_{\mathrm{bias}}(\lambda;x).
  \tag{BF}
\end{align}

For \(\delta\in(0,1)\), put
\[
  u_\delta:=\log {11r^2+5r\over \delta}.
\]
The budget in Main Theorem~\ref{thm:main} is
\[
  B^\sharp_\lambda(x,\delta)
  :=
  \min\{
    B_{\cspec},
    B_{\degact},
    B_{\bdens},
    B_{\hanova},
    B_{\bfdec}
  \}(\lambda;x,u_\delta),
\]
with \(B^\sharp_\lambda(x,\delta)=+\infty\) on \(E_{\rm rep}^c\).

\subsection{Edge--wedge envelope for Main Theorem~2}
\label{sec:edge-wedge-envelope-def}

We next define the deterministic envelope used in
Theorem \ref{thm:five-certificate-edge-wedge-envelope-final}.  It bounds the realized
budget \(B^\sharp_\lambda(x,\delta)\) by replacing the random action terms
\(A_\Delta,A_\Delta^\circ,\gamma_{\cspec}\) with edge, wedge, and maximum-degree
bounds.

For \(v>0\), write
\[
  U_{bc}(v):=U_{\KL}(N_{bc}^{\eff},q_{bc},v).
\]
Define
\[
  Q_E(v)
  :=
  \sum_{b\ne c}
  \widehat p_b\widehat p_c\,U_{bc}(v)
  +
  \sum_{b:n_b\ge2}
  \widehat p_b^2
  \left(1-{1\over n_b}\right)
  U_{bb}(v).
\]

For colors \(b,c,d\in[r]\), set
\[
  \mathcal T_{b;cd}
  :=
  \{(i,j,k):i\in I_b,\ j\in I_c,\ k\in I_d,\ i,j,k
  \text{ distinct}\},
  \qquad
  P_{b;cd}:=|\mathcal T_{b;cd}|.
\]
If \(P_{b;cd}>0\), define
\[
  \Delta_{b;cd}^{\wedge}
  :=
  \max_{\ell\in[n]}
  \left|
  \{(i,j,k)\in\mathcal T_{b;cd}:\ell\in\{i,j,k\}\}
  \right|,
  \qquad
  N_{b;cd}^{\wedge}
  :=
  {P_{b;cd}\over 3\Delta_{b;cd}^{\wedge}+1}.
\]
For an admissible substitution \(s\) of template \(b\), put
\[
  \pi_{b\to c}(s)
  :=
  \Pp_{V\sim\mu_{{\rm sub},c}}
  \bigl(\Bad_{bc}(s,V)\bigr),
\qquad \mbox{ and} \qquad
  \alpha_{b;cd}
  :=
  \E_{S\sim\mu_{{\rm sub},b}}
  \bigl[
    \pi_{b\to c}(S)\pi_{b\to d}(S)
  \bigr].
\]
Choose \(\bar\alpha_{b;cd}\in[\alpha_{b;cd},1]\).
Define
\[
  S_\wedge(v)
  :=
  \sum_{\substack{b,c,d\in[r]\\P_{b;cd}>0}}
  {P_{b;cd}\over n^3}
  U_{\KL}
  \bigl(
    N_{b;cd}^{\wedge},
    \bar\alpha_{b;cd},
    v
  \bigr), \quad \mbox{and } \quad
  A_{\ew}(v)
  :=
  {\delta_*\over n}
  +
  L_*
  \left(
    {Q_E(v)\over n}
    +
    S_\wedge(v)
  \right)^{1/2}.
\]

The maximum-degree bound is
\[
  \kappa_{bc}
  :=
  \sup_{s\in\supp(\mu_{{\rm sub},b})}
  \Pp_{V\sim\mu_{{\rm sub},c}}
  \bigl(\Bad_{bc}(s,V)\bigr),
\mbox{ and }
  \widehat\kappa_b:=\sum_c\widehat p_c\kappa_{bc},
  \
  \widehat\kappa_*:=\max_b\widehat\kappa_b,
\]
and
\[
  \overline d(v)
  :=
  \widehat\kappa_*
  +
  \sqrt{{2\widehat\kappa_*(v+\log n)\over n}}
  +
  {v+\log n\over3n}.
\]
Let
\[
  A_{\op}(v)
  :=
  {\delta_*\over n}
  +
  L_*\overline d(v),
  \qquad
  A_0(v):=\min\{A_{\ew}(v),A_{\op}(v)\},
\]
and
\[
  \bar\gamma_{\cspec}(v)
  :=
  {1\over4\lambda}A_{\op}(v).
\]

For the corrected path, define
\[
  \mu^{\lit}_{bc}:=N^\circ_{bc}-N_{bc},
  \qquad
  \delta_*^\circ:=\max_b |D_b-N^\circ_{bb}|,
\qquad
  r_i^\mu
  :=
  \sum_{j\ne i}|\mu^{\lit}_{\tau(i)\tau(j)}|,
  \qquad
  R_{2,\mu}
  :=
  \left(\sum_i(r_i^\mu)^2\right)^{1/2}.
\]
Set
\[
  A_{\corr}(v)
  :=
  {\delta_*^\circ\over n}
  +
  L_*
  \left(
    {Q_E(v)\over n}
    +
    S_\wedge(v)
  \right)^{1/2}
  +
  {R_{2,\mu}\over n^{3/2}}.
\]

At level \(u_\delta\), define
\begin{align*}
  B_{\cspec}^{\ew}(\lambda;x,\delta,v)
  &:=
  \operatorname{Curv}_\lambda
  \bigl(A_0(v),E_*(u_\delta),X_\star(u_\delta),\bar\gamma_{\cspec}(v)\bigr),\\
  B_{\degact}^{\ew}(\lambda;x,\delta,v)
  &:=
  \operatorname{Ord}_\lambda
  \bigl(A_0(v),0,Z_\star(u_\delta)\bigr),\\
  B_{\bdens}^{\ew}(\lambda;x,\delta,v)
  &:=
  \operatorname{Ord}_\lambda
  \bigl(A_0(v),D_{\bdens}(u_\delta),Z_{\bdens}(u_\delta)\bigr),\\
  B_{\hanova}^{\ew}(\lambda;x,\delta,v)
  &:=
  \operatorname{Ord}_\lambda
  \bigl(A_0(v),D_{\hanova}(u_\delta),Z_{\hanova}(u_\delta)\bigr),\\
  B_{\bfdec}^{\ew}(\lambda;x,\delta,v)
  &:=
  \operatorname{Ord}_\lambda
  \bigl(A_{\corr}(v),D_{\bfdec}(u_\delta),Z_{\bfdec}(u_\delta)\bigr)
  +
  B_{\mathrm{bias}}(\lambda;x).
\end{align*}
The envelope used in Main Theorem~2 is
\[
  B_{\lambda}(x,\delta,v)
  :=
  \min\left\{
    B_{\cspec}^{\ew},
    B_{\degact}^{\ew},
    B_{\bdens}^{\ew},
    B_{\hanova}^{\ew},
    B_{\bfdec}^{\ew}
  \right\}(\lambda;x,\delta,v).
\]

\section{Main results}
\label{sec:paper-main-results}

\begin{mainthm}[Collision-graph transfer]
\label{thm:main-appendix}
Assume the standing model conditions.  Fix \(x\in\cX_a^{\fr}\),
\(\lambda>0\), and \(\delta\in(0,1)\).  Then
\[
  \Pp\!\left(
    \left|\widehat f_\lambda(x)-S_\lambda^{\id}(x)\right|
    \le B^\sharp_\lambda(x,\delta)
  \right)
  \ge
  1-r e^{-n p_{\min}/8}-\delta .
\]

Moreover, let \(\varepsilon\ge0\) and \(s>0\).  There is a number
\(\lambda_{\varepsilon+s}^{\id}>0\), depending only on
\(\varepsilon+s\), \(p_{\min}\), \(\bN\), and \(y\), such that for every
\(0<\lambda<\lambda_{\varepsilon+s}^{\id}\),
\[
  \Pp\!\left(
    \{y_a\widehat f_\lambda(x)>\varepsilon\}
    \cup
    \{B^\sharp_\lambda(x,\delta)>s\}
  \right)
  \ge
  1-r e^{-n p_{\min}/8}-\delta .
\]
Equivalently, except on an event of probability at most
\(r e^{-n p_{\min}/8}+\delta\), the implication
\[
  B^\sharp_\lambda(x,\delta)\le s
  \quad\Longrightarrow\quad
  y_a\widehat f_\lambda(x)>\varepsilon
\]
holds.
\end{mainthm}

\begin{mainthm}[Edge--wedge envelope]
\label{thm:five-certificate-edge-wedge-envelope-final}
Assume the standing model conditions.  Fix \(x\in\cX_a^{\fr}\),
\(\lambda>0\), and \(\delta\in(0,1)\).  Given a realized template assignment
satisfying \(E_{\rm rep}\), for every \(v>0\),
\[
  \Pp\!\left(
    B^\sharp_\lambda(x,\delta)
    \le
    B_{\lambda}(x,\delta,v)
    \,\middle|\,
    \tau(1),\ldots,\tau(n)
  \right)
  \ge
  1-(r^2+r^3+1)e^{-v}.
\]
The same event also gives the corresponding domination for each of the five
terms entering the minimum defining \(B^\sharp_\lambda\).

Consequently, for
\[
  v_\eta:=\log {r^2+r^3+1\over\eta},
  \qquad
  \eta\in(0,1),
\]
one has
\[
  \Pp\!\left(
    E_{\rm rep}
    \cap
    \left\{
      B^\sharp_\lambda(x,\delta)
      \le
      B_{\lambda}(x,\delta,v_\eta)
    \right\}
  \right)
  \ge
  1-r e^{-np_{\min}/8}-\eta .
\]
Combining this with Main Theorem~\ref{thm:main} yields
\[
  \Pp\!\left(
    \left|
      \widehat f_\lambda(x)-S_\lambda^{\id}(x)
    \right|
    \le
    B_{\lambda}(x,\delta,v_\eta)
  \right)
  \ge
  1-r e^{-np_{\min}/8}-\delta-\eta .
\]
\end{mainthm}

\section{Proof machinery for Main Theorem 1}
\label{sec:proof-machinery-main1}
\label{sec:detailed-first-proof-machinery}
\subsection{Proofs of the main-text logistic-duality and curvature estimates}
\label{app:proofs-main-results}

The logistic-duality and curvature statements are embedded in the main text at the points where they are used.  We collect their proofs here to keep the main narrative focused on the template and collision-graph mechanisms.

\begin{proof}[Proof of Proposition~\ref{prop:template_margin}]
Let
\[
  \mathcal F:=\{g\in\R^r: y_ag_a\ge 1 \text{ for every } a\}.
\]
This set is nonempty, since the label vector itself belongs to it: if
\(g=y\), then \(y_ag_a=y_a^2=1\) for every \(a\).  By the standing
assumption \(\bN\succ0\), the quadratic form
\(g\mapsto g^\top\bN^{-1}g\) is coercive.  Hence it attains its minimum
on the closed set \(\mathcal F\).  Choose a minimizer \(g_*\), so that
\[
  g_*^\top\bN^{-1}g_*=R_*^2 .
\]
For any \(t>0\), the vector \(tg_*\) has signed template margin at least
\(t\).  Since \(\loss\) is decreasing by Lemma \ref{lem:logistic-aux},
\[
  \sum_a q_a\loss\bigl(y_a(tg_*)_a\bigr)\le \loss(t),
\]
where we use \(\sum_a q_a=1\).  Therefore
\[
  \Phi_{q,\lambda}(tg_*)
  \le
  \loss(t)+{\lambda t^2\over2}R_*^2 .
\]

Now suppose \(t>\Gamma\) and
\[
  \loss(t)+{\lambda t^2\over2}R_*^2
  <
  \underline p\,\loss(\Gamma).
\]
If a point \(g\) has \(\min_a y_ag_a\le \Gamma\), then for some index
\(a_*\) one has \(y_{a_*}g_{a_*}\le\Gamma\).  Since \(q_{a_*}\ge
\underline p\) and the ridge term is nonnegative,
\[
  \Phi_{q,\lambda}(g)
  \ge q_{a_*}\loss(y_{a_*}g_{a_*})
  \ge \underline p\,\loss(\Gamma).
\]
Thus such a point cannot minimize \(\Phi_{q,\lambda}\), because the
comparison point \(tg_*\) has strictly smaller objective value.  Hence every
minimizer has signed margin strictly larger than \(\Gamma\).

It remains to choose a uniform ridge threshold.  Since \(\bN\succ0\)
and every vector in \(\mathcal F\) is nonzero, \(R_*^2>0\).  Choose
\(t_\Gamma>\Gamma\) so large that
\[
  \loss(t_\Gamma)<\underline p\,\loss(\Gamma),
\]
which is possible because \(\loss(t)\to0\) as \(t\to\infty\).  Set
\[
  \lambda_\Gamma(\underline p)
  :=
  {2\bigl(\underline p\,\loss(\Gamma)-\loss(t_\Gamma)\bigr)
  \over t_\Gamma^2R_*^2}.
\]
Then every \(0<\lambda<\lambda_\Gamma(\underline p)\) satisfies the strict
comparison above with \(t=t_\Gamma\), proving the proposition.
\end{proof}

\begin{proof}[Proof of Theorem~\ref{thm:idealmargin}]
By the representation lemma, Lemma \ref{lem:representation},
\[
  \Pp(E_{\rm rep})\ge 1-r e^{-np_{\min}/8}.
\]
On \(E_{\rm rep}\), every empirical template mass obeys
\[
  \widehat p_a\ge {p_a\over2}\ge {p_{\min}\over2}.
\]
Apply Proposition~\ref{prop:template_margin} with \(q=\widehat p\) and
\(\underline p=p_{\min}/2\).  This gives a number
\(\lambda_\Gamma^{\id}=\lambda_\Gamma(p_{\min}/2)\), depending only on
\(\Gamma,p_{\min},\bN\), and \(y\), such that for every
\(0<\lambda<\lambda_\Gamma^{\id}\), the minimizer \(g_\lambda^{\id}\) of
\(\Phi_{\widehat p,\lambda}\) satisfies
\[
  \min_a y_a(g_\lambda^{\id})_a>\Gamma .
\]
If \(x\in\cX_a^{\fr}\), then by definition
\(S_\lambda^{\id}(x)=(g_\lambda^{\id})_a\).  Therefore
\[
  y_aS_\lambda^{\id}(x)>\Gamma
\]
for every represented template block and every globally fresh test string in
that class.
\end{proof}

\begin{proof}[Proof of Theorem~\ref{thm:dual}]
Let \(A:\cH\to\R^n\) be the signed evaluation operator
\[
  (Af)_i=Y_if(X_i).
\]
The primal objective is
\[
  \inf_{f\in\cH}
  \left\{
    {1\over n}\sum_i\loss((Af)_i)+{\lambda\over2}\|f\|_{\cH}^2
  \right\}.
\]
The Fenchel representation \eqref{eq:fenchel} gives
\[
  \loss(u)=\sup_{c\in[0,1]}\{-cu-\varphi(c)\}.
\]
Equivalently, Lemma \ref{lem:logistic-aux} says that the conjugate of
\(u\mapsto\loss(u)\) is \(\varphi(-s)\) for \(s\in[-1,0]\), and is
\(+\infty\) otherwise.  Applying Fenchel--Rockafellar duality to
\((1/n)\sum_i\loss((Af)_i)\) and then setting \(c_i=-ns_i\) yields the dual
minimization problem
\[
  \min_{c\in[0,1]^n}
  \left\{
    {1\over n}\sum_i\varphi(c_i)
    +{1\over 2\lambda n^2}
      (\mathbf y\odot c)^\top\widehat K_n(\mathbf y\odot c)
  \right\}.
\]
This is \(D_{\widehat K_n}\).

The entropy term is strictly convex on \((0,1)^n\).  Indeed,
\eqref{eq:fenchel} gives
\[
  {1\over n}\diag\bigl(\varphi''(c_1),\ldots,\varphi''(c_n)\bigr)
  \succeq {4\over n}I .
\]
Thus the whole dual objective is strictly convex.  Existence follows by
continuity on the compact cube \([0,1]^n\).  The derivative of \(\varphi\)
tends to \(-\infty\) at \(0\) and to \(+\infty\) at \(1\), while the
quadratic part has finite directional derivatives.  Hence no minimizer lies
on the boundary, and the unique minimizer belongs to \((0,1)^n\).

The Fenchel optimality condition is
\[
  \lambda\widehat f_\lambda+A^*s=0,
  \qquad
  s_i=-{c_i\over n}.
\]
Since
\[
  A^*s=\sum_i s_iY_iK(X_i,\cdot),
\]
we obtain
\[
  \widehat f_\lambda(\cdot)
  ={1\over\lambda n}\sum_{i=1}^n c_iY_iK(X_i,\cdot).
\]
Evaluating at a test point \(x\) gives
\[
  \widehat f_\lambda(x)
  ={1\over\lambda n}k_x^\top(\mathbf y\odot c_{\widehat K_n}),
\]
as claimed.
\end{proof}

\begin{proof}[Proof of Theorem \ref{thm:detperturb}]
The two dual objectives have the same entropy term, and that entropy term is
\((4/n)\)-strongly convex.  Since \(c_G\) minimizes \(D_G\), the standard
strong-convexity stability inequality gives
\[
  \|c_G-c_M\|_2
  \le {n\over4}\|\nabla D_G(c_M)\|_2.
\]
Because \(c_M\) minimizes \(D_M\), \(\nabla D_M(c_M)=0\).  Therefore
\[
  \|c_G-c_M\|_2
  \le {n\over4}\|\nabla D_G(c_M)-\nabla D_M(c_M)\|_2.
\]
The entropy gradients cancel, and the remaining gradient difference is
\[
  \nabla D_G(c)-\nabla D_M(c)
  ={1\over\lambda n^2}\mathbf Y(G-M)\mathbf Yc,
  \qquad
  \mathbf Y=\diag(\mathbf y).
\]
Since \(c_M\in[0,1]^n\), \(\|c_M\|_2\le\sqrt n\).  Hence
\[
  \|c_G-c_M\|_2
  \le {1\over4\lambda\sqrt n}\|G-M\|_{\op}.
\]

Now split the score difference into the part caused by the change in the dual
minimizer and the part caused by the change in the test vector:
\[
  s_G-s_M
  ={1\over\lambda n}k^\top(\mathbf y\odot(c_G-c_M))
  +{1\over\lambda n}(k-m)^\top(\mathbf y\odot c_M).
\]
Using \(\|k\|_2\le\kappa^2\sqrt n\), \(\|c_M\|_2\le\sqrt n\), and
Cauchy--Schwarz,
\[
  |s_G-s_M|
  \le
  {\kappa^2\over\lambda\sqrt n}\|c_G-c_M\|_2
  +{1\over\lambda\sqrt n}\|k-m\|_2.
\]
Substituting the preceding bound on \(\|c_G-c_M\|_2\) proves
\[
  |s_G-s_M|
  \le
  {\kappa^2\over4\lambda^2 n}\|G-M\|_{\op}
  +{1\over\lambda\sqrt n}\|k-m\|_2.
\]
\end{proof}

\begin{proof}[Proof of Theorem~\ref{thm:curvature-spectral-perturbation}]
Write
\[
  \delta:=c_G-c_M.
\]
The Hessian of the empirical dual objective \(D_G\) is
\[
  \nabla^2D_G(c)
  ={1\over n}\diag\bigl(\varphi''(c_1),\ldots,\varphi''(c_n)\bigr)
  +{1\over\lambda n^2}\mathbf YG\mathbf Y .
\]
Since \(\varphi''\ge4\) on \((0,1)\) and \(G=M+\Delta\),
\[
  \nabla^2D_G(c)
  \succeq
  {4\over n}I
  +{1\over\lambda n^2}\mathbf YM\mathbf Y
  +{1\over\lambda n^2}\mathbf Y\Delta\mathbf Y
  = C_M+E_\Delta .
\]
If \(\gamma_{G|M}\ge1\), the asserted bound is vacuous by the convention
\(\mathcal E_M=+\infty\).  We may therefore assume \(\gamma_{G|M}<1\).  By
the definition of \(\gamma_{G|M}\),
\[
  -\gamma_{G|M}C_M
  \preceq E_\Delta
  \preceq \gamma_{G|M}C_M,
\]
and hence
\[
  \nabla^2D_G(c)\succeq (1-\gamma_{G|M})C_M
\]
throughout the line segment between \(c_M\) and \(c_G\).

The optimality conditions give
\[
  \nabla D_G(c_G)=0,
  \qquad
  \nabla D_M(c_M)=0.
\]
Thus
\[
  \nabla D_G(c_M)
  =\nabla D_G(c_M)-\nabla D_M(c_M)
  ={1\over\lambda n^2}\mathbf Y(G-M)\mathbf Yc_M
  =b.
\]
By the fundamental theorem of calculus,
\begin{align*}
  -b^\top\delta
  &=\bigl(\nabla D_G(c_G)-\nabla D_G(c_M)\bigr)^\top\delta \\
  &=\int_0^1
    \delta^\top\nabla^2D_G(c_M+t\delta)\delta\,dt \\
  &\ge (1-\gamma_{G|M})\delta^\top C_M\delta .
\end{align*}
Consequently,
\begin{equation}
\label{eq:template-curvature-spectral-basic-ineq}
  (1-\gamma_{G|M})\delta^\top C_M\delta+b^\top\delta\le0.
\end{equation}
Let \(A:=(1-\gamma_{G|M})C_M\).  Completing the square in
\eqref{eq:template-curvature-spectral-basic-ineq} gives
\[
  \left(\delta+{1\over2}A^{-1}b\right)^\top
  A
  \left(\delta+{1\over2}A^{-1}b\right)
  \le {1\over4}b^\top A^{-1}b.
\]
Thus for some vector \(w\) with
\(\|w\|_2\le {1\over2}\sqrt{b^\top A^{-1}b}\),
\[
  \delta=-{1\over2}A^{-1}b+A^{-1/2}w.
\]
For any \(a\in\R^n\), Cauchy--Schwarz yields
\[
  |a^\top\delta|
  \le
  {1\over2}|a^\top A^{-1}b|
  +{1\over2}
  \sqrt{a^\top A^{-1}a}\sqrt{b^\top A^{-1}b}.
\]
Since \(A^{-1}=(1-\gamma_{G|M})^{-1}C_M^{-1}\), this is exactly
\begin{equation}
\label{eq:template-metric-linear-functional}
  |a^\top\delta|
  \le \mathcal E_M(a;b).
\end{equation}

Finally decompose the score.  With \(\zeta=k-m\),
\[
  s_G-s_M
  ={1\over\lambda n}k^\top\mathbf Y(c_G-c_M)
  +{1\over\lambda n}(k-m)^\top\mathbf Yc_M.
\]
The first term on the right equals
\[
  {1\over\lambda n}m^\top\mathbf Y\delta
  +{1\over\lambda n}\zeta^\top\mathbf Y\delta
  =a_m^\top\delta+a_\zeta^\top\delta.
\]
Applying \eqref{eq:template-metric-linear-functional} with \(a=a_m\) and
with \(a=a_\zeta\) gives
\[
  |s_G-s_M|
  \le
  {1\over\lambda n}|\zeta^\top\mathbf Yc_M|
  +\mathcal E_M(a_m;b)
  +\mathcal E_M(a_\zeta;b),
\]
which is \eqref{eq:curvature-spectral-perturbation-main}.
\end{proof}

\begin{proof}[Proof of Corollary~\ref{cor:curvature-template-identities}]
In the template-ideal case, \(M=P\bN P^\top\) and \(\mathbf YP=PY_r\).  The
block-constant subspace is invariant under \(C_M\).  Indeed, for any
\(u\in\R^r\),
\[
  C_M(Pu)
  ={1\over n}P\left(4I+{1\over\lambda}Y_r\bN QY_r\right)u.
\]
Equivalently, in the orthonormal block basis
\[
  U:=n^{-1/2}PQ^{-1/2},
\]
the restriction of \(C_M\) to the block-constant subspace is
\[
  {1\over n}\mathcal A_{0,\lambda,\widehat p},
  \qquad
  \mathcal A_{0,\lambda,\widehat p}
  =4I+{1\over\lambda}Y_rQ^{1/2}\bN Q^{1/2}Y_r.
\]
Therefore, for any block-constant vectors \(P\alpha\) and \(P\beta\),
\begin{equation}
\label{eq:block-metric-bilinear-proof}
  (P\alpha)^\top C_M^{-1}(P\beta)
  =n^2\alpha^\top R_{0,\lambda,\widehat p}\beta,
  \qquad
  R_{0,\lambda,\widehat p}
  =Q^{1/2}\mathcal A_{0,\lambda,\widehat p}^{-1}Q^{1/2}.
\end{equation}

Now
\[
  a_m={1\over\lambda n}\mathbf Ym_a
  =P\left({1\over\lambda n}s_a\right).
\]
The block average of \(b\) is
\[
  \bar b_d
  ={1\over n_d}\sum_{i\in I_d}
  {1\over\lambda n^2}
  Y_i\sum_{j=1}^n\Delta_{ij}Y_jc_{M,j}
  ={1\over\lambda n}\left(Y_r\overline\Delta Q\beta\right)_d
  ={1\over\lambda n}(g_\Delta)_d.
\]
Thus
\[
  b_\parallel=P\bar b
  =P\left({1\over\lambda n}g_\Delta\right).
\]
Applying \eqref{eq:block-metric-bilinear-proof} first with
\(\alpha=\beta=(\lambda n)^{-1}s_a\), and then with
\(\alpha=(\lambda n)^{-1}s_a\) and
\(\beta=(\lambda n)^{-1}g_\Delta\), gives
\[
  a_m^\top C_M^{-1}a_m
  ={1\over\lambda^2}s_a^\top R_{0,\lambda,\widehat p}s_a,
  \qquad
  a_m^\top C_M^{-1}b
  ={1\over\lambda^2}s_a^\top R_{0,\lambda,\widehat p}g_\Delta.
\]

It remains to compute \(b^\top C_M^{-1}b\).  Decompose
\[
  b=b_\parallel+b_\perp,
\]
where \(b_\perp\) is Euclidean-orthogonal to the block-constant subspace.  Both the block-constant
subspace and its Euclidean orthogonal complement are invariant under \(C_M\), so the two pieces are
orthogonal in the \(C_M^{-1}\)-metric.  The block part is
\[
  b_\parallel^\top C_M^{-1}b_\parallel
  ={1\over\lambda^2}g_\Delta^\top R_{0,\lambda,\widehat p}g_\Delta.
\]
For the orthogonal part, \(P^\top b_\perp=0\).  Since \(\mathbf Y\) is constant on each block,
\(P^\top\mathbf Yb_\perp=0\) as well.  Hence
\[
  \mathbf YM\mathbf Yb_\perp
  =\mathbf YP\bN P^\top\mathbf Yb_\perp
  =0.
\]
Therefore \(C_Mb_\perp=(4/n)b_\perp\), and
\[
  b_\perp^\top C_M^{-1}b_\perp
  ={n\over4}\|b_\perp\|_2^2.
\]
Combining the block and orthogonal pieces proves
\[
  b^\top C_M^{-1}b
  ={1\over\lambda^2}g_\Delta^\top R_{0,\lambda,\widehat p}g_\Delta
  +{n\over4}\|b_\perp\|_2^2,
\]
which is the third identity in \eqref{eq:curvature-spectral-template-identities}.
\end{proof}

\begin{proof}[Proof of Corollary~\ref{cor:Mn-level-curvature-spectral-upper-bound}]
We first prove a deterministic curvature bound.  The probability and graph estimates are inserted
only at the end.  By Lemma \ref{lem:template-kernel-aux}, the ideal sample matrix \(M_n\) is positive semidefinite,
so Theorem \ref{thm:curvature-spectral-perturbation} applies with
\[
  G=\widehat K_n,
  \qquad
  M=M_n=P\bN P^\top,
  \qquad
  k=k_x,
  \qquad
  m=m_a.
\]
The exact block reduction, Theorem \ref{thm:idealreduction}, identifies the reference score with
\(S_\lambda^{\id}(x)\).  With the global ideal dual vector \(c_M\) and
\(\zeta=k_x-m_a\), set
\[
  b:={1\over\lambda n^2}\mathbf Y\Delta\mathbf Yc_M.
\]
Introduce
\[
  z_x:={1\over n}|\zeta^\top\mathbf Yc_M|,
  \qquad
  e_x:={1\over\sqrt n}\|\zeta\|_2,
  \qquad
  \Delta_{\rm blk}:=\max_{b,c}|\overline\Delta_{bc}|,
\]
and
\[
  a_\Delta:={1\over n^{3/2}}\|\Delta\mathbf Yc_M\|_2.
\]
The perturbation theorem gives
\begin{equation}
\label{eq:Mn-level-curvature-master}
  \bigl|\widehat f_\lambda(x)-S_\lambda^{\id}(x)\bigr|
  \le
  {1\over\lambda}z_x
  +\mathcal E_M(a_m;b)
  +\mathcal E_M(a_\zeta;b),
\end{equation}
where
\[
  a_m={1\over\lambda n}\mathbf Ym_a,
  \qquad
  a_\zeta={1\over\lambda n}\mathbf Y\zeta.
\]
If \(\gamma_{\cspec}\ge1\), the stated curvature certificate is \(+\infty\), so there is nothing
to prove.  We therefore assume \(\gamma_{\cspec}<1\) for the remainder of the deterministic
estimates.

By Theorem \ref{thm:idealreduction}, the ideal dual vector is block-constant.  Write
\[
  c_M=P\theta,
  \qquad
  \beta:=y\odot\theta,
  \qquad
  |\beta_b|\le1.
\]
Let
\[
  s_a:=Y_r n_a^{\rm tmplt},
  \qquad
  g_\Delta:=Y_r\overline\Delta Q\beta,
  \qquad
  d_\Delta:=\left(g_\Delta^\top Qg_\Delta\right)^{1/2}.
\]
The template identities in Corollary \ref{cor:curvature-template-identities} use
\[
  R_{0,\lambda,\widehat p}
  =Q^{1/2}
  \left(4I+{1\over\lambda}Y_rQ^{1/2}\bN Q^{1/2}Y_r\right)^{-1}
  Q^{1/2}.
\]
Since \(\bN\succeq0\),
\[
  0\preceq R_{0,\lambda,\widehat p}\preceq {1\over4}Q.
\]
Also \(|N_{ab}|\le K_*\), so
\[
  s_a^\top R_{0,\lambda,\widehat p}s_a
  \le {1\over4}s_a^\top Qs_a
  \le {K_*^2\over4}.
\]
The block-action vector satisfies
\[
  |(g_\Delta)_b|
  \le \sum_c \widehat p_c |\overline\Delta_{bc}|\,|\beta_c|
  \le \Delta_{\rm blk},
\]
and therefore
\begin{equation}
\label{eq:Mn-level-delta-block-action-bound}
  d_\Delta\le\Delta_{\rm blk}.
\end{equation}
By Cauchy--Schwarz in the \(R_{0,\lambda,\widehat p}\)-metric,
\begin{equation}
\label{eq:Mn-level-alignment-sharp-bound}
  |a_m^\top C_M^{-1}b|
  ={1\over\lambda^2}|s_a^\top R_{0,\lambda,\widehat p}g_\Delta|
  \le {K_*d_\Delta\over4\lambda^2}.
\end{equation}

Next we control \(b^\top C_M^{-1}b\) without charging the block component twice.  Let
\[
  b_\parallel=P\bar b,
  \qquad
  \bar b={1\over\lambda n}g_\Delta,
  \qquad
  b_\perp=b-b_\parallel.
\]
By Corollary \ref{cor:curvature-template-identities},
\[
  b^\top C_M^{-1}b
  ={1\over\lambda^2}g_\Delta^\top R_{0,\lambda,\widehat p}g_\Delta
  +{n\over4}\|b_\perp\|_2^2.
\]
Because \(b_\parallel\) is the Euclidean block projection of \(b\),
\[
  \|b\|_2^2={a_\Delta^2\over\lambda^2 n},
  \qquad
  \|b_\parallel\|_2^2={g_\Delta^\top Qg_\Delta\over\lambda^2 n}
  ={d_\Delta^2\over\lambda^2 n}.
\]
Thus
\begin{equation}
\label{eq:Mn-level-Pythagorean-cancellation}
  {n\over4}\|b_\perp\|_2^2
  ={a_\Delta^2-d_\Delta^2\over4\lambda^2}.
\end{equation}
Since \(R_{0,\lambda,\widehat p}\preceq Q/4\),
\[
  {1\over\lambda^2}g_\Delta^\top R_{0,\lambda,\widehat p}g_\Delta
  \le {d_\Delta^2\over4\lambda^2}.
\]
Adding this to \eqref{eq:Mn-level-Pythagorean-cancellation} gives
\begin{equation}
\label{eq:Mn-level-residual-a2-only}
  b^\top C_M^{-1}b
  \le {a_\Delta^2\over4\lambda^2}.
\end{equation}
Combining \eqref{eq:Mn-level-alignment-sharp-bound},
\eqref{eq:Mn-level-residual-a2-only}, and
\[
  a_m^\top C_M^{-1}a_m
  ={1\over\lambda^2}s_a^\top R_{0,\lambda,\widehat p}s_a
  \le {K_*^2\over4\lambda^2},
\]
we obtain
\begin{equation}
\label{eq:Mn-level-am-curvature-term}
  \mathcal E_M(a_m;b)
  \le
  {K_*\over8\lambda^2(1-\gamma_{\cspec})}(d_\Delta+a_\Delta).
\end{equation}

For the test-residual term, use \(C_M\succeq(4/n)I\), so
\(C_M^{-1}\preceq(n/4)I\).  Then
\[
  a_\zeta^\top C_M^{-1}a_\zeta
  \le {n\over4}\|a_\zeta\|_2^2
  ={e_x^2\over4\lambda^2}.
\]
Together with \eqref{eq:Mn-level-residual-a2-only}, Cauchy--Schwarz gives
\[
  |a_\zeta^\top C_M^{-1}b|
  \le
  \sqrt{a_\zeta^\top C_M^{-1}a_\zeta}
  \sqrt{b^\top C_M^{-1}b}
  \le {e_xa_\Delta\over4\lambda^2}.
\]
The same product bounds the square-root term in \(\mathcal E_M(a_\zeta;b)\).  Hence
\begin{equation}
\label{eq:Mn-level-azeta-curvature-term}
  \mathcal E_M(a_\zeta;b)
  \le {e_xa_\Delta\over4\lambda^2(1-\gamma_{\cspec})}.
\end{equation}
Substituting \eqref{eq:Mn-level-am-curvature-term} and
\eqref{eq:Mn-level-azeta-curvature-term} into
\eqref{eq:Mn-level-curvature-master} gives
\begin{equation}
\label{eq:Mn-level-readable-curvature-spectral-bound}
  \bigl|\widehat f_\lambda(x)-S_\lambda^{\id}(x)\bigr|
  \le
  {1\over\lambda}z_x
  +{K_*d_\Delta+(K_*+2e_x)a_\Delta
    \over 8\lambda^2(1-\gamma_{\cspec})}.
\end{equation}
Using \eqref{eq:Mn-level-delta-block-action-bound}, we also have the certificate-ready bound
\begin{equation}
\label{eq:Mn-level-readable-curvature-spectral-bound-crude-action}
  \bigl|\widehat f_\lambda(x)-S_\lambda^{\id}(x)\bigr|
  \le
  {1\over\lambda}z_x
  +{K_*\Delta_{\rm blk}+(K_*+2e_x)a_\Delta
    \over 8\lambda^2(1-\gamma_{\cspec})}.
\end{equation}

We substitute the graph and KL certificates.  Since \(\mathbf Yc_M=P\beta\), the block-vector
identity from Lemma \ref{lem:blockproj-hessian} gives
\[
  {1\over n}\zeta^\top\mathbf Yc_M
  =\overline\zeta^{\,\top}Q\beta,
  \qquad
  \overline\zeta_b={1\over n_b}\sum_{i\in I_b}\zeta_i.
\]
Therefore
\[
  z_x\le\sum_b\widehat p_b|\overline\zeta_b|.
\]
On the KL test event from Theorem \ref{thm:klbdens-kl-block-envelope},
\[
  |\overline\zeta_b|
  \le L_*U_{\KL}(n_b,q_{bx},u)
  \qquad\text{for every }b,
\]
and hence
\begin{equation}
\label{eq:Mn-level-zx-certificate-bound}
  z_x\le L_*X_\star(u).
\end{equation}

Second, the test-edge support statement in Lemma \ref{lem:collision-support-front} gives
\(|\zeta_i|\le L_*\Ccol_{0i}\).  If
\[
  C_{b,x}:=\sum_{i\in I_b}\Ccol_{0i},
\]
then
\[
  e_x^2
  ={1\over n}\sum_i\zeta_i^2
  \le
  L_*^2\sum_b\widehat p_b{C_{b,x}\over n_b}.
\]
The same KL test event gives \(C_{b,x}/n_b\le U_{\KL}(n_b,q_{bx},u)\), so
\begin{equation}
\label{eq:Mn-level-ex-certificate-bound}
  e_x\le L_*\sqrt{X_\star(u)}.
\end{equation}

Third, the matrix part of Theorem \ref{thm:klbdens-kl-block-envelope} gives
\begin{equation}
\label{eq:Mn-level-block-certificate-bound}
  \Delta_{\rm blk}\le E_*(u).
\end{equation}
Finally, the action term satisfies two deterministic bounds.  The operator bound is
\[
  a_\Delta
  ={1\over n^{3/2}}\|\Delta\mathbf Yc_M\|_2
  \le {1\over n^{3/2}}\|\Delta\|_{\op}\|c_M\|_2
  \le {\|\Delta\|_{\op}\over n},
\]
because \(c_M\in[0,1]^n\).  The row-degree support bound from
Lemma \ref{lem:degcert-degree-action-support-control}, with
\(d_i=\sum_{j\ne i}\Ccol_{ij}\), gives
\[
  a_\Delta
  \le {\delta_*\over n}+L_*{d_2\over n^{3/2}}.
\]
Thus
\begin{equation}
\label{eq:Mn-level-action-certificate-bound}
  a_\Delta\le A_\Delta.
\end{equation}
Substituting
\eqref{eq:Mn-level-zx-certificate-bound}--\eqref{eq:Mn-level-action-certificate-bound} into
\eqref{eq:Mn-level-readable-curvature-spectral-bound-crude-action} gives
\[
  \bigl|\widehat f_\lambda(x)-S_\lambda^{\id}(x)\bigr|
  \le
  \operatorname{Curv}_\lambda
  \bigl(A_\Delta,E_*(u),X_\star(u),\gamma_{\cspec}\bigr)
  =B_{\cspec}(\lambda;x,u),
\]
which is the claimed substituted curvature-spectral certificate.

It remains to record the simple bound on \(\gamma_{\cspec}\).  Since
\(C_M\succeq(4/n)I\),
\[
  \|C_M^{-1/2}\|_{\op}\le\sqrt{n/4}.
\]
As \(\mathbf Y\) is an orthogonal diagonal sign matrix,
\[
  \gamma_{\cspec}
  =
  \left\|
  C_M^{-1/2}{1\over\lambda n^2}\mathbf Y\Delta\mathbf Y
  C_M^{-1/2}
  \right\|_{\op}
  \le
  {n\over4}{\|\Delta\|_{\op}\over\lambda n^2}
  ={\|\Delta\|_{\op}\over4\lambda n}.
\]  On \(E_{\rm rep}\), the optional closed-form
rates for \(E_*(u)\) and \(X_\star(u)\) follow directly from
Corollary \ref{cor:klbdens-closed-form-envelope}.
\end{proof}

\begin{proof}[Proof of Theorem~\ref{thm:idealreduction}]
By Lemma \ref{lem:template-kernel-aux}, the ideal sample-level matrix \(M_n\) is positive semidefinite, so the dual
formula and uniqueness statement in Theorem \ref{thm:dual} apply.  The ideal matrix is block-constant:
\[
  (M_n)_{ij}=N_{\tau(i)\tau(j)}.
\]
Thus the dual objective associated with \(M_n\) is invariant under permutations of sample indices
inside each template block \(I_b=\{i:\tau(i)=b\}\).  By uniqueness of the dual minimizer,
\[
  c_{M_n,i}=\theta_b
  \qquad\text{for every } i\in I_b.
\]
If \(I_b\) is empty, choose \(\theta_b\in(0,1)\) arbitrarily; every expression below multiplies
that coordinate by \(\widehat p_b=0\).

Define
\[
  g:={1\over\lambda}\bN(\widehat p\odot y\odot\theta).
\]
For a globally fresh \(x\in\cX_a^{\fr}\),
\[
  {1\over\lambda n}m_a^\top(\mathbf y\odot c_{M_n})
  ={1\over\lambda}\sum_b N_{ab}\widehat p_b y_b\theta_b
  =g_a.
\]
It remains to identify this vector \(g\) with the finite-template minimizer
\(g_\lambda^{\id}\).

The sample-level dual objective, restricted to block-constant vectors \(c=P\theta\), is
\[
  \sum_b\widehat p_b\varphi(\theta_b)
  +{1\over2\lambda}
    (\widehat p\odot y\odot\theta)^\top
    \bN
    (\widehat p\odot y\odot\theta).
\]
On nonempty blocks, the first-order equations are
\[
  \varphi'(\theta_b)+y_bg_b=0.
\]
Using Lemma \ref{lem:logistic-aux}, this is equivalently
\[
  \theta_b=-\loss'(y_bg_b).
\]
Also,
\[
  \bN^{-1}g={1\over\lambda}(\widehat p\odot y\odot\theta).
\]
Therefore, for every coordinate \(b\), including empty blocks where \(\widehat p_b=0\),
\[
  \widehat p_b\,y_b\loss'(y_bg_b)
  +\lambda(\bN^{-1}g)_b=0.
\]
This is the first-order condition for \(\Phi_{\widehat p,\lambda}\).  Since
\(\Phi_{\widehat p,\lambda}\) is strictly convex, \(g=g_\lambda^{\id}\).  Consequently
\[
  {1\over\lambda n}m_a^\top(\mathbf y\odot c_{M_n})
  =g_a
  =(g_\lambda^{\id})_a
  =S_\lambda^{\id}(x).
\]
\end{proof}

\section{Auxiliary results}
\label{sec:auxiliary-results}

We record the auxiliary facts used in the proof.

\begin{lemma}
\label{lem:logistic-aux}
Let
\[
\loss(u):=\log(1+e^{-u}).
\]
Then \(\loss\) is convex and differentiable on \(\R\), with
\[
\loss'(u)=-\frac{1}{1+e^u},
\qquad
|\loss'(u)|\le 1,
\qquad
\loss''(u)=\frac{e^u}{(1+e^u)^2}\le \frac14 .
\]
Moreover, for every \(u\in\R\),
\begin{equation}
 \label{eq:fenchel}
\loss(u)=\sup_{c\in[0,1]}\{-cu-\varphi(c)\},
\end{equation}
and the unique maximizer is
\[
c(u)=\frac{1}{1+e^u}.
\]
For \(0<c<1\),
\[
\varphi'(c)=\log\frac{c}{1-c},
\qquad
\varphi''(c)=\frac{1}{c}+\frac{1}{1-c}\ge 4.
\]
Finally, the convex conjugate of \(\loss\) is
\[
\loss^*(s)
=
\begin{cases}
\varphi(-s), & s\in[-1,0],\\[0.3em]
+\infty, & s\notin[-1,0].
\end{cases}
\]
\end{lemma}

\begin{proof}
The elementary differential identities follow by direct calculation:
\[
\loss'(u)
=
\frac{-e^{-u}}{1+e^{-u}}
=
-\frac{1}{1+e^u},
\qquad
\loss''(u)
=
\frac{e^u}{(1+e^u)^2}.
\]
Thus \(|\loss'(u)|\le 1\), and \(\loss''(u)\ge 0\), so \(\loss\) is convex. Writing \(v=e^{u/2}\), we also have
\[
\loss''(u)=\frac{v^2}{(1+v^2)^2}.
\]
Since
\[
(1+v^2)^2-4v^2=(v^2-1)^2\ge 0,
\]
it follows that \(\loss''(u)\le 1/4\).

For the variational formula, fix \(u\in\R\) and set
\[
H_u(c):=-cu-\varphi(c),
\qquad 0<c<1.
\]
Then
\[
H_u'(c)
=
-u-\log\frac{c}{1-c},
\qquad
H_u''(c)
=
-\frac{1}{c}-\frac{1}{1-c}<0.
\]
Hence \(H_u\) is strictly concave. Its critical point satisfies
\[
\log\frac{c}{1-c}=-u,
\]
so
\[
c_*=\frac{1}{1+e^u}.
\]
Since \(H_u\) extends continuously to \([0,1]\), this critical point is the unique maximizer on the closed interval. Evaluating at \(c_*\), using
\[
c_*=\frac{1}{1+e^u},
\qquad
1-c_*=\frac{e^u}{1+e^u},
\]
gives
\begin{align*}
H_u(c_*)
&=
-c_*u-c_*\log c_*-(1-c_*)\log(1-c_*)  \\
&=
-u+\log(1+e^u)
=
\log(1+e^{-u})
=
\loss(u).
\end{align*}
This proves the Fenchel representation and identifies the maximizer.

The displayed formulas for \(\varphi'\) and \(\varphi''\) follow from differentiating \(\varphi\). In particular,
\[
\varphi''(c)
=
\frac{1}{c(1-c)}
\ge 4,
\]
because \(c(1-c)\le 1/4\).

It remains to identify the conjugate. Rewriting the Fenchel representation with \(s=-c\), we get
\[
\loss(u)
=
\sup_{s\in[-1,0]}\{su-\varphi(-s)\}.
\]
Equivalently, \(\loss=q^*\), where
\[
q(s):=\varphi(-s)+I_{[-1,0]}(s).
\]
The function \(q\) is proper, closed, and convex, so Fenchel--Moreau gives
\[
\loss^*=q^{**}=q.
\]
This is the stated formula.
\end{proof}

\begin{lemma}
\label{lem:template-kernel-aux}
Under the standing assumptions on the templates and on the token-symmetric kernel \(K\), the following hold.
First, the label map
\[
y_*:\bigcup_{j=1}^r \cX_j\to \{-1,+1\},
\qquad
y_*(x):=y_j \ \text{whenever }x\in\cX_j,
\]
is well defined.
Second, if \((s,s')\) and \((t,t')\) are fresh admissible pairs for the same ordered pair \((z_i,z_j)\), then
\[
K(\subop(z_i,s),\subop(z_j,s'))
=
K(\subop(z_i,t),\subop(z_j,t')).
\]
Thus \(N_{ij}\) is independent of the fresh admissible choice.
Third, for each template \(a\), the self-kernel constant \(D_a\) is independent of the admissible substitution map.
Finally, \(\bN\) and \(M_n\) are symmetric positive semidefinite, and for every template index \(a\),
\[
D_a-N_{aa}\ge 0.
\]
In particular, \(\delta_*<\infty\).
\end{lemma}

\begin{proof}
The label map is well defined because the sets \(\cX_1,\ldots,\cX_r\) are pairwise disjoint. Hence every point in their union belongs to a unique \(\cX_j\).

The kernel invariance statements all follow from the same permutation argument. For the entries of \(\bN\), freshness makes the prescriptions
\[
\pi(a)=a \quad (a\in R),
\qquad
\pi(s(w))=t(w) \quad (w\in\cW(z_i)),
\qquad
\pi(s'(w))=t'(w) \quad (w\in\cW(z_j))
\]
consistent on the finite set of tokens involved. Extend this partial bijection to a permutation \(\pi\) of \(\cV\). Then
\[
\pi(\subop(z_i,s))=\subop(z_i,t),
\qquad
\pi(\subop(z_j,s'))=\subop(z_j,t'),
\]
and token symmetry gives
\[
K(\subop(z_i,s),\subop(z_j,s'))
=
K(\subop(z_i,t),\subop(z_j,t')).
\]
Thus \(N_{ij}\) is well defined.

The same argument applies to \(D_a\). If \(s\) and \(t\) are admissible substitutions for \(z_a\), then the assignments fixing \(L(z_a)\) and sending \(s(w)\) to \(t(w)\) for each \(w\in\cW(z_a)\) extend to a permutation of \(\cV\). Token symmetry therefore gives
\[
K(\subop(z_a,s),\subop(z_a,s))
=
K(\subop(z_a,t),\subop(z_a,t)),
\]
so \(D_a\) is well defined.

Symmetry of \(\bN\) follows by swapping the two substitutions and using symmetry of \(K\). For positive semidefiniteness, choose admissible substitutions \(s_1,\ldots,s_r\) whose nonliteral images are jointly fresh across the templates. By the preceding invariance,
\[
N_{ij}
=
K(\subop(z_i,s_i),\subop(z_j,s_j)).
\]
Thus \(\bN\) is the Gram matrix of the instantiated strings
\[
\subop(z_1,s_1),\ldots,\subop(z_r,s_r),
\]
and is positive semidefinite.

Now fix \(a\), choose a fresh admissible pair \((s,t)\) for \((z_a,z_a)\), and write
\[
x:=\subop(z_a,s),
\qquad
x':=\subop(z_a,t).
\]
Then
\[
K(x,x)=K(x',x')=D_a,
\qquad
K(x,x')=N_{aa}.
\]
Since \(K\) is positive semidefinite,
\[
0
\le
\begin{pmatrix}1&-1\end{pmatrix}
\begin{pmatrix}
D_a & N_{aa}\\
N_{aa} & D_a
\end{pmatrix}
\begin{pmatrix}1\\-1\end{pmatrix}
=
2(D_a-N_{aa}),
\]
so \(D_a-N_{aa}\ge 0\).

The boundedness of the kernel diagonal and Cauchy--Schwarz give
\[
D_a=K(x,x)\le \|K\|_\infty,
\qquad
|N_{aa}|=|K(x,x')|
\le
\sqrt{K(x,x)K(x',x')}
=
D_a
\le
\|K\|_\infty.
\]
Hence
\[
|D_a-N_{aa}|
\le
2\|K\|_\infty.
\]
Taking the maximum over finitely many template indices gives \(\delta_*<\infty\).

Finally, for the membership matrix \(P\),
\[
(P\bN P^\top)_{ij}
=
\sum_{b=1}^r\sum_{c=1}^r P_{ib}N_{bc}P_{jc}
=
N_{\tau(i)\tau(j)}
=
(M_n)_{ij}.
\]
Thus
\[
M_n=P\bN P^\top.
\]
Since \(\bN\) is symmetric positive semidefinite, the same is true of \(M_n\).
\end{proof}

\begin{lemma}
\label{lem:representation}
For the global representation event \(E_{\rm rep}\),
\[
\Pp(E_{\rm rep})\ge 1-r\,e^{-n p_{\min}/8}.
\]
\end{lemma}

\begin{proof}
For each \(j\),
\[
n_j\sim \mathrm{Binomial}(n,p_j),
\qquad
\E n_j=np_j.
\]
The multiplicative Chernoff lower-tail bound with parameter \(1/2\) gives
\[
\Pp\!\left(n_j\le \frac{np_j}{2}\right)
\le
\exp\!\left(-\frac{np_j}{8}\right).
\]
Since \(\widehat p_j=n_j/n\),
\[
\Pp\!\left(\widehat p_j\le \frac{p_j}{2}\right)
\le
e^{-np_j/8}
\le
e^{-n p_{\min}/8}.
\]
Taking a union bound over \(j=1,\ldots,r\),
\[
\Pp(E_{\rm rep}^c)
\le
\sum_{j=1}^r
\Pp\!\left(\widehat p_j<\frac{p_j}{2}\right)
\le
r e^{-n p_{\min}/8}.
\]
The claim follows by taking complements.
\end{proof}

\subsection{Sharp rates}
\label{app:sharp-transfer-machinery}
\label{app:certificates}

 \subsubsection{Block-averaged linear response}
\label{sec:block-averaged-linear-response}

After conditioning on the template assignment, the ideal problem admits a finite-dimensional block reduction.  A row-degree argument gives a direct bound on
\[
(\widehat K_n-M_n)(\mathbf y\odot c_{M_n}),
\]
but such a bound does not exploit the block structure of the ideal dual solution.  The ideal dual vector, the fresh-test representer, and the inverse Hessian of the ideal dual objective all lie in the block-constant subspace generated by the latent templates.  Consequently, the first-order score perturbation depends on the empirical kernel discrepancy only through block averages.

Throughout this subsection we condition on a realized template assignment for which every template is represented.  Let
\[
\mathcal B_n:=\{Pu:u\in\R^r\}\subseteq\R^n
\]
be the block-constant subspace.  For \(A\in\R^{n\times n}\) and \(z\in\R^n\), define their empirical block averages by
\[
\overline A_{bc}
:=
\frac{1}{n_b n_c}
\sum_{i\in I_b}\sum_{j\in I_c} A_{ij},
\qquad
\overline z_b
:=
\frac{1}{n_b}\sum_{i\in I_b}z_i.
\]
We also write
\[
Q:=\operatorname{diag}(\widehat p_1,\ldots,\widehat p_r).
\]

\begin{lemma}
\label{lem:blockproj-hessian}
For every \(A\in\R^{n\times n}\), \(z\in\R^n\), and \(u,v\in\R^r\),
\begin{equation}
\label{eq:blockproj-block-bilinear-identity}
(Pu)^\top A(Pv)
=
n^2 u^\top Q\overline A Qv,
\end{equation}
and
\begin{equation}
\label{eq:blockproj-block-vector-identity}
z^\top Pu
=
n\,\overline z^{\,\top}Qu.
\end{equation}

Let \(J_M\) be the Hessian of the ideal sample-level dual objective at \(c_M\):
\[
J_M
=
\frac1n\operatorname{diag}\bigl(\varphi''(c_{M,1}),\ldots,\varphi''(c_{M,n})\bigr)
+
\frac{1}{\lambda n^2}\mathbf Y M_n\mathbf Y .
\]
Write \(c_M=P\theta\), \(\beta=y\odot\theta\), and
\[
Y_r:=\operatorname{diag}(y_1,\ldots,y_r),
\qquad
H_\theta:=\operatorname{diag}\bigl(\varphi''(\theta_1),\ldots,\varphi''(\theta_r)\bigr).
\]
Define
\begin{equation}
\label{eq:blockproj-A-lambda-def}
A_{\lambda,\widehat p}
:=
H_\theta+\frac{1}{\lambda}Y_r\bN QY_r .
\end{equation}
Then \(\mathcal B_n\) and \(\mathcal B_n^\perp\) are invariant under \(J_M\).  Moreover, for every \(u,z\in\R^r\),
\begin{equation}
\label{eq:blockproj-HM-restriction}
J_M(Pu)
=
\frac1n P A_{\lambda,\widehat p}u,
\end{equation}
and
\begin{equation}
\label{eq:blockproj-HM-inverse-restriction}
J_M^{-1}(Pz)
=
nP A_{\lambda,\widehat p}^{-1}z.
\end{equation}
\end{lemma}

\begin{proof}
The block-average identities follow directly from the definitions:
\[
(Pu)^\top A(Pv)
=
\sum_{b,c}u_bv_c
\sum_{i\in I_b}\sum_{j\in I_c}A_{ij}
=
\sum_{b,c}u_bv_c\,n_bn_c\,\overline A_{bc}
=
n^2u^\top Q\overline A Qv.
\]
Similarly,
\[
z^\top Pu
=
\sum_b u_b\sum_{i\in I_b}z_i
=
\sum_b u_b n_b\overline z_b
=
n\,\overline z^{\,\top}Qu.
\]

We identify the action of \(J_M\) on the block-constant subspace.  Since \(c_M=P\theta\), the diagonal entropy part of \(J_M\) is constant on each block.  It therefore preserves both \(\mathcal B_n\) and \(\mathcal B_n^\perp\).  Also,
\[
M_n=P\bN P^\top,
\qquad
\mathbf YP=PY_r.
\]
If \(v\in\mathcal B_n^\perp\), then \(P^\top v=0\).  Since \(\mathbf Y\) is block-constant, \(P^\top\mathbf Yv=0\) as well, and therefore
\[
\mathbf YM_n\mathbf Yv
=
\mathbf YP\bN P^\top\mathbf Yv
=
0.
\]
Thus \(\mathcal B_n^\perp\) is invariant under the quadratic part of \(J_M\).

For \(u\in\R^r\),
\[
\mathbf YM_n\mathbf Y(Pu)
=
\mathbf YP\bN P^\top PY_ru
=
PY_r\bN \operatorname{diag}(n_1,\ldots,n_r)Y_ru
=
nP Y_r\bN QY_ru.
\]
Combining this with the entropy diagonal part gives
\[
J_M(Pu)
=
\frac1nPH_\theta u
+
\frac{1}{\lambda n}P Y_r\bN QY_ru
=
\frac1nP A_{\lambda,\widehat p}u,
\]
which proves \eqref{eq:blockproj-HM-restriction}.  Since \(J_M\) is positive definite, its restriction to \(\mathcal B_n\) is invertible, and \eqref{eq:blockproj-HM-inverse-restriction} follows by solving the last display for \(u\).
\end{proof}

For a globally fresh test string \(x\in\cX_a^{\mathrm{fr}}\), set
\[
n_a^{\mathrm{tmplt}}
:=
(N_{a1},\ldots,N_{ar})^\top,
\qquad
m_a:=Pn_a^{\mathrm{tmplt}},
\]
and define
\begin{equation}
\label{eq:blockproj-ra-def}
r_a
:=
Y_rA_{\lambda,\widehat p}^{-1}Y_r n_a^{\mathrm{tmplt}}.
\end{equation}

\begin{lemma}
\label{lem:blockproj-test-representer}
For every globally fresh \(x\in\cX_a^{\mathrm{fr}}\),
\begin{equation}
\label{eq:blockproj-Hinv-m-block}
\mathbf YJ_M^{-1}\mathbf Ym_a
=
nPr_a.
\end{equation}
Moreover,
\begin{equation}
\label{eq:blockproj-ra-weighted-bound}
\left(\sum_{b=1}^r\widehat p_b r_{ab}^2\right)^{1/2}
\le
\frac{K_*}{4},
\qquad
\sum_{b=1}^r\widehat p_b |r_{ab}|
\le
\frac{K_*}{4}.
\end{equation}
\end{lemma}

\begin{proof}
Since \(m_a=Pn_a^{\mathrm{tmplt}}\) and \(\mathbf YP=PY_r\),
\[
\mathbf Ym_a=PY_r n_a^{\mathrm{tmplt}}.
\]
Using \eqref{eq:blockproj-HM-inverse-restriction},
\[
J_M^{-1}\mathbf Ym_a
=
nP A_{\lambda,\widehat p}^{-1}Y_r n_a^{\mathrm{tmplt}}.
\]
Multiplying by \(\mathbf Y\) gives
\[
\mathbf YJ_M^{-1}\mathbf Ym_a
=
nP Y_rA_{\lambda,\widehat p}^{-1}Y_r n_a^{\mathrm{tmplt}}
=
nPr_a.
\]

For the bound, use the Hessian lower bound \(J_M\succeq (4/n)I_n\), which gives
\[
\|J_M^{-1}\|_{\mathrm{op}}\le \frac n4.
\]
Since every coordinate of \(m_a\) is bounded in absolute value by \(K_*\),
\[
\|m_a\|_2\le K_*\sqrt n.
\]
Therefore
\[
\|\mathbf YJ_M^{-1}\mathbf Ym_a\|_2
\le
\frac n4\|m_a\|_2
\le
\frac{K_*n^{3/2}}4.
\]
On the other hand, by \eqref{eq:blockproj-Hinv-m-block},
\[
\|nPr_a\|_2
=
n\left(\sum_{b=1}^r n_b r_{ab}^2\right)^{1/2}
=
n^{3/2}
\left(\sum_{b=1}^r\widehat p_b r_{ab}^2\right)^{1/2}.
\]
Dividing by \(n^{3/2}\) proves the first estimate in \eqref{eq:blockproj-ra-weighted-bound}.  The second follows from Cauchy--Schwarz and \(\sum_b\widehat p_b=1\).
\end{proof}

We next isolate the first-order block-averaged score perturbation.  Let \(\Delta\) and \(\zeta\) denote the global kernel and test-vector discrepancies, and write \(\overline\Delta\) and \(\overline\zeta\) for their empirical block averages.  Define
\begin{equation}
\label{eq:blockproj-linearized-block-score}
\mathcal L_a^{\mathrm{blk}}(\Delta,\zeta)
:=
-
\frac{1}{\lambda^2}
\sum_{b=1}^r\sum_{c=1}^r
\widehat p_b\widehat p_c\,r_{ab}\,\beta_c\,\overline\Delta_{bc}
+
\frac{1}{\lambda}
\sum_{b=1}^r
\widehat p_b\,\beta_b\,\overline\zeta_b .
\end{equation}
For later use, set
\begin{equation}
\label{eq:blockproj-Wa-def}
W_a:=\sum_{b=1}^r\widehat p_b|r_{ab}|,
\qquad
\Delta_{\mathrm{blk}}:=\max_{b,c}|\overline\Delta_{bc}|,
\qquad
Z_{\mathrm{blk}}:=\max_b|\overline\zeta_b|.
\end{equation}
We shall also use the weighted envelopes
\begin{align}
\label{eq:blockproj-weighted-block-score-envelopes}
\Delta_{\mathrm w}
&:=
\sum_{b=1}^r\sum_{c=1}^r
\widehat p_b\widehat p_c |r_{ab}|\,|\beta_c|\,|\overline\Delta_{bc}|, \\
Z_{\mathrm w}
&:=
\sum_{b=1}^r
\widehat p_b |\beta_b|\,|\overline\zeta_b|.
\end{align}
These satisfy
\begin{equation}
\label{eq:blockproj-weighted-vs-max}
\Delta_{\mathrm w}\le W_a\Delta_{\mathrm{blk}},
\qquad
Z_{\mathrm w}\le Z_{\mathrm{blk}},
\end{equation}
while the weighted quantities can be sharper when the large block averages occur on lightly weighted template pairs.

\begin{lemma}
\label{lem:blockproj-linearized-block-bound}
The block-linearized score obeys
\begin{equation}
\label{eq:blockproj-linearized-block-bound}
|\mathcal L_a^{\mathrm{blk}}(\Delta,\zeta)|
\le
\frac{1}{\lambda^2}\Delta_{\mathrm w}
+
\frac{1}{\lambda}Z_{\mathrm w}.
\end{equation}
Consequently,
\begin{equation}
\label{eq:blockproj-linearized-block-bound-max}
|\mathcal L_a^{\mathrm{blk}}(\Delta,\zeta)|
\le
\frac{W_a}{\lambda^2}\Delta_{\mathrm{blk}}
+
\frac{1}{\lambda}Z_{\mathrm{blk}},
\end{equation}
and
\begin{equation}
\label{eq:blockproj-linearized-block-bound-coarse}
|\mathcal L_a^{\mathrm{blk}}(\Delta,\zeta)|
\le
\frac{K_*}{4\lambda^2}\Delta_{\mathrm{blk}}
+
\frac{1}{\lambda}Z_{\mathrm{blk}}.
\end{equation}
\end{lemma}

\begin{proof}
Since \(0<\theta_b<1\), we have \(|\beta_b|=|y_b\theta_b|\le 1\).  Taking absolute values in \eqref{eq:blockproj-linearized-block-score} gives
\[
|\mathcal L_a^{\mathrm{blk}}(\Delta,\zeta)|
\le
\frac{1}{\lambda^2}
\sum_{b,c}\widehat p_b\widehat p_c |r_{ab}|\,|\beta_c|\,|\overline\Delta_{bc}|
+
\frac1\lambda
\sum_b\widehat p_b|\beta_b|\,|\overline\zeta_b|,
\]
which is \eqref{eq:blockproj-linearized-block-bound}.  The maximum bound follows from \eqref{eq:blockproj-weighted-vs-max}, and the final bound follows from the weighted estimate on \(r_a\).
\end{proof}

 \subsubsection{Block-average concentration}
\label{sec:blockavg-sharp-block-average-bounds}

The block reduction in Section \ref{sec:block-averaged-linear-response} shows that the first-order score perturbation is governed by block averages of \(\Delta\) and \(\zeta\), rather than by a global Frobenius or operator norm of \(\Delta\).  We now collect the concentration ingredients needed to control these averages.

There are two complementary estimates.  The first is a support bound, which treats every collision as a worst-case change in the corresponding kernel entry.  This is sharp when all collision contributions have the same sign.  The second is a signed, kernel-weighted bound, which keeps the centered entries and can be sharper when collision contributions have small mean, small variance, or cancel in block average.  In both estimates, literal collisions are measured only against the literals appearing in the two strings being compared, rather than against the full global literal set.

Throughout this subsection we condition on a realized template assignment for which every block is nonempty.

\begin{lemma}
\label{lem:blockavg-pair-specific-entrywise-agreement}
Condition on \(\tau(i)=b\) and \(\tau(j)=c\), with \(i\ne j\), and define
\begin{align}
\label{eq:blockavg-pair-bad-event}
\mathcal B_{ij}(b,c)
:=
&\{S_i(W_b)\cap L_c\ne\varnothing\}
\cup
\{S_j(W_c)\cap L_b\ne\varnothing\}
\nonumber\\
&\cup
\{S_i(W_b)\cap S_j(W_c)\ne\varnothing\}.
\end{align}
Then, on \(\mathcal B_{ij}(b,c)^c\),
\begin{equation}
\label{eq:blockavg-pair-entrywise-agree}
K(X_i,X_j)=N_{bc}.
\end{equation}
Moreover,
\begin{equation}
\label{eq:blockavg-pair-prob-bound}
\Pp\bigl(\mathcal B_{ij}(b,c)\mid \tau(i)=b,\tau(j)=c\bigr)
\le q_{bc}.
\end{equation}

Let \(x\in\cX_a^{\mathrm{fr}}\) be a fixed globally fresh test string.  Conditional on \(\tau(i)=b\), define
\begin{equation}
\label{eq:blockavg-test-bad-event}
\mathcal T_i(b,x)
:=
\{S_i(W_b)\cap L_a\ne\varnothing\}
\cup
\{S_i(W_b)\cap T_x\ne\varnothing\}.
\end{equation}
Then, on \(\mathcal T_i(b,x)^c\),
\begin{equation}
\label{eq:blockavg-test-entrywise-agree}
K(x,X_i)=N_{ab},
\end{equation}
and
\begin{equation}
\label{eq:blockavg-test-prob-bound}
\Pp\bigl(\mathcal T_i(b,x)\mid \tau(i)=b\bigr)
\le q_{bx}.
\end{equation}
\end{lemma}

\begin{proof}
On \(\mathcal B_{ij}(b,c)^c\), the wildcard image of \(X_i\) avoids the literals of \(z_c\), the wildcard image of \(X_j\) avoids the literals of \(z_b\), and the two wildcard images are disjoint.  Admissibility also gives
\[
S_i(W_b)\cap L_b=\varnothing,
\qquad
S_j(W_c)\cap L_c=\varnothing.
\]
Thus the wildcard tokens in the two strings are disjoint from one another and from all literals in \(L_b\cup L_c\).

Let \((u,v)\) be a fresh admissible pair used to define \(N_{bc}\).  The map that fixes \(L_b\cup L_c\), sends \(S_i(w)\) to \(u(w)\) for \(w\in W_b\), and sends \(S_j(w)\) to \(v(w)\) for \(w\in W_c\), is a consistent partial bijection on the finite set of tokens involved.  Extend it to a permutation \(\pi\) of the vocabulary.  Token symmetry then gives
\[
K(X_i,X_j)
=
K(\pi(X_i),\pi(X_j))
=
K(\subop(z_b,u),\subop(z_c,v))
=
N_{bc}.
\]

The probability bound follows by a union bound.  The probability that \(S_i(W_b)\) hits \(L_c\) is at most
\[
\sum_{t\in L_c}\Pp(t\in S_i(W_b)\mid \tau(i)=b)
=
\ell_{b\to c},
\]
and the analogous probability that \(S_j(W_c)\) hits \(L_b\) is at most \(\ell_{c\to b}\).  For the shared-wildcard term, conditional independence gives
\[
\Pp(t\in S_i(W_b),\ t\in S_j(W_c)\mid \tau(i)=b,\tau(j)=c)
=
p_{b,t}p_{c,t},
\]
and summing over \(t\) gives the contribution \(\chi_{bc}\).  Combining the three terms gives \eqref{eq:blockavg-pair-prob-bound}.

The test-string statement is identical in form.  On \(\mathcal T_i(b,x)^c\), the training wildcard image avoids both \(L_a\) and the fresh tokens \(T_x\), while admissibility already gives avoidance of \(L_b\).  Hence a vocabulary permutation fixing \(L_a\cup L_b\) carries the pair \((x,X_i)\) to a fresh admissible pair defining \(N_{ab}\), so \(K(x,X_i)=N_{ab}\).  The two collision probabilities are bounded by \(\ell_{b\to a}\) and \(\chi_{b,x}\), respectively.
\end{proof}

We next record the elementary edge colorings used to separate dependent kernel entries into independent classes.

\begin{lemma}
\label{lem:blockavg-block-edge-colorings}
For two nonempty blocks \(I_b,I_c\), the following partitions exist.

If \(b\ne c\), the complete bipartite edge set \(I_b\times I_c\) can be partitioned into
\[
\max\{n_b,n_c\}
\]
matchings.  If \(b=c\), the unordered edge set
\[
\{\{i,j\}:i,j\in I_b,\ i<j\}
\]
can be partitioned into \(m_{n_b}\) matchings, where
\[
m_{n_b}
:=
\begin{cases}
n_b-1, & n_b\text{ even},\\
n_b, & n_b\text{ odd}.
\end{cases}
\]
\end{lemma}

\begin{proof}
For \(b\ne c\), this is the standard edge coloring of \(K_{n_b,n_c}\).  Assuming \(n_b\le n_c\), label the two parts by \(\{1,\ldots,n_b\}\) and \(\{1,\ldots,n_c\}\), and place \((i,j)\) in color class \(q\in\{0,\ldots,n_c-1\}\) when
\[
j-i\equiv q\pmod{n_c}.
\]
Each color class is a matching, and the \(n_c\) classes cover all edges.

For \(b=c\), use the usual round-robin one-factorization.  If \(n_b\) is even, it partitions the complete graph on \(I_b\) into \(n_b-1\) perfect matchings.  If \(n_b\) is odd, add one dummy vertex, apply the even construction to \(n_b+1\) vertices, and then delete the dummy-incident edges.  The remaining \(n_b\) matchings partition the original edge set.
\end{proof}

The coloring lets us apply Bernstein's inequality color by color.

\begin{lemma}
\label{lem:blockavg-signed-partition-bernstein}
Let \(\mathcal I\) be a finite index set and let \((X_\alpha)_{\alpha\in\mathcal I}\) be real random variables satisfying
\[
|X_\alpha|\le L
\qquad\text{almost surely}.
\]
Assume that
\[
\mathcal I=\mathcal I_1\sqcup\cdots\sqcup\mathcal I_m
\]
and that, for each \(q\), the variables \(\{X_\alpha:\alpha\in\mathcal I_q\}\) are jointly independent.  Define
\[
S:=\sum_{\alpha\in\mathcal I}(X_\alpha-\E X_\alpha),
\qquad
V:=\sum_{\alpha\in\mathcal I}\Var(X_\alpha).
\]
Then, for every \(u>0\),
\begin{equation}
\label{eq:blockavg-signed-partition-bernstein}
\Pp\left(
|S|>
\sqrt{2mVu}+\frac{4}{3}mLu
\right)
\le 2e^{-u}.
\end{equation}
\end{lemma}

\begin{proof}
Write \(Z_\alpha=X_\alpha-\E X_\alpha\).  Then \(\E Z_\alpha=0\), \(\Var(Z_\alpha)=\Var(X_\alpha)\), and \(|Z_\alpha|\le 2L\).  The standard Bernstein moment estimate gives, for \(0\le s<3/(2L)\),
\[
\E e^{sZ_\alpha}
\le
\exp\left(
\frac{s^2\Var(Z_\alpha)}{2(1-2Ls/3)}
\right).
\]

For \(\theta\ge0\) with \(m\theta<3/(2L)\), Holder's inequality gives
\[
\E e^{\theta S}
\le
\prod_{q=1}^m
\left[
\E\exp\left(m\theta\sum_{\alpha\in\mathcal I_q}Z_\alpha\right)
\right]^{1/m}.
\]
Within each color class the variables are independent, so the Bernstein moment estimate yields
\[
\log\E e^{\theta S}
\le
\frac{m\theta^2V}{2(1-2Lm\theta/3)}.
\]
Chernoff's bound then gives
\[
\Pp(S\ge t)
\le
\exp\left(
-\frac{t^2}{2m(V+2Lt/3)}
\right).
\]
The stated choice
\[
t=\sqrt{2mVu}+\frac{4}{3}mLu
\]
gives the desired one-sided bound.  Applying the same argument to \(-S\) and taking a union bound proves \eqref{eq:blockavg-signed-partition-bernstein}.
\end{proof}

 \subsubsection{Deterministic transfer reductions}
\label{sec:degcert-resulting-sharp-transfer}

The concentration estimates above control the random block averages and collision actions.  We now record the deterministic transfer step that turns those quantities into score bounds.  The first reduction is a row-action stability estimate; it uses only the uniform curvature of the logistic dual entropy and therefore requires neither a coordinate lower bound on the ideal dual solution nor a perturbative smallness condition.  The curvature-spectral path is already contained in Corollary \ref{cor:Mn-level-curvature-spectral-upper-bound}; we only restate below how it enters the budget notation of Section \ref{sec:graph-budgets}.

\begin{lemma}[Degree-action stability]
\label{thm:degcert-global-degree-action-transfer}
Let \(G\succeq0\) be a Gram matrix with test vector \(k\), and let \(M\succeq0\) be a reference Gram matrix with test vector \(m\).  Let \(c_G,c_M\) be the corresponding dual minimizers, and write
\[
s_G={1\over\lambda n}k^\top(\mathbf y\odot c_G),
\qquad
s_M={1\over\lambda n}m^\top(\mathbf y\odot c_M).
\]
Then, with \(\Delta_{G|M}:=G-M\) and \(\zeta_{k|m}:=k-m\),
\begin{equation}
\label{eq:degcert-global-degree-action-transfer-general}
|s_G-s_M|
\le
\frac{K_*}{4\lambda^2}
\frac{\|\Delta_{G|M}\mathbf Yc_M\|_2}{n^{3/2}}
+
\frac1\lambda
\frac{|\zeta_{k|m}^\top\mathbf Yc_M|}{n},
\end{equation}
provided \(\|k\|_2\le K_*\sqrt n\).
\end{lemma}

\begin{proof}
By Theorem \ref{thm:dual}, the two dual objectives differ only in their quadratic kernel terms.  Since
\[
\varphi''(c)\ge4
\]
on \((0,1)\), the dual objective associated with \(G\) is \((4/n)\)-strongly convex.  The standard stability estimate for strongly convex functions gives
\[
\|c_G-c_M\|_2
\le
\frac n4
\|\nabla D_G(c_M)-\nabla D_M(c_M)\|_2 .
\]
The entropy gradients cancel, and the remaining difference is
\[
\nabla D_G(c_M)-\nabla D_M(c_M)
=
\frac{1}{\lambda n^2}\mathbf Y(G-M)\mathbf Yc_M.
\]
Thus
\begin{equation}
\label{eq:degcert-global-cdiff-bound}
\|c_G-c_M\|_2
\le
\frac{1}{4\lambda n}
\|\Delta_{G|M}\mathbf Yc_M\|_2 .
\end{equation}

Decomposing the score difference,
\[
s_G-s_M
=
\frac1{\lambda n}k^\top\mathbf Y(c_G-c_M)
+
\frac1{\lambda n}\zeta_{k|m}^\top\mathbf Yc_M .
\]
The first term is bounded by Cauchy--Schwarz and \(\|k\|_2\le K_*\sqrt n\):
\[
\frac1{\lambda n}
|k^\top\mathbf Y(c_G-c_M)|
\le
\frac{K_*}{\lambda\sqrt n}
\|c_G-c_M\|_2 .
\]
Substituting \eqref{eq:degcert-global-cdiff-bound} gives the first term in
\eqref{eq:degcert-global-degree-action-transfer-general}; the second term is already in the desired form.
\end{proof}

The following consequence is the form used by the block-density, Hoeffding--ANOVA, and bias--fluctuation paths.  All quantities appearing in it were defined earlier in Section \ref{sec:graph-budgets} and Section \ref{sec:block-averaged-linear-response}.

\begin{corollary}[Fresh-target deterministic transfer]
\label{thm:degcert-deterministic-sharp-transfer}
\label{lem:degcert-zid-block-control}
\label{lem:degcert-degree-action-support-control}
For the fresh ideal reference \((M_n,m_a)\),
\begin{equation}
\label{eq:degcert-deterministic-sharp-transfer}
\bigl|\widehat f_\lambda(x)-S_\lambda^{\id}(x)\bigr|
\le
\operatorname{Ord}_\lambda(A_\Delta,0,Z_{\mathrm w}).
\end{equation}
Equivalently,
\[
\bigl|\widehat f_\lambda(x)-S_\lambda^{\id}(x)\bigr|
\le
\frac{K_*}{4\lambda^2}A_\Delta
+
\frac1\lambda Z_{\mathrm w}.
\]
The projection-augmented form
\begin{equation}
\label{eq:degcert-projection-augmented-certificate}
\bigl|\widehat f_\lambda(x)-S_\lambda^{\id}(x)\bigr|
\le
\operatorname{Ord}_\lambda(A_\Delta,\Delta_{\mathrm w},Z_{\mathrm w})
\end{equation}
also holds.

In particular, on any event where
\[
\Delta_{\mathrm w}\le D,
\qquad
Z_{\mathrm w}\le Z,
\]
one has
\begin{equation}
\label{eq:degcert-generic-projected-budget}
\bigl|\widehat f_\lambda(x)-S_\lambda^{\id}(x)\bigr|
\le
\operatorname{Ord}_\lambda(A_\Delta,D,Z).
\end{equation}
The same statement holds for any block-constant positive-semidefinite reference target, with the corresponding reference quantities substituted throughout.
\end{corollary}

\begin{proof}
Apply Theorem \ref{thm:degcert-global-degree-action-transfer} with
\[
(G,M,k,m)=(\widehat K_n,M_n,k_x,m_a).
\]
By Theorem \ref{thm:idealreduction}, the reference score is \(S_\lambda^{\id}(x)\).  Also \(\|k_x\|_2\le K_*\sqrt n\), since each coordinate of \(k_x\) is bounded by \(K_*\).

It remains to connect the two terms in
\eqref{eq:degcert-global-degree-action-transfer-general} with the previously defined budget quantities.  First,
\[
\frac{\|\Delta\mathbf Yc_M\|_2}{n^{3/2}}
\le
\frac{\|\Delta\|_{\op}}{n},
\]
because \(c_M\in[0,1]^n\).  The collision-support bound also gives
\[
\frac{\|\Delta\mathbf Yc_M\|_2}{n^{3/2}}
\le
\frac{\delta_*}{n}
+
L_*\frac{d_2}{n^{3/2}}.
\]
Therefore, by the definition of \(A_\Delta\) in Section \ref{sec:graph-budgets},
\[
\frac{\|\Delta\mathbf Yc_M\|_2}{n^{3/2}}
\le A_\Delta.
\]

For the test term, \(c_M=P\theta\), so \(\mathbf Yc_M=P\beta\).  Using the block-vector identity from Lemma \ref{lem:blockproj-hessian},
\[
\frac1n\zeta^\top\mathbf Yc_M
=
\overline\zeta^{\,\top}Q\beta
=
\sum_b\widehat p_b\,\overline\zeta_b\,\beta_b.
\]
Taking absolute values gives
\[
\frac1n|\zeta^\top\mathbf Yc_M|
\le
\sum_b\widehat p_b|\beta_b|\,|\overline\zeta_b|
=
Z_{\mathrm w}.
\]
Substituting these two estimates into
\eqref{eq:degcert-global-degree-action-transfer-general} proves
\eqref{eq:degcert-deterministic-sharp-transfer}.  The projection-augmented bound follows by adding the nonnegative term \(\lambda^{-2}\Delta_{\mathrm w}\).  The final statement is the same argument applied to any block-constant reference target.
\end{proof}

It is sometimes convenient to name the action part of the ordinary budget:
\begin{equation}
\label{eq:degcert-degree-remainder-def}
R_{\degact}(\lambda;x)
:=
\frac{K_*}{4\lambda^2}A_\Delta .
\end{equation}
For corrected targets we use the analogous shorthand
$
R_{\degact}^{\circ}(\lambda;x)
:=
\frac{K_*}{4\lambda^2}A_\Delta^\circ .
$
Thus \eqref{eq:degcert-deterministic-sharp-transfer} may be written as
\[
\bigl|\widehat f_\lambda(x)-S_\lambda^{\id}(x)\bigr|
\le
R_{\degact}(\lambda;x)
+
\frac1\lambda Z_{\mathrm w}.
\]

\begin{corollary}[Curvature-spectral transfer in budget form]
\label{thm:degcert-global-curvature-spectral-transfer}
On the KL block/test event of Theorem \ref{thm:klbdens-kl-block-envelope},
\begin{equation}
\label{eq:degcert-global-curvature-spectral-transfer}
\bigl|\widehat f_\lambda(x)-S_\lambda^{\id}(x)\bigr|
\le
B_{\cspec}(\lambda;x,u).
\end{equation}
\end{corollary}

\begin{proof}
This is the \(M_n\)-level curvature-spectral specialization in
Corollary \ref{cor:Mn-level-curvature-spectral-upper-bound}, written using the budget notation from
Equation \ref{eq:five-certificates-compact}.
\end{proof}

   \subsubsection{Block-density KL envelope}
\label{sec:klbdens-sharpness-exact-counts}

The deterministic transfer step reduces the block-density path to controlling the empirical block averages
\[
\overline\Delta_{bc}
\qquad\text{and}\qquad
\overline\zeta_b .
\]
The envelopes \(E_{bc}(u)\), \(T_b(u)\), \(D_{\bdens}(u)\), \(Z_{\bdens}(u)\), and the budget
\(B_{\bdens}(\lambda;x,u)\) were defined in Section \ref{sec:graph-budgets}.  We now prove that these
envelopes hold with high probability.  The KL inverse keeps the exact Bernoulli tail for the relevant collision counts; the Bernstein-type form is used only as a later relaxation.

\begin{lemma}[Colored Bernoulli KL bound]
\label{lem:klbdens-partition-kl}
\label{lem:klbdens-kl-bernstein-relaxation}
Let \(\mathcal I\) be a finite index set partitioned as
\[
\mathcal I=\mathcal I_1\sqcup\cdots\sqcup\mathcal I_m .
\]
Let \(B_\alpha\in\{0,1\}\), \(\alpha\in\mathcal I\), satisfy
\[
\Pp(B_\alpha=1)\le q .
\]
Assume that, within each color class \(\mathcal I_s\), the variables
\(\{B_\alpha:\alpha\in\mathcal I_s\}\) are jointly independent.  Put
\[
P:=|\mathcal I|,
\qquad
N_{\mathrm{eff}}:=P/m,
\qquad
\overline B:=P^{-1}\sum_{\alpha\in\mathcal I}B_\alpha .
\]
Then, for every \(p\in[q,1]\),
\begin{equation}
\label{eq:klbdens-partition-kl-tail}
\Pp(\overline B\ge p)
\le
\exp\{-N_{\mathrm{eff}}D_{\rm B}(p\|q)\}.
\end{equation}
Consequently, for every \(u>0\),
\begin{equation}
\label{eq:klbdens-partition-kl-envelope}
\Pp\left(
\overline B\le U_{\KL}(N_{\mathrm{eff}},q,u)
\right)
\ge 1-e^{-u}.
\end{equation}
Moreover, whenever \(N_{\mathrm{eff}}>0\),
\begin{equation}
\label{eq:klbdens-kl-bernstein-relaxation}
U_{\KL}(N_{\mathrm{eff}},q,u)
\le
q+
\sqrt{\frac{2q(1-q)u}{N_{\mathrm{eff}}}}
+
\frac{2u}{3N_{\mathrm{eff}}}
\le
q+
\sqrt{\frac{2qu}{N_{\mathrm{eff}}}}
+
\frac{2u}{3N_{\mathrm{eff}}}.
\end{equation}
\end{lemma}

\begin{proof}
Write \(S_s:=\sum_{\alpha\in\mathcal I_s}B_\alpha\) and \(S:=\sum_sS_s\).  For \(\theta\ge0\), Holder's inequality gives
\[
\E e^{\theta S}
=
\E\prod_{s=1}^m e^{\theta S_s}
\le
\prod_{s=1}^m
\left(\E e^{m\theta S_s}\right)^{1/m}.
\]
Inside each color class the variables are independent, and since each \(B_\alpha\) is Bernoulli
with success probability at most \(q\),
\[
\E e^{m\theta S_s}
\le
\prod_{\alpha\in\mathcal I_s}
\left(1+q(e^{m\theta}-1)\right).
\]
Thus
\[
\E e^{\theta S}
\le
\left(1+q(e^{m\theta}-1)\right)^{P/m}.
\]
Chernoff's bound gives, for \(p\ge q\),
\[
\Pp(S\ge Pp)
\le
\inf_{\theta\ge0}
\exp\left\{
-\theta Pp
+
\frac{P}{m}\log\bigl(1+q(e^{m\theta}-1)\bigr)
\right\}.
\]
With \(\eta=m\theta\), the exponent is
\[
-\frac{P}{m}
\left\{
\eta p-\log(1+q(e^\eta-1))
\right\}.
\]
Optimizing over \(\eta\ge0\) gives the Bernoulli Cramer transform
\(D_{\rm B}(p\|q)\), proving \eqref{eq:klbdens-partition-kl-tail}.  The inverse-envelope statement follows from the definition of \(U_{\KL}\).

For the closed-form relaxation, use the standard Bennett lower bound
\[
D_{\rm B}(q+s\|q)
\ge
\frac{s^2}{2q(1-q)+2s/3},
\qquad 0\le s\le 1-q .
\]
Taking
\[
s=
\sqrt{\frac{2q(1-q)u}{N_{\mathrm{eff}}}}
+
\frac{2u}{3N_{\mathrm{eff}}},
\]
one checks that
\[
D_{\rm B}(q+s\|q)\ge \frac{u}{N_{\mathrm{eff}}}
\]
whenever \(q+s\le1\); if \(q+s>1\), the claim is trivial.  This proves the first inequality in
\eqref{eq:klbdens-kl-bernstein-relaxation}, and the second follows from \(1-q\le1\).
\end{proof}

\begin{theorem}[KL block and test envelopes]
\label{thm:klbdens-kl-block-envelope}
Condition on a template assignment with \(n_b\ge1\) for every \(b\).  Fix \(u>0\).  With conditional probability at least
\[
1-(r^2+r)e^{-u},
\]
the following bounds hold simultaneously:
\begin{equation}
\label{eq:klbdens-block-envelope-compact}
|\overline\Delta_{bc}|\le E_{bc}(u)
\qquad
(1\le b,c\le r),
\end{equation}
and
\begin{equation}
\label{eq:klbdens-test-envelope-compact}
|\overline\zeta_b|\le T_b(u)
\qquad
(1\le b\le r).
\end{equation}
\end{theorem}

\begin{proof}
Fix first two distinct blocks \(b\ne c\), and let
\[
B_{bc}:=
\sum_{i\in I_b}\sum_{j\in I_c}
\mathbf 1_{\mathcal B_{ij}(b,c)} .
\]
By Lemma \ref{lem:blockavg-pair-specific-entrywise-agreement}, the indicators have conditional means at most \(q_{bc}\), and \(\Delta_{ij}=0\) outside \(\mathcal B_{ij}(b,c)\), while \(|\Delta_{ij}|\le L_*\) always.  Hence
\[
|\overline\Delta_{bc}|
\le
L_*\frac{B_{bc}}{n_bn_c}.
\]
The complete bipartite edge set \(I_b\times I_c\) can be partitioned into
\(\max\{n_b,n_c\}\) matchings by Lemma \ref{lem:blockavg-block-edge-colorings}.  Within each matching the collision indicators are independent, since no training substitution appears twice.  Applying Lemma \ref{lem:klbdens-partition-kl} gives the bound
\[
|\overline\Delta_{bc}|\le E_{bc}(u)
\]
with failure probability at most \(e^{-u}\).

For a diagonal block, write
\[
\overline\Delta_{bb}
=
\frac1{n_b^2}\sum_{i\in I_b}\Delta_{ii}
+
\frac1{n_b^2}
\sum_{\substack{i,j\in I_b\\ i\ne j}}\Delta_{ij}.
\]
The diagonal contribution is bounded by \(\delta_*/n_b\).  The off-diagonal contribution is controlled by the unordered collision count
\[
\sum_{\substack{i<j\\ i,j\in I_b}}
\mathbf 1_{\mathcal B_{ij}(b,b)} .
\]
The round-robin coloring from Lemma \ref{lem:blockavg-block-edge-colorings}, followed again by
Lemma \ref{lem:klbdens-partition-kl}, gives the within-block part of \(E_{bb}(u)\).  If \(n_b=1\), there is no off-diagonal term.

For the test block, set
\[
C_{b,x}:=\sum_{i\in I_b}\mathbf 1_{\mathcal T_i(b,x)} .
\]
The indicators are independent over \(i\in I_b\), each has conditional probability at most
\(q_{bx}\), and \(|\overline\zeta_b|\le L_*C_{b,x}/n_b\).  The one-color version of
Lemma \ref{lem:klbdens-partition-kl} gives
\[
|\overline\zeta_b|\le T_b(u)
\]
with failure probability at most \(e^{-u}\).

A union bound over the \(r^2\) Gram blocks and \(r\) test blocks proves the theorem.
\end{proof}

\begin{corollary}[Closed-form scale]
\label{cor:klbdens-closed-form-envelope}
On the event of Theorem \ref{thm:klbdens-kl-block-envelope}, every occurrence of the KL envelope
\(U_{\KL}(N_{\mathrm{eff}},q,u)\) in the block and test bounds may be relaxed to
\[
q+
\sqrt{\frac{2qu}{N_{\mathrm{eff}}}}
+
\frac{2u}{3N_{\mathrm{eff}}}.
\]
In particular, on \(E_{\rm rep}\), the block-density terms have the scale
\begin{equation}
\label{eq:klbdens-representation-scale}
q+
\sqrt{\frac{q u}{np_{\min}}}
+
\frac{u}{np_{\min}},
\end{equation}
up to universal constants, with the additional same-block diagonal contribution
\(\delta_*/(np_{\min})\).
\end{corollary}

\begin{proof}
The first statement is exactly \eqref{eq:klbdens-kl-bernstein-relaxation}.  On \(E_{\rm rep}\),
\[
n_b\ge \frac{np_{\min}}2
\]
for every \(b\).  Thus the effective sample sizes appearing in the cross-block and test envelopes
are bounded below by a constant multiple of \(np_{\min}\).  The same is true for the within-block
effective sample size whenever \(n_b\ge2\).  The displayed scale follows after absorbing numerical constants.
\end{proof}

\begin{corollary}[Block-density transfer certificate]
\label{cor:klbdens-kl-refined-transfer}
Conditional on a template assignment with all block sizes positive, with probability at least
\[
1-(r^2+r)e^{-u},
\]
one has
\begin{equation}
\label{eq:klbdens-kl-refined-transfer}
|\widehat f_\lambda(x)-S_\lambda^{\id}(x)|
\le
B_{\bdens}(\lambda;x,u).
\end{equation}
On \(E_{\rm rep}\), the block-density part of this budget has the closed-form scale
\[
\frac{1}{\lambda^2}
\left(
q_{\Delta,*}
+
\sqrt{\frac{q_{\Delta,*}u}{np_{\min}}}
+
\frac{u}{np_{\min}}
\right)
+
\frac{1}{\lambda}
\left(
q_{x,*}
+
\sqrt{\frac{q_{x,*}u}{np_{\min}}}
+
\frac{u}{np_{\min}}
\right),
\]
up to fixed kernel and template constants, where
\[
q_{\Delta,*}:=\max_{b,c}q_{bc},
\qquad
q_{x,*}:=\max_b q_{bx}.
\]
The realized weighted budget \(B_{\bdens}\) can be smaller than this max-envelope relaxation.
\end{corollary}

\begin{proof}
Use the projection-augmented deterministic transfer certificate
\eqref{eq:degcert-projection-augmented-certificate}:
\[
|\widehat f_\lambda(x)-S_\lambda^{\id}(x)|
\le
\operatorname{Ord}_\lambda(A_\Delta,\Delta_{\mathrm w},Z_{\mathrm w}).
\]
On the event from Theorem \ref{thm:klbdens-kl-block-envelope}, the bounds
\[
|\overline\Delta_{bc}|\le E_{bc}(u),
\qquad
|\overline\zeta_b|\le T_b(u)
\]
hold for every \(b,c\).  Substituting these bounds into the weighted definitions of
\(\Delta_{\mathrm w}\) and \(Z_{\mathrm w}\) gives
\[
\Delta_{\mathrm w}\le D_{\bdens}(u),
\qquad
Z_{\mathrm w}\le Z_{\bdens}(u).
\]
Therefore
\[
|\widehat f_\lambda(x)-S_\lambda^{\id}(x)|
\le
\operatorname{Ord}_\lambda(A_\Delta,D_{\bdens}(u),Z_{\bdens}(u))
=
B_{\bdens}(\lambda;x,u),
\]
as claimed.  The probability is inherited from Theorem \ref{thm:klbdens-kl-block-envelope}, and the
closed-form display follows from Corollary \ref{cor:klbdens-closed-form-envelope}.
\end{proof}

  \subsubsection{Hoeffding--ANOVA block envelopes}
\label{app:projection}
\label{sec:hanova-hanova-rates}

The KL block-density bound treats each discrepancy as a supported collision event.  The next refinement keeps the actual signed kernel discrepancy.  For each template pair, the block average admits a Hoeffding--ANOVA decomposition into a deterministic bias, two first-order projection terms, and a canonical residual.  This separates the nondegenerate fluctuations, which scale with the individual block sizes, from the genuinely degenerate two-sample fluctuation, which can scale with the product of the block sizes.

Throughout this subsection we condition on a realized template assignment with all blocks nonempty.  For independent substitutions \(U_b\sim\mu_{\mathrm{sub},b}\) and \(V_c\sim\mu_{\mathrm{sub},c}\), define
\begin{equation}
\label{eq:hanova-Hbc-def}
H_{bc}(u,v)
:=
K\bigl(\subop(z_b,u),\subop(z_c,v)\bigr)-N_{bc}.
\end{equation}
For a fixed globally fresh test string \(x\in\cX_a^{\fr}\), define
\begin{equation}
\label{eq:hanova-Gbx-def}
G_{b,x}(u)
:=
K\bigl(x,\subop(z_b,u)\bigr)-N_{ab}.
\end{equation}
By Cauchy--Schwarz in the RKHS,
\begin{equation}
\label{eq:hanova-HG-bounds}
|H_{bc}(u,v)|\le L_*,
\qquad
|G_{b,x}(u)|\le L_*.
\end{equation}

Set
\begin{equation}
\label{eq:hanova-mu-nu-def}
\mu_{bc}:=\E H_{bc}(U_b,V_c),
\qquad
\nu_{b,x}:=\E G_{b,x}(U_b).
\end{equation}
These are signed collision biases.  They obey the support envelopes
$
|\mu_{bc}|\le L_*q_{bc},
\
|\nu_{b,x}|\le L_*q_{bx},
$
but can be much smaller because positive and negative discrepancies may cancel.

For \(b\ne c\), define the first projections
\begin{align}
\label{eq:hanova-two-sample-projections}
h_{b\to c}(u)
&:=
\E_{V_c}H_{bc}(u,V_c)-\mu_{bc}, \qquad
h_{c\to b}(v)
 :=
\E_{U_b}H_{bc}(U_b,v)-\mu_{bc},
\end{align}
and the canonical part
\begin{equation}
\label{eq:hanova-two-sample-canonical}
h_{bc}^{\circ}(u,v)
:=
H_{bc}(u,v)-\mu_{bc}-h_{b\to c}(u)-h_{c\to b}(v).
\end{equation}
Then
$
\E h_{b\to c}(U_b)=0,
\
\E h_{c\to b}(V_c)=0,
$
and
$
\E[h_{bc}^{\circ}(U_b,V_c)\mid U_b]=0,
\
\E[h_{bc}^{\circ}(U_b,V_c)\mid V_c]=0.
$
Let
\begin{equation}
\label{eq:hanova-variance-components}
\sigma_{b\to c}^2:=\E h_{b\to c}(U_b)^2,
\qquad
\sigma_{c\to b}^2:=\E h_{c\to b}(V_c)^2,
\qquad
\omega_{bc}^2:=\E h_{bc}^{\circ}(U_b,V_c)^2,
\end{equation}
and
\begin{equation}
\label{eq:hanova-projection-envelopes}
M_{b\to c}:=\|h_{b\to c}\|_\infty,
\qquad
M_{c\to b}:=\|h_{c\to b}\|_\infty.
\end{equation}
From \eqref{eq:hanova-HG-bounds},
$
M_{b\to c}\le 2L_*,
\
M_{c\to b}\le 2L_*,
\
\|h_{bc}^{\circ}\|_\infty\le 4L_*.
$

For \(b=c\), let
\[
\mu_{bb}:=\E H_{bb}(U_b,V_b),
\qquad
h_b(u):=\E_{V_b}H_{bb}(u,V_b)-\mu_{bb},
\]
and
\[
h_{bb}^{\circ}(u,v)
:=
H_{bb}(u,v)-\mu_{bb}-h_b(u)-h_b(v),
\]
where \(U_b,V_b\) are independent.  Define
\begin{equation}
\label{eq:hanova-same-block-variance-def}
\sigma_b^2:=\E h_b(U_b)^2,
\qquad
\omega_{bb}^2:=\E h_{bb}^{\circ}(U_b,V_b)^2,
\qquad
M_b:=\|h_b\|_\infty.
\end{equation}
Then \(M_b\le 2L_*\) and \(\|h_{bb}^{\circ}\|_\infty\le 4L_*\).

\begin{lemma}[Hoeffding decompositions of block averages]
\label{lem:hanova-two-sample-decomp}
\label{lem:hanova-same-block-decomp}
For \(b\ne c\),
\begin{align}
\label{eq:hanova-two-sample-decomp}
\overline\Delta_{bc}-\mu_{bc}
&=
\frac1{n_b}\sum_{i=1}^{n_b}h_{b\to c}(U_{b,i})
+
\frac1{n_c}\sum_{j=1}^{n_c}h_{c\to b}(V_{c,j})
+
\frac1{n_bn_c}
\sum_{i=1}^{n_b}\sum_{j=1}^{n_c}
h_{bc}^{\circ}(U_{b,i},V_{c,j}).
\end{align}
The three terms on the right are pairwise orthogonal in \(L^2\), and
\begin{equation}
\label{eq:hanova-two-sample-variance}
\Var(\overline\Delta_{bc})
=
\frac{\sigma_{b\to c}^2}{n_b}
+
\frac{\sigma_{c\to b}^2}{n_c}
+
\frac{\omega_{bc}^2}{n_bn_c}.
\end{equation}
If \(b=c\) and \(n_b\ge2\), define the ordered off-diagonal average
\[
\overline\Delta_{bb}^{\mathrm{off}}
:=
\frac{1}{n_b(n_b-1)}
\sum_{\substack{i,j\in I_b\\ i\ne j}}
\left(K(X_i,X_j)-N_{bb}\right).
\]
Then
\begin{equation}
\label{eq:hanova-same-block-decomp}
\overline\Delta_{bb}^{\mathrm{off}}-\mu_{bb}
=
\frac{2}{n_b}\sum_{i\in I_b}h_b(U_{b,i})
+
\frac{1}{n_b(n_b-1)}
\sum_{\substack{i,j\in I_b\\ i\ne j}}
h_{bb}^{\circ}(U_{b,i},U_{b,j}),
\end{equation}
the two terms are orthogonal in \(L^2\), and
\begin{equation}
\label{eq:hanova-same-block-variance}
\Var(\overline\Delta_{bb}^{\mathrm{off}})
=
\frac{4\sigma_b^2}{n_b}
+
\frac{2\omega_{bb}^2}{n_b(n_b-1)}.
\end{equation}
Finally, the diagonal-including same-block average satisfies
\begin{equation}
\label{eq:hanova-diagonal-reduction}
|\overline\Delta_{bb}|
\le
\frac{\delta_*}{n_b}
+
|\overline\Delta_{bb}^{\mathrm{off}}|
\qquad (n_b\ge2),
\end{equation}
with the convention \(|\overline\Delta_{bb}|\le\delta_*\) when \(n_b=1\).
\end{lemma}

\begin{proof}
For \(b\ne c\), substitute
\[
H_{bc}(u,v)
=
\mu_{bc}+h_{b\to c}(u)+h_{c\to b}(v)+h_{bc}^{\circ}(u,v)
\]
into the rectangular block average.  The two projection terms collapse to averages over their own blocks, while the canonical part remains a two-sample average, giving \eqref{eq:hanova-two-sample-decomp}.

Orthogonality follows from independence of the two blocks, centering of the projections, and the conditional mean-zero identities for \(h_{bc}^{\circ}\).  Expanding the variance of the canonical average, all mixed terms vanish unless the two ordered pairs coincide; this leaves \(\omega_{bc}^2/(n_bn_c)\).  The projection variances are \(\sigma_{b\to c}^2/n_b\) and \(\sigma_{c\to b}^2/n_c\), giving \eqref{eq:hanova-two-sample-variance}.

The same-block decomposition is identical, except that the ordered off-diagonal statistic is a one-sample \(U\)-statistic.  Each \(h_b(U_{b,i})\) appears once as a first argument and once as a second argument, producing the coefficient \(2/n_b\).  The canonical variance receives contributions only from identical or reversed ordered pairs; by symmetry of \(K\), both contribute \(\omega_{bb}^2\).  This gives the factor \(2\omega_{bb}^2/[n_b(n_b-1)]\).  The diagonal reduction follows from \(|\Delta_{ii}|\le\delta_*\).
\end{proof}

We use Bernstein's inequality for the projection averages.  For a centered random variable \(Z\) with \(|Z|\le M\) and variance \(\sigma^2\),
\begin{equation}
\label{eq:hanova-bernstein-template}
\Pp\left(
\left|\frac1N\sum_{i=1}^N Z_i\right|
>
\sqrt{\frac{2\sigma^2u}{N}}
+
\frac{2Mu}{3N}
\right)
\le
2e^{-u}.
\end{equation}
For the canonical pieces, we keep the exact tail quantile.  For \(b\ne c\), let
\begin{equation}
\label{eq:hanova-two-sample-canonical-quantile}
Q^0_{bc}(u;n_b,n_c)
:=
\inf\left\{q\ge0:
\Pp\left(
\left|
\frac1{n_bn_c}
\sum_{i=1}^{n_b}\sum_{j=1}^{n_c}
h_{bc}^{\circ}(U_{b,i},V_{c,j})
\right|>q
\right)
\le e^{-u}
\right\}.
\end{equation}
For \(b=c\), define \(Q^0_{bb}(u;n_b)\) analogously for the ordered off-diagonal canonical average.  Any available canonical-chaos inequality may be substituted for these exact quantiles; for example, one may use bounds of the form
\[
Q^0_{bc}(u;n_b,n_c)
\lesssim
\frac{\omega_{bc}}{\sqrt{n_bn_c}}(u+1)
+
\frac{L_*}{n_bn_c}(u+1)^2,
\]
and similarly
\[
Q^0_{bb}(u;n_b)
\lesssim
\frac{\omega_{bb}}{n_b}(u+1)
+
\frac{L_*}{n_b^2}(u+1)^2.
\]

Define the block envelopes as follows.  For \(b\ne c\),
\begin{align}
\label{eq:hanova-offdiag-block-envelope}
A_{bc}^{\hanova}(u)
:=
&|\mu_{bc}|
+
\sqrt{\frac{2\sigma_{b\to c}^2u}{n_b}}
+
\frac{2M_{b\to c}u}{3n_b}
+
\sqrt{\frac{2\sigma_{c\to b}^2u}{n_c}}
+
\frac{2M_{c\to b}u}{3n_c}
+
Q^0_{bc}(u;n_b,n_c).
\end{align}
For \(b=c\) and \(n_b\ge2\),
\begin{align}
\label{eq:hanova-diag-block-envelope}
A_{bb}^{\hanova}(u)
:=
\frac{\delta_*}{n_b}
+
|\mu_{bb}|
+
2\sqrt{\frac{2\sigma_b^2u}{n_b}}
+
\frac{4M_bu}{3n_b}
+
Q^0_{bb}(u;n_b),
\end{align}
and set \(A_{bb}^{\hanova}(u):=\delta_*\) when \(n_b=1\).

For the test vector, define
\begin{equation}
\label{eq:hanova-test-variance-def}
\sigma_{b,x}^2:=\Var(G_{b,x}(U_b)),
\qquad
M_{b,x}:=\|G_{b,x}-\nu_{b,x}\|_\infty,
\end{equation}
so \(M_{b,x}\le2L_*\), and set
\begin{equation}
\label{eq:hanova-test-envelope}
T_{bx}^{\hanova}(u)
:=
|\nu_{b,x}|
+
\sqrt{\frac{2\sigma_{b,x}^2u}{n_b}}
+
\frac{2M_{b,x}u}{3n_b}.
\end{equation}

\begin{theorem}[Hoeffding--ANOVA block and test envelopes]
\label{thm:hanova-hanova-block}
\label{thm:hanova-hanova-test}
Conditional on the template assignment, with probability at least
\[
1-(5r^2+2r)e^{-u},
\]
the following inequalities hold simultaneously:
\begin{equation}
\label{eq:hanova-hanova-block-conclusion}
|\overline\Delta_{bc}|
\le
A_{bc}^{\hanova}(u)
\qquad
(1\le b,c\le r),
\qquad \mbox{and} \qquad
|\overline\zeta_b|
\le
T_{bx}^{\hanova}(u)
\qquad
(1\le b\le r).
\end{equation}
\end{theorem}

\begin{proof}
Fix \(b\ne c\).  The decomposition in \eqref{eq:hanova-two-sample-decomp} gives
\[
\overline\Delta_{bc}-\mu_{bc}=P_b+P_c+C_{bc},
\]
where \(P_b\) and \(P_c\) are the two projection averages and \(C_{bc}\) is the canonical average.  Bernstein's inequality \eqref{eq:hanova-bernstein-template} controls \(P_b\) and \(P_c\) with failure probabilities \(2e^{-u}\) each, while the definition of \(Q^0_{bc}\) controls \(C_{bc}\) with failure probability \(e^{-u}\).  The triangle inequality gives the envelope \(A_{bc}^{\hanova}(u)\), with failure probability at most \(5e^{-u}\) for this ordered pair.

For \(b=c\), the same argument uses \eqref{eq:hanova-same-block-decomp}, Bernstein for the one-sample projection average, the quantile \(Q^0_{bb}\) for the canonical part, and the diagonal reduction \eqref{eq:hanova-diagonal-reduction}.  The resulting bound is \(A_{bb}^{\hanova}(u)\), with failure probability at most \(3e^{-u}\), and hence at most \(5e^{-u}\).

For the test block,
\[
\overline\zeta_b-\nu_{b,x}
=
\frac1{n_b}\sum_{i\in I_b}
\bigl(G_{b,x}(U_{b,i})-\nu_{b,x}\bigr).
\]
The summands are independent, centered, bounded by \(M_{b,x}\), and have variance \(\sigma_{b,x}^2\).  Bernstein's inequality gives the bound \(T_{bx}^{\hanova}(u)\) with failure probability at most \(2e^{-u}\) for the fixed block.  A union bound over all ordered template pairs and all test blocks proves the theorem.
\end{proof}

The weighted Hoeffding--ANOVA quantities used in the transfer budget are
\begin{align}
\label{eq:hanova-weighted-proj-def}
D_{\hanova}(u)
&:=
\sum_{b=1}^r\sum_{c=1}^r
\widehat p_b\widehat p_c
|r_{ab}|\,|\beta_c|\,A_{bc}^{\hanova}(u),
 \qquad
Z_{\hanova}(u)
 :=
\sum_{b=1}^r
\widehat p_b|\beta_b|\,T_{bx}^{\hanova}(u).
\end{align}

\begin{corollary}[Hoeffding--ANOVA transfer certificate]
\label{cor:hanova-hanova-transfer}
Conditional on the template assignment, with probability at least
\[
1-(5r^2+2r)e^{-u},
\]
\begin{equation}
\label{eq:hanova-hanova-transfer}
|\widehat f_\lambda(x)-S_\lambda^{\id}(x)|
\le
B_{\hanova}(\lambda;x,u),
\end{equation}
where \(B_{\hanova}\) is the budget defined in Equation \ref{eq:five-certificates-compact}.
\end{corollary}

\begin{proof}
On the event of Theorem \ref{thm:hanova-hanova-block},
\[
|\overline\Delta_{bc}|\le A_{bc}^{\hanova}(u),
\qquad
|\overline\zeta_b|\le T_{bx}^{\hanova}(u).
\]
Substitution into the weighted block quantities from
\eqref{eq:blockproj-weighted-block-score-envelopes} gives
\[
\Delta_{\mathrm w}\le D_{\hanova}(u),
\qquad
Z_{\mathrm w}\le Z_{\hanova}(u).
\]
The projection-augmented deterministic transfer certificate
\eqref{eq:degcert-projection-augmented-certificate} therefore yields
\[
|\widehat f_\lambda(x)-S_\lambda^{\id}(x)|
\le
\operatorname{Ord}_\lambda(A_\Delta,D_{\hanova}(u),Z_{\hanova}(u))
=
B_{\hanova}(\lambda;x,u).
\]
\end{proof}

\subsubsection{Literal-aware corrected target}
\label{app:literal}
\label{sec:bfdec-literal-correction}

The previous certificates exploit directional, block-constant, and Hoeffding--ANOVA structure.  One avoidable loss remains.  A global literal-collision event treats a wildcard hit on any literal in
\[
R=\bigcup_{a=1}^r L(z_a)
\]
as corrupting every template pair.  For a pair of templates \((b,c)\), however, only literals in \(L_c\) matter for the wildcard image of template \(b\), and only literals in \(L_b\) matter for the wildcard image of template \(c\).  The mean effect of these pair-specific collisions can be absorbed into a corrected block target.  After this correction, the remaining block fluctuations are centered; the deterministic \(q\)-scale appears only when we compare the corrected target back to the fully fresh target.

For an admissible substitution \(s\) of template \(z_b\), write
\[
W_b(s):=s(W_b).
\]
For independent substitutions \(U_b\sim\mu_{\mathrm{sub},b}\) and \(V_c\sim\mu_{\mathrm{sub},c}\), define the pair-specific collision event
\begin{align}
\label{eq:bfdec-pair-bad-event}
\mathcal B_{bc}^{\lit}(U_b,V_c)
:=
&\{W_b(U_b)\cap L_c\neq\varnothing\}
\cup
\{W_c(V_c)\cap L_b\neq\varnothing\}
\nonumber\\
&\cup
\{W_b(U_b)\cap W_c(V_c)\neq\varnothing\}.
\end{align}
For a globally fresh test string \(x=\subop(z_a,s_x)\), let \(T_x=s_x(W_a)\) and define
\begin{equation}
\label{eq:bfdec-test-bad-event}
\mathcal T_{bx}^{\lit}(U_b)
:=
\{W_b(U_b)\cap L_a\neq\varnothing\}
\cup
\{W_b(U_b)\cap T_x\neq\varnothing\}.
\end{equation}

\begin{lemma}[Pair-specific agreement and corrected target]
\label{lem:bfdec-corrected-target}
On \((\mathcal B_{bc}^{\lit})^c\),
\[
K\bigl(\subop(z_b,U_b),\subop(z_c,V_c)\bigr)=N_{bc},
\]
and
\[
\Pp(\mathcal B_{bc}^{\lit})\le q_{bc}.
\]
Similarly, on \((\mathcal T_{bx}^{\lit})^c\),
\[
K\bigl(x,\subop(z_b,U_b)\bigr)=N_{ab},
\]
and
\[
\Pp(\mathcal T_{bx}^{\lit})\le q_{bx}.
\]

Let \(\phi:\cX\to\cH\) be the RKHS feature map and define
\[
\psi_b^\circ
:=
\E_{U_b}\phi(\subop(z_b,U_b)).
\]
Set
\begin{equation}
\label{eq:bfdec-corrected-N-def}
N_{bc}^{\circ}
:=
\langle \psi_b^\circ,\psi_c^\circ\rangle_{\cH}
=
\E_{U_b,V_c}K\bigl(\subop(z_b,U_b),\subop(z_c,V_c)\bigr),
\end{equation}
and
\begin{equation}
\label{eq:bfdec-corrected-test-row-def}
m_{x,b}^{\circ}
:=
\langle \phi(x),\psi_b^\circ\rangle_{\cH}
=
\E_{U_b}K\bigl(x,\subop(z_b,U_b)\bigr).
\end{equation}
Then \(\bN^\circ=(N_{bc}^{\circ})_{b,c}\) is symmetric positive semidefinite.  Moreover, with
\[
\mu_{bc}^{\lit}:=N_{bc}^{\circ}-N_{bc},
\qquad
\nu_{b,x}^{\lit}:=m_{x,b}^{\circ}-N_{ab},
\]
one has
\begin{equation}
\label{eq:bfdec-mu-nu-envelope}
|\mu_{bc}^{\lit}|\le L_*q_{bc},
\qquad
|\nu_{b,x}^{\lit}|\le L_*q_{bx}.
\end{equation}
Consequently, for \(Q_q=(q_{bc})_{b,c}\),
\begin{equation}
\label{eq:bfdec-N-bias-op-envelope}
\|\bN^\circ-\bN\|_{\op}
\le
L_*\|Q_q\|_{\op}
\le
L_*\sqrt{\|Q_q\|_1\|Q_q\|_\infty}.
\end{equation}
\end{lemma}

\begin{proof}
If \(\mathcal B_{bc}^{\lit}\) does not occur, the wildcard image of each substitution avoids the literals of both templates, and the two wildcard images are disjoint.  A vocabulary permutation fixing \(L_b\cup L_c\) can therefore send the realized pair to any fresh admissible pair for \((z_b,z_c)\).  Token symmetry gives the equality with \(N_{bc}\).  The probability bound follows from the union bound:
\[
\Pp(\mathcal B_{bc}^{\lit})
\le
\ell_{b\to c}+\ell_{c\to b}+\chi_{bc}
\le q_{bc}.
\]
The test statement is the same argument with \(x\) fixed.  On \((\mathcal T_{bx}^{\lit})^c\), the training wildcard image avoids both \(L_a\) and \(T_x\), so token symmetry gives the fresh value \(N_{ab}\), and the union bound gives
\[
\Pp(\mathcal T_{bx}^{\lit})\le q_{bx}.
\]

The matrix \(\bN^\circ\) is a Gram matrix of the mean features \(\psi_b^\circ\), hence it is symmetric positive semidefinite.  The bias bounds follow from the agreement just proved.  Indeed, the discrepancy
\[
K\bigl(\subop(z_b,U_b),\subop(z_c,V_c)\bigr)-N_{bc}
\]
vanishes off \(\mathcal B_{bc}^{\lit}\) and has absolute value at most \(L_*\), so
\[
|\mu_{bc}^{\lit}|
\le
L_*\Pp(\mathcal B_{bc}^{\lit})
\le
L_*q_{bc}.
\]
The test bias is identical.  Finally, the entrywise inequality
\[
|N_{bc}^{\circ}-N_{bc}|
\le
L_*(Q_q)_{bc}
\]
implies
\[
|u^\top(\bN^\circ-\bN)v|
\le
L_* |u|^\top Q_q |v|
\]
for unit vectors \(u,v\), which gives the operator bound.
\end{proof}

Using the membership matrix, define
\begin{equation}
\label{eq:bfdec-corrected-sample-objects}
M_n^\circ:=P\bN^\circ P^\top,
\qquad
m_x^\circ(i):=m_{x,\tau(i)}^\circ,
\end{equation}
and let \(c_{M_n^\circ}\) be the corresponding dual minimizer.  The corrected sample-level score is
\begin{equation}
\label{eq:bfdec-corrected-score-def}
S_\lambda^\circ(x)
:=
\frac{1}{\lambda n}
(m_x^\circ)^\top(\mathbf y\odot c_{M_n^\circ}).
\end{equation}

\begin{lemma}[Corrected-to-fresh bias]
\label{lem:bfdec-corrected-vs-fresh-bias}
For every globally fresh test string \(x\),
\begin{equation}
\label{eq:bfdec-corrected-vs-fresh-bias}
|S_\lambda^\circ(x)-S_\lambda^{\id}(x)|
\le
B_{\mathrm{bias}}(\lambda;x),
\end{equation}
where \(B_{\mathrm{bias}}\) is the explicit bias envelope from Equation \ref{eq:main-bias-envelope}.
\end{lemma}

\begin{proof}
Apply the deterministic dual-stability bound \(Theorem \ref{thm:detperturb}\) to the two ideal sample-level problems
\[
(G,M,k,m)=(M_n^\circ,M_n,m_x^\circ,m_a).
\]
Both \(M_n^\circ\) and \(M_n\) are positive semidefinite, and all test-vector coordinates are bounded by \(K_*\).  Since
\[
M_n^\circ-M_n
=
P(\bN^\circ-\bN)P^\top
\]
and \(P^\top P=nQ\),
\[
\|M_n^\circ-M_n\|_{\op}
=
n\left\|Q^{1/2}(\bN^\circ-\bN)Q^{1/2}\right\|_{\op}.
\]
Similarly,
\[
\|m_x^\circ-m_a\|_2^2
=
n\sum_b \widehat p_b(\nu_{b,x}^{\lit})^2.
\]
Thus
\[
|S_\lambda^\circ(x)-S_\lambda^{\id}(x)|
\le
\frac{K_*}{4\lambda^2}
\left\|Q^{1/2}(\bN^\circ-\bN)Q^{1/2}\right\|_{\op}
+
\frac1\lambda
\left(\sum_b\widehat p_b(\nu_{b,x}^{\lit})^2\right)^{1/2}.
\]
Using \eqref{eq:bfdec-mu-nu-envelope} gives exactly the envelope
\(B_{\mathrm{bias}}(\lambda;x)\).
\end{proof}

We now center the stochastic comparison around the corrected target.  Define
\begin{equation}
\label{eq:bfdec-centered-HG-def}
H_{bc}^\circ(u,v)
:=
K\bigl(\subop(z_b,u),\subop(z_c,v)\bigr)-N_{bc}^\circ,
\qquad
G_{b,x}^\circ(u)
:=
K\bigl(x,\subop(z_b,u)\bigr)-m_{x,b}^\circ.
\end{equation}
Then
\[
\E H_{bc}^\circ(U_b,V_c)=0,
\qquad
\E G_{b,x}^\circ(U_b)=0.
\]
Moreover,
\begin{equation}
\label{eq:bfdec-centered-bounds}
|H_{bc}^\circ|\le L_*,
\qquad
|G_{b,x}^\circ|\le L_*,
\end{equation}
and
\begin{equation}
\label{eq:bfdec-centered-variance-envelope}
\Var(H_{bc}^\circ(U_b,V_c))\le L_*^2q_{bc},
\qquad
\Var(G_{b,x}^\circ(U_b))\le L_*^2q_{bx}.
\end{equation}

For \(b\ne c\), take the Hoeffding decomposition
\[
H_{bc}^\circ(u,v)
=
h_{b\to c}^\circ(u)
+
h_{c\to b}^\circ(v)
+
h_{bc}^{\circ\circ}(u,v),
\]
where the first two terms are the conditional expectations and the last term is canonical.  Let
\[
(\sigma_{b\to c}^{\circ})^2,
\quad
(\sigma_{c\to b}^{\circ})^2,
\quad
(\omega_{bc}^{\circ})^2
\]
be the corresponding component variances, and let
\[
M_{b\to c}^\circ:=\|h_{b\to c}^\circ\|_\infty,
\qquad
M_{c\to b}^\circ:=\|h_{c\to b}^\circ\|_\infty.
\]
Then
\[
(\sigma_{b\to c}^{\circ})^2,
(\sigma_{c\to b}^{\circ})^2,
(\omega_{bc}^{\circ})^2
\le
L_*^2q_{bc},
\qquad
M_{b\to c}^\circ,M_{c\to b}^\circ\le L_*,
\qquad
\|h_{bc}^{\circ\circ}\|_\infty\le 3L_*.
\]
The same one-sample Hoeffding decomposition is used when \(b=c\), with notation
\[
h_b^\circ,\quad
h_{bb}^{\circ\circ},\quad
(\sigma_b^\circ)^2,\quad
(\omega_{bb}^\circ)^2,\quad
M_b^\circ.
\]
For the test vector, write
\[
(\sigma_{b,x}^{\circ})^2:=\Var(G_{b,x}^\circ(U_b)),
\qquad
M_{b,x}^\circ:=\|G_{b,x}^\circ\|_\infty.
\]

As in the Hoeffding--ANOVA section, keep the canonical quantiles explicit.  Let
\(Q_{bc}^{00}(u;n_b,n_c)\) denote the \(e^{-u}\)-tail envelope for the rectangular canonical average when \(b\ne c\), and let \(Q_{bb}^{00}(u;n_b)\) denote the analogous ordered off-diagonal same-block canonical envelope.  Define
\begin{align}
\label{eq:bfdec-block-envelope-off}
A_{bc}^{\bfdec}(u)
:=
&
\sqrt{\frac{2(\sigma_{b\to c}^{\circ})^2u}{n_b}}
+
\frac{2M_{b\to c}^{\circ}u}{3n_b}
+
\sqrt{\frac{2(\sigma_{c\to b}^{\circ})^2u}{n_c}}
+
\frac{2M_{c\to b}^{\circ}u}{3n_c}
+
Q_{bc}^{00}(u;n_b,n_c),
\qquad b\ne c,
\end{align}
and, for \(n_b\ge2\),
\begin{align}
\label{eq:bfdec-block-envelope-same}
A_{bb}^{\bfdec}(u)
:=
&
\frac{\delta_*^\circ}{n_b}
+
2\sqrt{\frac{2(\sigma_b^\circ)^2u}{n_b}}
+
\frac{4M_b^\circ u}{3n_b}
+
Q_{bb}^{00}(u;n_b),
\end{align}
with \(A_{bb}^{\bfdec}(u):=\delta_*^\circ\) when \(n_b=1\).  The test envelope is
\begin{equation}
\label{eq:bfdec-test-envelope}
T_{bx}^{\bfdec}(u)
:=
\sqrt{\frac{2(\sigma_{b,x}^{\circ})^2u}{n_b}}
+
\frac{2M_{b,x}^{\circ}u}{3n_b}.
\end{equation}

\begin{theorem}[Centered corrected-target envelopes]
\label{thm:bfdec-centered-block-bound}
\label{thm:bfdec-centered-test-bound}
Conditional on the template assignment, with probability at least
\[
1-(5r^2+2r)e^{-u},
\]
the centered corrected block and test averages satisfy
\begin{equation}
\label{eq:bfdec-centered-block-bound}
\left|
\frac1{n_bn_c}
\sum_{i\in I_b}\sum_{j\in I_c}
\bigl(K(X_i,X_j)-N_{bc}^{\circ}\bigr)
\right|
\le
A_{bc}^{\bfdec}(u)
\qquad(1\le b,c\le r),
\end{equation}
and
\begin{equation}
\label{eq:bfdec-centered-test-bound}
\left|
\frac1{n_b}
\sum_{i\in I_b}
\bigl(K(x,X_i)-m_{x,b}^{\circ}\bigr)
\right|
\le
T_{bx}^{\bfdec}(u)
\qquad(1\le b\le r).
\end{equation}
\end{theorem}

\begin{proof}
For \(b\ne c\), average the Hoeffding decomposition of \(H_{bc}^\circ\) over \(I_b\times I_c\).  The two projection averages are controlled by Bernstein's inequality, and the canonical average is controlled by the definition of \(Q_{bc}^{00}\).  The failure probability for a fixed ordered pair is at most \(5e^{-u}\).

For \(b=c\), use the ordered off-diagonal one-sample Hoeffding decomposition and add the diagonal contribution \(\delta_*^\circ/n_b\).  This gives the bound \(A_{bb}^{\bfdec}(u)\).  A union bound over all ordered template pairs gives the block failure probability \(5r^2e^{-u}\).

For the test average, the summands \(G_{b,x}^\circ(U_{b,i})\) are independent, centered, bounded by \(M_{b,x}^\circ\), and have variance \((\sigma_{b,x}^\circ)^2\).  Bernstein's inequality gives the displayed test bound with failure probability \(2e^{-u}\) per block.  A union bound over \(b\) completes the proof.
\end{proof}

Let \(c_{M_n^\circ}=P\theta^\circ\) be the dual minimizer for the corrected block problem and set
\[
  \beta^\circ:=y\odot\theta^\circ,
  \qquad
  m_x^{\rm blk}:=(m_{x,1}^\circ,\ldots,m_{x,r}^\circ)^\top .
\]
Define the corrected finite-dimensional Hessian and test representer by
\[
  A_{\lambda,\widehat p}^{\circ}
  :=
  \diag\bigl(\varphi''(\theta_1^\circ),\ldots,\varphi''(\theta_r^\circ)\bigr)
  +{1\over\lambda}Y_r\bN^\circ QY_r,
  \qquad
  r_x^\circ
  :=
  Y_r(A_{\lambda,\widehat p}^{\circ})^{-1}Y_rm_x^{\rm blk}.
\]
The weighted centered envelopes used by the bias--fluctuation certificate are
\begin{align}
\label{eq:bfdec-weighted-lit-cent-def}
D_{\bfdec}(u)
&:=
\sum_{b,c}\widehat p_b\widehat p_c
|r_{xb}^{\circ}|\,|\beta_c^{\circ}|\,A_{bc}^{\bfdec}(u),\\
Z_{\bfdec}(u)
&:=
\sum_b\widehat p_b|\beta_b^{\circ}|\,T_{bx}^{\bfdec}(u).
\end{align}
Let \(R_{\degact}^\circ\) be the corrected degree-action remainder.

\begin{corollary}[Bias--fluctuation transfer]
\label{cor:bfdec-corrected-transfer}
Conditional on the template assignment, with probability at least
\[
1-(5r^2+2r)e^{-u},
\]
\begin{equation}
\label{eq:bfdec-fresh-target-transfer}
|\widehat f_\lambda(x)-S_\lambda^{\id}(x)|
\le
\frac{1}{\lambda^2}D_{\bfdec}(u)
+
\frac1\lambda Z_{\bfdec}(u)
+
R_{\degact}^{\circ}(\lambda;x)
+
B_{\mathrm{bias}}(\lambda;x).
\end{equation}
Equivalently,
\[
|\widehat f_\lambda(x)-S_\lambda^{\id}(x)|
\le
B_{\bfdec}(\lambda;x,u),
\]
with \(B_{\bfdec}\) as defined in Equation \ref{eq:five-certificates-compact}.
\end{corollary}

\begin{proof}
Apply the corrected-target version of the projection-augmented deterministic transfer certificate to the comparison between \((\widehat K_n,k_x)\) and \((M_n^\circ,m_x^\circ)\).  On the event of Theorem \ref{thm:bfdec-centered-block-bound}, the centered block and test averages are bounded by \(A_{bc}^{\bfdec}(u)\) and \(T_{bx}^{\bfdec}(u)\), so the weighted projection terms are bounded by \(D_{\bfdec}(u)\) and \(Z_{\bfdec}(u)\).  This yields the corrected-target score bound with the remainder \(R_{\degact}^\circ\).  Adding and subtracting \(S_\lambda^\circ(x)\), then using Lemma \ref{lem:bfdec-corrected-vs-fresh-bias}, gives the stated fresh-target bound.
\end{proof}

 \subsubsection{Combining the transfer certificates}
\label{sec:mincert-final-paper-theorem}

We now combine the five transfer bounds into the single certificate used in the main theorem.  Each path gives an upper bound for the same error
\[
\bigl|\widehat f_\lambda(x)-S_\lambda^{\id}(x)\bigr|.
\]
Thus, on the event where all five paths are valid, the error is bounded by the minimum of the five realized budgets.  This is the reason for defining
\[
B_\lambda^\sharp(x,\delta)
=
\min\{
B_{\cspec},
B_{\degact},
B_{\bdens},
B_{\hanova},
B_{\bfdec}
\}(\lambda;x,u_\delta),
\qquad
u_\delta:=\log\frac{11r^2+5r}{\delta}.
\]

\begin{lemma}
\label{lem:conditional-five-certificate-transfer}
Condition on a realized template assignment for which \(E_{\rm rep}\) holds.  Then
\[
\Pp\!\left(
\left|\widehat f_\lambda(x)-S_\lambda^{\id}(x)\right|
\le
B^\sharp_\lambda(x,\delta)
\,\middle|\,
\tau(1),\ldots,\tau(n)
\right)
\ge
1-\delta .
\]
\end{lemma}

\begin{proof}
On \(E_{\rm rep}\), every block is nonempty, so all block averages and all five realized budgets are well defined.

The deterministic reductions show that the score error may be controlled through the action term \(A_\Delta\), the block averages of \(\Delta\), and the block averages of \(\zeta\).  In particular, the degree-action path gives \(B_{\degact}\), and the curvature-spectral specialization gives \(B_{\cspec}\), with the convention that \(B_{\cspec}=+\infty\) when \(\gamma_{\cspec}\ge1\).

The other three paths insert high-probability block-average estimates into the same deterministic transfer inequality.  The KL block-density event of Theorem \ref{thm:klbdens-kl-block-envelope} validates \(B_{\bdens}\), and also supplies the block and test envelopes used in \(B_{\degact}\) and \(B_{\cspec}\).  Its conditional failure probability is at most
\[
(r^2+r)e^{-u}.
\]
The Hoeffding--ANOVA block and test event of Theorem \ref{thm:hanova-hanova-block} validates \(B_{\hanova}\), with conditional failure probability at most
\[
(5r^2+2r)e^{-u}.
\]
Finally, the literal-corrected centered block and test event of
Theorem \ref{thm:bfdec-centered-block-bound}, together with the corrected-to-fresh bias bound in Lemma \ref{lem:bfdec-corrected-vs-fresh-bias}, validates \(B_{\bfdec}\).  Its conditional failure probability is at most
\[
(5r^2+2r)e^{-u}.
\]

Therefore, with conditional probability at least
\[
1-(11r^2+5r)e^{-u},
\]
all five realized budgets are valid upper bounds for the same scalar error.  On this intersection,
\[
\left|\widehat f_\lambda(x)-S_\lambda^{\id}(x)\right|
\le
\min\{
B_{\cspec},
B_{\degact},
B_{\bdens},
B_{\hanova},
B_{\bfdec}
\}(\lambda;x,u).
\]
Taking \(u=u_\delta\) gives
\[
(11r^2+5r)e^{-u_\delta}=\delta,
\]
and hence
\[
\left|\widehat f_\lambda(x)-S_\lambda^{\id}(x)\right|
\le
B^\sharp_\lambda(x,\delta)
\]
with conditional probability at least \(1-\delta\).
\end{proof}

\subsection{Proof of Main Theorem 1}
\begin{proof}[Proof of Theorem \ref{thm:main}]
Fix the globally fresh test string \(x\in\cX_a^{\fr}\), ridge parameter \(\lambda>0\), and confidence level \(\delta\in(0,1)\).  Let
\[
  \mathcal F_\delta
  :=
  \left\{
  \left|\widehat f_\lambda(x)-S_\lambda^{\id}(x)\right|
  > B_\lambda^\sharp(x,\delta)
  \right\}
\]
be the transfer-failure event.  On every realized template assignment satisfying \(E_{\rm rep}\), Lemma \ref{lem:conditional-five-certificate-transfer} gives
\[
  \Pp(\mathcal F_\delta\mid \tau(1),\ldots,\tau(n))\le\delta.
\]
On \(E_{\rm rep}^c\), the convention \(B_\lambda^\sharp(x,\delta)=+\infty\) makes \(\mathcal F_\delta\) empty.  Hence
\[
  \Pp(\mathcal F_\delta)
  =
  \E\left[
    \1_{E_{\rm rep}}
    \Pp(\mathcal F_\delta\mid \tau(1),\ldots,\tau(n))
  \right]
  \le \delta.
\]
Thus the first display in the theorem actually holds with probability at least \(1-\delta\).  The stated bound
\[
  1-r e^{-np_{\min}/8}-\delta
\]
follows by weakening this estimate; this keeps the supplement aligned with the stated theorem and the later edge--wedge bookkeeping.

For the margin statement, fix \(\varepsilon\ge0\) and \(s>0\).  Apply the ideal-margin part of Theorem \ref{thm:idealmargin} with \(\Gamma=\varepsilon+s\).  There is \(\lambda_{\varepsilon+s}^{\id}>0\), depending only on \(\varepsilon+s\), \(p_{\min}\), \(\bN\), and \(y\), such that on \(E_{\rm rep}\), whenever \(0<\lambda<\lambda_{\varepsilon+s}^{\id}\),
\[
  y_aS_\lambda^{\id}(x)>\varepsilon+s.
\]
Let
\[
  \mathcal G
  :=
  \{y_a\widehat f_\lambda(x)\le\varepsilon\}
  \cap
  \{B_\lambda^\sharp(x,\delta)\le s\}
\]
be the failure of the desired union event.  On \(E_{\rm rep}^c\), \(B_\lambda^\sharp(x,\delta)=+\infty\), so \(\mathcal G\) cannot occur.  On \(E_{\rm rep}\cap\mathcal F_\delta^c\), the preceding ideal-margin bound gives
\[
  y_a\widehat f_\lambda(x)
  \ge
  y_aS_\lambda^{\id}(x)
  -
  \left|\widehat f_\lambda(x)-S_\lambda^{\id}(x)\right|
  >
  (\varepsilon+s)-B_\lambda^\sharp(x,\delta).
\]
Therefore, if \(B_\lambda^\sharp(x,\delta)\le s\), then \(y_a\widehat f_\lambda(x)>\varepsilon\).  Hence
\[
  \mathcal G\subseteq E_{\rm rep}\cap\mathcal F_\delta,
  \qquad
  \Pp(\mathcal G)\le\delta.
\]
This again gives the displayed theorem bound after weakening to
\(1-r e^{-np_{\min}/8}-\delta\), and the equivalent implication follows from the definition of \(\mathcal G\).
\end{proof}

\section{Proof machinery for Main Theorem 2}
\label{sec:proof-machinery-main2}
\label{sec:detailed-edge-wedge-machinery}

This section replaces the realized action terms in \(B_\lambda^\sharp\) by the deterministic edge--wedge envelopes from Section \ref{sec:edge-wedge-envelope-def}.  We use edge and wedge concentration to control the row-action term \(d_2\), and a separate maximum-degree estimate to control the operator and curvature terms.  After those action inputs are dominated, monotonicity of the score functionals gives the comparison between the realized certificates and their edge--wedge counterparts.

\begin{proposition}[Edge--wedge domination]
\label{lem:ew-edge-wedge}
\label{lem:ew-degree-corrected}
\label{lem:ew-domination}
Condition on the realized template assignment and assume \(E_{\rm rep}\).  Then, with conditional
probability at least
\[
1-(r^2+r^3+1)e^{-v},
\]
the following inequalities hold simultaneously:
\begin{align}
\label{eq:ew-d2-bound}
\frac{d_2}{n^{3/2}}
&\le
\left(
\frac{Q_E(v)}{n}
+
S_\wedge(v)
\right)^{1/2},\\
\label{eq:ew-op-curv-bound}
\frac{\|\Delta\|_{\op}}{n}
&\le A_{\op}(v),
\qquad
\gamma_{\cspec}\le \bar\gamma_{\cspec}(v),\\
\label{eq:ew-action-domination}
A_\Delta&\le A_0(v),
\qquad
A_\Delta^\circ\le A_{\corr}(v).
\end{align}
Consequently,
\[
B_{\cspec}\le B_{\cspec}^{\ew},\quad
B_{\degact}\le B_{\degact}^{\ew},\quad
B_{\bdens}\le B_{\bdens}^{\ew},\quad
B_{\hanova}\le B_{\hanova}^{\ew},\quad
B_{\bfdec}\le B_{\bfdec}^{\ew},
\]
with all quantities evaluated at the arguments of Theorem \ref{thm:five-certificate-edge-wedge-envelope-final}.  In particular,
\[
B^\sharp_\lambda(x,\delta)
\le
B_\lambda^{\ew}(x,\delta,v).
\]
\end{proposition}

\begin{proof}
Define the realized ordered edge and wedge densities
\[
\widehat q_E
:=
\frac1{n^2}\sum_{i\ne j}\Ccol_{ij},
\qquad
\widehat s_E
:=
\frac1{n^3}
\sum_{i=1}^n
\sum_{\substack{j,k\ne i\\ j\ne k}}
\Ccol_{ij}\Ccol_{ik}.
\]
Since \(d_i=\sum_{j\ne i}\Ccol_{ij}\),
\[
d_2^2
=
\sum_i d_i^2
=
\sum_i\sum_{j\ne i}\Ccol_{ij}
+
\sum_i
\sum_{\substack{j,k\ne i\\j\ne k}}
\Ccol_{ij}\Ccol_{ik}
=
n^2\widehat q_E+n^3\widehat s_E.
\]
Thus it is enough to prove
\[
\widehat q_E\le Q_E(v),
\qquad
\widehat s_E\le S_\wedge(v).
\]

For the edge term, decompose the ordered edge set by template colors.  Cross-color rectangles are
partitioned into matchings, and same-color pairs are handled by the round-robin coloring used in
the block-density argument.  Within each matching the collision indicators are independent, since no
training substitution appears twice.  Applying the partition-KL bound color by color gives the
edge-density envelope with effective sample size \(N_{bc}^{\eff}\).  Summing over color pairs with
the weights appearing in \(Q_E(v)\) gives
\[
\widehat q_E\le Q_E(v)
\]
with failure probability at most \(r^2e^{-v}\).

For the wedge term, fix colors \((b,c,d)\) and consider triples
\[
(i,j,k)\in\mathcal T_{b;cd}.
\]
Set
\[
B_{ijk}:=\Ccol_{ij}\Ccol_{ik}.
\]
Conditional on the anchor substitution \(S_i=s\), the two leaf substitutions are independent, so
\[
\Pp(B_{ijk}=1\mid S_i=s,\tau)
=
\pi_{b\to c}(s)\pi_{b\to d}(s).
\]
After averaging over \(S_i\), the mean is \(\alpha_{b;cd}\le \bar\alpha_{b;cd}\).  View
\(\mathcal T_{b;cd}\) as a \(3\)-uniform hypergraph.  A triple intersects at most
\(3\Delta_{b;cd}^{\wedge}\) other triples, hence a greedy coloring partitions the triples into at
most
\[
3\Delta_{b;cd}^{\wedge}+1
\]
vertex-disjoint classes.  Within each class the variables \(B_{ijk}\) are independent.  The
partition-KL bound with effective sample size \(N_{b;cd}^{\wedge}\) gives the defining wedge
envelope.  A union bound over the \(r^3\) color triples yields
\[
\widehat s_E\le S_\wedge(v)
\]
with failure probability at most \(r^3e^{-v}\).  Combining the edge and wedge estimates proves
\eqref{eq:ew-d2-bound} with failure probability at most \((r^2+r^3)e^{-v}\).

It remains to control the operator and curvature terms.  Fix a vertex \(i\) of color \(b\), and
condition on its substitution.  For \(j\ne i\), the indicators \(\Ccol_{ij}\) are independent over
\(j\), and their conditional means are bounded by \(\kappa_{b,\tau(j)}\).  Hence
\[
\E(d_i\mid S_i,\tau)
\le
\sum_c n_c\kappa_{bc}
=
n\widehat\kappa_b
\le
n\widehat\kappa_*.
\]
Bernstein's inequality with exponent \(v+\log n\), followed by a union bound over \(i\), gives
\[
\frac{d_{\max}}{n}\le \overline d(v)
\]
with failure probability at most \(e^{-v}\).

By the collision-support bound,
\[
|\Delta_{ij}|\le L_*\Ccol_{ij}\quad (i\ne j),
\qquad
|\Delta_{ii}|\le \delta_*.
\]
Since \(\Delta\) is symmetric,
\[
\|\Delta\|_{\op}
\le
\max_i\sum_j|\Delta_{ij}|
\le
\delta_*+L_*d_{\max}.
\]
Dividing by \(n\) gives
\[
\frac{\|\Delta\|_{\op}}{n}\le A_{\op}(v).
\]
Also \(C_M\succeq (4/n)I\), and therefore
\[
\gamma_{\cspec}
\le
\frac{n}{4}
\frac{\|\Delta\|_{\op}}{\lambda n^2}
=
\frac{1}{4\lambda}\frac{\|\Delta\|_{\op}}{n}
\le
\bar\gamma_{\cspec}(v).
\]
This proves \eqref{eq:ew-op-curv-bound}.

Now combine the two action estimates.  The edge--wedge bound gives
\[
\frac{\delta_*}{n}
+
L_*\frac{d_2}{n^{3/2}}
\le
A_{\ew}(v),
\]
while the operator bound gives
\[
\frac{\|\Delta\|_{\op}}{n}\le A_{\op}(v).
\]
By the definition of \(A_\Delta\),
\[
A_\Delta\le A_0(v).
\]

For the corrected action, let
\[
\Delta^\circ:=\widehat K_n-M_n^\circ,
\qquad
v^\circ:=Yc_{M_n^\circ}.
\]
Since \(|v_j^\circ|\le1\), and since for \(i\ne j\)
\[
|\Delta_{ij}^\circ|
\le
|K(X_i,X_j)-N_{\tau(i)\tau(j)}|
+
|N^\circ_{\tau(i)\tau(j)}-N_{\tau(i)\tau(j)}|
\le
L_*\Ccol_{ij}
+
|\mu^{\lit}_{\tau(i)\tau(j)}|,
\]
while \(|\Delta_{ii}^\circ|\le\delta_*^\circ\), each row satisfies
\[
|(\Delta^\circ v^\circ)_i|
\le
\delta_*^\circ
+
L_*d_i
+
r_i^\mu.
\]
Taking \(\ell_2\) norms and using \eqref{eq:ew-d2-bound} gives
\[
A_\Delta^\circ
\le
\frac{\delta_*^\circ}{n}
+
L_*
\left(
\frac{Q_E(v)}{n}+S_\wedge(v)
\right)^{1/2}
+
\frac{R_{2,\mu}}{n^{3/2}}
=
A_{\corr}(v).
\]
This proves \eqref{eq:ew-action-domination}.

Finally, the five certificate dominations follow from monotonicity of the score functionals.  The
degree-action, block-density, and Hoeffding--ANOVA budgets have the form
\[
\operatorname{Ord}_\lambda(A_\Delta,D,Z),
\]
so \(A_\Delta\le A_0(v)\) gives the corresponding edge--wedge bounds.  The curvature-spectral
budget is monotone in both \(A_\Delta\) and \(\gamma_{\cspec}\); if
\(\bar\gamma_{\cspec}(v)\ge1\), its edge--wedge version is \(+\infty\), and otherwise the
monotonicity is immediate.  For the bias--fluctuation path, the only changed action input is
\(A_\Delta^\circ\), which is bounded by \(A_{\corr}(v)\); the weighted block, test, and bias terms
are unchanged.  Taking the minimum over the five pointwise dominations gives
\[
B_\lambda^\sharp(x,\delta)
\le
B_\lambda^{\ew}(x,\delta,v).
\]
\end{proof}

\subsection{Proof of Main Theorem 2}

\begin{proof}[Proof of Theorem \ref{thm:five-certificate-edge-wedge-envelope-final}]
Condition on the realized template assignment and assume \(E_{\rm rep}\).  By
Lemma \ref{lem:ew-edge-wedge}, with conditional probability at least
\[
1-(r^2+r^3+1)e^{-v},
\]
the realized certificate is dominated by its edge--wedge envelope:
\[
B_\lambda^\sharp(x,\delta)
\le
B_\lambda^{\ew}(x,\delta,v).
\]
This proves the conditional statement of the theorem, including the corresponding domination for
each of the five certificate terms.

Now choose
\[
v=v_\eta:=\log\frac{r^2+r^3+1}{\eta}.
\]
The conditional failure probability is then at most \(\eta\).  Integrating over the template
assignment and using
\[
\Pp(E_{\rm rep})\ge 1-r e^{-np_{\min}/8}
\]
gives
\[
\Pp\!\left(
E_{\rm rep}
\cap
\left\{
B_\lambda^\sharp(x,\delta)
\le
B_\lambda(x,\delta,v_\eta)
\right\}
\right)
\ge
1-r e^{-np_{\min}/8}-\eta.
\]

Finally, intersect this event with the conditional transfer event from
Lemma \ref{lem:conditional-five-certificate-transfer}.  On \(E_{\rm rep}\), that event has conditional
failure probability at most \(\delta\).  Hence, except on an event of probability at most
\[
r e^{-np_{\min}/8}+\delta+\eta,
\]
we have
\[
|\widehat f_\lambda(x)-S_\lambda^{\id}(x)|
\le
B_\lambda^\sharp(x,\delta)
\le
B_\lambda(x,\delta,v_\eta).
\]
Therefore
\[
\Pp\!\left(
\left|\widehat f_\lambda(x)-S_\lambda^{\id}(x)\right|
\le
B_\lambda(x,\delta,v_\eta)
\right)
\ge
1-r e^{-np_{\min}/8}-\delta-\eta,
\]
which is the asserted score-error consequence.
\end{proof}

 \section{Interface with the frozen-feature transformer kernel}
\label{app:transformer-kernel-interface}

We consider a depth-one transformer architecture, without skip connections or layer normalization. The input is a one-hot matrix
$
X=
 (
e_{x_1}^\top,  
\cdots
e_{x_k}^\top
)^\top
\in \R^{k\times m},
$
where the last token is a distinguished CLS token shared by all inputs. The parameters are:
 (i) $H$ attention heads with matrices
$
W_{K,h},W_{Q,h},W_{V,h},W_{O,h}\in\R^{d_{\mathrm{head}}\times d_{\mathrm{emb}}},
\ h\in[H];
$
(ii)  an embedding matrix
$
W_E\in\R^{m\times d_{\mathrm{emb}}};
$
(iii) positional embeddings
$
P\in\R^{k\times d_{\mathrm{emb}}};
$
(iv) an MLP block with matrices
$
W_A,W_B\in\R^{d_{\mathrm{mlp}}\times d_{\mathrm{emb}}};
$
(v)  and a trainable output head
$
w_U\in\R^{d_{\mathrm{emb}}}.
$

The network output is
\begin{align*}
f_{\mathrm{rf}}(X;\theta)
&=
\frac{1}{\sqrt{d_{\mathrm{emb}}}}\,w_U^\top z_2(X),
\tag{Unembedding}
\end{align*}
where
\begin{align*}
z_2(X)
&=
\frac{1}{\sqrt{d_{\mathrm{mlp}}}}W_B^\top
\sigma\!\left(\frac{1}{\sqrt{d_{\mathrm{emb}}}}W_A z_1(X)\right)
\in\R^{d_{\mathrm{emb}}},
\tag{MLP layer}\\
z_1(X)
&=
\frac{1}{\sqrt H}\sum_{h\in[H]}A_h(X)^\top e_k
\in\R^{d_{\mathrm{emb}}},
\tag{Attention layer output at CLS token}\\
A_h(X)
&=
\operatorname{softmax}\!\left(
\frac{\beta Z_0(X)W_{K,h}^\top W_{Q,h} Z_0(X)^\top}{d_{\mathrm{emb}}\sqrt{d_{\mathrm{head}}}}
\right)
Z_0(X)\frac{W_{V,h}^\top W_{O,h}}{\sqrt{d_{\mathrm{head}}d_{\mathrm{emb}}}}
\in\R^{k\times d_{\mathrm{emb}}},
\tag{Attention heads}\\
Z_0(X)
&=
XW_E+\gamma P
\in\R^{k\times d_{\mathrm{emb}}}.
\tag{Embedding layer}
\end{align*}
Here $\beta,\gamma\ge 0$ are the inverse-temperature and positional-embedding parameters. Only the attention output at the final CLS token is used.

We initialize every entry of
$
W_E,\;P,\;W_{K,h},\;W_{Q,h},\;W_{V,h},\;W_{O,h},\;W_A,\;W_B
$
independently from $N(0,1)$, initialize $w_U=0$, and freeze all parameters except $w_U$.
With the frozen feature map
$
\Psi_\theta(X):=\frac{1}{\sqrt{d_{\mathrm{emb}}}}\,z_2(X)\in\R^{d_{\mathrm{emb}}},
$
the predictor is linear in the trainable output head $w_U$
\[
f_{\mathrm{rf}}(X;\theta)=\langle w_U,\Psi_\theta(X)\rangle.
\]
The corresponding empirical random-features kernel is
\begin{equation}
\label{eq:Hfaithful_hatKtrans}
\hat K_{\mathrm{trans}}(X,Y)
=
\langle \Psi_\theta(X),\Psi_\theta(Y)\rangle
=
\frac{z_2(X)^\top z_2(Y)}{d_{\mathrm{emb}}}.
\end{equation}

This we already know.
Fix two one-hot inputs $X,Y\in\R^{k\times m}$ with a common CLS token in position $k$.
Define
$
C_{XY}^{(d_{\mathrm{emb}})}
:=
\frac{1}{d_{\mathrm{emb}}}Z_0(X)Z_0(Y)^\top.
$
Then, as $d_{\mathrm{emb}}\to\infty$,
$
C_{XY}^{(d_{\mathrm{emb}})}
\to
XY^\top+\gamma^2 I_k
$
entrywise almost surely and hence in probability. Moreover,   for a single head, let 
$$
m(X)
:=
\frac{
e_k^\top Z_0(X)W_K^\top W_Q Z_0(X)^\top
}{d_{\mathrm{emb}}\sqrt{d_{\mathrm{head}}}}
\in\R^k,
\ 
m(Y)
:=
\frac{
e_k^\top Z_0(Y)W_K^\top W_Q Z_0(Y)^\top
}{d_{\mathrm{emb}}\sqrt{d_{\mathrm{head}}}}
\in\R^k.
$$
Then, as $d_{\mathrm{emb}},d_{\mathrm{head}}\to\infty$,
\begin{equation}
\label{eq:Hfaithful_single_head_clt_limit}
\begin{bmatrix}
m(X)\\
m(Y)
\end{bmatrix}
\Rightarrow
N\!\left(
0,
(1+\gamma^2)
\begin{bmatrix}
XX^\top+\gamma^2 I_k & XY^\top+\gamma^2 I_k\\
YX^\top+\gamma^2 I_k & YY^\top+\gamma^2 I_k
\end{bmatrix}
\right).
\end{equation}

\begin{proposition}[Probabilistic limit of the attention kernel]
\label{prop:Hfaithful_attention_kernel_limit}
For fixed inputs $X,Y$, define
\[
\hat K_{\mathrm{attn}}(X,Y):=\frac{z_1(X)^\top z_1(Y)}{d_{\mathrm{emb}}}.
\]
Then, as $H,d_{\mathrm{emb}},d_{\mathrm{head}}\to\infty$,
\begin{equation}
\label{eq:Hfaithful_attention_prob_limit}
\hat K_{\mathrm{attn}}(X,Y)
\to
K_{\mathrm{attn}}(X,Y)
:=
\E_{m(X),m(Y)}
\Bigl[
\operatorname{softmax}(\beta m(X))^\top
(XY^\top+\gamma^2 I_k)
\operatorname{softmax}(\beta m(Y))
\Bigr]
\end{equation}
in probability.
\end{proposition}

\begin{proof}For a single head, write
\[
s_h(X):=S_h(X)^\top e_k\in\R^k,
\qquad
u_h(X):=\frac{A_h(X)^\top e_k}{\sqrt{d_{\mathrm{emb}}}}\in\R^{d_{\mathrm{emb}}},
\]
and similarly for $Y$. Then
\[
\hat K_{\mathrm{attn}}(X,Y)
=
\frac1H\sum_{h,\ell=1}^H u_h(X)^\top u_\ell(Y).
\]
We split this into the diagonal average
\[
D_H:=\frac1H\sum_{h=1}^H u_h(X)^\top u_h(Y)
\]
and the off-diagonal average
\[
R_H:=\frac1H\sum_{h\neq \ell} u_h(X)^\top u_\ell(Y).
\]

\medskip 
Let
\[
\mathcal G_H:=\sigma\!\bigl(W_E,P,\{W_{K,h},W_{Q,h}\}_{h=1}^H\bigr).
\]
Conditional on $\mathcal G_H$, the pairs $\{u_h(X),u_h(Y)\}_{h=1}^H$ are independent across heads, because the only remaining randomness comes from the head-specific matrices $W_{V,h}$ and $W_{O,h}$.

For each head define
\[
v_h(X):=Z_0(X)^\top s_h(X)\in\R^{d_{\mathrm{emb}}},
\qquad
a_h(X):=\frac{\|v_h(X)\|_2^2}{d_{\mathrm{emb}}}
=
s_h(X)^\top C_{XX}^{(d_{\mathrm{emb}})} s_h(X),
\]
and similarly define $v_h(Y)$ and $a_h(Y)$.
Because
\[
u_h(X)=\frac{1}{d_{\mathrm{emb}}\sqrt{d_{\mathrm{head}}}}\,W_{O,h}^\top W_{V,h} v_h(X),
\]
conditioning first on $(\mathcal G_H,W_{V,h})$ and then on $\mathcal G_H$ gives
\begin{align*}
\E\!\left[u_h(X)u_h(X)^\top \mid \mathcal G_H\right]
&=
\frac{1}{d_{\mathrm{emb}}^2 d_{\mathrm{head}}}
\,\E\!\left[W_{O,h}^\top W_{V,h} v_h(X) v_h(X)^\top W_{V,h}^\top W_{O,h}\,\middle|\,\mathcal G_H\right]\\
&=
\frac{1}{d_{\mathrm{emb}}^2 d_{\mathrm{head}}}
\,\E\!\left[\|W_{V,h} v_h(X)\|_2^2\,I_{d_{\mathrm{emb}}}\,\middle|\,\mathcal G_H\right]\\
&=
\frac{\|v_h(X)\|_2^2}{d_{\mathrm{emb}}^2}\,I_{d_{\mathrm{emb}}}
=
\frac{a_h(X)}{d_{\mathrm{emb}}}\,I_{d_{\mathrm{emb}}}.
\end{align*}
The same formula holds with $Y$ in place of $X$.

Now fix distinct heads $h\neq \ell$. Since $u_h(X)$ and $u_\ell(Y)$ are conditionally independent and conditionally centered given $\mathcal G_H$,
\[
\E\!\left[(u_h(X)^\top u_\ell(Y))^2 \,\middle|\, \mathcal G_H\right]
=
\operatorname{tr}\!\Bigl(
\E[u_h(X)u_h(X)^\top\mid\mathcal G_H]\,
\E[u_\ell(Y)u_\ell(Y)^\top\mid\mathcal G_H]
\Bigr)
=
\frac{a_h(X)a_\ell(Y)}{d_{\mathrm{emb}}}.
\]
If $(h,\ell)\neq(h',\ell')$, then the mixed covariance between
\[
u_h(X)^\top u_\ell(Y)
\qquad\text{and}\qquad
u_{h'}(X)^\top u_{\ell'}(Y)
\]
vanishes conditional on $\mathcal G_H$: disjoint ordered pairs are conditionally independent, and if exactly one head index is shared then, after conditioning on the shared head, the remaining factor is conditionally centered and independent. Therefore
\begin{align*}
\E[R_H^2\mid \mathcal G_H]
&=
\frac1{H^2}\sum_{h\neq \ell}
\E\!\left[(u_h(X)^\top u_\ell(Y))^2 \,\middle|\, \mathcal G_H\right]\\
&=
\frac1{d_{\mathrm{emb}}H^2}\sum_{h\neq \ell} a_h(X)a_\ell(Y)\\
&\le
\frac1{d_{\mathrm{emb}}}
\left(\frac1H\sum_{h=1}^H a_h(X)\right)
\left(\frac1H\sum_{\ell=1}^H a_\ell(Y)\right).
\end{align*}
Because each $s_h(X)$ and $s_h(Y)$ is a probability vector,
\[
a_h(X)\le \|C_{XX}^{(d_{\mathrm{emb}})}\|_{\mathrm{op}},
\qquad
a_h(Y)\le \|C_{YY}^{(d_{\mathrm{emb}})}\|_{\mathrm{op}}.
\]
Since $k$ is fixed,  
\[
\|C_{XX}^{(d_{\mathrm{emb}})}\|_{\mathrm{op}}=O_{\Pp}(1),
\qquad
\|C_{YY}^{(d_{\mathrm{emb}})}\|_{\mathrm{op}}=O_{\Pp}(1),
\]
for example by the bound $\|A\|_{\mathrm{op}}\le k\max_{a,b}|A_{ab}|$.
Hence
\[
\E[R_H^2\mid \mathcal G_H]=O_{\Pp}(d_{\mathrm{emb}}^{-1}),
\]
and therefore
\[
R_H\to 0
\]
in probability by conditional Chebyshev.

\medskip 
Let
\[
\mathcal G_0:=\sigma(W_E,P).
\]
Conditional on $\mathcal G_0$, the random variables
\[
T_h:=u_h(X)^\top u_h(Y)
\]
are i.i.d.\ across heads, because the head-specific parameter collections
\[
(W_{K,h},W_{Q,h},W_{V,h},W_{O,h})
\]
are i.i.d.
Write
\[
\mu_{d_{\mathrm{emb}},d_{\mathrm{head}}}(X,Y)
:=
\E[T_1\mid \mathcal G_0].
\]
We claim that
\begin{equation}
\label{eq:attn-diagonal-conditional-lln-v13}
D_H-\mu_{d_{\mathrm{emb}},d_{\mathrm{head}}}(X,Y)\to 0
\qquad\text{in probability.}
\end{equation}

We first compute the conditional mean more explicitly. Conditioning on $(\mathcal G_0,W_{K,h},W_{Q,h},W_{V,h})$ makes the vectors
\[
g_h(X):=\frac{W_{V,h}v_h(X)}{\sqrt{d_{\mathrm{head}}}},
\qquad
g_h(Y):=\frac{W_{V,h}v_h(Y)}{\sqrt{d_{\mathrm{head}}}}
\]
 , and then
\[
u_h(X)=\frac{W_{O,h}^\top g_h(X)}{d_{\mathrm{emb}}},
\qquad
u_h(Y)=\frac{W_{O,h}^\top g_h(Y)}{d_{\mathrm{emb}}}.
\]
Conditional on $(\mathcal G_0,W_{K,h},W_{Q,h},W_{V,h})$, the pair $(u_h(X),u_h(Y))$ is therefore centered Gaussian with coordinatewise covariance
\[
\E[u_h(X)_p u_h(Y)_q\mid \mathcal G_0,W_{K,h},W_{Q,h},W_{V,h}]
=
\frac{\langle g_h(X),g_h(Y)\rangle}{d_{\mathrm{emb}}^2}\,\delta_{pq}.
\]
Consequently,
\[
\E[T_h\mid \mathcal G_0,W_{K,h},W_{Q,h},W_{V,h}]
=
\frac{\langle g_h(X),g_h(Y)\rangle}{d_{\mathrm{emb}}}
=
s_h(X)^\top C_{XY}^{(d_{\mathrm{emb}})} s_h(Y),
\]
because
\[
\E[\langle g_h(X),g_h(Y)\rangle \mid \mathcal G_0,W_{K,h},W_{Q,h}]
=
\frac{v_h(X)^\top v_h(Y)}{d_{\mathrm{head}}}
=
d_{\mathrm{emb}}\,s_h(X)^\top C_{XY}^{(d_{\mathrm{emb}})} s_h(Y).
\]
Taking expectation over $W_{V,h}$ yields the stated formula for $\mu_{d_{\mathrm{emb}},d_{\mathrm{head}}}(X,Y)$.

We next bound the conditional second moment of $T_h$. Conditional on $(\mathcal G_0,W_{K,h},W_{Q,h},W_{V,h})$, Isserlis' formula for the jointly Gaussian vector $(u_h(X),u_h(Y))$ gives
\begin{align*}
\E[T_h^2\mid \mathcal G_0,W_{K,h},W_{Q,h},W_{V,h}]
&=
\sum_{p,q=1}^{d_{\mathrm{emb}}}
\E[u_h(X)_p u_h(Y)_p u_h(X)_q u_h(Y)_q\mid \cdots]\\
&=
\frac{\|g_h(X)\|_2^2\|g_h(Y)\|_2^2}{d_{\mathrm{emb}}^3}
+
\frac{d_{\mathrm{emb}}+1}{d_{\mathrm{emb}}^3}\,\langle g_h(X),g_h(Y)\rangle^2.
\end{align*}
Since $\langle g_h(X),g_h(Y)\rangle^2\le \|g_h(X)\|_2^2\|g_h(Y)\|_2^2$, we obtain
\[
\E[T_h^2\mid \mathcal G_0,W_{K,h},W_{Q,h},W_{V,h}]
\le
\frac{2}{d_{\mathrm{emb}}^2}\,\|g_h(X)\|_2^2\|g_h(Y)\|_2^2.
\]
Now
\[
\|g_h(X)\|_2^2=\frac{\|W_{V,h}v_h(X)\|_2^2}{d_{\mathrm{head}}},
\qquad
\|g_h(Y)\|_2^2=\frac{\|W_{V,h}v_h(Y)\|_2^2}{d_{\mathrm{head}}}.
\]
For   vectors $u,v$, a row-wise Gaussian calculation gives
\[
\E[\|W_{V,h}u\|_2^2\|W_{V,h}v\|_2^2]
\le
(d_{\mathrm{head}}^2+2d_{\mathrm{head}})\|u\|_2^2\|v\|_2^2.
\]
Applying this with $u=v_h(X)$ and $v=v_h(Y)$ yields
\[
\E[T_h^2\mid \mathcal G_0,W_{K,h},W_{Q,h}]
\le
2\Bigl(1+\frac{2}{d_{\mathrm{head}}}\Bigr)
a_h(X)a_h(Y).
\]
Since $a_h(X)\le \|C_{XX}^{(d_{\mathrm{emb}})}\|_{\mathrm{op}}$ and $a_h(Y)\le \|C_{YY}^{(d_{\mathrm{emb}})}\|_{\mathrm{op}}$, we conclude that
\[
\E[T_h^2\mid \mathcal G_0]
\le
6\,\|C_{XX}^{(d_{\mathrm{emb}})}\|_{\mathrm{op}}\,
\|C_{YY}^{(d_{\mathrm{emb}})}\|_{\mathrm{op}}
=
O_{\Pp}(1).
\]
Therefore
\[
\operatorname{Var}(D_H\mid \mathcal G_0)
=
\frac1H\operatorname{Var}(T_1\mid \mathcal G_0)
=
O_{\Pp}(H^{-1}),
\]
and conditional Chebyshev proves \eqref{eq:attn-diagonal-conditional-lln-v13}.

\medskip 
Define, for positive-semidefinite matrices $C_{XX},C_{XY},C_{YY}\in\R^{k\times k}$,
\[
\Phi(C_{XX},C_{XY},C_{YY})
:=
\E_{g_X,g_Y}\!\Bigl[
\operatorname{softmax}(\beta g_X)^\top
C_{XY}
\operatorname{softmax}(\beta g_Y)
\Bigr],
\]
where $(g_X,g_Y)$ is centered Gaussian with covariance
\[
(1+\gamma^2)
\begin{bmatrix}
C_{XX} & C_{XY}\\
C_{XY}^\top & C_{YY}
\end{bmatrix}.
\]
Because softmax vectors lie in the simplex,
\[
\left|
\operatorname{softmax}(\beta g_X)^\top
C_{XY}
\operatorname{softmax}(\beta g_Y)
\right|
\le
\|C_{XY}\|_\infty.
\]
The entries of $C_{XY}^{(d_{\mathrm{emb}})}$ are empirical averages of i.i.d.\ finite-variance random variables, so
\[
M_{d_{\mathrm{emb}}}:=\|C_{XY}^{(d_{\mathrm{emb}})}\|_\infty
\]
is uniformly $L^2$-bounded and hence uniformly integrable.

Conditional on $\mathcal G_0$,   the law of the single-head score pair $(m_1(X),m_1(Y))$ is multivariate Gaussian up to the covariance triple $(C_{XX}^{(d_{\mathrm{emb}})},C_{XY}^{(d_{\mathrm{emb}})},C_{YY}^{(d_{\mathrm{emb}})})$. We now make the convergence of expectations explicit.
Fix $\varepsilon>0$. Uniform integrability of $(M_{d_{\mathrm{emb}}})$ gives an $M<\infty$ such that
\[
\sup_{d_{\mathrm{emb}}}\E\!\left[M_{d_{\mathrm{emb}}}\mathbf 1_{\{M_{d_{\mathrm{emb}}}>M\}}\right]<\varepsilon.
\]
On the event $\{M_{d_{\mathrm{emb}}}\le M\}$, the integrand
\[
(u,v)\longmapsto
\operatorname{softmax}(\beta u)^\top C_{XY}^{(d_{\mathrm{emb}})} \operatorname{softmax}(\beta v)
\]
is bounded by $M$ and continuous in $(u,v)$. Hence the conditional weak convergence   implies
\[
\E\!\left[
\operatorname{softmax}(\beta m_1(X))^\top
C_{XY}^{(d_{\mathrm{emb}})}
\operatorname{softmax}(\beta m_1(Y))
\mathbf 1_{\{M_{d_{\mathrm{emb}}}\le M\}}
\ \middle|\ \mathcal G_0
\right]
-
\Phi\!\left(
C_{XX}^{(d_{\mathrm{emb}})},
C_{XY}^{(d_{\mathrm{emb}})},
C_{YY}^{(d_{\mathrm{emb}})}
\right)
\mathbf 1_{\{M_{d_{\mathrm{emb}}}\le M\}}
\to 0
\]
in probability.
The contribution of the complement is bounded in absolute value by
\[
2\,\E\!\left[M_{d_{\mathrm{emb}}}\mathbf 1_{\{M_{d_{\mathrm{emb}}}>M\}}\ \middle|\ \mathcal G_0\right],
\]
whose expectation is at most $2\varepsilon$. Since $\varepsilon$ is arbitrary, we conclude that
\[
\mu_{d_{\mathrm{emb}},d_{\mathrm{head}}}(X,Y)
-
\Phi\!\left(
C_{XX}^{(d_{\mathrm{emb}})},
C_{XY}^{(d_{\mathrm{emb}})},
C_{YY}^{(d_{\mathrm{emb}})}
\right)
\to 0
\]
in probability.

Observe that
\[
C_{XX}^{(d_{\mathrm{emb}})}\to XX^\top+\gamma^2 I_k,
\qquad
C_{XY}^{(d_{\mathrm{emb}})}\to XY^\top+\gamma^2 I_k,
\qquad
C_{YY}^{(d_{\mathrm{emb}})}\to YY^\top+\gamma^2 I_k
\]
in probability.
The map $\Phi$ is continuous on the cone of positive-semidefinite covariance triples: if $(C_{XX}^{(n)},C_{XY}^{(n)},C_{YY}^{(n)})\to(C_{XX},C_{XY},C_{YY})$, then the corresponding Gaussian pairs converge in distribution, and the integrands are uniformly integrable because they are bounded by $\|C_{XY}^{(n)}\|_\infty$. Therefore
\[
\Phi\!\left(
C_{XX}^{(d_{\mathrm{emb}})},
C_{XY}^{(d_{\mathrm{emb}})},
C_{YY}^{(d_{\mathrm{emb}})}
\right)
\to
K_{\mathrm{attn}}(X,Y)
\]
in probability, and hence
\[
\mu_{d_{\mathrm{emb}},d_{\mathrm{head}}}(X,Y)\to K_{\mathrm{attn}}(X,Y)
\]
in probability.

\medskip 
We have 
\[
D_H-\mu_{d_{\mathrm{emb}},d_{\mathrm{head}}}(X,Y)\to 0
\]
in probability, and   the limit of the conditional mean:
\[
\mu_{d_{\mathrm{emb}},d_{\mathrm{head}}}(X,Y)\to K_{\mathrm{attn}}(X,Y)
\]
in probability. Therefore
\[
D_H\to K_{\mathrm{attn}}(X,Y)
\]
in probability.
We showed that $R_H\to 0$ in probability. Since
\[
\hat K_{\mathrm{attn}}(X,Y)=D_H+R_H,
\]
we conclude that
\[
\hat K_{\mathrm{attn}}(X,Y)\to K_{\mathrm{attn}}(X,Y)
\]
in probability, which is exactly \eqref{eq:Hfaithful_attention_prob_limit}.
\end{proof}

\begin{theorem}[Probabilistic frozen-feature transformer kernel limit]
\label{thm:Hfaithful_kernel_prob_limit}
Let $\mathcal S=\{X^{(1)},\dots,X^{(N)}\}$ be any fixed finite set of inputs.
Then, as $H,d_{\mathrm{emb}},d_{\mathrm{head}},d_{\mathrm{mlp}}\to\infty$,
\begin{equation}
\label{eq:Hfaithful_kernel_prob_limit}
\max_{1\le a,b\le N}
\left|
\hat K_{\mathrm{trans}}(X^{(a)},X^{(b)})
-
K_{\mathrm{trans}}(X^{(a)},X^{(b)})
\right|
\to 0
\end{equation}
in probability, where
\begin{align}
\label{eq:Hfaithful_transformer_rf_kernel}
K_{\mathrm{trans}}(X,Y)
&=
\mathbb E_{u,v}[\sigma(u)\sigma(v)]
\quad\text{for}\quad
u,v\sim N\!\left(
0,
\begin{bmatrix}
K_{\mathrm{attn}}(X,X) & K_{\mathrm{attn}}(X,Y)\\
K_{\mathrm{attn}}(Y,X) & K_{\mathrm{attn}}(Y,Y)
\end{bmatrix}
\right),
\\
K_{\mathrm{attn}}(X,Y)
&=
\mathbb E_{m(X),m(Y)}
\Bigl[
\operatorname{softmax}(\beta m(X))^\top
(XY^\top+\gamma^2 I_k)
\operatorname{softmax}(\beta m(Y))
\Bigr],
\label{eq:Hfaithful_Kattn}
\end{align}
\end{theorem}

\begin{proof}Fix one pair $(X,Y)\in\mathcal S\times\mathcal S$.
By Proposition~\ref{prop:Hfaithful_attention_kernel_limit},
\[
\hat K_{\mathrm{attn}}(X,Y)\to K_{\mathrm{attn}}(X,Y)
\]
in probability, and the same holds for $(X,X)$ and $(Y,Y)$.

We now pass through the MLP layer.
Let
\[
a_X:=\sigma\!\left(\frac1{\sqrt{d_{\mathrm{emb}}}}W_A z_1(X)\right)\in\R^{d_{\mathrm{mlp}}},
\qquad
a_Y:=\sigma\!\left(\frac1{\sqrt{d_{\mathrm{emb}}}}W_A z_1(Y)\right)\in\R^{d_{\mathrm{mlp}}}.
\]
Then
\[
\hat K_{\mathrm{trans}}(X,Y)
=
\frac{1}{d_{\mathrm{emb}}d_{\mathrm{mlp}}}
a_X^\top W_BW_B^\top a_Y.
\]

\medskip 
Let $b_1,\dots,b_{d_{\mathrm{emb}}}\in\R^{d_{\mathrm{mlp}}}$ denote the columns of $W_B$.
They are i.i.d.\ $N(0,I_{d_{\mathrm{mlp}}})$, and
\[
a_X^\top W_BW_B^\top a_Y
=
\sum_{s=1}^{d_{\mathrm{emb}}} (b_s^\top a_X)(b_s^\top a_Y).
\]
Conditional on $(W_E,P,W_K,W_Q,W_V,W_O,W_A)$, the vectors $a_X,a_Y$ are deterministic, so the summands
\[
Z_s:=(b_s^\top a_X)(b_s^\top a_Y)
\]
are i.i.d.
For a centered Gaussian vector $b\sim N(0,I_{d_{\mathrm{mlp}}})$,
\[
\E[(b^\top a_X)(b^\top a_Y)\mid a_X,a_Y]=a_X^\top a_Y
\]
and Isserlis' formula gives
\[
\operatorname{Var}\!\bigl((b^\top a_X)(b^\top a_Y)\mid a_X,a_Y\bigr)
=
\|a_X\|_2^2\|a_Y\|_2^2+(a_X^\top a_Y)^2
\le
2\|a_X\|_2^2\|a_Y\|_2^2.
\]
Therefore
\[
\E[\hat K_{\mathrm{trans}}(X,Y)\mid a_X,a_Y]
=
\frac{a_X^\top a_Y}{d_{\mathrm{mlp}}}
\]
and
\[
\operatorname{Var}(\hat K_{\mathrm{trans}}(X,Y)\mid a_X,a_Y)
=
\frac{\operatorname{Var}(Z_1\mid a_X,a_Y)}{d_{\mathrm{emb}}d_{\mathrm{mlp}}^2}
\le
\frac{2\,\|a_X\|_2^2\|a_Y\|_2^2}{d_{\mathrm{emb}}d_{\mathrm{mlp}}^2}.
\]
It's easy to see that 
\[
\frac{\|a_X\|_2^2}{d_{\mathrm{mlp}}}=O_{\Pp}(1),
\qquad
\frac{\|a_Y\|_2^2}{d_{\mathrm{mlp}}}=O_{\Pp}(1).
\]
Hence the conditional variance is $O_{\Pp}(d_{\mathrm{emb}}^{-1})$, and conditional Chebyshev yields
\[
\hat K_{\mathrm{trans}}(X,Y)-\frac{a_X^\top a_Y}{d_{\mathrm{mlp}}}\to 0
\]
in probability.

\medskip 
Conditional on the attention outputs $z_1(X),z_1(Y)$, the rows of $W_A$ are i.i.d.
For each row $r$,
\[
u_r:=\frac{\langle W_{A,r\bullet},z_1(X)\rangle}{\sqrt{d_{\mathrm{emb}}}},
\qquad
v_r:=\frac{\langle W_{A,r\bullet},z_1(Y)\rangle}{\sqrt{d_{\mathrm{emb}}}}
\]
form an i.i.d.\ sequence of centered Gaussian pairs with conditional covariance
\[
\begin{bmatrix}
\hat K_{\mathrm{attn}}(X,X) & \hat K_{\mathrm{attn}}(X,Y)\\
\hat K_{\mathrm{attn}}(Y,X) & \hat K_{\mathrm{attn}}(Y,Y)
\end{bmatrix}.
\]
Therefore,
\[
\frac{a_X^\top a_Y}{d_{\mathrm{mlp}}}
=
\frac1{d_{\mathrm{mlp}}}\sum_{r=1}^{d_{\mathrm{mlp}}}\sigma(u_r)\sigma(v_r).
\]
Because $\sigma$ has linear growth, there exists $C<\infty$ such that
\[
|\sigma(s)\sigma(t)|
\le
C(1+s^2+t^2)
\qquad (s,t\in\R).
\]
Conditional on the attention outputs, the Gaussian pair $(u_r,v_r)$ has second moments
\[
\E[u_r^2\mid z_1(X),z_1(Y)] = \hat K_{\mathrm{attn}}(X,X),
\qquad
\E[v_r^2\mid z_1(X),z_1(Y)] = \hat K_{\mathrm{attn}}(Y,Y),
\]
so
\[
\E[\sigma(u_r)^2\sigma(v_r)^2\mid z_1(X),z_1(Y)]
\le
C\bigl(1+\hat K_{\mathrm{attn}}(X,X)+\hat K_{\mathrm{attn}}(Y,Y)\bigr)^2.
\]
Since Proposition~\ref{prop:Hfaithful_attention_kernel_limit} implies that the diagonal attention kernels are tight, the right-hand side is $O_{\Pp}(1)$. Consequently
\[
\operatorname{Var}\!\left(\frac{a_X^\top a_Y}{d_{\mathrm{mlp}}}\ \middle|\ z_1(X),z_1(Y)\right)
=
\frac1{d_{\mathrm{mlp}}}\operatorname{Var}(\sigma(u_1)\sigma(v_1)\mid z_1(X),z_1(Y))
=
O_{\Pp}(d_{\mathrm{mlp}}^{-1}).
\]
Define
\[
\psi(c_{xx},c_{xy},c_{yy})
:=
\E_{u,v}\bigl[\sigma(u)\sigma(v)\bigr],
\]
where $(u,v)$ is centered Gaussian with covariance matrix
$
\begin{bmatrix}
c_{xx} & c_{xy}\\
c_{xy} & c_{yy}
\end{bmatrix}.
$
Then
\[
\E\!\left[\left.\frac{a_X^\top a_Y}{d_{\mathrm{mlp}}}\ \right|\ z_1(X),z_1(Y)\right]
=
\psi\!\left(
\hat K_{\mathrm{attn}}(X,X),
\hat K_{\mathrm{attn}}(X,Y),
\hat K_{\mathrm{attn}}(Y,Y)
\right),
\]
and conditional Chebyshev gives
\[
\frac{a_X^\top a_Y}{d_{\mathrm{mlp}}}
-
\psi\!\left(
\hat K_{\mathrm{attn}}(X,X),
\hat K_{\mathrm{attn}}(X,Y),
\hat K_{\mathrm{attn}}(Y,Y)
\right)
\to 0
\]
in probability.

\medskip 
Because $\sigma$ has linear growth, the map $\psi$ is continuous on the cone of positive semidefinite $2\times2$ covariance matrices.
Indeed, if covariance matrices $\Sigma_n\to\Sigma$, the corresponding centered Gaussian pairs converge in distribution, and the family
\[
\sigma(U_n)\sigma(V_n)
\]
is uniformly integrable because
\[
|\sigma(U_n)\sigma(V_n)|
\le C(1+U_n^2+V_n^2)
\]
and the second moments of $(U_n,V_n)$ are bounded on compact covariance sets.
Proposition~\ref{prop:Hfaithful_attention_kernel_limit} gives
\[
\bigl(
\hat K_{\mathrm{attn}}(X,X),\hat K_{\mathrm{attn}}(X,Y),\hat K_{\mathrm{attn}}(Y,Y)
\bigr)
\to
\bigl(
K_{\mathrm{attn}}(X,X),K_{\mathrm{attn}}(X,Y),K_{\mathrm{attn}}(Y,Y)
\bigr)
\]
in probability.
Hence
\[
\psi\!\left(
\hat K_{\mathrm{attn}}(X,X),
\hat K_{\mathrm{attn}}(X,Y),
\hat K_{\mathrm{attn}}(Y,Y)
\right)
\to
K_{\mathrm{trans}}(X,Y)
\]
in probability.

Combining all of the above
\[
\hat K_{\mathrm{trans}}(X,Y)\to K_{\mathrm{trans}}(X,Y)
\]
in probability for the fixed pair $(X,Y)$.

Because $\mathcal S$ is finite, the same convergence holds simultaneously for every pair in $\mathcal S\times\mathcal S$ by a union bound.
This proves \eqref{eq:Hfaithful_kernel_prob_limit}.
\end{proof}

  Let $(X_i,Y_i)_{i=1}^n$ be a fixed binary-labeled dataset with $Y_i\in\{-1,+1\}$.
For each finite width, training only the output head with binary logistic loss and an $\ell_2$ penalty is the optimization problem
\begin{equation}
\label{eq:Hfaithful_finitewidth_logistic}
\min_{u\in\R^{d_{\mathrm{emb}}}}
\frac1n\sum_{i=1}^n
\log\!\bigl(1+e^{-Y_i\langle u,\Psi_\theta(X_i)\rangle}\bigr)
+
\frac{\lambda}{2}\|u\|_2^2.
\end{equation}
With
\[
\hat K_n
:=
\bigl(\hat K_{\mathrm{trans}}(X_i,X_j)\bigr)_{i,j=1}^n,
\]
the finite-width optimization problem reduces exactly to
\begin{equation}
\label{eq:Hfaithful_finitewidth_dual_random}
\min_{\alpha\in\R^n}
\frac1n\sum_{i=1}^n
\log\!\bigl(1+e^{-Y_i(\hat K_n\alpha)_i}\bigr)
+
\frac{\lambda}{2}\alpha^\top \hat K_n\alpha.
\end{equation}
Moreover, by Theorem~\ref{thm:Hfaithful_kernel_prob_limit},
\[
\hat K_n \to K_n
:=
\bigl(K_{\mathrm{trans}}(X_i,X_j)\bigr)_{i,j=1}^n
\]
entrywise in probability, and hence for every fixed $\alpha\in\R^n$ the random objective in \eqref{eq:Hfaithful_finitewidth_dual_random} converges in probability to
\begin{equation}
\label{eq:Hfaithful_dual}
\min_{\alpha\in\R^n}
\frac1n\sum_{i=1}^n
\log\!\bigl(1+e^{-Y_i(K_n\alpha)_i}\bigr)
+
\frac{\lambda}{2}\alpha^\top K_n\alpha.
\end{equation}

 The   limit problem \eqref{eq:Hfaithful_dual} is exactly the coefficient form of the RKHS optimization problem
\begin{equation}
\label{eq:Hfaithful_rkhs_logistic}
\min_{f\in\mathcal H_{K_{\mathrm{trans}}}}
\frac1n\sum_{i=1}^n
\log\!\bigl(1+e^{-Y_i f(X_i)}\bigr)
+
\frac{\lambda}{2}\|f\|_{\mathcal H_{K_{\mathrm{trans}}}}^2.
\end{equation}

 Next
we give a sufficient condition under which the ReLU frozen-feature template
matrix is strictly positive definite.  The condition   only requires that, on a
positive-probability set of frozen features, one template's response dominates
the total response of the others.

Let the limiting ReLU frozen-feature kernel admit the random-feature
representation
$
  K_{\mathrm{ReLU}}(x,x')
  =
  \mathbb E_{\omega}
  \bigl[
    \phi_\omega(x)\phi_\omega(x')
  \bigr],
  \ 
  \phi_\omega(x)
  =
  \operatorname{ReLU}(T_\omega(x)).
$
For each template \(z_a\), let \(S_a\) denote an admissible fresh substitution,
and set
$
  X_a:=\Sub(z_a,S_a).
$
Define the mean ReLU response of the frozen feature \(\omega\) on template
class \(a\) by
$
  \mu_a(\omega)
  :=
  \mathbb E_{S_a}
  \left[
    \phi_\omega(\Sub(z_a,S_a))
  \right].
$
The corresponding fresh-template matrix is
\[
  N^{\operatorname{ReLU}}_{ab}
  :=
  \mathbb E_{\omega}
  \bigl[
    \mu_a(\omega)\mu_b(\omega)
  \bigr].
\] 

Assume that for every template \(a\in[r]\), there exist a measurable set of
frozen features
$
  \Omega_a, \ p_a:=\mathbb P(\Omega_a)>0
$
and constants
$
  \eta_a>0,
  \ 
  \theta_a\in[0,1),
$
such that
  for every \(\omega\in\Omega_a\),
\begin{equation}
\label{eq:relu-dominant-response-condition}
  \mu_a(\omega)\ge \eta_a,
  \qquad
  \sum_{b\neq a}\mu_b(\omega)
  \le
  \theta_a\mu_a(\omega).
\end{equation}
In words, on \(\Omega_a\), the feature responds to template \(a\) by at least
\(\eta_a\), and the combined response on all other templates is at most a
\(\theta_a\)-fraction of the response on template \(a\).

Then
$
  \mathbf N^{\operatorname{ReLU}}
  :=
  \bigl(N^{\operatorname{ReLU}}_{ab}\bigr)_{a,b=1}^r
$
is strictly positive definite in that for 
$
  p_*:=\min_{a\in[r]}p_a,
  \ 
  \delta_*:=\min_{a\in[r]}(1-\theta_a)\eta_a,
$
 we have for every \(c\in\mathbb R^r\),
\begin{equation}
\label{eq:relu-dominant-response-lower-bound}
  c^\top \mathbf N^{\operatorname{ReLU}}c
  \ge
  p_*\delta_*^2\|c\|_\infty^2
  \ge
  {p_*\delta_*^2\over r}\|c\|_2^2.
\end{equation}

Indeed, by the random-feature representation,
\[
\begin{aligned}
  c^\top \mathbf N^{\operatorname{ReLU}}c
  &=
  \sum_{a,b=1}^r c_ac_b
  \mathbb E_\omega[
    \mu_a(\omega)\mu_b(\omega)
  ]  =
  \mathbb E_\omega
  \left[
    \left(
      \sum_{a=1}^r c_a\mu_a(\omega)
    \right)^2
  \right].
\end{aligned}
\]
Thus \(\mathbf N^{\operatorname{ReLU}}\succeq0\).  We now prove the strict
lower bound.

First, the sets \(\Omega_a\) are pairwise disjoint.  Suppose, to the contrary,
that \(\omega\in\Omega_a\cap\Omega_b\) for some \(a\neq b\).  Since
\(\omega\in\Omega_a\),
\[
  \mu_b(\omega)
  \le
  \sum_{j\neq a}\mu_j(\omega)
  \le
  \theta_a\mu_a(\omega).
\]
Since \(\omega\in\Omega_b\),
\[
  \mu_a(\omega)
  \le
  \sum_{j\neq b}\mu_j(\omega)
  \le
  \theta_b\mu_b(\omega).
\]
Combining the above  gives
$
  \mu_a(\omega)
  \le
  \theta_a\theta_b\mu_a(\omega).
$
But \(\theta_a\theta_b<1\) and
\(\mu_a(\omega)\ge\eta_a>0\), a contradiction.  Hence
\[
  \Omega_a\cap\Omega_b=\varnothing
  \qquad
  (a\neq b).
\]

For \(a,b\in[r]\), define the conditional response
$
  M_{ab}
  :=
  \mathbb E[
    \mu_b(\omega)
    \mid
    \omega\in\Omega_a
  ].
$
Since ReLU is nonnegative, \(M_{ab}\ge0\).  The dominance condition gives
$
  M_{aa}\ge \eta_a
$
and
\[
  \sum_{b\neq a}M_{ab}
  =
  \mathbb E
  \biggl[
    \sum_{b\neq a}\mu_b(\omega)
    \, | \,
    \omega\in\Omega_a
  \biggl]
  \le
  \theta_a
  \mathbb E[
    \mu_a(\omega)
    \mid
    \omega\in\Omega_a
  ]
  =
  \theta_a M_{aa}.
\]
Therefore
\begin{equation}
\label{eq:relu-M-row-dominance}
  M_{aa}
  -
  \sum_{b\neq a}M_{ab}
  \ge
  (1-\theta_a)M_{aa}
  \ge
  (1-\theta_a)\eta_a
  \ge
  \delta_* .
\end{equation}

Now fix \(c\in\mathbb R^r\), and write
$
  F_c(\omega):=\sum_{b=1}^r c_b\mu_b(\omega).
$
Using the disjointness of the sets \(\Omega_a\),
\[
\begin{aligned}
  c^\top \mathbf N^{\operatorname{ReLU}}c
  &=
  \mathbb E[F_c(\omega)^2]  \ge
  \sum_{a=1}^r
  \mathbb E[
    F_c(\omega)^2\mathbf 1_{\Omega_a}
  ]   =
  \sum_{a=1}^r
  p_a\,
  \mathbb E[
    F_c(\omega)^2
    \mid
    \omega\in\Omega_a
  ].
\end{aligned}
\]
By Jensen's inequality,
\[
  \mathbb E[
    F_c(\omega)^2
    \mid
    \omega\in\Omega_a
  ]
  \ge
  \left(
    \mathbb E[
      F_c(\omega)
      \mid
      \omega\in\Omega_a
    ]
  \right)^2.
\]
Since
\[
  \mathbb E[
    F_c(\omega)
    \mid
    \omega\in\Omega_a
  ]
  =
  \sum_{b=1}^r c_bM_{ab},
\]
we obtain
\begin{equation}
\label{eq:relu-jensen-lower}
  c^\top \mathbf N^{\operatorname{ReLU}}c
  \ge
  \sum_{a=1}^r
  p_a
  \left(
    \sum_{b=1}^r c_bM_{ab}
  \right)^2.
\end{equation}

Choose \(a_*\in[r]\) such that
$
  |c_{a_*}|=\|c\|_\infty.
$
Then
\[
\begin{aligned}
  \left|
    \sum_{b=1}^r c_bM_{a_*b}
  \right|
  &=
  \left|
    c_{a_*}M_{a_*a_*}
    +
    \sum_{b\neq a_*}c_bM_{a_*b}
  \right| \ge
  |c_{a_*}|M_{a_*a_*}
  -
  \sum_{b\neq a_*}|c_b|M_{a_*b} \\
  &\ge
  \|c\|_\infty
  \left(
    M_{a_*a_*}
    -
    \sum_{b\neq a_*}M_{a_*b}
  \right)  \ge
  \delta_*\|c\|_\infty,
\end{aligned}
\]
where the last step uses \eqref{eq:relu-M-row-dominance}.  Keeping only the
\(a_*\)-term in \eqref{eq:relu-jensen-lower} gives
\[
  c^\top \mathbf N^{\operatorname{ReLU}}c
  \ge
  p_{a_*}\delta_*^2\|c\|_\infty^2
  \ge
  p_*\delta_*^2\|c\|_\infty^2.
\]
Finally,
\[
  \|c\|_\infty^2\ge {1\over r}\|c\|_2^2,
\]
so
$
  c^\top \mathbf N^{\operatorname{ReLU}}c
  \ge
  {p_*\delta_*^2\over r}\|c\|_2^2.
$
This proves \eqref{eq:relu-dominant-response-lower-bound}.    

\section{Abstract prompting through the transformer kernel}
\label{sec:abstract-prompting-ktrans}

We next extend the collision-perturbation viewpoint to abstraction-based
reasoning methods such as AbstRaL~\cite{gao2026abstral},
Chain-of-Abstraction~\cite{gao2025coa}, and SyReLM~\cite{dutta2024syrelm}.
Although these methods use intermediate abstractions to improve reasoning,
their theoretical effect on fresh-symbol generalization has not, to our
knowledge, been formalized.  We formalize this phenomenon by augmenting each
template-generated input with an abstraction channel, representing prompted
intermediate reasoning.  The analysis reveals a perhaps surprising limitation:
abstraction is not inherently beneficial.  It improves fresh-symbol
classification only when it strengthens the template-level margin, and only
when this gain outweighs the additional transformer-kernel interactions
introduced by the abstraction tokens.

Let \(\psi\) denote the parameters of an abstraction generator after upstream
training.  Given a concrete input \(X\), the abstraction is
\[
    A=\Anc_\psi(X,\xi),
\]
where \(\xi\) denotes prompting, decoding, retrieval, or sampling randomness.
Equivalently, \(A\sim\nu_\psi(\cdot\mid X)\).  Thus \(A\) may be an arbitrary
learned function of \(X\), possibly randomized.  We condition on \(\psi\)
throughout this section.

We use the limiting frozen-feature transformer kernel derived in
Appendix~\ref{app:transformer-kernel-interface}.  Recall that
\[
    K_{\mathrm{trans}}(X,Y)
    =
    \E_{u,v}[\sigma(u)\sigma(v)],
\]
where \((u,v)\) is a centered Gaussian pair with covariance
\[
    \begin{pmatrix}
        K_{\mathrm{attn}}(X,X) & K_{\mathrm{attn}}(X,Y)\\
        K_{\mathrm{attn}}(Y,X) & K_{\mathrm{attn}}(Y,Y)
    \end{pmatrix}.
\]
The attention kernel \(K_{\mathrm{attn}}\) is
\[
K_{\mathrm{attn}}(X,Y)
=
\E_{m(X),m(Y)}
\left[
\operatorname{softmax}(\beta m(X))^\top
(XY^\top+\gamma^2I)
\operatorname{softmax}(\beta m(Y))
\right],
\]
where
\[
\begin{bmatrix}
m(X)\\
m(Y)
\end{bmatrix}
\sim
N\!\left(
0,
(1+\gamma^2)
\begin{bmatrix}
XX^\top+\gamma^2I & XY^\top+\gamma^2I\\
YX^\top+\gamma^2I & YY^\top+\gamma^2I
\end{bmatrix}
\right).
\]

The abstraction enters by changing the sequence fed to this same kernel.  Let
\[
    U_\eta(x,a)
\]
be the fixed-length embedded sequence obtained from the concrete input \(x\)
and abstraction \(a\), where \(\eta\ge0\) controls the strength of the
abstraction prompt.  We assume the prompt is turned off at \(\eta=0\):
\[
    U_0(x,a)=U_{\mathrm{base}}(x)
    \qquad\text{for all }a.
\]
Define the prompted transformer kernel by composition:
\[
    K_{\mathrm{trans},\eta}((x,a),(x',a'))
    :=
    K_{\mathrm{trans}}
    \!\left(
        U_\eta(x,a),U_\eta(x',a')
    \right).
\]
This definition does not impose an additive decomposition.  The effect of the
abstraction may pass through attention between original tokens and abstraction
tokens, through the MLP covariance map, or through both.
 
We assume that \(U_\eta(x,a)\) is \(C^1\) in \(\eta\), and that for some
constants \(B,B_1<\infty\),
\[
    \sup_{\eta\in[0,\eta_0]}
    \|U_\eta(x,a)\|_{\op}\le B,
    \qquad
    \sup_{\eta\in[0,\eta_0]}
    \|\partial_\eta U_\eta(x,a)\|_{\op}\le B_1
\]
for all admissible \(x,a\).  This covers bounded abstraction prompts inserted
through a smooth gate.  For differentiating the MLP covariance map, we also
assume that \(\sigma\in C^2\) and that
\[
    |\sigma(t)|+|\sigma'(t)|+|\sigma''(t)|
    \le C_\sigma(1+|t|^m)
\]
for some finite \(C_\sigma,m\).  The convergence result for
\(K_{\mathrm{trans}}\) itself only needs weaker growth assumptions; the extra
smoothness is used here only to take the first derivative with respect to the
prompt strength \(\eta\).

\begin{theorem}[First-order regularity of abstract prompting in \(K_{\mathrm{trans}}\)]
\label{thm:abst-ktrans-regularity}
Assume the smooth bounded prompting condition above.  For fixed augmented
inputs \((x,a)\) and \((x',a')\), the map
\[
    \eta
    \mapsto
    K_{\mathrm{trans},\eta}((x,a),(x',a'))
\]
is differentiable on \([0,\eta_0]\).  Moreover, there exists a finite envelope
\(M(x,a,x',a')\) such that, for all sufficiently small \(\eta>0\),
\[
    \left|
    {K_{\mathrm{trans},\eta}((x,a),(x',a'))
    -
    K_{\mathrm{trans},0}((x,a),(x',a'))\over \eta}
    \right|
    \le
    M(x,a,x',a').
\]
If the abstraction space is bounded, \(M\) may be chosen uniformly.  If the
abstraction is stochastic, it is enough that
\(M(X,A,X',A')\) be integrable.
\end{theorem}

\begin{proof}
Write
\[
    U_\eta:=U_\eta(x,a),
    \qquad
    V_\eta:=U_\eta(x',a').
\]
We first check differentiability of the attention kernel along this path.
Define
\[
    C_\eta^{UU}:=U_\eta U_\eta^\top+\gamma^2I,\qquad
    C_\eta^{UV}:=U_\eta V_\eta^\top+\gamma^2I,\qquad
    C_\eta^{VV}:=V_\eta V_\eta^\top+\gamma^2I.
\]
The bounded \(C^1\) assumption on \(U_\eta,V_\eta\) implies that these matrices
are \(C^1\) in \(\eta\), with uniformly bounded operator norms and uniformly
bounded derivatives.

Let
\[
    G_\eta=(G_{\eta,U},G_{\eta,V})
\]
be centered Gaussian with covariance
\[
    \Sigma_\eta^{\mathrm{attn}}
    :=
    (1+\gamma^2)
    \begin{pmatrix}
        C_\eta^{UU} & C_\eta^{UV}\\
        (C_\eta^{UV})^\top & C_\eta^{VV}
    \end{pmatrix}.
\]
Then
\[
    K_{\mathrm{attn}}(U_\eta,V_\eta)
    =
    \E\left[
        s(G_{\eta,U})^\top C_\eta^{UV}s(G_{\eta,V})
    \right],
    \qquad
    s(u):=\operatorname{softmax}(\beta u).
\]
Set
\[
    f_\eta(u,v):=s(u)^\top C_\eta^{UV}s(v).
\]
The softmax map has uniformly bounded first and second derivatives, and
\(s(u)\) lies in the probability simplex.  Hence
\[
    |\partial_\eta f_\eta(u,v)|
    \le
    \|\partial_\eta C_\eta^{UV}\|_{\op},
\]
and the Hessian blocks of \(f_\eta\) satisfy
\[
    \|\nabla^2_{uu}f_\eta(u,v)\|_{\op}
    +
    \|\nabla^2_{vv}f_\eta(u,v)\|_{\op}
    +
    \|\nabla^2_{uv}f_\eta(u,v)\|_{\op}
    \lesssim_\beta
    \|C_\eta^{UV}\|_{\op}.
\]
Both right-hand sides are uniformly bounded on \([0,\eta_0]\).

We use the standard Gaussian covariance differentiation identity: if
\(Z_\eta\sim N(0,\Sigma_\eta)\), \(\Sigma_\eta\) is \(C^1\), and \(f_\eta\)
has bounded \(\partial_\eta f_\eta\) and bounded Hessian, then
\[
    {d\over d\eta}\E f_\eta(Z_\eta)
    =
    \E[\partial_\eta f_\eta(Z_\eta)]
    +
    {1\over2}
    \operatorname{tr}
    \left(
        \partial_\eta\Sigma_\eta\,
        \E[\nabla^2 f_\eta(Z_\eta)]
    \right).
\]
For nonsingular covariances this follows by differentiating the Gaussian
density and applying Gaussian integration by parts; the singular case follows
by adding \(\varepsilon I\) and sending \(\varepsilon\downarrow0\).  Therefore
\[
    \eta\mapsto K_{\mathrm{attn}}(U_\eta,V_\eta)
\]
is differentiable and has a locally bounded derivative.  The same argument
applies to
\[
    K_{\mathrm{attn}}(U_\eta,U_\eta),
    \qquad
    K_{\mathrm{attn}}(V_\eta,V_\eta).
\]

Now pass through the MLP covariance map.  Define
\[
    \Sigma_\eta^{\mathrm{mlp}}
    :=
    \begin{pmatrix}
        K_{\mathrm{attn}}(U_\eta,U_\eta)
        &
        K_{\mathrm{attn}}(U_\eta,V_\eta)\\
        K_{\mathrm{attn}}(V_\eta,U_\eta)
        &
        K_{\mathrm{attn}}(V_\eta,V_\eta)
    \end{pmatrix}.
\]
The preceding paragraph shows that
\(\Sigma_\eta^{\mathrm{mlp}}\) is \(C^1\) with bounded derivative.  The full
transformer kernel is
\[
    K_{\mathrm{trans},\eta}((x,a),(x',a'))
    =
    \E[\sigma(Z_{\eta,1})\sigma(Z_{\eta,2})],
    \qquad
    (Z_{\eta,1},Z_{\eta,2})\sim N(0,\Sigma_\eta^{\mathrm{mlp}}).
\]
Apply the same Gaussian covariance differentiation identity to
\[
    h(z_1,z_2)=\sigma(z_1)\sigma(z_2).
\]
The Hessian of \(h\) is
\[
    \nabla^2 h(z_1,z_2)
    =
    \begin{pmatrix}
        \sigma''(z_1)\sigma(z_2)
        &
        \sigma'(z_1)\sigma'(z_2)\\
        \sigma'(z_1)\sigma'(z_2)
        &
        \sigma(z_1)\sigma''(z_2)
    \end{pmatrix}.
\]
By the polynomial-growth assumption on
\(\sigma,\sigma',\sigma''\), its entries have finite Gaussian moments under the
bounded covariance family \(\Sigma_\eta^{\mathrm{mlp}}\).  Hence
\[
    \eta\mapsto K_{\mathrm{trans},\eta}((x,a),(x',a'))
\]
is differentiable and has a locally bounded derivative.  The mean-value theorem
then gives
\[
    \left|
    {K_{\mathrm{trans},\eta}-K_{\mathrm{trans},0}\over \eta}
    \right|
    \le
    \sup_{t\in[0,\eta]}
    |\partial_t K_{\mathrm{trans},t}|,
\]
which is the desired envelope.  The uniform and stochastic-integrable versions
follow from the corresponding boundedness or integrability of the augmented
prompt map.
\end{proof}

For templates \(z_a,z_b\), let
\[
    X_a=\Sub(z_a,S_a),
    \qquad
    A_a=\Anc_\psi(X_a,\xi_a),
\]
and let \((X_b',A_b')\) be an independent fresh augmented instantiation of
\(z_b\).  Define
\[
    N_{\eta,ab}
    :=
    \E\!\left[
        K_{\mathrm{trans},\eta}
        ((X_a,A_a),(X_b',A_b'))
    \right].
\]
Since \(U_0(x,a)=U_{\mathrm{base}}(x)\), the abstraction is absent at
\(\eta=0\), and \(N_{0,ab}=N_{ab}\).

\begin{theorem}[Template expansion induced by abstract prompting]
\label{thm:abst-ktrans-template-expansion}
Under the conditions of Theorem~\ref{thm:abst-ktrans-regularity},
\[
    N_{\eta,ab}
    =
    N_{ab}
    +
    \eta H_{ab}
    +
    o(\eta),
\]
where
\[
    H_{ab}
    :=
    \E\!\left[
        \dot K_{\mathrm{trans},0}
        ((X_a,A_a),(X_b',A_b'))
    \right],
    \qquad
    \dot K_{\mathrm{trans},0}
    :=
    \left.{d\over d\eta}K_{\mathrm{trans},\eta}\right|_{\eta=0}.
\]
Equivalently,
\[
    \bN_\eta=\bN+\eta H+o(\eta).
\]
The matrix \(H\) is the first-order template-level effect of abstract prompting.
It includes abstraction-token effects, cross-attention effects, and the
downstream MLP covariance response.  In general \(H\) need not be positive
semidefinite.
\end{theorem}

\begin{proof}
By definition,
\[
    {N_{\eta,ab}-N_{0,ab}\over\eta}
    =
    \E\!\left[
        {K_{\mathrm{trans},\eta}((X_a,A_a),(X_b',A_b'))
        -
        K_{\mathrm{trans},0}((X_a,A_a),(X_b',A_b'))\over\eta}
    \right].
\]
Since \(U_0(x,a)=U_{\mathrm{base}}(x)\), the abstraction is absent at
\(\eta=0\), so \(N_{0,ab}=N_{ab}\).  Theorem~\ref{thm:abst-ktrans-regularity}
gives pointwise convergence of the difference quotient to
\[
    \dot K_{\mathrm{trans},0}((X_a,A_a),(X_b',A_b')),
\]
and it also supplies an integrable envelope.  Dominated convergence gives
\[
    \lim_{\eta\downarrow0}
    {N_{\eta,ab}-N_{ab}\over\eta}
    =
    \E\!\left[
        \dot K_{\mathrm{trans},0}((X_a,A_a),(X_b',A_b'))
    \right]
    =
    H_{ab}.
\]
The template set is finite, so the entrywise expansion gives
\[
    \bN_\eta=\bN+\eta H+o(\eta).
\]
Symmetry of \(H\) follows from symmetry of \(K_{\mathrm{trans},\eta}\).  Since
\(H\) is a derivative direction of positive-semidefinite matrices, it need not
itself be positive semidefinite.
\end{proof}

For \(q\in\Delta_r\), define the ideal prompted template objective
\[
    \Phi_{\eta,q,\lambda}(g)
    :=
    \sum_{a=1}^r q_a\loss(y_ag_a)
    +
    {\lambda\over2}g^\top\bN_\eta^{-1}g,
    \qquad g\in\R^r,
\]
and let
\[
    g_\eta:=\arg\min_{g\in\R^r}\Phi_{\eta,q,\lambda}(g).
\]

\begin{theorem}[First-order value of abstract prompting]
\label{thm:abst-ktrans-first-order-value}
Assume \(\bN\succ0\), \(q_a>0\) for every \(a\), and \(\lambda>0\).  Let
\(g_0=g_\eta|_{\eta=0}\), and define
\[
    J_0
    :=
    \diag\!\left(
        q_a\loss''(y_a(g_0)_a)
    \right)_{a=1}^r
    +
    \lambda\bN^{-1}.
\]
Then
\[
    \left.{d\over d\eta}g_\eta\right|_{\eta=0}
    =
    \lambda J_0^{-1}\bN^{-1}H\bN^{-1}g_0.
\]
Consequently, the first-order change in the signed margin of template \(a\) is
\[
    \mathfrak C_a(H)
    :=
    y_a e_a^\top
    \lambda J_0^{-1}\bN^{-1}H\bN^{-1}g_0.
\]
If an augmented certificate \(B_{\lambda,\eta}\) is right-differentiable at
\(\eta=0\), then the certified margin
\[
    \mathcal M_a(\eta):=y_a(g_\eta)_a-B_{\lambda,\eta}
\]
satisfies
\[
    \left.{d_+\over d\eta}\mathcal M_a(\eta)\right|_{\eta=0}
    =
    \mathfrak C_a(H)
    -
    \left.{d_+\over d\eta}B_{\lambda,\eta}\right|_{\eta=0}.
\]
Thus abstract prompting improves the certified margin to first order exactly
when its template-margin gain exceeds the additional certificate cost.
\end{theorem}

\begin{proof}
By Theorem~\ref{thm:abst-ktrans-template-expansion},
\[
    \bN_\eta=\bN+\eta H+o(\eta).
\]
Since \(\bN\succ0\), \(\bN_\eta\succ0\) for all sufficiently small \(\eta\).
The first-order condition for \(g_\eta\) is
\[
    F(g_\eta,\eta)=0,
\]
where
\[
    F(g,\eta)
    :=
    q\odot y\odot\loss'(y\odot g)
    +
    \lambda\bN_\eta^{-1}g .
\]
At \((g_0,0)\),
\[
    D_gF(g_0,0)
    =
    \diag\!\left(
        q_a\loss''(y_a(g_0)_a)
    \right)_{a=1}^r
    +
    \lambda\bN^{-1}
    =
    J_0 .
\]
Because \(\loss''\ge0\) and \(\bN^{-1}\succ0\), \(J_0\succ0\).  Hence the
implicit function theorem applies.

The inverse expansion
\[
    \bN_\eta^{-1}
    =
    \bN^{-1}
    -
    \eta\bN^{-1}H\bN^{-1}
    +
    o(\eta)
\]
gives
\[
    D_\eta F(g_0,0)
    =
    -\lambda\bN^{-1}H\bN^{-1}g_0 .
\]
Differentiating \(F(g_\eta,\eta)=0\) at \(\eta=0\) yields
\[
    J_0\dot g_0
    -
    \lambda\bN^{-1}H\bN^{-1}g_0
    =
    0.
\]
Therefore
\[
    \dot g_0
    =
    \lambda J_0^{-1}\bN^{-1}H\bN^{-1}g_0.
\]
Multiplying coordinate \(a\) by \(y_a\) gives \(\mathfrak C_a(H)\).  The
certified-margin derivative follows by subtracting the right derivative of
\(B_{\lambda,\eta}\).
\end{proof}

Any valid augmented certificate \(B_{\lambda,\eta}(x,\delta)\) now yields the
same margin-transfer statement as Theorem~\ref{thm:main}: if
\[
    y_a(g_\eta)_a>\varepsilon+s
    \qquad\text{and}\qquad
    B_{\lambda,\eta}(x,\delta)\le s,
\]
then
\[
    y_a\widehat f_{\eta,\lambda}(x,A_x)>\varepsilon
\]
on the corresponding high-probability event.  Canonical abstractions can be
handled by applying the original collision graph to the prompted transformer
kernel.  Stochastic abstractions contribute additional centered block-average
or bias--fluctuation terms, exactly as in the existing certificates.

% Supplementary writeup for reverse-generated template examples.
% This is intended to be pasted into the paper supplement.

\section{Details for  Figure~\ref{fig:reverse-template-certificates}}
\label{sec:supp-reverse-template-examples}

This section gives the complete finite-template construction used for
Figure~\ref{fig:reverse-template-certificates}.  The purpose of the construction is not to choose
certificate inputs by hand, but rather to generate concrete training examples first and then compute all
certificate quantities from the resulting collision graph.  Thus each case below is a genuine template
learning task: we specify the templates, instantiate their wildcard slots, build the colored collision graph,
and only then evaluate the five routes in Theorem~3.2.

The common template family is
\[
z_1=(\mathtt{LA},\alpha,\beta),\qquad
z_2=(\mathtt{LB},\alpha,\beta),\qquad
z_3=(\mathtt{LC},\alpha,\beta),
\]
with labels
\[
y_1=+1,\qquad y_2=-1,\qquad y_3=+1.
\]
The test point has template type $z_1$ and fresh wildcard tokens $\mathtt{TX}$ and $\mathtt{TY}$.
Vertices are indexed in block order: the first $n_{T1}$ vertices have type $z_1$, the next $n_{T2}$
have type $z_2$, and the final $n_{T3}$ have type $z_3$.  Unless explicitly listed below, every wildcard
slot receives a unique fresh token.

Edges are then induced directly by token collisions.  A train--train edge is added whenever two training
examples share a wildcard token, or whenever a wildcard token in one example equals the literal token of
the other example's template.  A train--test edge is added whenever a training wildcard equals one of the
test wildcards $\mathtt{TX},\mathtt{TY}$, or equals the test literal $\mathtt{LA}$.  The sign or magnitude
reported with each substitution is the corresponding signed collision weight used in the discrepancy
matrix.

From the resulting graph we form the signed train--train discrepancy matrix $\Delta$ and the signed
test-collision vector $\zeta$.  The script then computes the derived quantities used by Theorem~3.2,
namely $A_\Delta$, $A_\Delta^\circ$, $\gamma_{\rm cs}$, $E_*$, $X_*$, and the block/test slots
$(D_{\rm bd},Z_{\rm bd})$, $(D_{\rm anova},Z_{\rm anova})$, and $(D_{\rm bf},Z_{\rm bf})$.  Applying
the ordinary and curvature score maps gives the five candidate graph certificates
\[
B_{\rm cs},\qquad B_{\rm deg},\qquad B_{\rm bd},\qquad B_{\rm anova},\qquad B_{\rm bf}.
\]
We report
\[
B^\sharp_\lambda
  =\min\{B_{\rm cs},B_{\rm deg},B_{\rm bd},B_{\rm anova},B_{\rm bf}\},
\]
together with the selected minimizing route.  For comparison, we also report the scalar proxy
$B_\rho=w_{\max}/\rho$.

\begin{center}
\begin{tabular}{c|l|c|c|c|c}
Case & Regime & Train--train edges & Test edges & Selected route & $(B^\sharp_\lambda,B_\rho)$ \\
\hline
$C_1$ & sample-poor & $3$ & $2$ & BD & $(0.0056,0.5000)$ \\
$C_2$ & degree-dominated & $50$ & $0$ & DEG & $(0.0047,1.0000)$ \\
$C_3$ & sample-rich block-local & $13$ & $2$ & BD & $(0.0017,0.1250)$ \\
$C_4$ & row-balanced signed & $8$ & $2$ & ANOVA & $(0.0022,0.1667)$ \\
$C_5$ & curvature-benign & $1$ & $0$ & CS & $(0.0001,0.1429)$ \\
$C_6$ & pair-specific & $3$ & $1$ & BD & $(0.0011,0.1667)$ \\
$C_7$ & projection-degenerate & $12$ & $2$ & ANOVA & $(0.0022,0.1667)$ \\
$C_8$ & literal-correctable & $47$ & $3$ & BF & $(0.0054,0.3333)$
\end{tabular}
\end{center}

\paragraph{Case $C_1$: sample-poor.}
Set $n_{T1}=n_{T2}=n_{T3}=4$.  Starting from fresh wildcard tokens in all slots, impose the following
identifications:
\begin{itemize}
\item vertices $\{0,4\}$ share token $\mathtt{C1\_pair\_A}$ with weight $+1.0$, creating one isolated
cross-template collision;
\item vertices $\{1,8\}$ share token $\mathtt{C1\_pair\_B}$ with weight $+1.0$, creating a second
isolated cross-template collision;
\item vertices $\{5,9\}$ substitute the test token $\mathtt{TX}$ with weight $+1.0$, creating a small
test-star collision in a small sample.
\end{itemize}
This produces $3$ train--train edges and $2$ test edges.  The selected graph certificate is
$B^\sharp_\lambda=0.0056$, achieved by the BD route; the scalar proxy is $B_\rho=0.5000$.

\paragraph{Case $C_2$: degree-dominated.}
Set $n_{T1}=n_{T2}=n_{T3}=10$.  Impose the following identifications:
\begin{itemize}
\item vertex $\{0\}$ substitutes the literal token $\mathtt{LB}$ with weight $+1.0$, so a wildcard in a
$z_1$ example hits the $z_2$ literal and creates a literal-induced hub;
\item vertices $\{0,1,2,3,4,10,11,12,13,14\}$ share token $\mathtt{C2\_degree\_core}$ with weight
$+1.0$, forming a dense degree core with no train--test collision.
\end{itemize}
This produces $50$ train--train edges and no test edges.  The selected graph certificate is
$B^\sharp_\lambda=0.0047$, achieved by the DEG route; the scalar proxy is $B_\rho=1.0000$.

\paragraph{Case $C_3$: sample-rich block-local.}
Set $n_{T1}=n_{T2}=n_{T3}=16$.  Impose four balanced local cliques and one test-star component:
\begin{itemize}
\item vertices $\{0,16,32\}$ share token $\mathtt{C3\_local\_0}$ with weight $+1.0$;
\item vertices $\{2,18,34\}$ share token $\mathtt{C3\_local\_1}$ with weight $+1.0$;
\item vertices $\{4,20,36\}$ share token $\mathtt{C3\_local\_2}$ with weight $+1.0$;
\item vertices $\{6,22,38\}$ share token $\mathtt{C3\_local\_3}$ with weight $+1.0$;
\item vertices $\{9,25\}$ substitute the test token $\mathtt{TX}$ with weight $+1.0$.
\end{itemize}
This produces $13$ train--train edges and $2$ test edges.  The selected graph certificate is
$B^\sharp_\lambda=0.0017$, achieved by the BD route; the scalar proxy is $B_\rho=0.1250$.

\paragraph{Case $C_4$: row-balanced signed.}
Set $n_{T1}=n_{T2}=n_{T3}=12$.  Impose signed collisions with approximately balanced row contributions:
\begin{itemize}
\item vertices $\{0,12,24\}$ share token $\mathtt{C4\_plus\_1}$ with weight $+1.0$;
\item vertices $\{1,13,25\}$ share token $\mathtt{C4\_minus\_1}$ with weight $-1.0$;
\item vertices $\{2,14\}$ share token $\mathtt{C4\_plus\_2}$ with weight $+1.0$;
\item vertices $\{3,15\}$ share token $\mathtt{C4\_minus\_2}$ with weight $-1.0$;
\item vertex $\{4\}$ substitutes the test token $\mathtt{TX}$ with weight $+1.0$;
\item vertex $\{16\}$ substitutes the test token $\mathtt{TY}$ with weight $-1.0$.
\end{itemize}
This produces $8$ train--train edges and $2$ test edges.  The selected graph certificate is
$B^\sharp_\lambda=0.0022$, achieved by the ANOVA route; the scalar proxy is $B_\rho=0.1667$.

\paragraph{Case $C_5$: curvature-benign.}
Set $n_{T1}=n_{T2}=n_{T3}=14$.  Impose a single weak collision:
\begin{itemize}
\item vertices $\{14,28\}$ share token $\mathtt{C5\_weak}$ with weight $+0.6$.
\end{itemize}
This produces $1$ train--train edge and no test edges.  The selected graph certificate is
$B^\sharp_\lambda=0.0001$, achieved by the CS route; the scalar proxy is $B_\rho=0.1429$.

\paragraph{Case $C_6$: pair-specific.}
Set $n_{T1}=n_{T2}=n_{T3}=12$.  Concentrate the train--train collisions on the $z_2$--$z_3$ pair and
place the test hit in the same color:
\begin{itemize}
\item vertices $\{12,24\}$ share token $\mathtt{C6\_23\_0}$ with weight $+1.0$;
\item vertices $\{14,26\}$ share token $\mathtt{C6\_23\_1}$ with weight $+1.0$;
\item vertices $\{16,28\}$ share token $\mathtt{C6\_23\_2}$ with weight $+1.0$;
\item vertex $\{19\}$ substitutes the test token $\mathtt{TX}$ with weight $+1.0$.
\end{itemize}
This produces $3$ train--train edges and $1$ test edge.  The selected graph certificate is
$B^\sharp_\lambda=0.0011$, achieved by the BD route; the scalar proxy is $B_\rho=0.1667$.

\paragraph{Case $C_7$: projection-degenerate.}
Set $n_{T1}=n_{T2}=n_{T3}=12$.  Impose alternating signed cliques designed to cancel in the projected
ANOVA components, together with a signed test pair:
\begin{itemize}
\item vertices $\{0,12,24\}$ share token $\mathtt{C7\_cancel\_0}$ with weight $+1.0$;
\item vertices $\{1,13,25\}$ share token $\mathtt{C7\_cancel\_1}$ with weight $-1.0$;
\item vertices $\{2,14,26\}$ share token $\mathtt{C7\_cancel\_2}$ with weight $+1.0$;
\item vertices $\{3,15,27\}$ share token $\mathtt{C7\_cancel\_3}$ with weight $-1.0$;
\item vertex $\{5\}$ substitutes the test token $\mathtt{TX}$ with weight $+1.0$;
\item vertex $\{17\}$ substitutes the test token $\mathtt{TY}$ with weight $-1.0$.
\end{itemize}
This produces $12$ train--train edges and $2$ test edges.  The selected graph certificate is
$B^\sharp_\lambda=0.0022$, achieved by the ANOVA route; the scalar proxy is $B_\rho=0.1667$.

\paragraph{Case $C_8$: literal-correctable.}
Set $n_{T1}=n_{T2}=n_{T3}=12$.  Create a large literal-induced backbone, plus a small residual collision
after literal correction:
\begin{itemize}
\item vertices $\{0,1\}$ substitute the literal token $\mathtt{LB}$ with weight $+1.0$, so $z_1$
wildcards hit the $z_2$ literal;
\item vertices $\{12,13\}$ substitute the literal token $\mathtt{LA}$ with weight $+1.0$, so $z_2$
wildcards hit the $z_1$ literal;
\item vertices $\{16,28\}$ share token $\mathtt{C8\_resid}$ with weight $+0.7$;
\item vertices $\{12,28\}$ substitute the test token $\mathtt{TX}$ with weight $+1.0$.
\end{itemize}
This produces $47$ train--train edges and $3$ test edges.  The selected graph certificate is
$B^\sharp_\lambda=0.0054$, achieved by the BF route; the scalar proxy is $B_\rho=0.3333$.

 Figure~\ref{fig:reverse-template-certificates} shows that the active certificate route changes with the collision geometry: BD is selected for sparse or block-local structure, DEG for a dense degree core, ANOVA for signed/projection cancellations, CS for weak curvature, and BF for literal-correctable collisions.  Thus the best graph certificate \(B^\sharp_\lambda\) adapts to the mechanism present in the generated collision graph, rather than behaving like a single worst-case scalar bound.

 \begin{figure}[t]
    \centering
    \includegraphics[width=\textwidth]{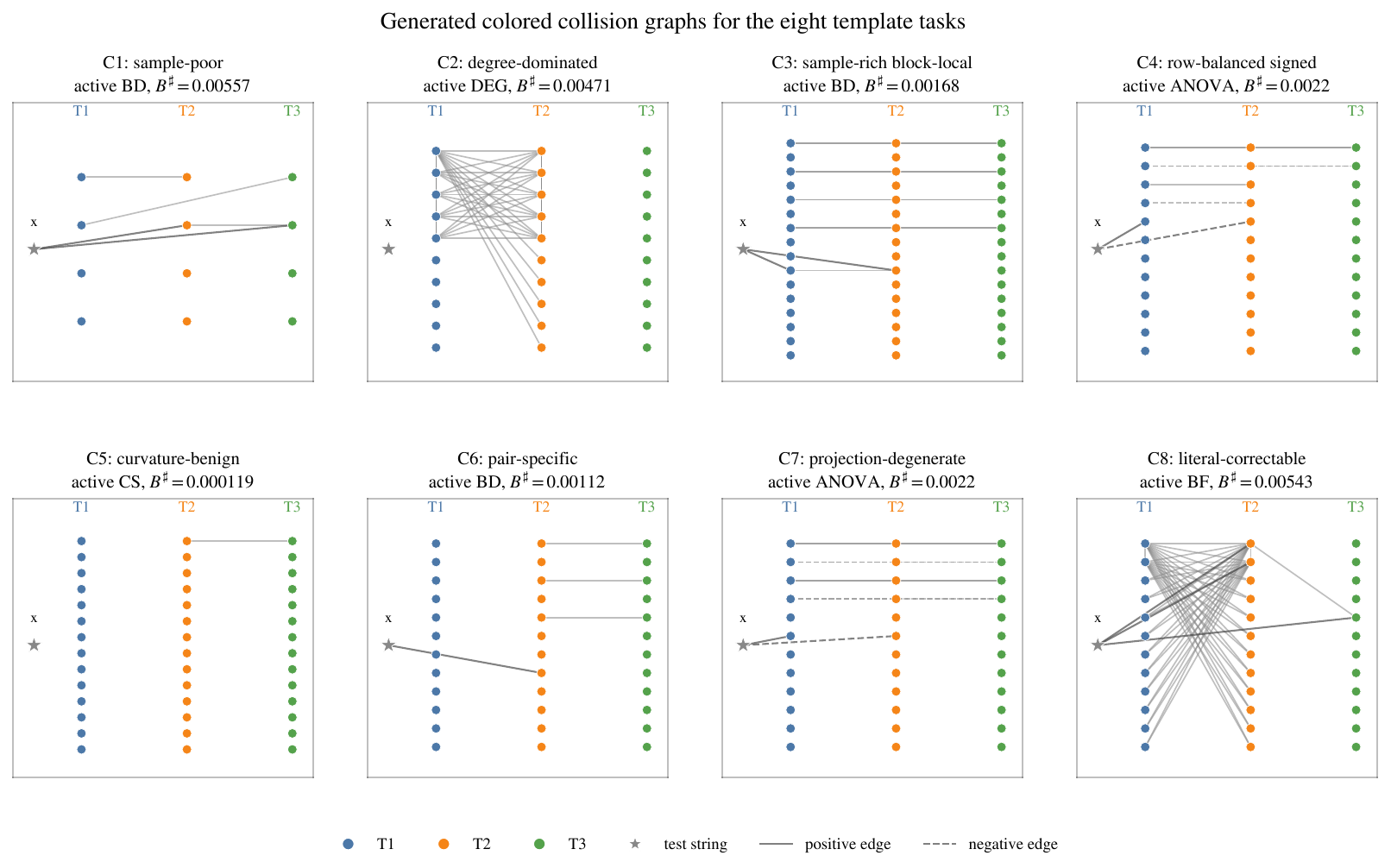}
    \caption{Generated colored collision graphs for the eight template tasks, showing how different wildcard and literal substitutions induce distinct graph geometries.}
    \label{fig:reverse-template-certificates2}
\end{figure}

 Figure~\ref{fig:reverse-template-certificates2} visualizes the colored collision graphs generated by the eight finite template tasks.  The panels show that the same template family can produce qualitatively different collision geometries, including isolated pairs, dense hubs, block-local components, signed cancellations, test-star interactions, projection-degenerate structure, and literal-induced backbones.  These examples illustrate that template collisions are graph-structured objects, not merely scalar worst-case events.

\section{Extensive simulation details}
\label{app:simulation}

The code for this simulation is available in the supplementary material, which is modified from \href{https://github.com/eboix/relational-reasoning}{this GitHub repository} for our purpose.

All experiments instantiate a common synthetic data-generation framework. A
\emph{task} is specified by a tuple $(\cZ, \cW,  \templ, \cY)$,
where $\cZ$ is a finite set of fixed-length \emph{templates} (strings of equal length $L$), $\cW$ is the \emph{wildcard alphabet}, $ \templ$ is the template distribution, and $\cY$ are template labels.  Here we describe the generation for general template tasks, where a sample is generated as follows. (i)~A template $z_a$ is drawn with probability $p_a = \templ(a)$. (ii)~For each \emph{distinct} wildcard symbol $w$ in
$z_a$, a fresh substitution token is drawn without replacement from a
substitution alphabet according to a substitution map, which is split disjointly into train,
validation, and test sub-alphabets so that test substitutions are never seen
during training. (iii)~Wildcard occurrences in $z_a$ are replaced consistently
using this map, while non-wildcard tokens are left intact. (iv)~The label is
either constant (i.e., $\cY = \{1, 2, \ldots, K \}$ for $K$-class tasks) or itself is a wildcard. When the label is a wildcard, then it should be the substituted token under the same map (allowing a ``copy-a-wildcard'' output
mode). The total vocabulary size is therefore
\begin{equation}
	V \;=\; |\mathcal{W}_{\text{used}}| + |\mathcal{R}_{\text{used}}|
	+ S_{\text{train}} + S_{\text{val}} + S_{\text{test}},
\end{equation}
where $\mathcal{R}_{\text{used}}$ collects the non-wildcard tokens that occur
in any template (or any non-wildcard label) and $S_{\text{split}}$ is the
substitution-alphabet size for the corresponding split. Throughout this
appendix we set $S_{\text{val}}=S_{\text{test}}=100$ and
$S_{\text{train}}=n$ (the number of training samples), with $100$ validation
and $100$ test samples. This decoupling separates \emph{identity} content
(the substitution tokens, which are unseen at test time) from \emph{relational}
content (the equality pattern between positions), so solving the task requires
position-wise relational features rather than memorization of substitution
tokens.

\subsection{Template Tasks}
\label{app:tasks}

We instantiate the framework with four template tasks. The symbols
\texttt{\#} and \texttt{\$} denote distinct wildcards, corresponding to $\alpha$ and $\beta$ in the main text; tokens at those positions are sampled independently per example.

\paragraph{Binary same/different
	(in \href{run:experiments/experiment_binary.ipynb}{\texttt{experiment\_binary.ipynb}}).}
\begin{equation*}
	\mathcal{T}=\{\texttt{\#\$\$},\,\texttt{\#\$\#}\},\quad
	\mathcal{W}=\{\texttt{\#},\texttt{\$}\},\quad
	\boldsymbol{\pi}=(0.5,0.5),\quad
	\mathcal{L}=(\texttt{A},\texttt{B}).
\end{equation*}
Each example is a length-3 sequence in which the first and second tokens are
always different. The model emits class \texttt{A} iff the third token equals
the second, and \texttt{B} iff it equals the first. With fresh substitution
tokens at every example, the task is purely relational.

\paragraph{Four-class relational task
	(in \href{run:experiments/experiment_multiclass.ipynb}{\texttt{experiment\_multiclass.ipynb}}).}
\begin{equation*}
	\mathcal{T}=\{\texttt{\#\#\#},\,\texttt{\#\#\$},\,\texttt{\#\$\#},\,\texttt{\#\$\$}\},\quad
	\mathcal{W}=\{\texttt{\#},\texttt{\$}\},\quad
	\boldsymbol{\pi}=(\tfrac14,\tfrac14,\tfrac14,\tfrac14),\quad
	\mathcal{L}=(\texttt{A},\texttt{B},\texttt{C},\texttt{D}).
\end{equation*}
The natural 4-way generalization of the binary task. The model must determine,
from a length-3 sequence with two distinct underlying tokens, which of the
four equality patterns produced it.

\paragraph{Majority / copy-wildcard task
	(in \href{run:experiments/experiment_new.ipynb}{\texttt{experiment\_new.ipynb}}).}
\begin{equation*}
	\mathcal{T}=\{\texttt{\#\#\#},\,\texttt{\#\#\$},\,\texttt{\#\$\#},\,\texttt{\#\$\$}\},\quad
	\mathcal{W}=\{\texttt{\#},\texttt{\$}\},\quad
	\boldsymbol{\pi}=(\tfrac14,\tfrac14,\tfrac14,\tfrac14),\quad
	\mathcal{L}=(\texttt{\#},\texttt{\#},\texttt{\#},\texttt{\$}).
\end{equation*}
Templates are identical to the four-class task, but the label is itself a
wildcard: the model must output the actual substitution token that occupies
the \emph{majority} position. Because the label varies per example with the
substitution and is drawn from a much larger token space than the four
templates, this task requires a contextual \emph{copy} on top of the
relational classification.

\paragraph{Variable-assignment / ``\texttt{print}'' task
	(\texttt{experiment\_new.ipynb}, \emph{teaser}).}
\begin{equation*}
	\mathcal{T}=\bigl\{
	\texttt{\#="\$";?="\%";print(\#)},\;
	\texttt{\#="\$";?="\%";print(?)}
	\bigr\},
\end{equation*}
\begin{equation*}
	\mathcal{W}=\{\texttt{\#},\texttt{?},\texttt{\$},\texttt{\%}\},\quad
	\boldsymbol{\pi}=(0.5,0.5),\quad
	\mathcal{L}=(\texttt{\$},\texttt{\%}).
\end{equation*}
A minimal code-interpretation task: two variables (\texttt{\#} and \texttt{?})
are assigned values (\texttt{\$} and \texttt{\%} respectively) and one of them
is printed. The label is the \emph{value} of the printed variable, drawn anew
at every example, so the model must perform a contextual copy through
substitution tokens it has never seen during training.

\subsection{Model Architectures}
\label{app:models}

\paragraph{Vanilla Transformer encoder.}
A standard pre-norm encoder: token embedding $W_E$, learned positional
embedding, a learnable \texttt{[CLS]} token appended to the sequence, $D$
blocks of multi-head self-attention with separated $QK$ and $V$ projections,
followed by a pre-norm GELU MLP with hidden width $2d$. Classification reads
the \texttt{[CLS]} position through an unaffine \texttt{LayerNorm} and a
linear head. Attention uses scaling $(d_h)^{-1/2}\cdot\texttt{attn\_mult}$. 

\paragraph{Tied embedding/unembedding Transformer.}
Identical to the vanilla encoder except (i) read-out is the final sequence
position rather than \texttt{[CLS]}, and (ii) the unembedding matrix is tied
to $W_E$, so the logits are $W_E\,\mathrm{LN}(h_L)$.

\paragraph{Trainable identity multipliers ($KQ$ and $VO$).}
Both architectures support an attention-identity augmentation that is the
focus of our theoretical results. For each head $h$, two trainable scalars
$\alpha_h^{KQ},\alpha_h^{VO}\in\mathbb{R}$ (initialized to zero) and two
fixed multipliers $M^{KQ},M^{VO}\in\mathbb{R}_{\ge 0}$ modify attention as
\begin{align}
	\mathrm{logits}_h
	&\;=\; \frac{(QK^\top)_h}{\sqrt{d_h}\cdot\texttt{attn\_mult}}
	\;+\; M^{KQ}\,\alpha_h^{KQ}\,X X^\top, \\
	\text{output}
	&\;=\; W_O\,\mathrm{Concat}_h\!\bigl(\mathrm{Attn}_h\,V_h\bigr)
	\;+\; M^{VO}\!\sum_h \alpha_h^{VO}\,\mathrm{Attn}_h\,X.
\end{align}
Setting $M^{KQ}=M^{VO}=0$ recovers the vanilla model. Setting $M^{KQ}=100$
(resp.\ $M^{VO}=100$) gives the \emph{KQ-identity} (resp.\ \emph{VO-identity})
variant; setting both yields the \emph{KQ + VO} variant. The trainable
scalars allow the model to dial these multipliers on or off, but their
\emph{scale} is set by the multiplier.

\subsection{Training Protocol}
\label{app:training}

All runs use the same training loop. The loss is cross-entropy for the
multiclass tasks of Sec.~\ref{app:tasks}. The optimizer is Adam with a single
learning rate (no scheduler, no weight decay); two parameter groups
(\texttt{to\_qk} parameters vs.\ all others) are exposed but receive the
same learning rate. Mini-batches of size $1024$ are drawn from a shuffled
training set; each epoch performs a full pass. Train, validation, and test
metrics are recorded every epoch. We report
\begin{itemize}\itemsep0pt
	\item \textbf{test accuracy:} test accuracy at the epoch where validation
	loss is minimal;
	\item \textbf{test error:} $1$ minus test accuracy at the epoch where
	validation accuracy is maximal;
	\item \textbf{train accuracy:} the maximum training accuracy across all
	epochs. We don't include training accuracy curves in the appendix but it is available in our code.
\end{itemize}
Means and standard errors of the mean are computed over the trial dimension
and rendered as shaded confidence bands in the plots. Every run is cached
on disk and reused on subsequent invocations to keep the entire suite
deterministic.

\subsection{Hyperparameter Settings}
\label{app:hparams}

We organize the experiments into two regimes that share the data-generation
procedure, optimizer, and loss but differ in which axes are swept.

\subsubsection{Architecture comparison}
\label{app:hparams-arch}

This regime fixes a single large model and sweeps the number of training
samples. It produces the per-architecture training/test plots (vanilla,
$+VO\!\times\!100$, $+KQ\!\times\!100$, $+KQ\!\times\!100 + VO\!\times\!100$,
their tied-embedding/unembedding counterparts, and the GPT-2 baseline) for
each of the binary, four-class, and majority tasks
(see Fig.~\ref{fig:binary-arch}, Fig.~\ref{fig:multiclass-arch},
Fig.~\ref{fig:majority-arch}).

\begin{table}[h]
	\centering
	\small
	\begin{tabular}{ll}
		\toprule
		\textbf{Hyperparameter} & \textbf{Value} \\
		\midrule
		Train sample sizes $n$ & $\{16,\,64,\,256,\,1024,\,4096\}$ \\
		$S_{\text{train}}$ & $n$ \\
		$S_{\text{val}}=S_{\text{test}}$ & $100$ \\
		$N_{\text{val}}=N_{\text{test}}$ & $100$ \\
		Embedding/residual dim $d$ & $2048$ (binary, multiclass), $1024$ (majority) \\
		MLP hidden width & $2d$ \\
		Depth $D$ & $2$ \\
		Heads $H$ & $16$ \\
		Per-head dim $d_h$ & $64$ \\
		Dropout / emb dropout & $0$ \\
		\texttt{attn\_mult} & $1.0$ \\
		Identity multipliers $(M^{KQ},M^{VO})$ &
		$\{(0,0),(100,0),(0,100),(100,100)\}$ \\
		Optimizer & Adam, lr $10^{-4}$ \\
		Batch size & $1024$ \\
		Epochs & $1000$ \\
		Trials & $2$ \\
		\bottomrule
	\end{tabular}
	\caption{Hyperparameters for the architecture-comparison regime
		(Sec.~\ref{app:hparams-arch}).}
	\label{tab:hparams-arch}
\end{table}

\subsubsection{Width-scaling regime}
\label{app:hparams-width}

This regime studies how test error decays with $n$ as the embedding
dimension $d$ varies, at fixed identity-multiplier settings, and produces
the width-sweep panels (Fig.~\ref{fig:binary-width-vanilla}, Fig.~\ref{fig:binary-width-tied}, 
Fig.~\ref{fig:multiclass-width-vanilla}, Fig.~\ref{fig:multiclass-width-tied}, Fig.~\ref{fig:majority-width-vanilla}, Fig.~\ref{fig:majority-width-tied}
Fig.~\ref{fig:teaser-width}).

\begin{table}[h]
	\centering
	\small
	\begin{tabular}{ll}
		\toprule
		\textbf{Hyperparameter} & \textbf{Value} \\
		\midrule
		Train sample sizes $n$ & $\{4,8,16,32,64,128,256,512,1024,2048\}$ \\
		& (print task additionally adds $\{4096,16384,65536,262144\}$) \\
		Embedding dim $d$ & $\{64,128,256,512,1024,2048,4096\}$ \\
		Learning rate & $10^{-3}$ if $d\le 512$,\quad $10^{-4}$ if $d\ge 1024$ \\
		Identity multipliers $(M^{KQ},M^{VO})$ &
		binary, four-class: $\{(0, 0), (0, 100), (100,0),(100,100)\}$; \\
		& majority, print: $\{(0,0),(100,100)\}$ \\
		Depth $D$ & $2$ \\
		Heads $H$ & $16$ \\
		Per-head dim $d_h$ & $64$ \\
		Dropout / \texttt{attn\_mult} & $0$ / $1.0$ \\
		Optimizer & Adam \\
		Batch size & $1024$ \\
		Epochs & $1000$ ($100$ for the print task) \\
		Trials & $5$ \\
		Architectures & vanilla and tied-embedding/unembedding Transformers \\
		Reporting metric & test error \\
		\bottomrule
	\end{tabular}
	\caption{Hyperparameters for the width-scaling regime
		(Sec.~\ref{app:hparams-width}).}
	\label{tab:hparams-width}
\end{table}

In this regime, vanilla Transformers (with both $M^{KQ}=M^{VO}=0$) fail to
generalize across a wide range of widths and sample sizes, whereas adding
identity multipliers $(M^{KQ},M^{VO})\in\{(100,0),(100,100)\}$ recovers the
predicted scaling-law behavior $\mathrm{err}(n)\to 0$ as $n$ grows.

\subsection{Figures}
\label{app:figures}

The simulation results are shown in the following figures.

% --- Binary task ---
\begin{figure}[h]
	\centering
	\begin{subfigure}{0.48\textwidth}
		\centering
		\includegraphics[width=\textwidth]{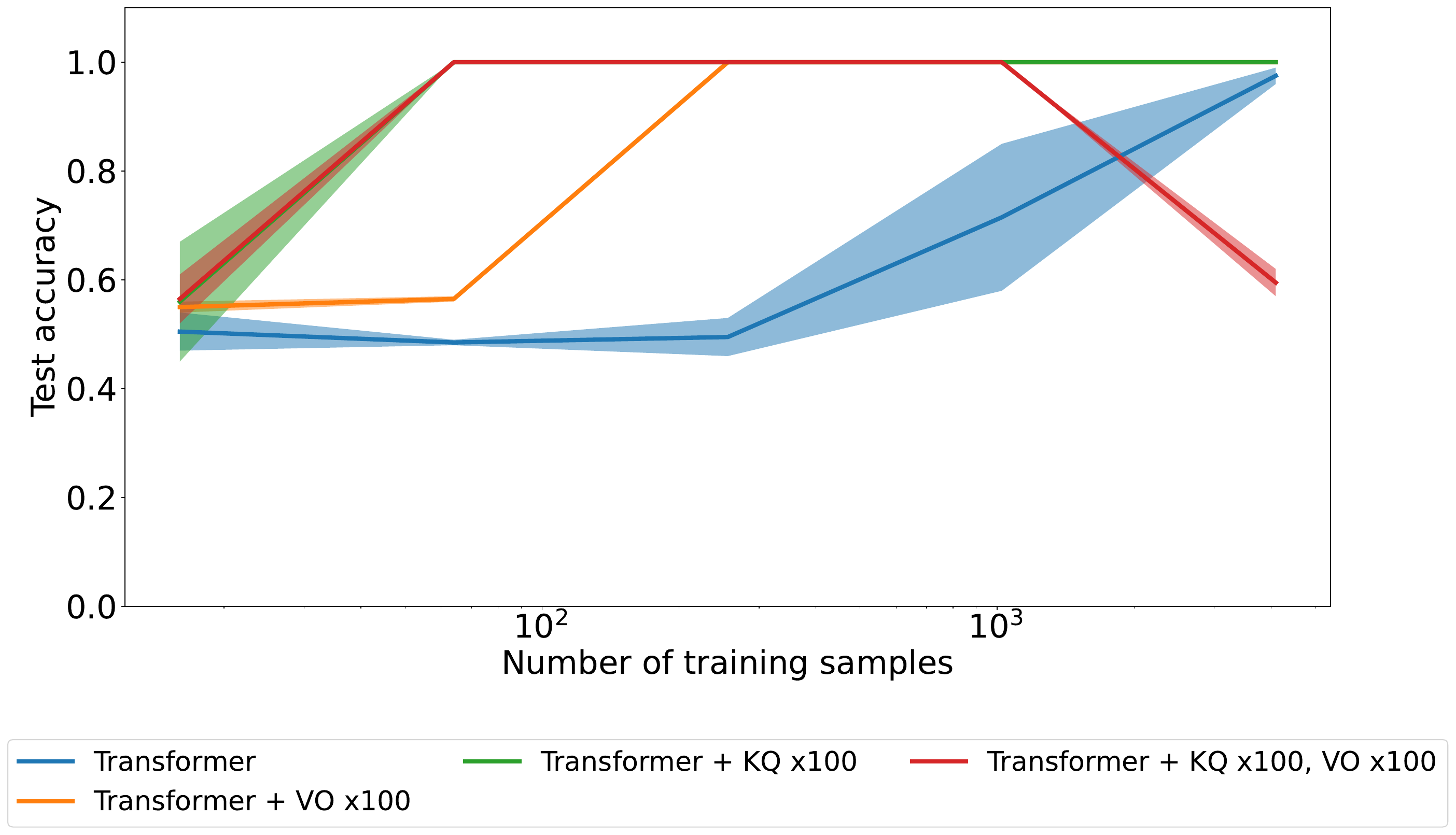}
	\end{subfigure}
	\hfill
	\begin{subfigure}{0.48\textwidth}
		\includegraphics[width=\linewidth]{figures/binary_tied_transformers_vs_gpt2_test.pdf}
	\end{subfigure}
	\caption{\textbf{Binary same/different task.} Architecture comparison
		under the regime of Sec.~\ref{app:hparams-arch}. Left: vanilla Transformers. Right: tied-embedding/unembedding transformers.}
	\label{fig:binary-arch}
\end{figure}

\begin{figure}[h]
	\centering
	\begin{subfigure}{0.48\textwidth}
		\centering
		\includegraphics[width=\textwidth]{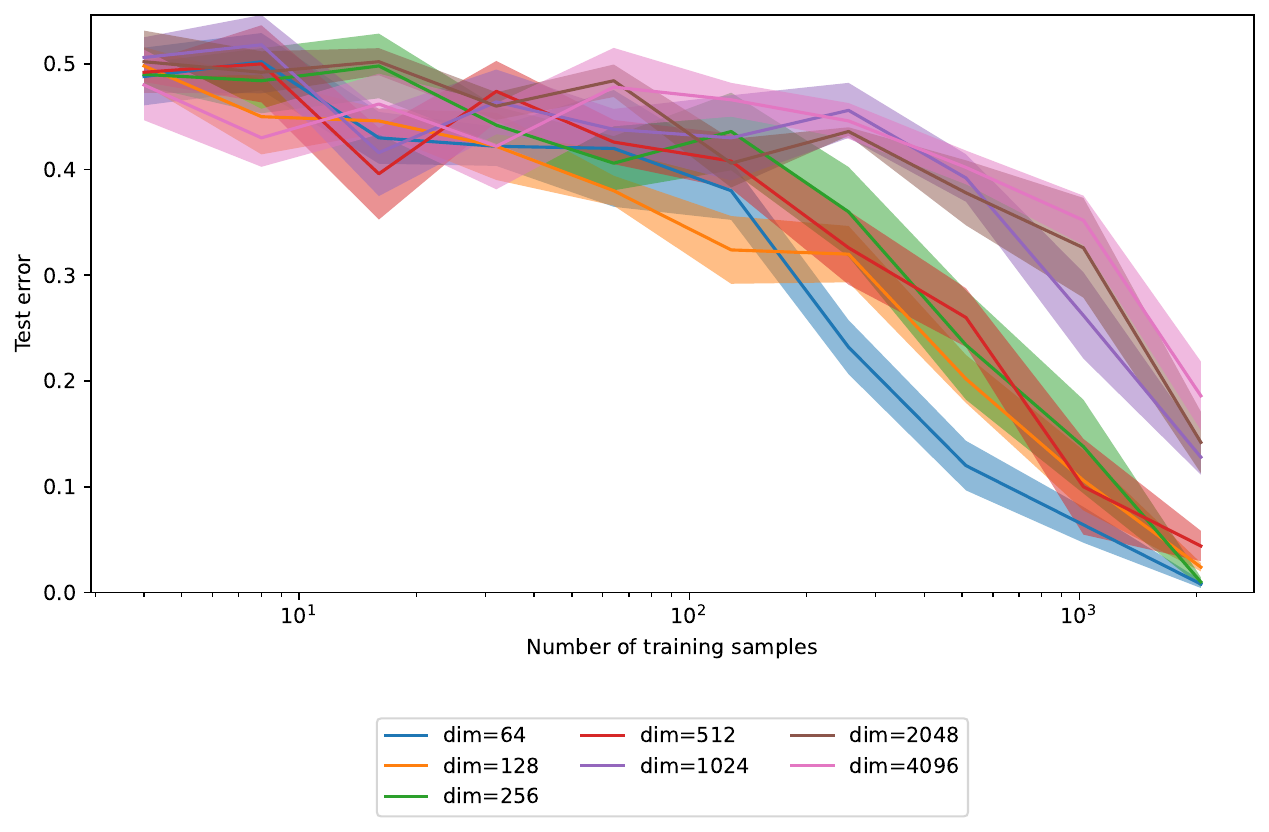}
		\caption{$(M^{KQ},M^{VO})=(0,0)$}
	\end{subfigure}
	\hfill
	\begin{subfigure}{0.48\textwidth}
		\centering
		\includegraphics[width=\linewidth]{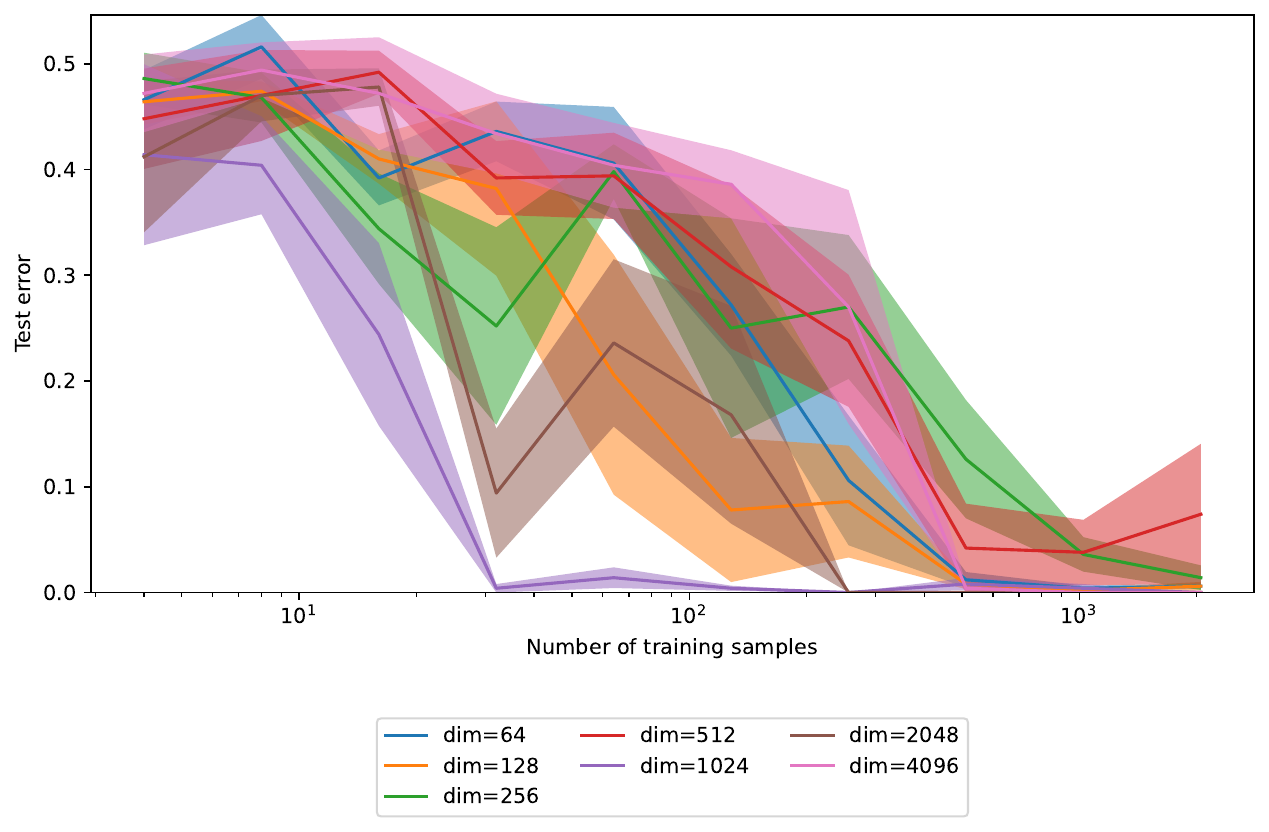}
		\caption{$(M^{KQ},M^{VO})=(0,100)$}
	\end{subfigure}
	\\
	\begin{subfigure}{0.48\textwidth}
		\centering
		\includegraphics[width=\textwidth]{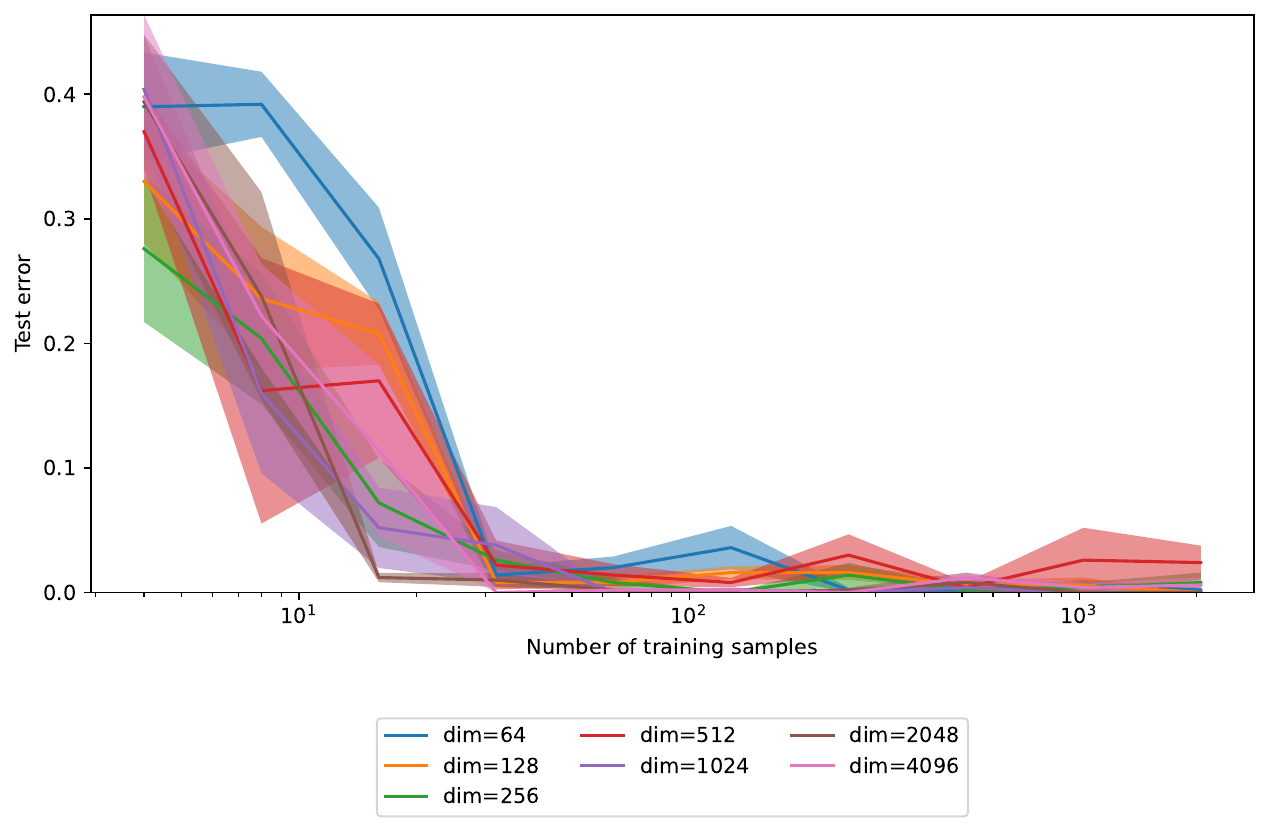}
		\caption{$(M^{KQ},M^{VO})=(100,0)$}
	\end{subfigure}
	\hfill
	\begin{subfigure}{0.48\textwidth}
		\centering
		\includegraphics[width=\textwidth]{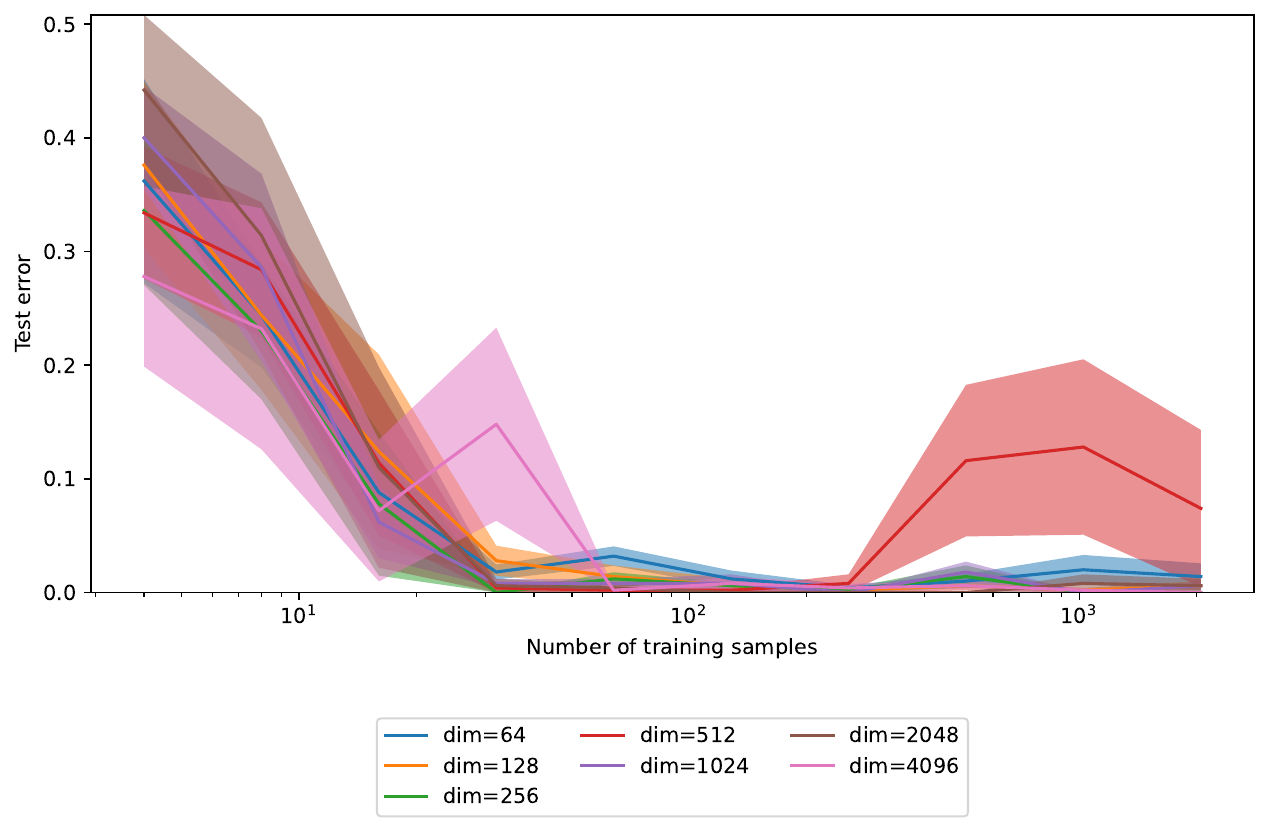}
		\caption{$(M^{KQ},M^{VO})=(100,100)$}
	\end{subfigure}
	\caption{\textbf{Binary task: width sweep.} Test error vs.\ $n$ for
		several embedding dimensions $d$ for different $KQ$ and $VO$ multipliers, for vanilla Transformers (Sec.~\ref{app:hparams-width}).}
	\label{fig:binary-width-vanilla}
\end{figure}

\begin{figure}[h]
	\centering
	\begin{subfigure}{0.48\textwidth}
		\centering
		\includegraphics[width=\textwidth]{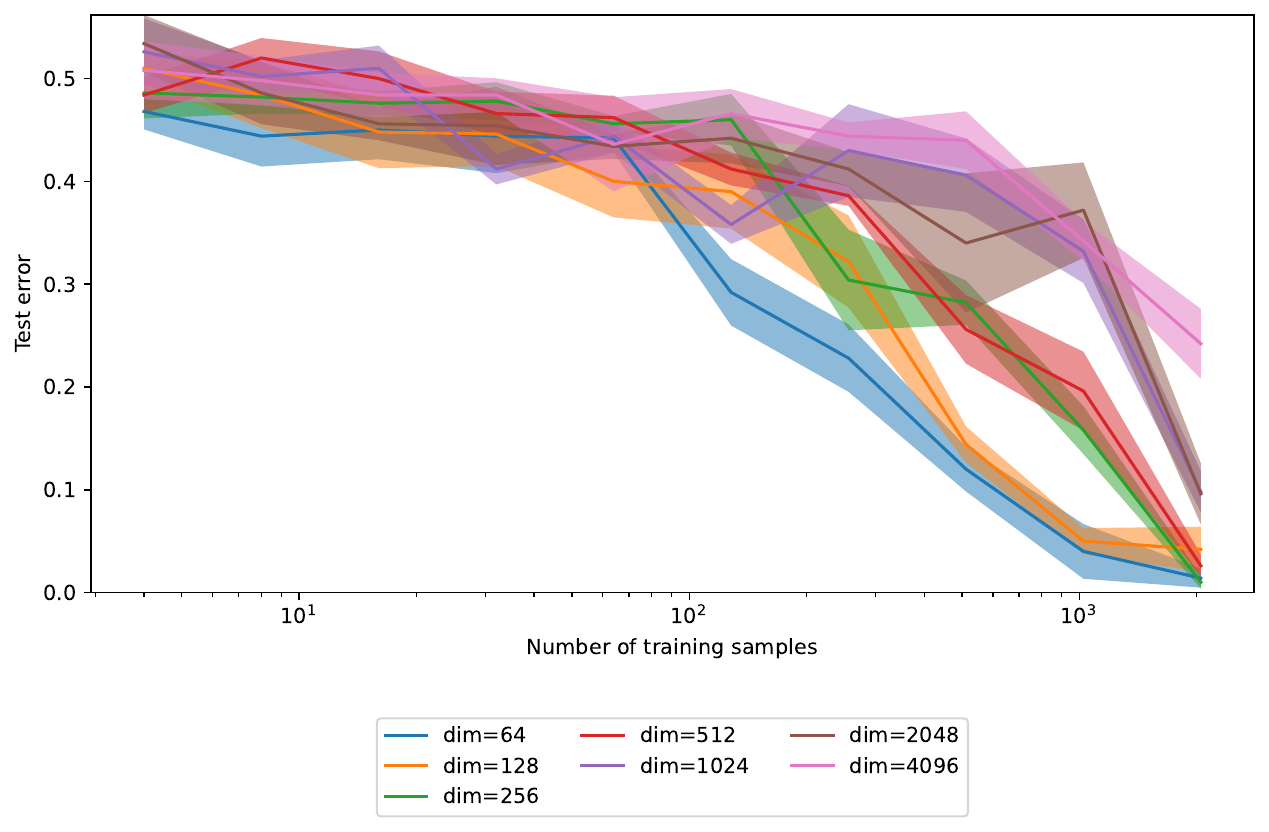}
		\caption{$(M^{KQ},M^{VO})=(0,0)$}
	\end{subfigure}
	\hfill
	\begin{subfigure}{0.48\textwidth}
		\centering
		\includegraphics[width=\linewidth]{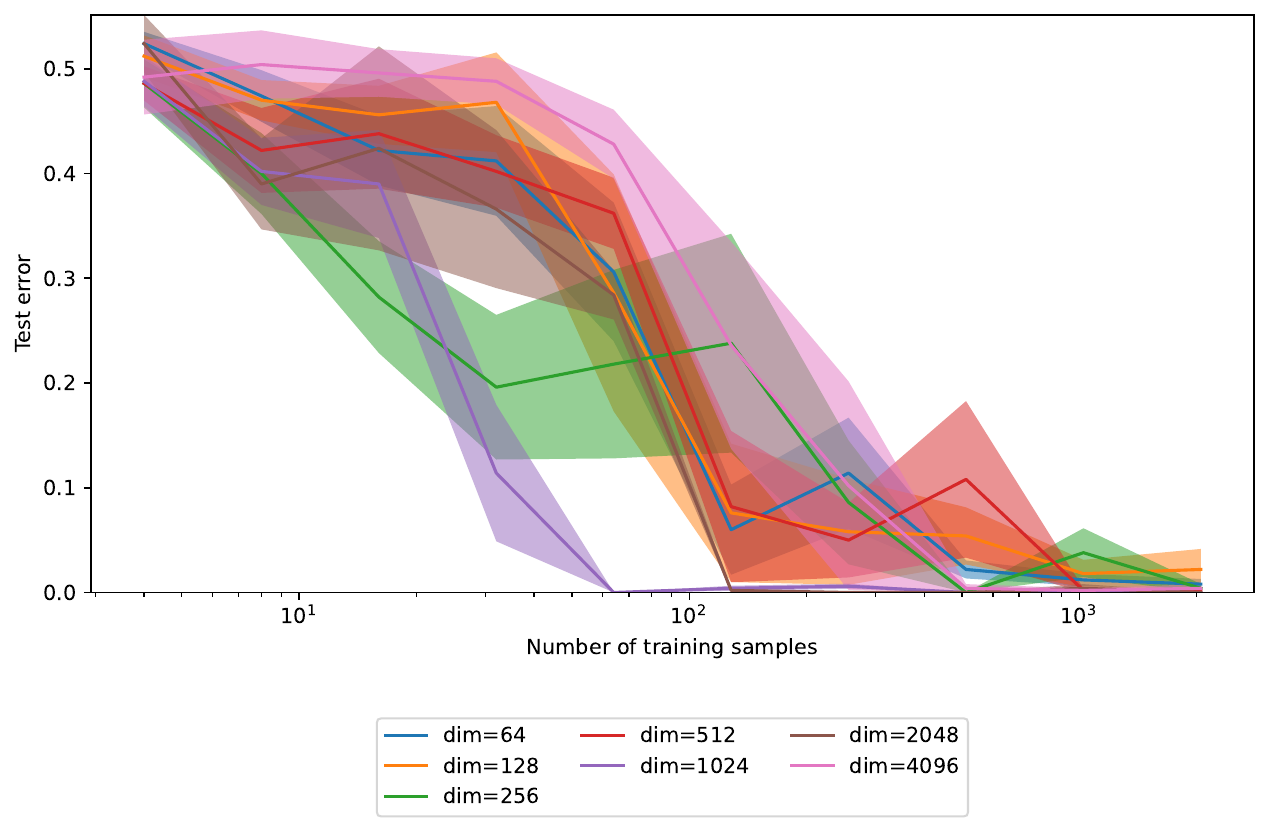}
		\caption{$(M^{KQ},M^{VO})=(0,100)$}
	\end{subfigure}
	\\
	\begin{subfigure}{0.48\textwidth}
		\centering
		\includegraphics[width=\textwidth]{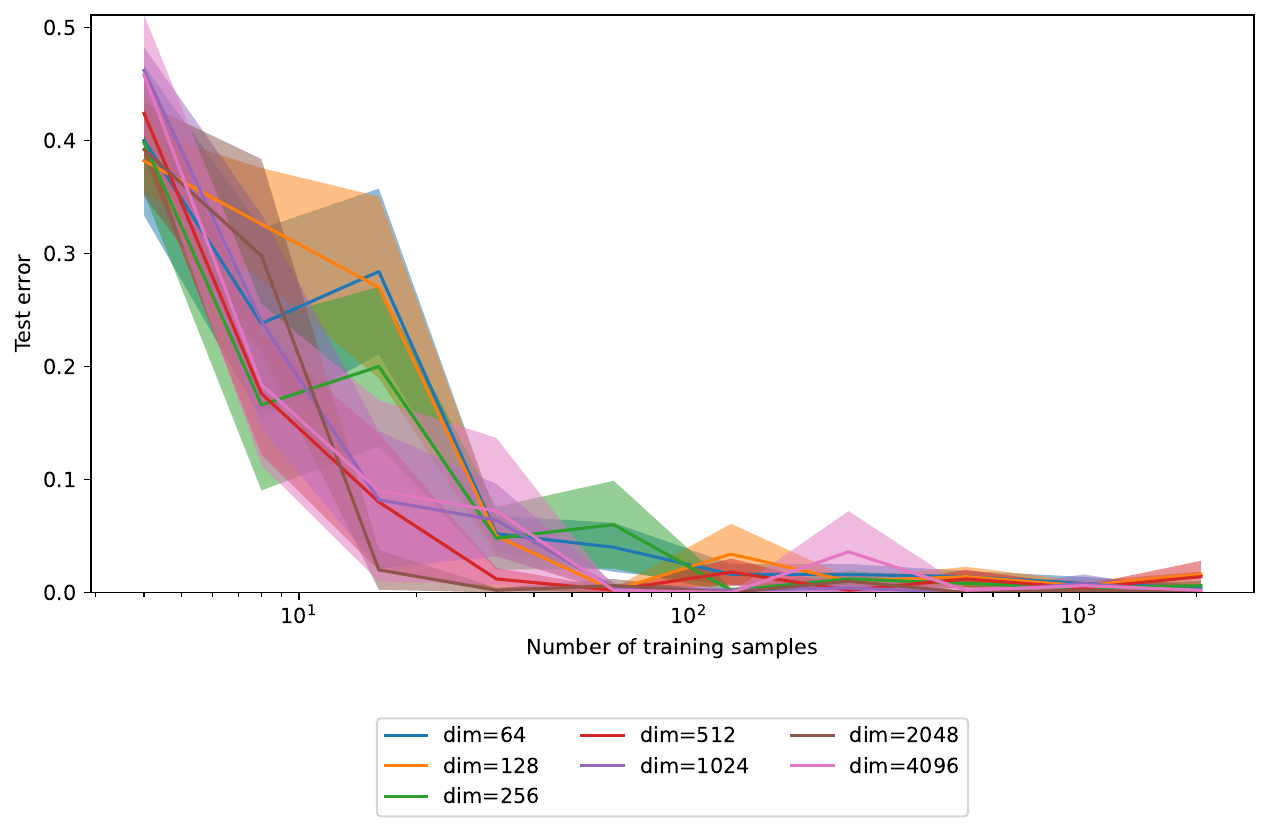}
		\caption{$(M^{KQ},M^{VO})=(100,0)$}
	\end{subfigure}
	\hfill
	\begin{subfigure}{0.48\textwidth}
		\centering
		\includegraphics[width=\textwidth]{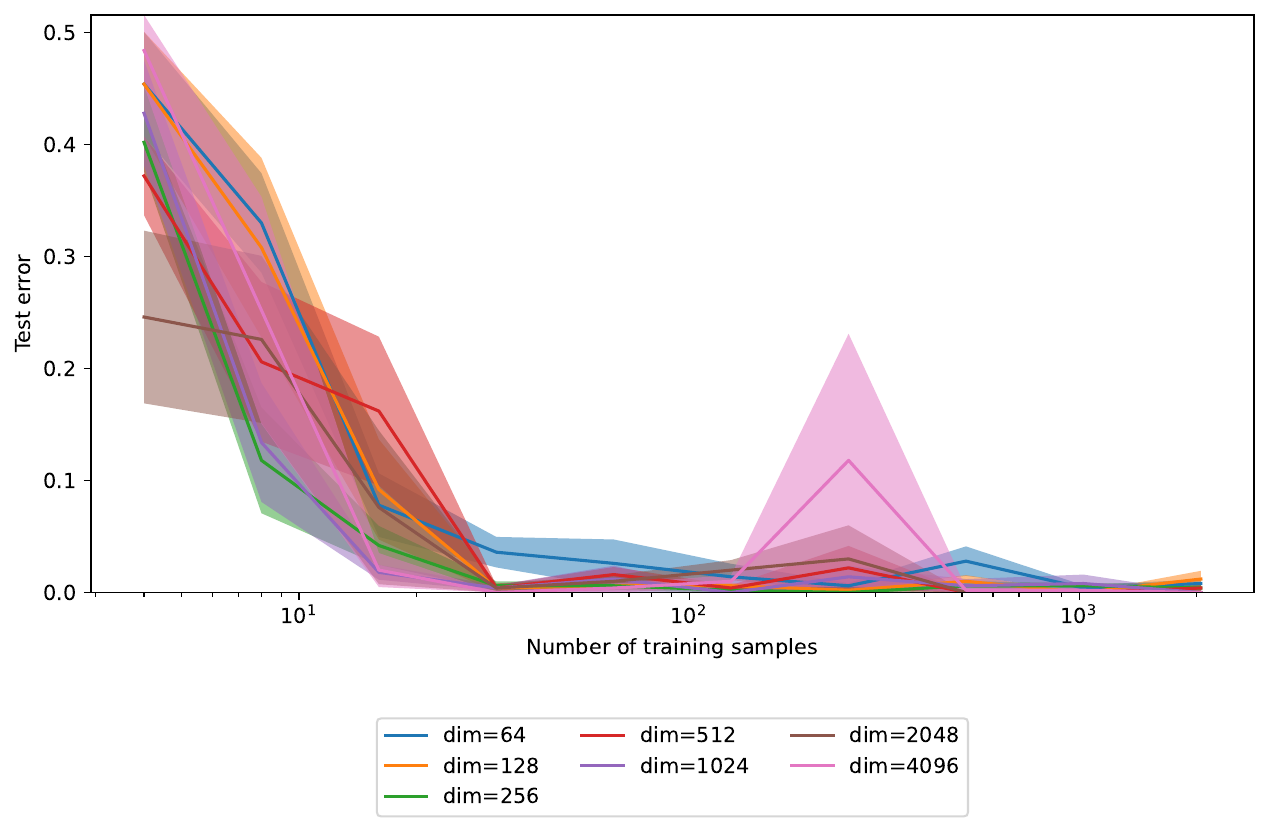}
		\caption{$(M^{KQ},M^{VO})=(100,100)$}
	\end{subfigure}
	\caption{\textbf{Binary task: width sweep.} Test error vs.\ $n$ for
		several embedding dimensions $d$ for different $KQ$ and $VO$ multipliers, for tied-embedding/unembedding Transformers (Sec.~\ref{app:hparams-width}).}
	\label{fig:binary-width-tied}
\end{figure}

% --- Four-class task ---
\begin{figure}[h]
	\centering
	\begin{subfigure}{0.48\textwidth}
		\centering
		\includegraphics[width=\textwidth]{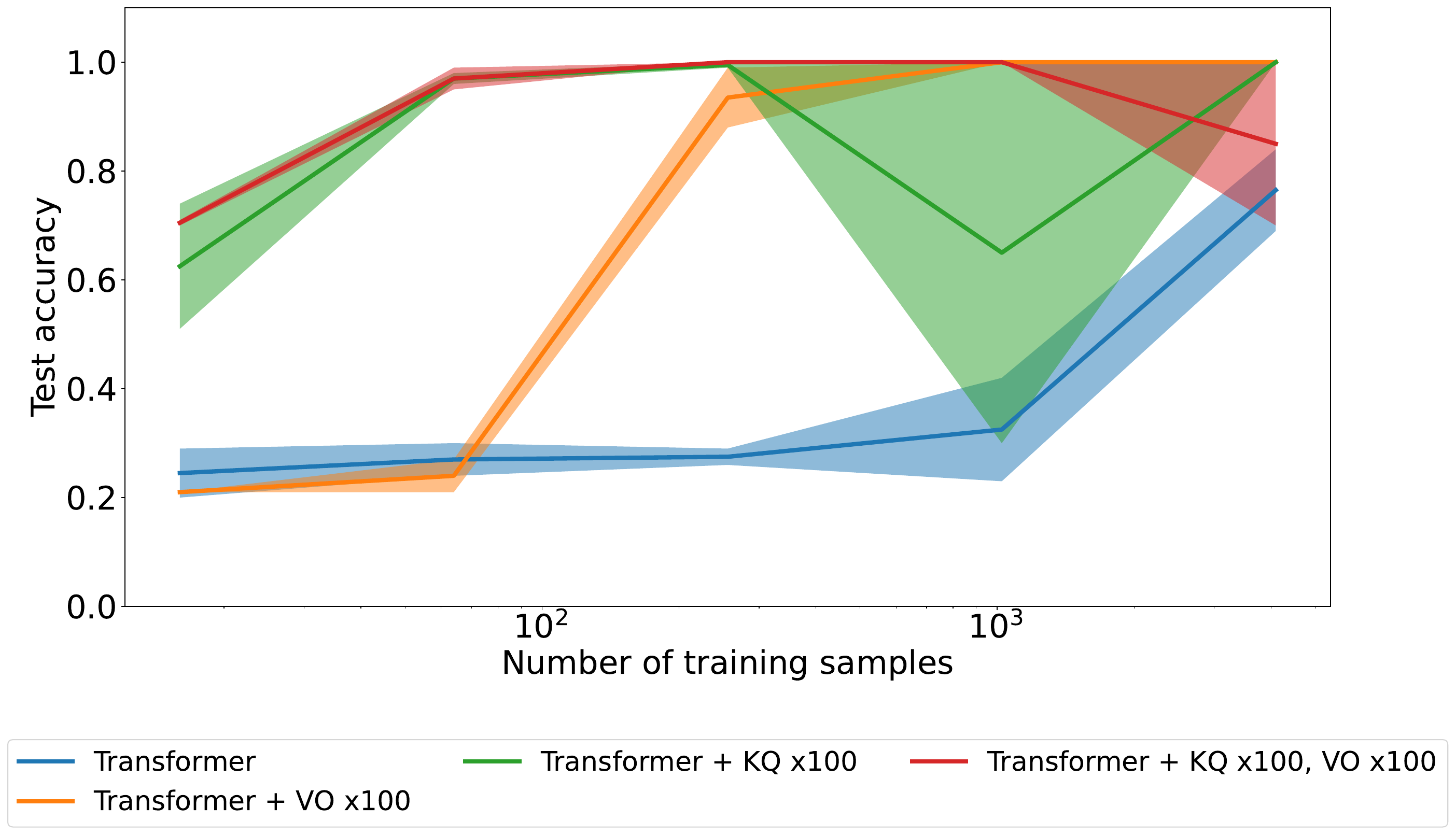}
	\end{subfigure}
	\hfill
	\begin{subfigure}{0.48\textwidth}
		\centering
		\includegraphics[width=\linewidth]{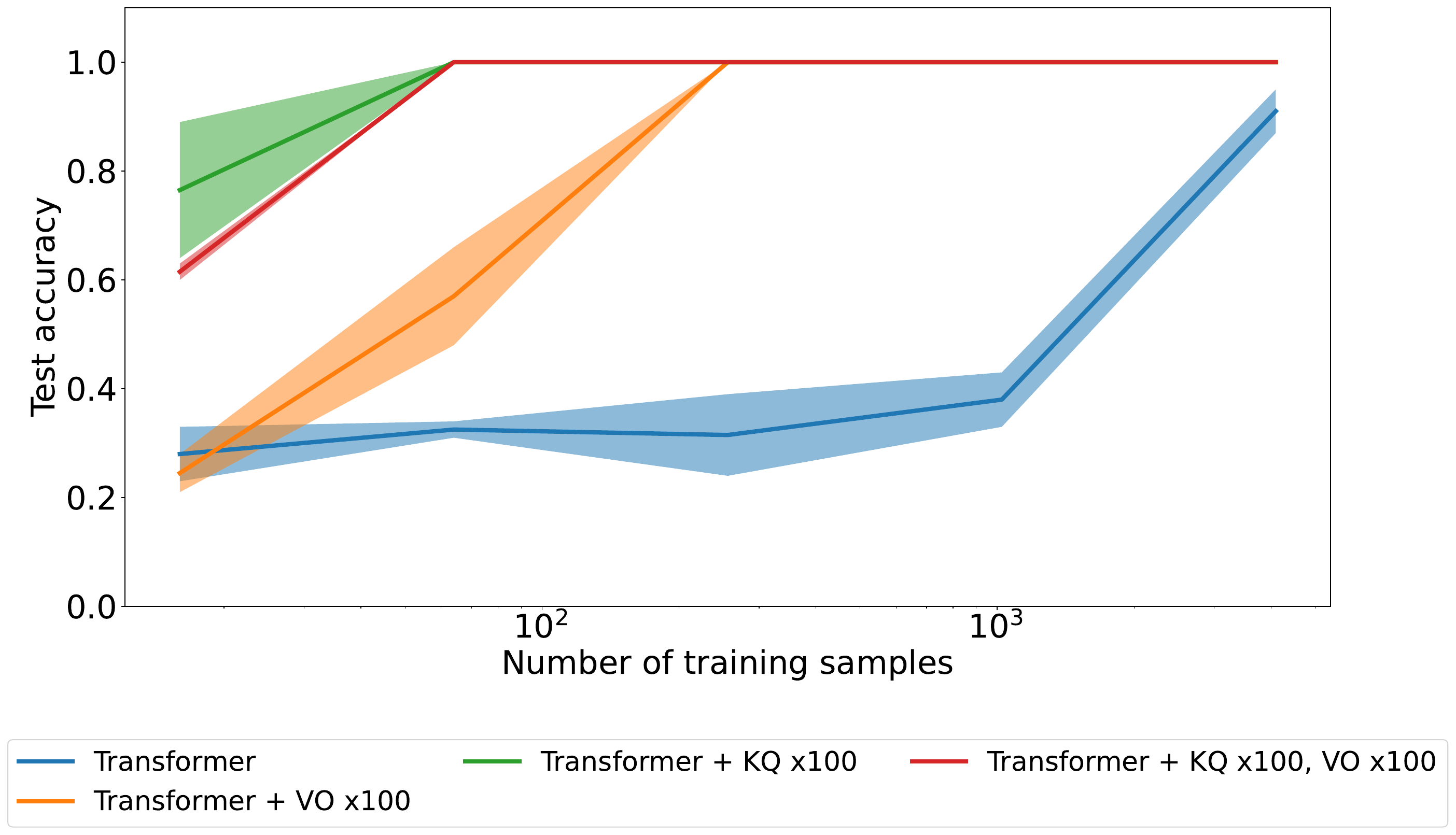}
	\end{subfigure}
	\caption{\textbf{Four-class relational task.} Architecture comparison. Left: vanilla Transformers. Right: tied-embedding/unembedding transformers.}
	\label{fig:multiclass-arch}
\end{figure}

\begin{figure}[h]
	\centering
	\begin{subfigure}{0.48\textwidth}
		\centering
		\includegraphics[width=\textwidth]{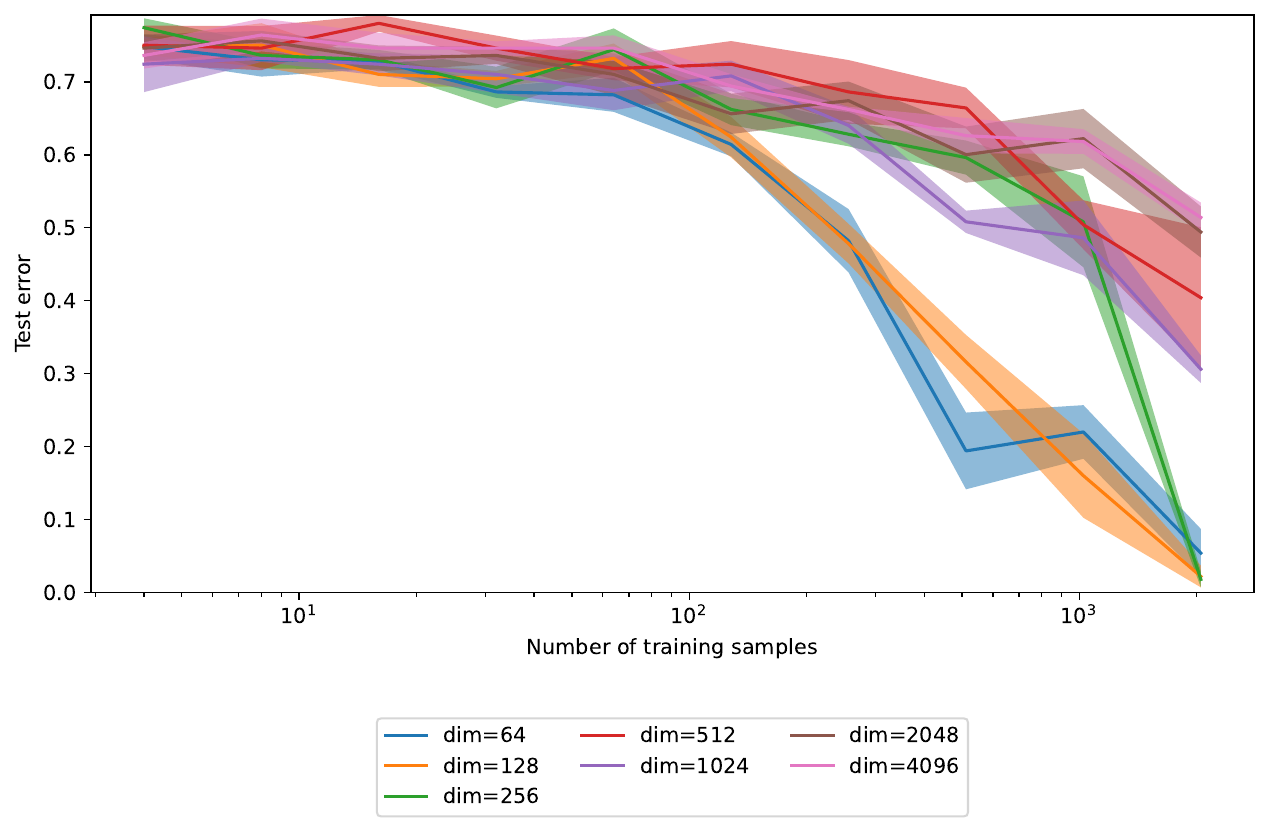}
		\caption{$(M^{KQ},M^{VO})=(0,0)$}
	\end{subfigure}
	\hfill
	\begin{subfigure}{0.48\textwidth}
		\centering
		\includegraphics[width=\linewidth]{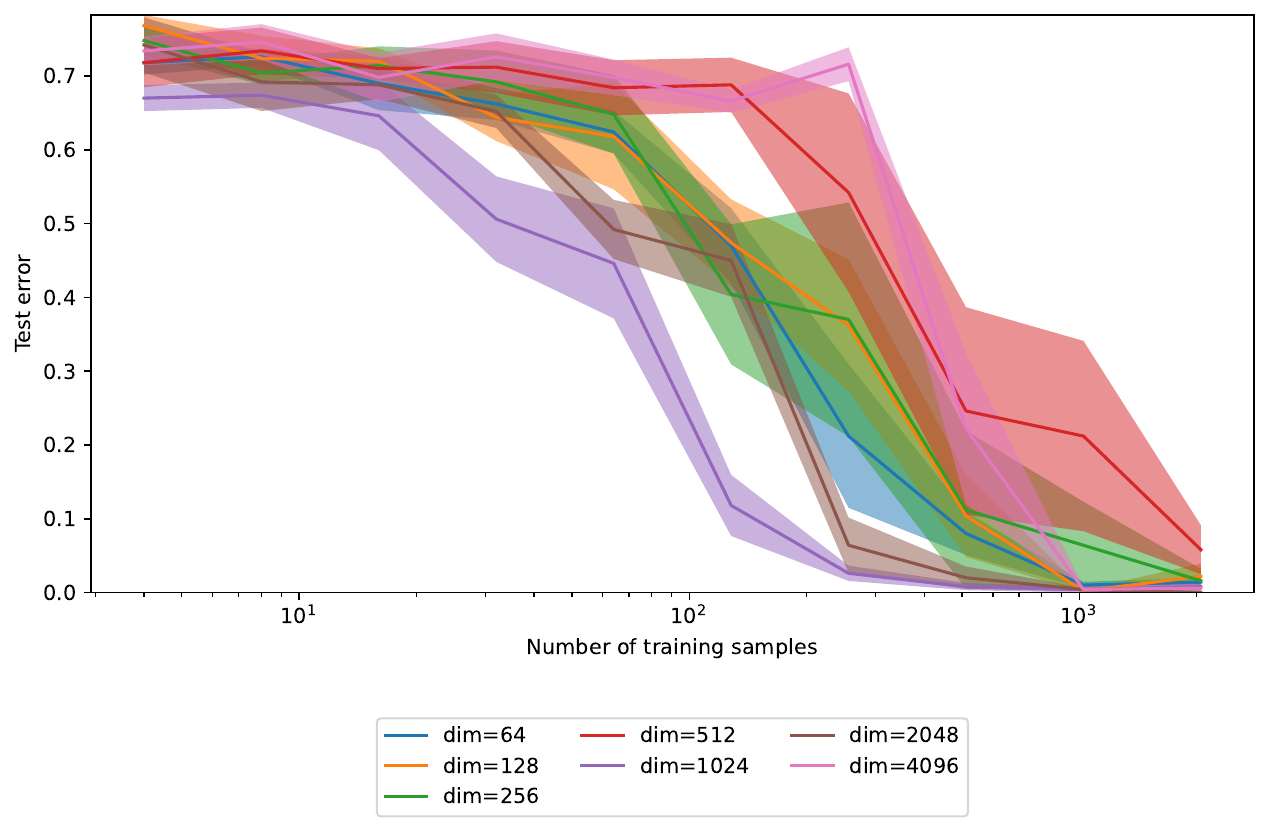}
		\caption{$(M^{KQ},M^{VO})=(0,100)$}
	\end{subfigure}
	\\
	\begin{subfigure}{0.48\textwidth}
		\centering
		\includegraphics[width=\textwidth]{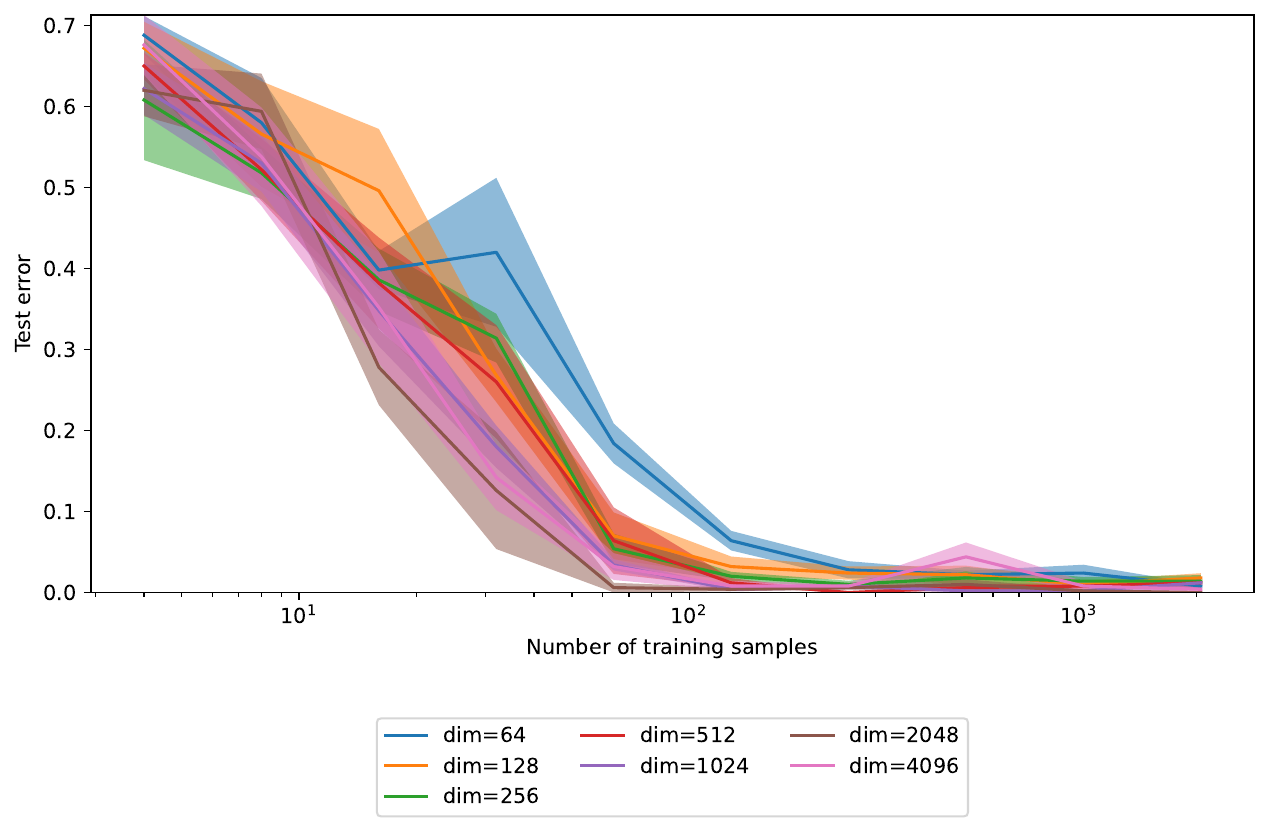}
		\caption{$(M^{KQ},M^{VO})=(100,0)$}
	\end{subfigure}
	\hfill
	\begin{subfigure}{0.48\textwidth}
		\centering
		\includegraphics[width=\textwidth]{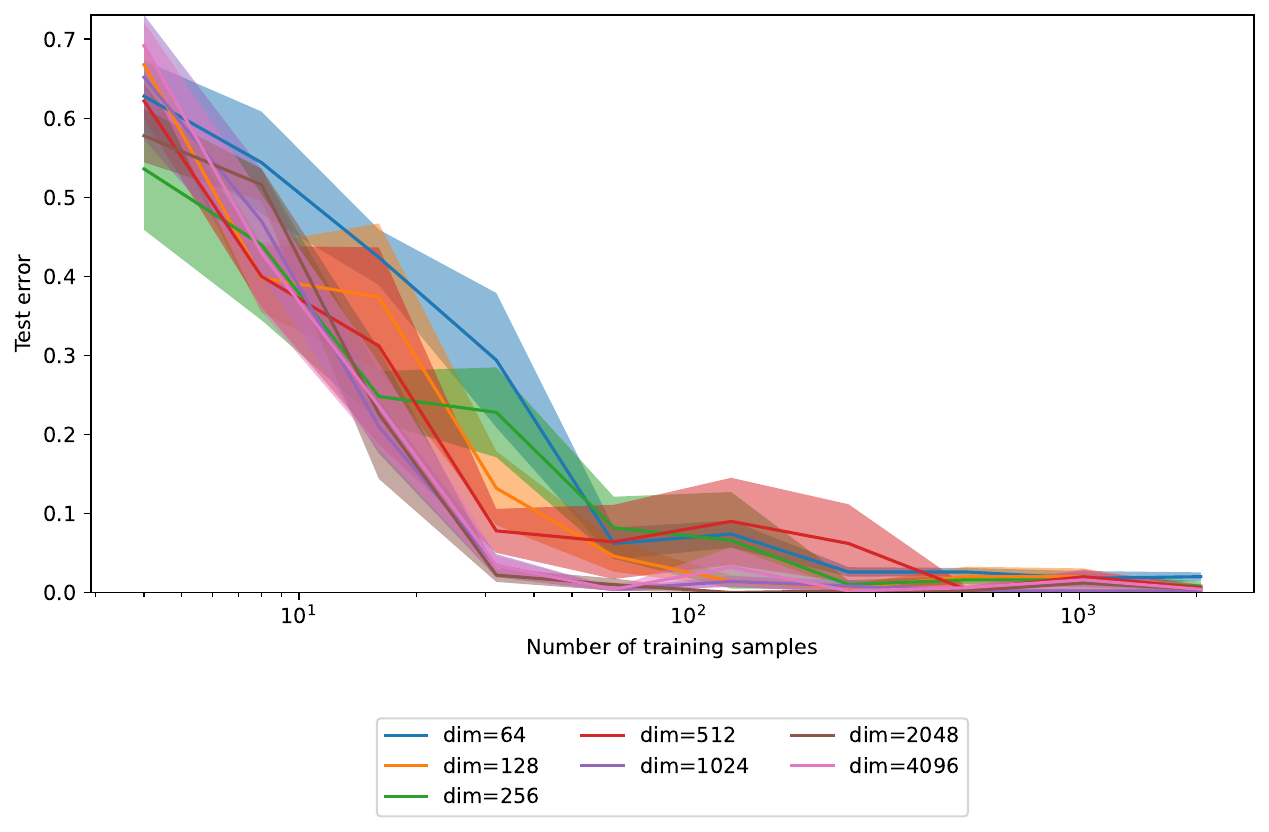}
		\caption{$(M^{KQ},M^{VO})=(100,100)$}
	\end{subfigure}
	\caption{\textbf{Multiclass task: width sweep.} Test error vs.\ $n$ for
		several embedding dimensions $d$ for different $KQ$ and $VO$ multipliers, for vanilla Transformers (Sec.~\ref{app:hparams-width}).}
	\label{fig:multiclass-width-vanilla}
\end{figure}

\begin{figure}[h]
	\centering
	\begin{subfigure}{0.48\textwidth}
		\centering
		\includegraphics[width=\textwidth]{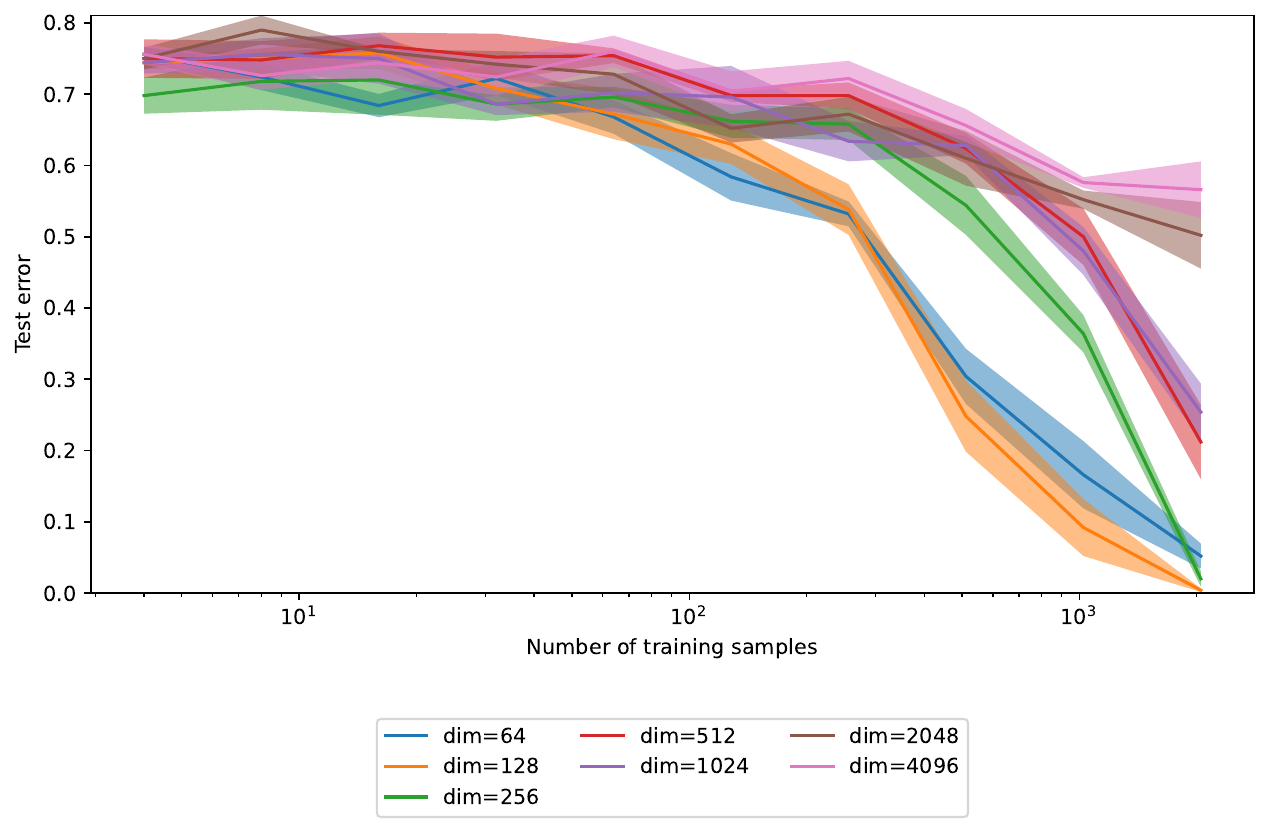}
		\caption{$(M^{KQ},M^{VO})=(0,0)$}
	\end{subfigure}
	\hfill
	\begin{subfigure}{0.48\textwidth}
		\centering
		\includegraphics[width=\linewidth]{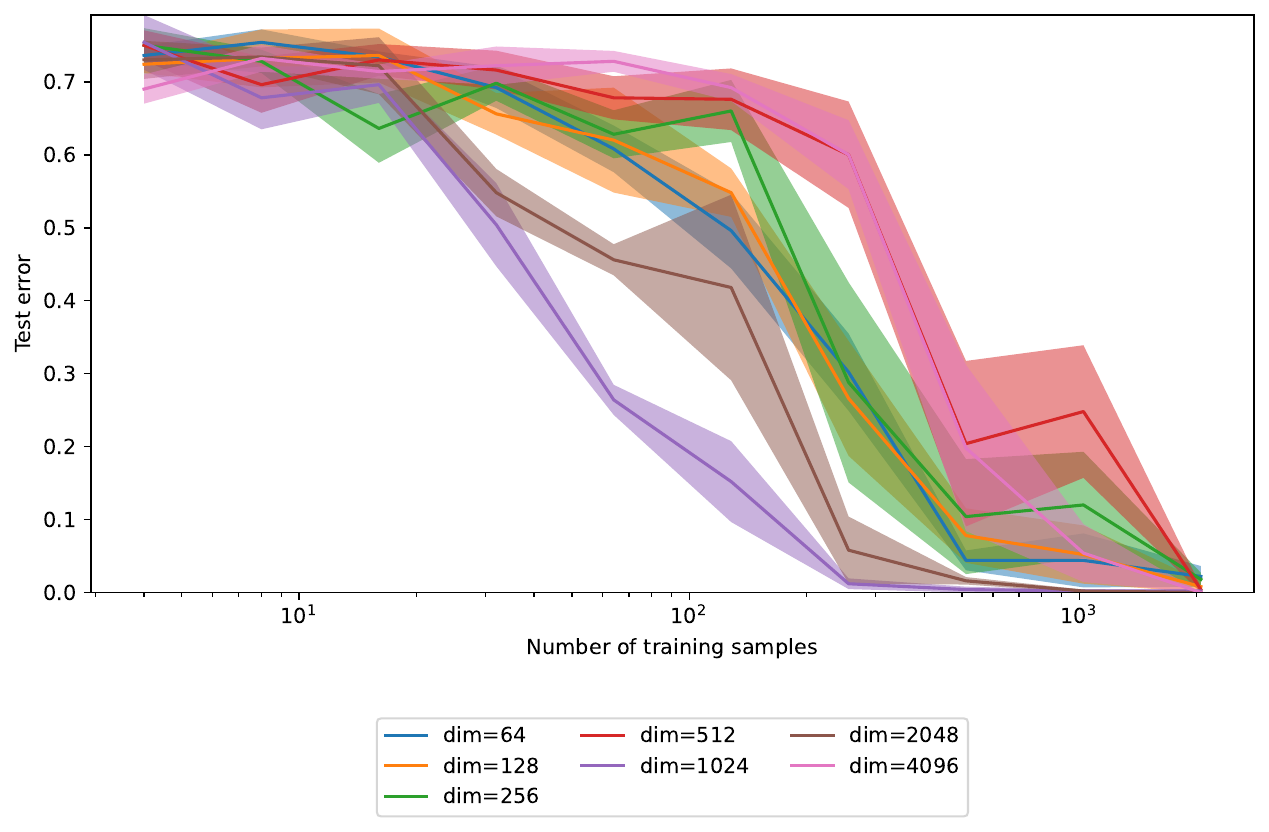}
		\caption{$(M^{KQ},M^{VO})=(0,100)$}
	\end{subfigure}
	\\
	\begin{subfigure}{0.48\textwidth}
		\centering
		\includegraphics[width=\textwidth]{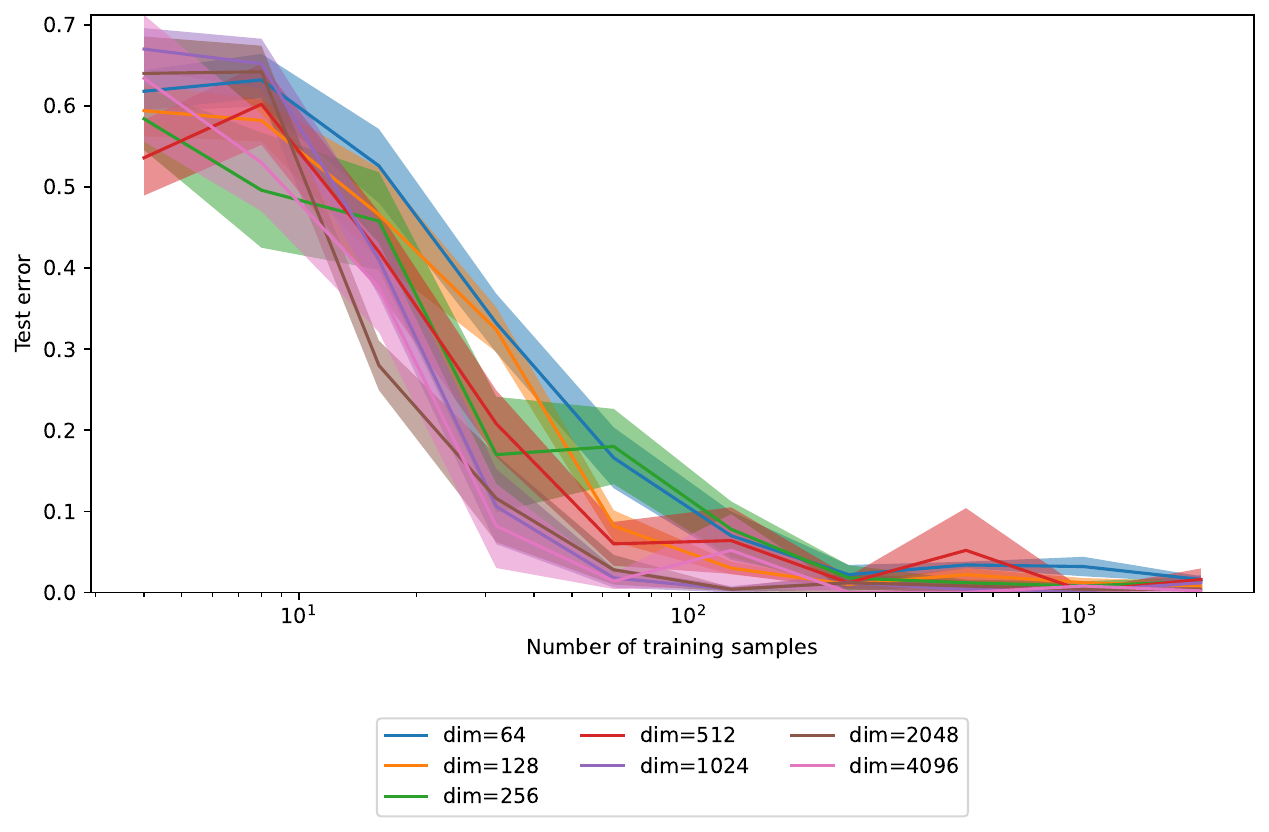}
		\caption{$(M^{KQ},M^{VO})=(100,0)$}
	\end{subfigure}
	\hfill
	\begin{subfigure}{0.48\textwidth}
		\centering
		\includegraphics[width=\textwidth]{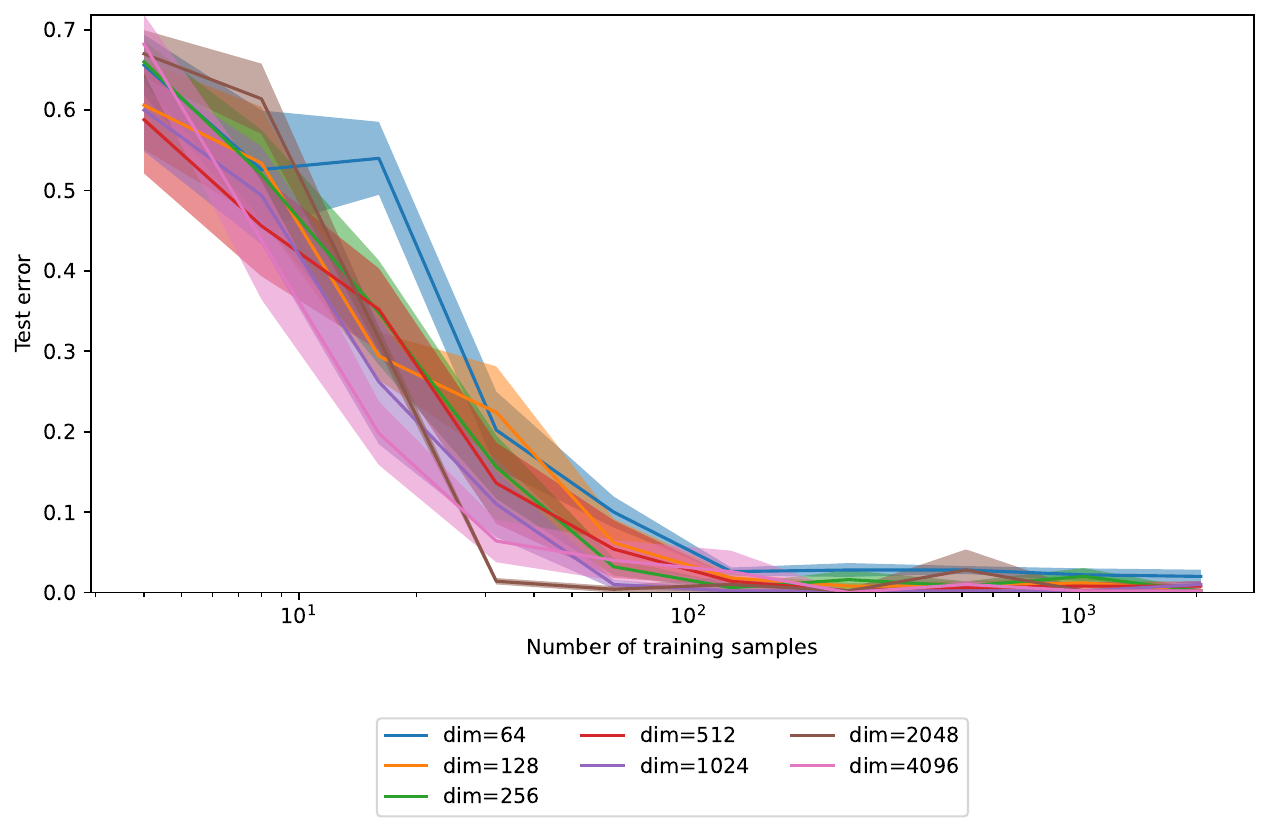}
		\caption{$(M^{KQ},M^{VO})=(100,100)$}
	\end{subfigure}
	\caption{\textbf{Multiclass task: width sweep.} Test error vs.\ $n$ for
		several embedding dimensions $d$ for different $KQ$ and $VO$ multipliers, for tied-embedding/unembedding Transformers (Sec.~\ref{app:hparams-width}).}
	\label{fig:multiclass-width-tied}
\end{figure}

% --- Majority task ---
\begin{figure}[h]
	\centering
	\begin{subfigure}{0.48\textwidth}
		\centering
		\includegraphics[width=\linewidth]{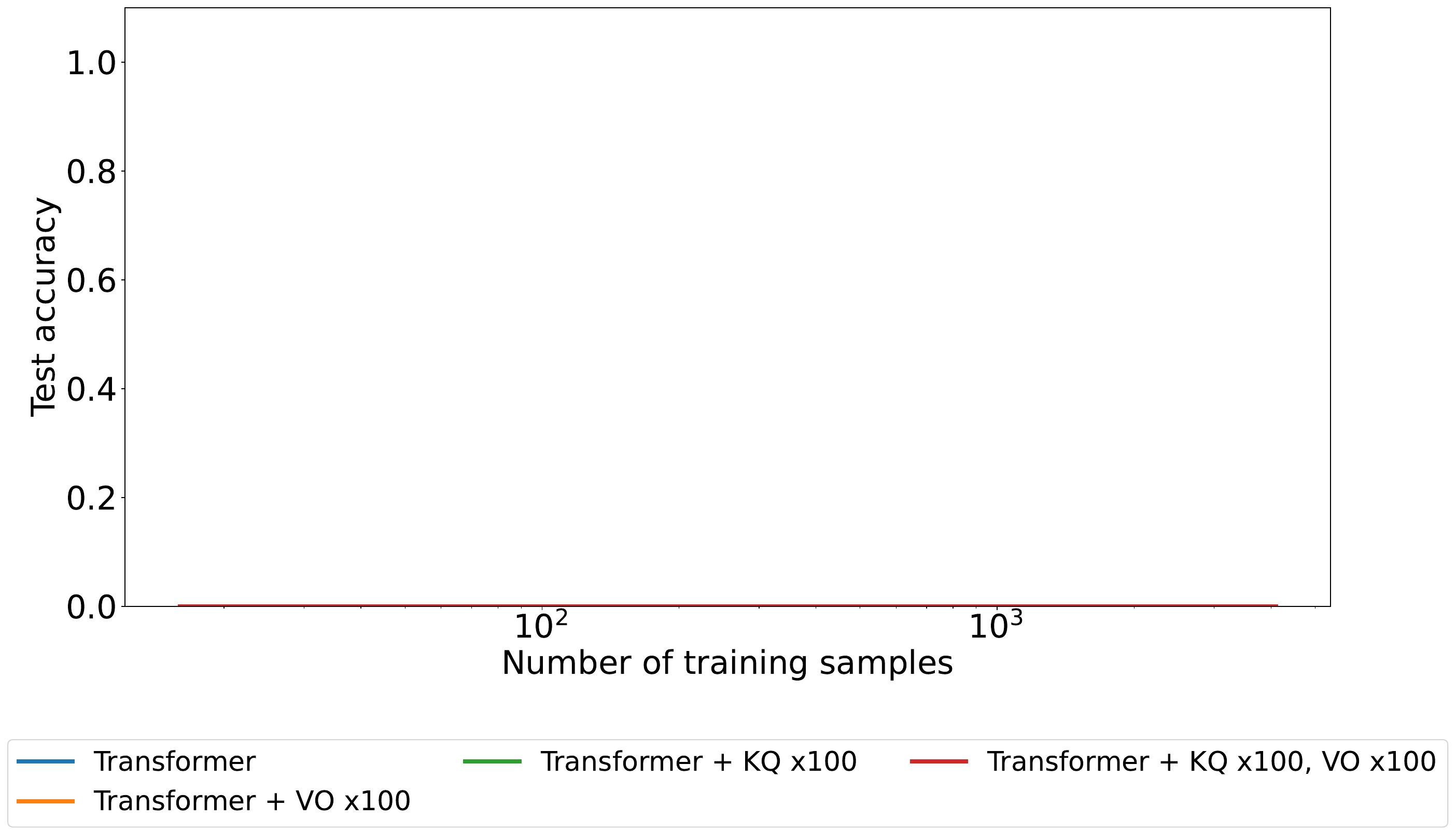}
	\end{subfigure}
	\hfill
	\begin{subfigure}{0.48\textwidth}
		\centering
		\includegraphics[width=\linewidth]{figures/majority_tied_transformers_vs_gpt2_test.pdf}
	\end{subfigure}
	\caption{\textbf{Majority / copy-wildcard task.} Architecture comparison. Left: vanilla Transformers. Right: tied-embedding/unembedding transformers.}
	\label{fig:majority-arch}
\end{figure}

\begin{figure}[h]
	\centering
	\begin{subfigure}{0.48\textwidth}
		\centering
		\includegraphics[width=\textwidth]{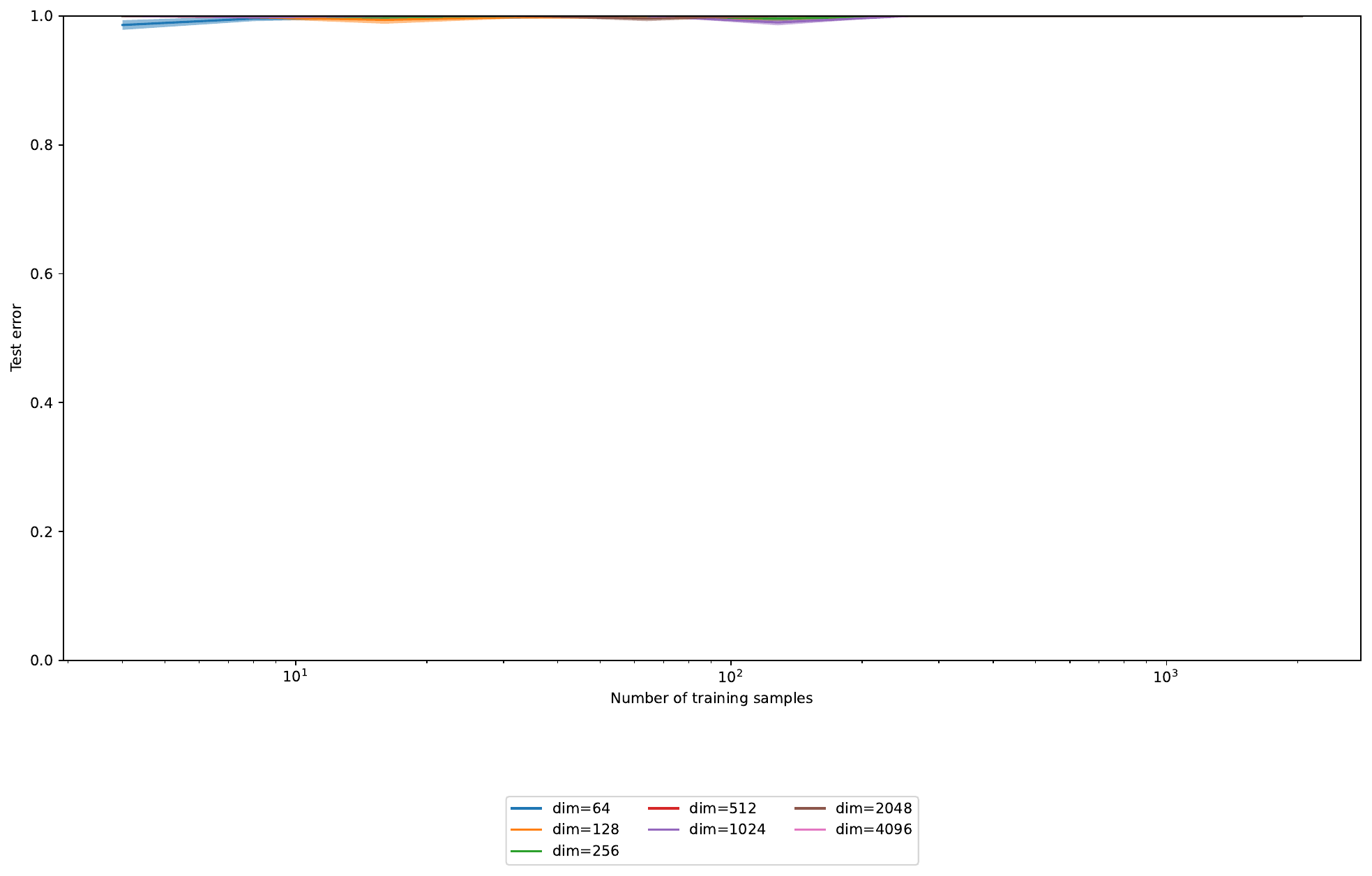}
		\caption{$(M^{KQ},M^{VO})=(0,0)$}
	\end{subfigure}
	\hfill
	\begin{subfigure}{0.48\textwidth}
		\centering
		\includegraphics[width=\textwidth]{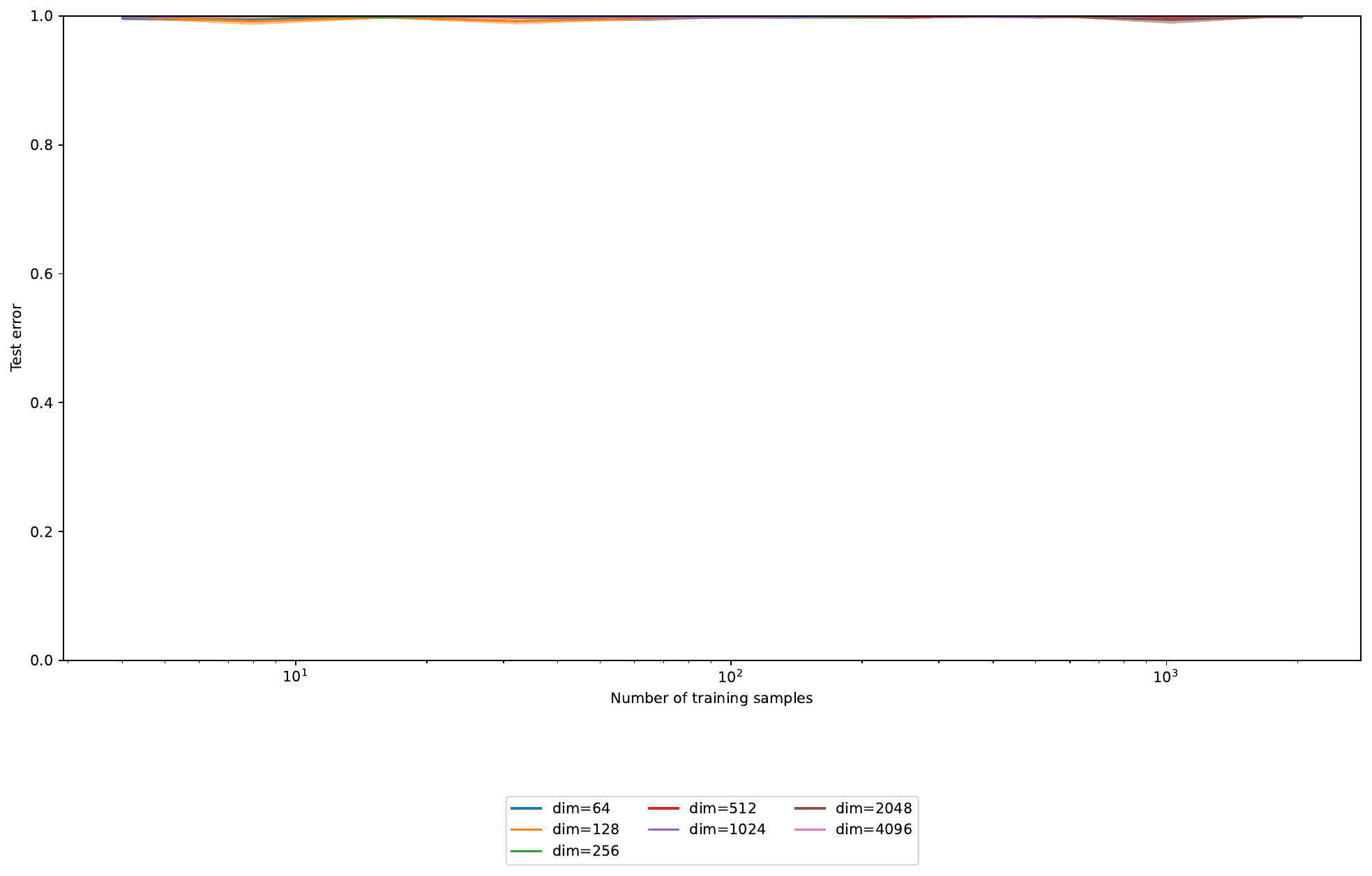}
		\caption{$(M^{KQ},M^{VO})=(100,100)$}
	\end{subfigure}
	\caption{\textbf{Majority task: width sweep.} For vanilla Transformers.}
	\label{fig:majority-width-vanilla}
\end{figure}

\begin{figure}[h]
	\centering
	\begin{subfigure}{0.48\textwidth}
		\centering
		\includegraphics[width=\textwidth]{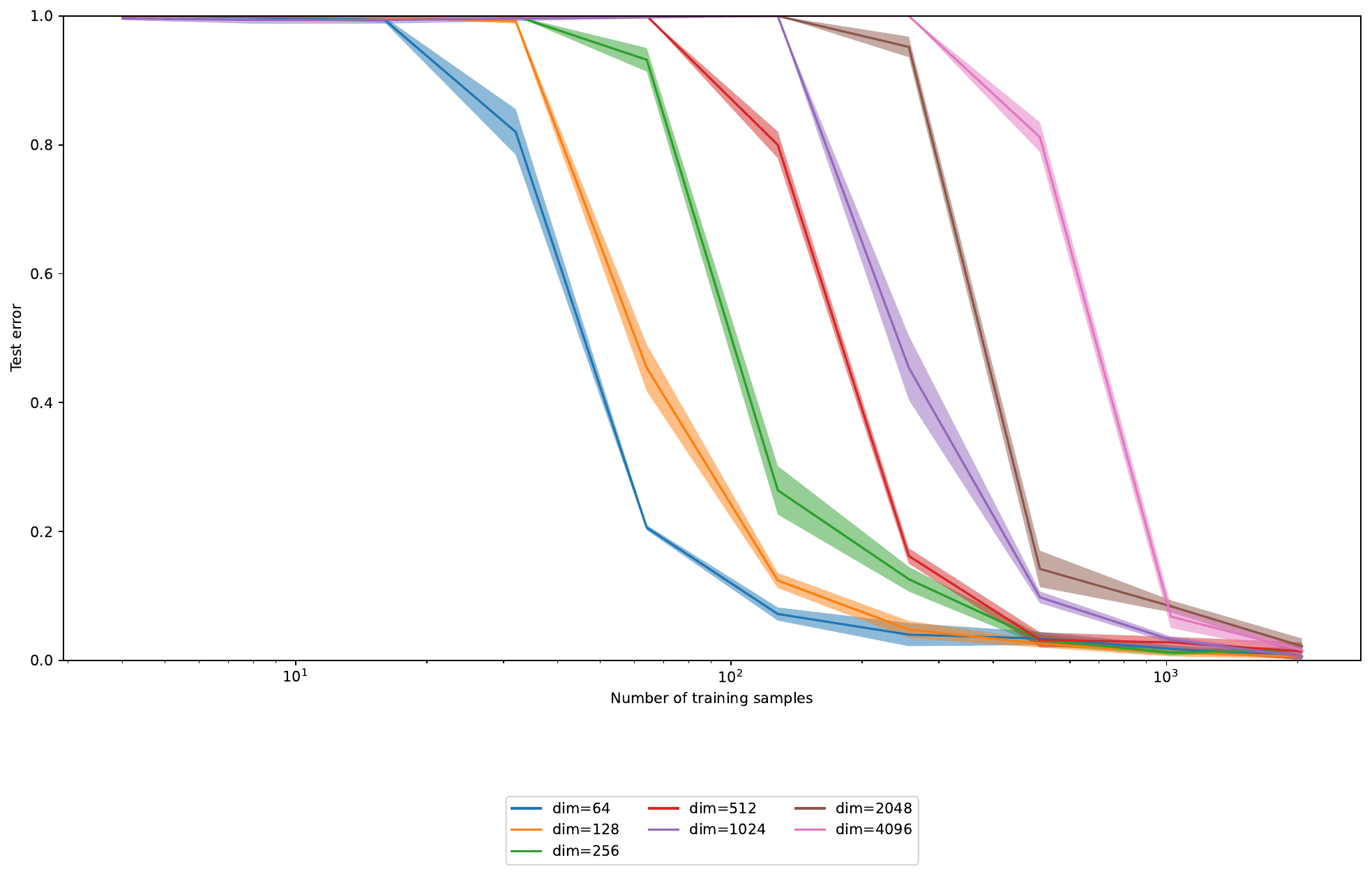}
		\caption{$(M^{KQ},M^{VO})=(0,0)$}
	\end{subfigure}
	\hfill
	\begin{subfigure}{0.48\textwidth}
		\centering
		\includegraphics[width=\textwidth]{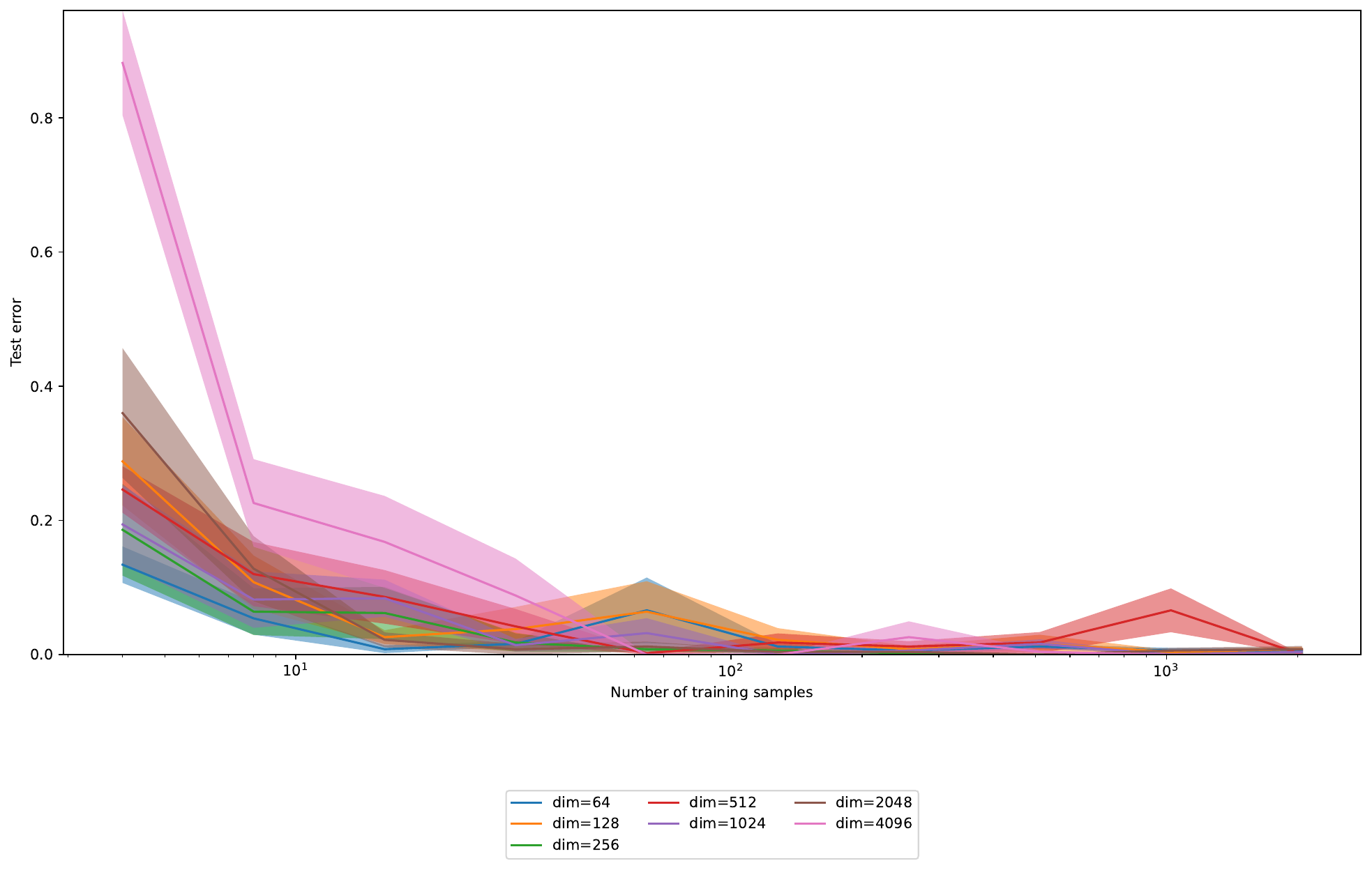}
		\caption{$(M^{KQ},M^{VO})=(100,100)$}
	\end{subfigure}
	\caption{\textbf{Majority task: width sweep.} For tied-embedding/unembedding Transformers.}
	\label{fig:majority-width-tied}
\end{figure}

% --- Print task (teaser) ---
\begin{figure}[h]
	\centering
	\includegraphics[width=0.75\linewidth]{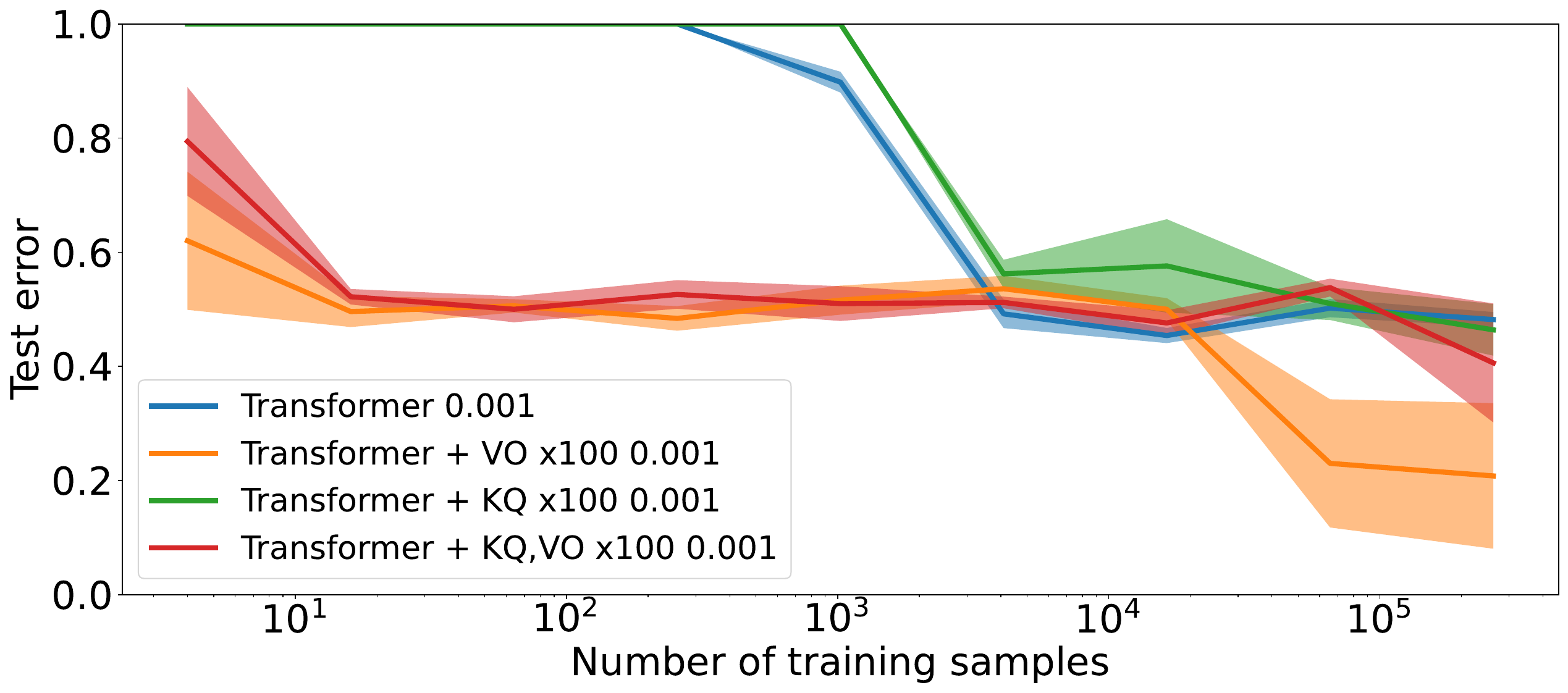}
	\caption{\textbf{Variable-assignment / \texttt{print} task.} Test error
		vs.\ number of training samples for vanilla Transformer and the three
		identity-multiplier variants, at $d=1024$, $D=2$
		(Sec.~\ref{app:hparams-width}).}
	\label{fig:teaser-width}
\end{figure}

% NeurIPS checklist removed for the AMS-style version.

\end{document}